\documentclass[10pt,journal,compsoc]{IEEEtran}
\ifCLASSOPTIONcompsoc
  \usepackage[nocompress]{cite}
\else
  \usepackage{cite}
\fi
\ifCLASSINFOpdf
  \usepackage[pdftex]{graphicx}
\else
\fi
\usepackage{amsmath}
\usepackage{amsthm}
\usepackage{amssymb}
\usepackage{bbm}
\usepackage{amsfonts}
\usepackage{mathtools}
\usepackage{color}
\usepackage[ruled,vlined,linesnumbered,lined,boxed,commentsnumbered]{algorithm2e}
\usepackage{subcaption}
\usepackage{multirow}
\usepackage{booktabs}
\usepackage{ragged2e}
\usepackage{enumitem}
\usepackage{bigints}
\usepackage{mathrsfs}
\usepackage[mathscr]{euscript}

\usepackage{multicol}
\usepackage{balance}
\newtheorem{theorem}{Theorem}

\newtheorem{lemma}{Lemma}
\newtheorem{definition}{Definition}
\newtheorem{assumption}{Assumption}
\newtheorem{proposition}{Proposition}

\usepackage{makecell}
\usepackage{cellspace}
\usepackage{array}
\usepackage{thmtools}
\usepackage{thm-restate}

\usepackage{titletoc}
\usepackage{nccmath}
\usepackage[framemethod=tikz]{mdframed}
\definecolor{softred}{HTML}{FE8A71}
\usepackage{titlesec}
\usepackage[toc,page,header]{appendix}
\usepackage[colorlinks=true, menucolor=blue, linkcolor=blue]{hyperref}
\definecolor{Gray}{HTML}{7f7f7f}
\usepackage{soul}
\begin{document}
% \startcontents[sections]
% \printcontents[sections]{l}{1}{\setcounter{tocdepth}{2}}
% \title{HodgeRank and Spectral Method Are Both Vulnerable for Target Attack}
\title{Sequential Manipulation against Rank Aggregation: Theory and Algorithm}
\author
{
    Ke~Ma,~\IEEEmembership{Member,~IEEE,}
    Qianqian~Xu$^*$,~\IEEEmembership{Senior Member,~IEEE,}
    Jinshan~Zeng,
    Wei~Liu,~\IEEEmembership{Fellow,~IEEE,}\\
    Xiaochun~Cao,\IEEEmembership{Senior Member,~IEEE,}
    Yingfei~Sun,
    and~Qingming~Huang$^*$,~\IEEEmembership{Fellow,~IEEE}% <-this % stops a space
    \IEEEcompsocitemizethanks
    {
        % \IEEEcompsocthanksitem M. Shell was with the Department of Electrical and Computer Engineering, Georgia Institute of Technology, Atlanta, GA, 30332.\protect\\
        % % note need leading \protect in front of \\ to get a newline within \thanks as
        % % \\ is fragile and will error, could use \hfil\break instead.
        % E-mail: see http://www.michaelshell.org/contact.html
        % \IEEEcompsocthanksitem J. Doe and J. Doe are with Anonymous University.
        % M. Shell is with the Department of Electrical and Computer Engineering, Georgia Institute of Technology, Atlanta, GA, 30332.\protect\\
        \IEEEcompsocthanksitem K. Ma and Y. Sun are with the School of Electronic, Electrical and Communication Engineering, University of Chinese Academy of Sciences, Beijing 100049, China. E-mail: make@ucas.ac.cn, yfsun@ucas.ac.cn.\hfil\break
        \IEEEcompsocthanksitem Q. Xu is with the Key Laboratory of Intelligent Information Processing, Institute of Computing Technology, Chinese Academy of Sciences, Beijing 100190, China. E-mail: qianqian.xu@vipl.ict.ac.cn, xuqianqian@ict.ac.cn.\hfil\break
        \IEEEcompsocthanksitem J. Zeng is with the School of Computer and Information Engineering, Jiangxi Normal University, Nanchang, Jiangxi 330022, China. E-mail: jinshanzeng@jxnu.edu.cn.\hfil\break
        \IEEEcompsocthanksitem W. Liu is with the Tencent Data Platform, Shenzhen 518054, China. E-mail: wl2223@columbia.edu.\hfil\break
        \IEEEcompsocthanksitem X. Cao is with the School of Cyber Science and Technology, Shenzhen Campus of Sun Yat-sen University, Shenzhen 518107, China. E-mail: caoxiaochun@mail.sysu.edu.cn.\hfil\break
        \IEEEcompsocthanksitem Q. Huang is with the School of Computer Science and Technology, University of Chinese Academy of Sciences, Beijing 100049, China, also with the Key Laboratory of Big Data Mining and Knowledge Management (BDKM), University of Chinese Academy of Sciences, Beijing 100049, China, also with the Key Laboratory of Intelligent Information Processing, Institute of Computing Technology, Chinese Academy of Sciences, Beijing 100190, China. E-mail: qmhuang@ucas.ac.cn.\hfil\break
        \IEEEcompsocthanksitem $^*$ Corresponding author.
        % \IEEEcompsocthanksitem $^*$ Corresponding author.
    }% <-this % stops an unwanted space
    % \thanks{Manuscript received Oct 31, 2020; revised Mar 1, 2021.}
}


\markboth{Journal of \LaTeX\ Class Files,~Vol.~14, No.~8, August~2015}%
{Shell \MakeLowercase{\textit{et al.}}: Bare Demo of IEEEtran.cls for Computer Society Journals}



\IEEEtitleabstractindextext
{%
    \begin{abstract}
        \justifying
        Rank aggregation with pairwise comparisons is widely encountered in sociology, politics, economics, psychology, sports, \textit{etc}. Given the enormous social impact and the consequent incentives, the potential adversary has a strong motivation to manipulate the ranking list. However, the ideal attack opportunity and the excessive adversarial capability cause the existing methods to be impractical. To fully explore the potential risks, we leverage an online attack on the vulnerable data collection process. Since it is independent of rank aggregation and lacks effective protection mechanisms, we disrupt the data collection process by fabricating pairwise comparisons without knowledge of the future data or the true distribution. From the game-theoretic perspective, the confrontation scenario between the online manipulator and the ranker who takes control of the original data source is formulated as a distributionally robust game that deals with the uncertainty of knowledge. Then we demonstrate that the equilibrium in the above game is potentially favorable to the adversary by analyzing the vulnerability of the sampling algorithms such as Bernoulli and reservoir methods. According to the above theoretical analysis, different sequential manipulation policies are proposed under a Bayesian decision framework and a large class of parametric pairwise comparison models. For attackers with complete knowledge, we establish the asymptotic optimality of the proposed policies. To increase the success rate of the sequential manipulation with incomplete knowledge, a distributionally robust estimator, which replaces the maximum likelihood estimation in a saddle point problem, provides a conservative data generation solution. Finally, the corroborating empirical evidence shows that the proposed method manipulates the results of rank aggregation methods in a sequential manner.
        % In this problem, the data generation process can be modeled as a sampling method responsible for sequentially collecting pairwise comparisons from the external environment. Since it is independent of rank aggregation and lacks effective protection mechanisms, the adversary has strong motivation and incentives to manipulate the sampling algorithms online. To fully expose the potential risks in this paper, we first draw a portrait of the purposeful adversary who sequentially inserts the pairwise comparisons without knowledge of future data nor the true comparison distribution. Then we prove that the well-known Bernoulli method is vulnerable to the purposeful adversary in an online game. The attack strategy of the adversary in this game is an active data generation process that can be decomposed into two sub-problems: a worst-case estimation and a dynamic rule selection. The distributionally robust optimization and mirror-descent algorithms are developed to solve the proposed sequential attack operation. 
    \end{abstract}
}

\maketitle

\IEEEdisplaynontitleabstractindextext

% \fi
\IEEEpeerreviewmaketitle

\IEEEraisesectionheading{\section{Introduction}\label{sec:introduction}}

\IEEEPARstart{R}{ank} aggregation has wide-ranging applications in social choice theory \cite{arrow2012social}, psychology \cite{skrondal2003multilevel}, economics \cite{saari2000mathematics}, statistic \cite{Jiang2011}, bioinformatic \cite{badgeley2015hybrid}, and other fields. In pursuit of large benefits, the potential attackers have strong motivations to manipulate the ranking aggregation algorithms which are utilized in high-stakes scenarios, \textit{e.g.} elections \cite{bartholdi1989voting}, sports competitions \cite{keener1993perron}, and recommendations \cite{DBLP:conf/kdd/RadlinskiJ07}. A profit-seeking adversary will try his/her best to designate the ranking list and fulfill his/her demands. In addition to statistical \cite{chen2019spectral} and computational \cite{DBLP:conf/aistats/0001SR20} properties, the integrity issue of ranking results becomes a new direction in the study of rank aggregation algorithms. 

\begin{figure*}[h!]
    \centering
    % \includegraphics[width=\textwidth]{framework.pdf}
    % \caption{Need to re-draw. At least two flows. One is the original case. The other is the adversarial case.}
    % \label{fig:framework}
    \begin{subfigure}[b]{0.49\textwidth}
        \centering
        \includegraphics[width=\textwidth]{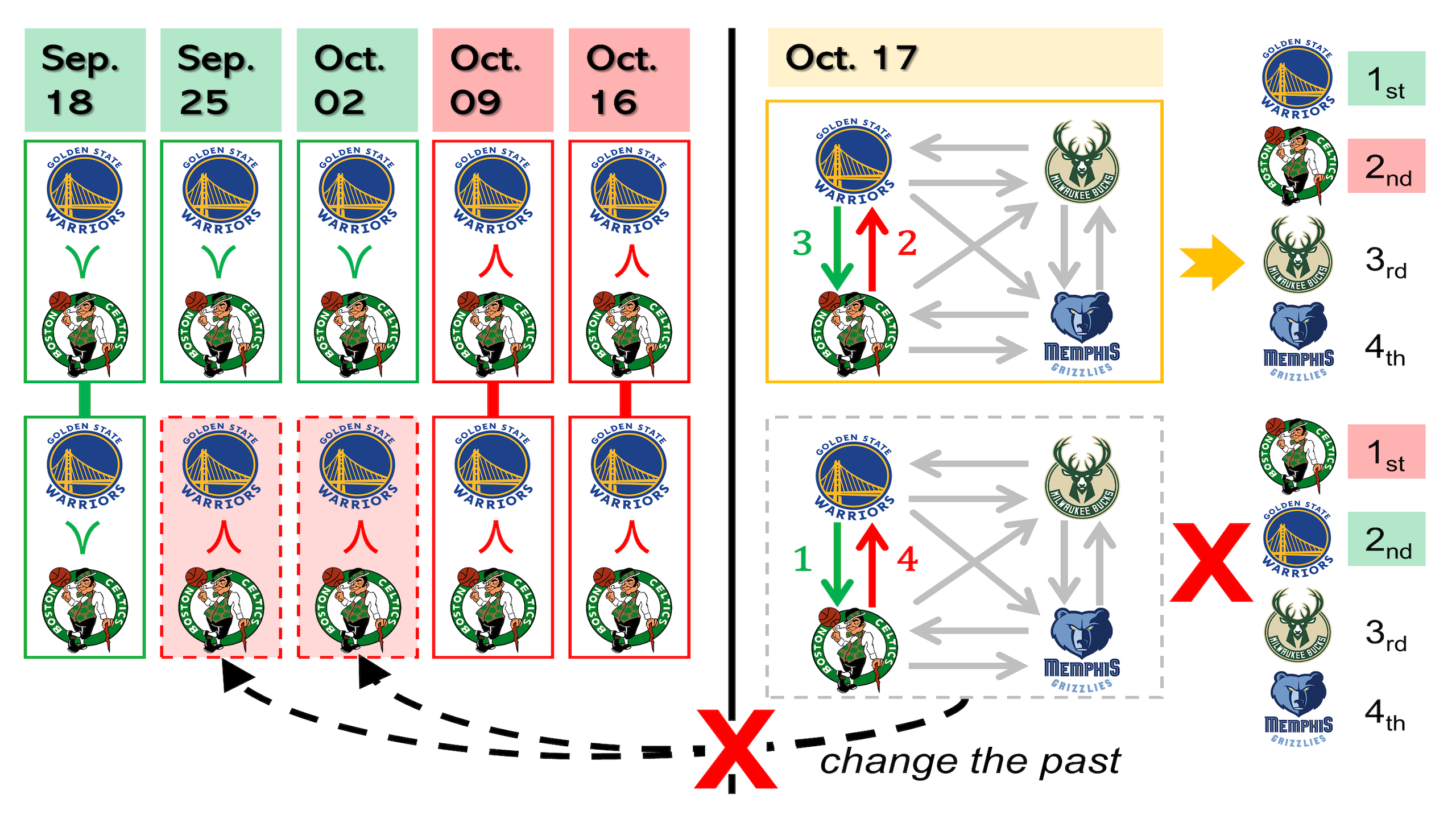}
        \caption{offline attack against rank aggregation}
    \end{subfigure}
    \hfill
    \begin{subfigure}[b]{0.49\textwidth}
        \centering
        \includegraphics[width=\textwidth]{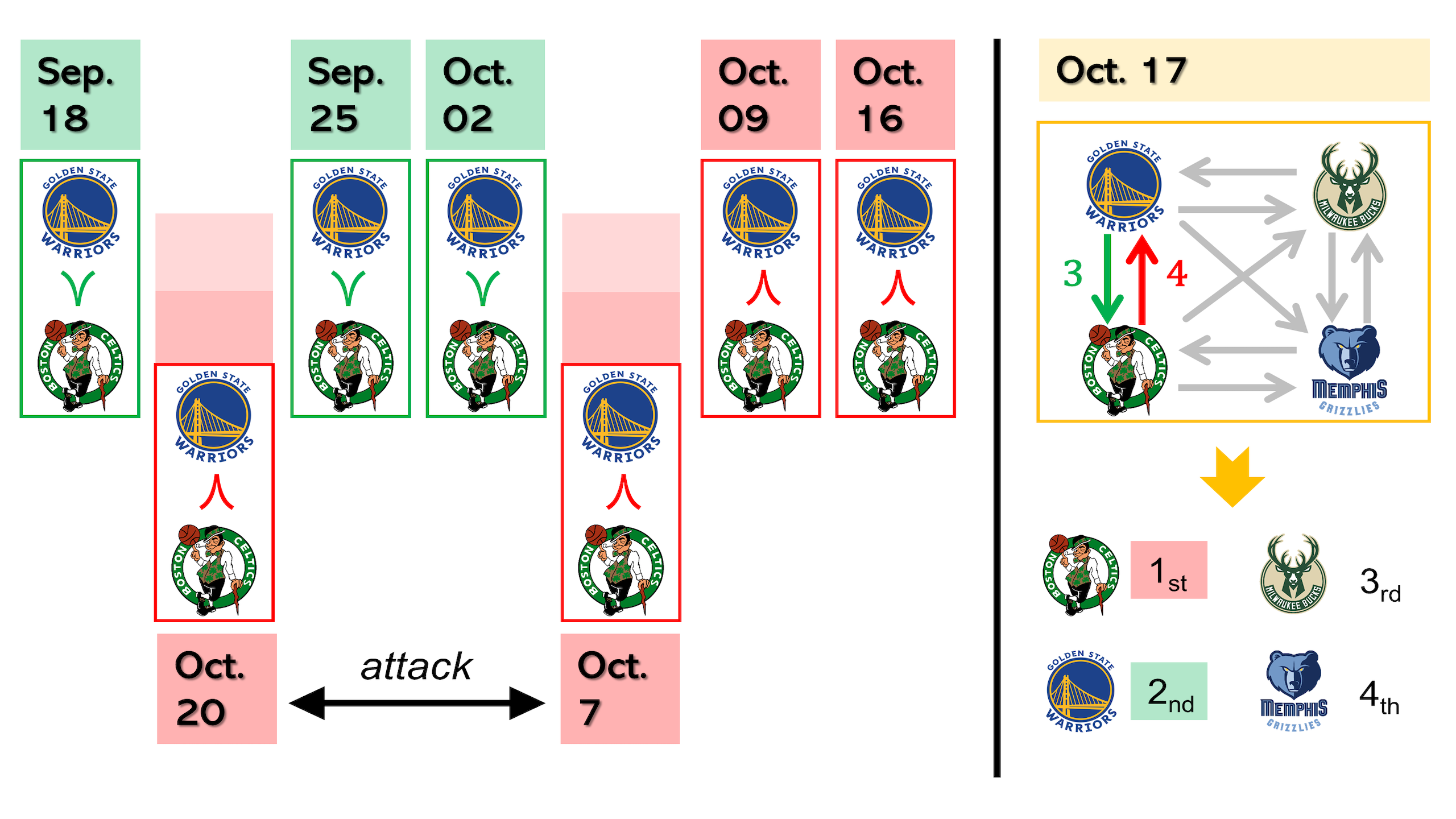}
        \caption{online attack against rank aggregation}
    \end{subfigure}
    \caption{Overview of the \textbf{\textit{offline}} and \textbf{\textit{online}} adversarial settings. (a) In the offline confrontation scenario, the adversary observes the whole comparison graph on Oct. 17 and he/her obtains the attack strategy which needs to flip the comparisons which have occurred on Sep. 25 and Oct. 2. However, no one can return to the past and change what has happened. Moreover, bypassing the defense mechanisms of the rank aggregation to modify the completed comparison graph is really a challenging task. (b) Different form the offline attack methods, we consider the sequential manipulation strategies which has no knowledge of all future observed pairwise comparisons. The proposed online attack method inserts malicious into the data stream before the construction of comparison graph.}
    \label{fig:framework}
\end{figure*}

The pioneer in conducting security-related research on rank aggregation is \cite{10.1109/TPAMI.2021.3087514}. \cite{10.1109/TPAMI.2021.3087514} develops a strong threat model for perturbing the aggregated results. The adversary has complete knowledge of the initial truthful data and corresponding feedback of victims. He/her can corrupt the original data by inserting, deleting, or flipping any pairwise comparisons with limitations on quantity of modification. \cite{10.1109/TPAMI.2021.3087514} also considers the adversary with incomplete knowledge, who lacks the preference score generated by the victims. The attack strategies are solved by maximizing the objective functions of the victims with global modification on the weights of comparison graph. Their results show that the rank aggregation algorithms are vulnerable to these attackers. Concurrent to \cite{10.1109/TPAMI.2021.3087514}, \cite{lerman2021robust} and \cite{DBLP:conf/icml/0001AKP20} restrict the modification scope and degree of weights towards specific families of comparison graphs, then provide the recovery guarantees for the ground truth ranking with the proposed procedures. It is noteworthy that these weaker threat models could not be translated into any defense mechanism against the unregulated attackers. Furthermore, \cite{9830042} poses the manipulation problem against rank aggregation algorithms. The purposeful attackers are not satisfied with simply perturbing the ranking list, but with designating it. The attack behavior with a target ranking list is a fixed point belonging to the composition of the adversary and the victim from the perspective of the dynamical system. The manipulation strategies equal to the conditions that the weights of comparison graph should satisfy when the victims obtain the target ranking list. From the above analysis, we conclude that the existing methods study the security issue of rank aggregation in an ``\underline{\textbf{\textit{offline}}}'' adversarial scenario \cite{DBLP:journals/pami/WeiGY23,DBLP:journals/pami/WeiGYZ23,DBLP:journals/pami/WeiWY23}. \ul{In general, the attackers from the existing methods try to modify the pairwise comparisons that have already been collected. These offline attacks must occur after the construction of the comparison graph and before the victims aggregate its results}. The rank aggregation algorithms would wait for the adversary to complete his/her malicious actions and unconditionally accept the modifications to the data before they can begin their own jobs. Opportunity for such attacks affords the adversary some privileges. There exists an implicit assumption that the adversary is capable of changing the existed data in the possession of the victims. \ul{However, the data held by the rank aggregation algorithms is often immutable in practice}. In sports competitions, the final ranking is only produced when all the races have been completed. Theoretically, the existing methods could perturb or manipulate the ranking lists of all teams or players. But no one can travel to the past and change the outcome of a match that has already finished. Once a vote has been cast at the polling station, the ballot will not be changed by any third party. In the partial confrontation scenarios, the existing methods assume that they have completely bypassed the constraints of time and space. Therefore, these offline methods fail to profile the capability boundary of the potential attackers and illuminate the underlying risks of ranking aggregation algorithms. 
% \textcolor{red}{However, the existing attack methods \cite{lerman2021robust,10.1109/TPAMI.2021.3087514,DBLP:conf/icml/0001AKP20,9830042} can perturb or manipulate the aggregated results only in the ideal case. In particular, these methods try to modify the pairwise comparisons that have already been collected, namely to change the topology of the comparison graph to be aggregated. The above assumption gives some privileges to the adversary. . In other words, the attack opportunity of existing methods lies between the completion of data collection and the start of aggregation, while the modification of the collected data is the corresponding attack capability. It is worth noting that aggregation methods, like voting systems, are always protected by certain physical defenses. Once the collection is complete, the paired data will receive protection from the defense mechanisms of aggregation. The existing methods must completely bypass the victim's defenses. The ideal opportunity and the excessive capability cause the existing methods to be impractical.}

To address the above challenges, we need a new online paradigm for manipulating rank aggregation algorithms. In terms of attack opportunities, attackers need to seek more chances for archiving his/her goal and bypass the time and space constraints. \ul{The whole process of obtaining a ranking list can be divided into two parts: \textbf{\textit{online}} data collection and \textbf{\textit{offline}} data aggregation. Compared to offline aggregation, online collection is much more vulnerable.} As a distributed and asynchronous process, data collection can't be done in a controlled environment and is therefore independent of rank aggregation. The defense mechanisms of aggregation often fail to protect online data collection. In addition, data collection always takes a long time and the attacker has sufficient chances to execute his/her actions. Consequently, disrupting the data collection process by online falsifying pairwise comparisons is more sophisticated than offline changing the collected data. Having determined the attack opportunity, it is necessary to identify the attacker's capabilities during data collection. During the collection process, a data source generates many pairwise comparisons waiting to be sampled. Once a comparison is sampled, it is used to construct the comparison graph and cannot be modified. If the attacker performed malicious actions during data collection, all he/she could do was mimic the behavior of normal data. The adversary could construct an adversarial data source which generates specific pairwise comparisons and insert them into the data stream. Since the cost of authenticating data sources is much greater than the cost of fabricating a pairwise comparison with malicious intention, an attacker can effectively bypass the victim's defenses. To the best of our knowledge, manipulating aggregation results by fabricating the data source and continuously injecting malicious pairwise comparisons into the data stream is a new formulation for attacking against rank aggregation algorithms, which is still under-explored.

The core of this paper is to make the above analysis rigorous by establishing a principle framework for sequential manipulating rank aggregation algorithms. The main methodology and theoretical contributions are summarized as follows.
\begin{itemize}
    %  Describe the methods used to overcome these technical difficulties, and ensure this is reflected as part of the core contribution in the introduction.
    \item Under a distributionally robust game theoretic framework, we construct the confrontation model between the online manipulator and the ranker who is bound to the original data sources. We then prove the existence of distributionally robust Nash equilibrium in such a game, which guarantees the possibility of sequential manipulation. This adversarial game describes the goal, knowledge and capability of the attacker, with particular emphasis on the uncertainty that all players must deal with. 
    \item We characterize the data collection process as a sampling algorithm and focus on two of the most basic and well-known sampling algorithms: Bernoulli sampling and reservoir sampling. Our theoretical analysis shows that the sampled results could be representative with respect to the mixture of the original and adversarial data sources. Such results suggest that the actions of adversary could resist the effects of randomness in the original data source and the data collection. 
    \item Different sequential manipulation policies are proposed under a Bayesian decision framework and a large class of parametric pairwise comparison models. The underlying Bayes risk consists of the expected Kendall's tau distance and the expected relative generation cost. We then derive the asymptotic optimality of the proposed policies with complete knowledge. 
    \item To increase the success rate of the sequential manipulation with incomplete knowledge, we empower the generation rule against uncertainty. A distributionally robust estimator replaces the maximum likelihood estimation in a saddle point optimization problem. Then the corresponding conservative generation rule is obtained by mirror descent algorithm. 
\end{itemize}

The rest of the paper is organized as follows. In Section \ref{sec:preliminary}, we introduce the basic concept of rank aggregation and two representative algorithms as \textbf{HodgeRank} \cite{Jiang2011} and \textbf{RankCentrality} \cite{DBLP:journals/ior/NegahbanOS17}. Section \ref{sec:3} establishes the general framework for sequential manipulating rank aggregation algorithms. We present the details of manipulation strategies and the theoretical results in Section \ref{sec:adv_gen_pro}. Section \ref{sec:experiment} illustrates the simulated and real-world data results, followed by concluding remarks in Section \ref{sec:conclusion}. Technical proofs are provided in the supplementary material.

\section{Preliminary}
\label{sec:preliminary}
We begin with a formal description of the parametric model for binary comparisons, \textit{a.k.a} Bradley-Terry-Luce (\textbf{BTL}) model \cite{bradley1952rank}. Then we revisit the comparison graph and the Laplacian matrix which are essential for the ranking algorithms tailored to the \textbf{BTL} model. Two popular approaches which rank the items based on appropriate estimation of the latent preference scores, named \textbf{HodgeRank} \cite{Jiang2011} and \textbf{Rank Centrality} \cite{DBLP:journals/ior/NegahbanOS17}, are chosen as the victims to motivate our target attack strategies. In the remainder of this paper, we will use positive integers to indicate alternatives and voters. Let $\boldsymbol{V}$ be the set $[n]=\{1,\ \dots,\ n\}$ which denotes a set of alternatives to be ranked. $\boldsymbol{U}=\{\boldsymbol{u}_1,\ \dots,\ \boldsymbol{u}_m\}$ denotes a set of voters. We will adopt the following notation from combinatorics:
\begin{equation*}
    \begin{bmatrix}\boldsymbol{V} \\
    l \end{bmatrix}:=\text{set of all}\ l\ \text{elements subset of}\ \boldsymbol{V}.
\end{equation*}
In particular, 
\begin{equation*}
    \begin{aligned}
        &  \begin{bmatrix}\boldsymbol{V} \\
            2 \end{bmatrix}&:=&\ \ \text{set of all unordered pairs of elements of}\ \boldsymbol{V}\\
        & &:=&\left\{[i,\ j]\ \Big\vert\ \forall\ i,\ j\in\boldsymbol{V},\ i\neq j\right\}.
    \end{aligned}
\end{equation*}
Moreover, for any $i,j\in\boldsymbol{V},\ i\neq j$, we write $i\succ j$ to mean that alternative $i$ is preferred over alternative $j$. Such a comparison could be converted into an ordered pair $(i,\ j)$. The set of ordered pair will be denoted as
\begin{equation*}
    \boldsymbol{V}\times\boldsymbol{V}:=\left\{(i,\ j)\ \Big\vert\ i\succ j,\ \forall\ i,\ j\in\boldsymbol{V},\ i\neq j\right\}. 
\end{equation*}
Ordered and unordered pairs will be delimited by parentheses $(i,j)$ and braces $\{i,j\}$ respectively. If we wish to emphasize the preference judgment from a particular voter $\boldsymbol{u}\in\boldsymbol{U}$, we will write $i\succ_{\boldsymbol{u}} j$. 

\subsection{Parametric Model and Pairwise Comparisons}

Given a collection $\boldsymbol{V}$ of $n$ alternatives, the parametric model of pairwise comparisons assumes that each $i\in\boldsymbol{V}$ has a certain numeric quality score $\theta^*_i$. Suppose that $\boldsymbol{\theta}^*\in\mathbb{R}^n$
\begin{equation}
    \boldsymbol{\theta}^* = \left[\theta^*_1,\ \dots,\ \theta^*_n\right]^\top
\end{equation}
comprises the underlying preference scores assigned to each of the $n$ items. Without loss of generality, $\boldsymbol{\theta}^*$ could be positive as 
\begin{equation*}
    \theta^*_i > 0,\ \forall\ i\in[n].
\end{equation*}
Specifically, a comparison of any pair $\{i,j\}\in\begin{bmatrix}\boldsymbol{V}\\2 \end{bmatrix}$ is generated via the comparing between the corresponding scores $\theta^*_i$ and $\theta^*_j$ (in the presence of noise) by the \textbf{BTL} model. Let $y^*_{ij}$ denote the outcome of the comparison of the pair $i$ and $j$ based on $\boldsymbol{\theta}^*$, such that $y^*_{ij}=1$ if $i$ is preferred over $j$ and $y^*_{ij}=-1$ otherwise. Then, according to the \textbf{BTL} model,
\begin{equation}
    \label{eq:BTL}
    y^*_{ij} = \left\{
    \begin{array}{rl}
        1,  & \text{with probability}\ \theta^*_i/(\theta^*_i+\theta^*_j),\\[5pt]
        -1, & \text{otherwise}.
    \end{array}
    \right.
\end{equation}
Since the \textbf{BTL} model is invariant under the scaling of the scores, the latent preference score is not unique. Indeed, under the \textbf{BTL} model, a score vector $\boldsymbol{\theta}^*\in\mathbb{R}^n_+$ is the equivalence class
\begin{equation*}
    \boldsymbol{\Theta}^*=\left\{\boldsymbol{\theta}\ \Big\vert\ \text{there exists}\ \alpha>0\ \text{such that}\ \boldsymbol{\theta}=\alpha\boldsymbol{\theta}^*\right\}.
\end{equation*}
The outcome of a comparison depends on the equivalence class $\boldsymbol{\Theta}^*$.

\subsection{Comparison Graph and Combinatorial Laplacian}
A graph structure, named comparison graph, arises naturally from pairwise comparisons as follows. Let $\boldsymbol{\mathcal{G}}=(\boldsymbol{V},\boldsymbol{E})$ stand for a comparison graph, where the vertex set $\boldsymbol{V}=[n]$ represents the $n$ candidates. In our problem setting, we pay attention to the complete graph setting: the directed edge set $\boldsymbol{E}=\boldsymbol{V}\times\boldsymbol{V}$ and $N:=|\boldsymbol{E}| = n(n-1)$. One can further associate weights $\boldsymbol{w}^*$ on $\boldsymbol{E}$ as voters $\boldsymbol{U}$ would have rated, \textit{i.e.} assigned cardinal scores or given an ordinal ordering to, the complete set of the alternatives $\boldsymbol{V}$. But no matter how incomplete the rated portion is, one may always convert such judgments into pairwise rankings that have no missing values as follows. For each voter $\boldsymbol{u}\in\boldsymbol{U}$, the pairwise ranking matrix is a skew-symmetric matrix $\boldsymbol{Y}^{\boldsymbol{u}}=\{y^{\boldsymbol{u}}_{ij}\}\in\{-1,0,1\}^{n\times n}$ as
\begin{equation}
    y^{\boldsymbol{u}}_{ij} = -y^{\boldsymbol{u}}_{ji},\ \forall\ (i,\ j)\in\boldsymbol{E},\ \forall\ \boldsymbol{u}\in\boldsymbol{U},
\end{equation}
where
\begin{equation}
    \label{eq:direction_ind}
    y^{\boldsymbol{u}}_{ij} = \left\{
    \begin{array}{rl}
        1, & \text{if}\ i\succ_{\boldsymbol{u}} j,\\[5pt]
        -1,& \text{if}\ j\succ_{\boldsymbol{u}} i,\\[5pt]
        0, & \text{otherwise}.
    \end{array}
    \right.
\end{equation}
Furthermore, we associate weight with each directed edge as $\boldsymbol{w}^*=[w^*_{12},w^*_{13},\dots,w^*_{n,n-1}]^\top\in\mathbb{Z}_+$
\begin{equation}
    \label{eq:weight}
    w^*_{ij}:= \underset{\boldsymbol{u}\in\boldsymbol{U}}{\sum}\ \mathbb{I}[y^{\boldsymbol{u}}_{ij}> 0] + \mathbb{I}[y^{\boldsymbol{u}}_{ji}<0], 
\end{equation}
where $\mathbb{I}[\cdot]$ is the Iverson bracket. Consequently, we can represent any pairwise ranking data as a comparison graph $\boldsymbol{\mathcal{G}}$ with edge weights $\boldsymbol{w}^*$.

Given a graph $\boldsymbol{\mathcal{G}}$ and weights $\boldsymbol{w}^*$, it is common to consider the weight matrix $\boldsymbol{W}^*$ with $w^*_{ij}$ as matrix elements, as well as the diagonal degree matrix $\boldsymbol{D}^*=\textbf{diag}(d^*_1,\dots,d^*_n)$ given by $d^*_{i} = \sum_{j\in\boldsymbol{V}}w^*_{ij}$, which represents the volume taken by each node in the graph $\boldsymbol{\mathcal{G}}$. The combinatorial Laplacian $\mathcal{L}_0$ is defined as
\begin{equation}
    \mathcal{L}_0 = \boldsymbol{D}^* - \boldsymbol{W}^*. 
    % = \underset{(i,j)\in\boldsymbol{E},i>j}{\sum}\big(\boldsymbol{e}_i-\boldsymbol{e}_j\big)\big(\boldsymbol{e}_i-\boldsymbol{e}_j\big)^\top,
\end{equation}
% where $\{\boldsymbol{e}_i\}_{i\in\boldsymbol{V}}$ are the standard basis vectors in $\mathbb{R}^n$. 
In both solving process and the theoretical analysis, the combinatorial Laplacian $\mathcal{L}_0$ plays a vital role in the popular approaches based on the parametric model.

\noindent\textbf{Remark 1.} In this paper, we select \textbf{HodgeRank} \cite{Jiang2011} and \textbf{RankCentrality} \cite{DBLP:journals/ior/NegahbanOS17} to verify that online manipulation behavior is a potentially significant threat to rank aggregation methods. This is due to the following considerations. First, these two representative methods that have received much recent attention have been well studied by \cite{Jiang2011,chen2019spectral,10.1214/21-AOS2166} and their theoretical properties guarantee the promising recovery performance. The successful manipulation will be in stark contrast to the original aggregated results. Second, the variants of \textbf{HodgeRank} and \textbf{RankCentrality} are hot topics of the literature \cite{doi:10.1287/opre.2022.2313,10.1214/22-AOS2175,pmlr-v180-li22g,pmlr-v162-bong22a}. The online attack method proposed in this paper has a large potential victimization. Third, the destructive results of these two estimators for the famous Bradley–Terry–Luce (\textbf{BTL}) model will prompt researchers to focus on the security issue of rank aggregation in the high-stakes applications. 

\noindent\textbf{Remark 2.} When exists a purposeful adversary, the collected pairwise comparisons would be a mixture of the data which supports the original ranking list and the fabricated data by the adversary. To manipulate the aggregated results, the attacker will predict the ranker's behavior with incomplete information and fabricate the suitable pairwise comparisons. Therefore, we need a mathematical tool to formulate the ranker's and the adversary's behaviors, which has been extensively modeled as a two-player, non-cooperative game in the adversarial learning\cite{DBLP:journals/aim/DasguptaC19}. 
Specifically, the confrontation scenario between the online manipulator and the ranker who takes control of the original data source is formulated as a distributionally robust game that deals with the uncertainty of knowledge. The ranker's set of actions corresponds to selecting pairwise comparisons and minimizing the difference between the aggregation result and the original ranking list. Meanwhile, the adversary's set of actions corresponds to generate pairwise comparisons and minimizing the difference between the aggregation result and the desired ranking list. For two players, the upcoming data is the uncertain knowledge. 

\noindent\textbf{Remark 3.} Although the {\textbf{\textit{offline}}} \cite{9830042} and {\textbf{\textit{online}}} attackers have the same goal, different behavioral patterns result in the two having different knowledge and capabilities. Specifically, let $T_0$ be the stopping time of data collection, the {\textbf{\textit{offline}}} attacker has full/partial knowledge of the comparison graph weight $\boldsymbol{w}_{\boldsymbol{\mathcal{A}}}(T_0)$. {Then the {\textbf{\textit{offline}}} manipulator has the ability to modify $\boldsymbol{w}(T_0)$ in its entirety}, {\textbf{\textit{increasing}}} or {\textbf{\textit{decreasing}}} the values of $\boldsymbol{w}_{\boldsymbol{\mathcal{A}}}(T_0)$ at arbitrary position\footnote{Please see Eq. (53)-(56), (73) and (81) of \cite{9830042} for the detailed utilization of the {\textbf{\textit{offline}}} knowledge.}. Meanwhile, the {\textbf{\textit{online}}} manipulator of this paper sequentially obtains his/her knowledge but knows nothing about the forthcoming pairwise comparisons. More importantly, the {\textbf{\textit{online}}} attacker will execute his/her strategy based on the knowledge $\boldsymbol{w}_{\boldsymbol{\mathcal{A}}}(t)$ at each time step $t$ instead of waiting for the moment $T_0$. Thus, the greatest limitation on the ability of the {\textbf{\textit{online}}} attacker is that he/she can only {\textbf{\textit{insert}}} fabricated pairwise comparisons. The {\textbf{\textit{online}}} attack paradigm could bypass the existing defense mechanisms of rank aggregation algorithms and break the barrier of time. We provide an example in Fig. \ref{fig:framework}. It is noteworthy that utilizing the {\textbf{\textit{offline}}} method at each time step can't achieve a similar result as the {\textbf{\textit{online}}} method, since the {\textbf{\textit{offline}}} method does not guarantee that \textbf{\textit{the collected data keep unchanged}}.

\noindent\textbf{Remark 4.} In order to accomplish an effective online attack without modifying the collected data, the adversary will generate the most destructive data to inject based on the current partial information and stop when the ambiguity of ranking list falls below a certain level. This paper develops a general framework against the parametric models of rank aggregation, especially the \textbf{BTL} model. The proposed adversarial generation process, corresponding to the third core contribution, can designate the leading candidate of the aggregated ranking lists by \textbf{HodgeRank} and \textbf{RankCentrality}. In addition, the offline attack methods \cite{10.1109/TPAMI.2021.3087514,9830042} cannot yield the available attack results in the online manipulation setting of this paper. 
% Describe the methods used to overcome these technical difficulties, and ensure this is reflected as part of the core contribution in the introduction.

\section{General Framework}
    \label{sec:3}
    % In this section, we systematically introduce the general framework for sequential manipulating against pairwise ranking. To mathematically characterize the interaction between the online data collection and the adversary, we profile the online attacker's goal, knowledge and capability in Sec. \ref{sec:threat_model}. Based on the above threat modelling, we treat the online data collection and the attacker as the original and adversarial data sources respectively. Then we develop the game-theoretic formulation between two data sources in Sec. \ref{sec:drg} with particular emphasis on the uncertainty that the online attacker must deal with. Meanwhile the existence of the distributionally robust Nash equilibrium is also established. Furthermore, we are interested in whether the underlying equilibrium state favors the adversarial data source. To answer this question, we analyze the outcomes after the adversarial games which equal to the output of some sampling algorithms in \ref{sec:sampling}. The theoretical results show that the adversarial data source has the opportunity to dominate the outcomes after the adversarial games, which indicate the successful manipulation.
    In this section, we systematically introduce the general framework for sequential manipulating against pairwise ranking algorithms. To mathematically characterize the successive interaction between the manipulator and the victims, we perform threat modeling to profile the attacker's goal, knowledge and capability in Sec. \ref{sec:threat_model} and dissect the online adversarial behavior. Then we develop the game-theoretic formulation between the online adversary and the offline rank aggregation procedure in Sec. \ref{sec:drg} with particular emphasis on the uncertainty that the online manipulator must deal with. Such a game with fundamental uncertainty about future data and the opponent's strategies and the settings of \cite{10.1109/TPAMI.2021.3087514,9830042} are significantly different. Meanwhile the existence of the distributionally robust Nash equilibrium is also established. 
    % Based on the above threat modelling, we treat the online data collection and the attacker as the original and adversarial data sources respectively.
    % Meanwhile the existence of the distributionally robust Nash equilibrium is also established. Furthermore, we are interested in whether the underlying equilibrium state favors the adversarial data source. To answer this question, we analyze the outcomes after the adversarial games which equal to the output of some sampling algorithms in \ref{sec:sampling}. The theoretical results show that the adversarial data source has the opportunity to dominate the outcomes after the adversarial games, which indicate the successful manipulation.

% In this section, we formulate the online adversarial interaction between the purposeful attacker and the pairwise comparison generation process. Specifically, we define the threat model of the purposeful adversary (Sec. \ref{sec:threat_model}), establish a distributionally robust game between difference data source (Sec. \ref{sec:drg}) and present the existence of equilibrium which will be favor to the attacker (Sec. \ref{sec:sampling}). 

    \subsection{Threat Model of Online Adversary}
    \label{sec:threat_model}
    Here we present the threat model of the manipulator to specify his/her goal, knowledge and capability with online behavioral pattern. The threat model helps to establish the online interactions between the purposeful attacker and the rank aggregation with pairwise comparisons. 
    
    \vspace{2pt}
    \noindent\textbf{The Goal of Online Adversary.}
    Inducing the threatened rank aggregation approaches to produce the designated ranking is the goal of a manipulator. On the one hand, the adversary cannot interact directly with the threatened rank aggregation procedure due to the inevitable defense mechanisms. On the other hand, the collection of pairwise comparisons is an online process which is independent of the subsequent rank aggregation method. It often takes place in open environments and lacks adequate supervision. If the attacker could interfere with the data collection procedure, he/she has a high possibility of bypassing defense mechanisms and accomplishing manipulation.  The data collection procedure is always treated as a random sampling process. All possible pairwise comparisons consist of the data stream. A random sampling algorithm will receive and choose the data which constructs the comparison graph. To archive manipulation, the adversary proactively disguises the crafted malicious data as part of the data stream. Then these malicious data could be adopted by sampling algorithms and used to construct a comparison graph. After sampling, the ranker produces the aggregated result based on the comparison graph. These sequential actions of the adversary will induce the ranker to produce a designated ranking result. If the ranking list meets the demand of adversary, we will say that the adversary has executed a successful manipulation.

    We denote $\boldsymbol{\mathcal{A}}$ and $\boldsymbol{\mathcal{R}}$ be the adversary and the ranker respectively. Let $\boldsymbol{C}=\{c_1,c_2,\dots\}$ be a sequence of recurring pairwise comparisons involving at most $n$ candidates. The perturbed sequence by $\boldsymbol{\mathcal{A}}$ is $\boldsymbol{C}'$. The sequence of pairwise comparisons will be transferred into the comparison graphs as \eqref{eq:weight}. Suppose that $\boldsymbol{\mathcal{G}}(\boldsymbol{C})$ is the comparison graph constructed by $\boldsymbol{C}$. The relative ranking scores $\boldsymbol{\theta}$ and $\boldsymbol{\theta}'$ are produced by $\boldsymbol{\mathcal{R}}$ with $\boldsymbol{\mathcal{G}}(\boldsymbol{C})$ and $\boldsymbol{\mathcal{G}}(\boldsymbol{C}')$ accordingly. The non-adversarial rank aggregation can be portrayed as
    \begin{equation}
        \label{eq:normal_case}
        \boldsymbol{\mathcal{R}}(\boldsymbol{\mathcal{G}}(\boldsymbol{C})) = \boldsymbol{\theta}.\ 
    \end{equation}
    % where
    % \begin{equation}
    %     \boldsymbol{\mathcal{S}}(\boldsymbol{C}) = \boldsymbol{\mathcal{G}}.
    % \end{equation}
    Then the rank aggregation result under online manipulation strategies would be
    \begin{equation}
        \boldsymbol{\mathcal{R}}(\boldsymbol{\mathcal{G}}(\boldsymbol{C}')) = \boldsymbol{\theta}'.
    \end{equation}
    % where
    % \begin{equation}
    %     \boldsymbol{\mathcal{S}}(\boldsymbol{C}') =\boldsymbol{\mathcal{G}}',\ \boldsymbol{\mathcal{A}}(\boldsymbol{C}) =\boldsymbol{C}'.
    % \end{equation}
    Although $\boldsymbol{\mathcal{A}}$ is able to achieve multiple objectives with the help of $\boldsymbol{\theta}'$, designating the winner will be the most desired achievement of $\boldsymbol{\mathcal{A}}$. Therefore, we consider the following scenario: after the action of $\boldsymbol{\mathcal{R}}$, it holds that
    \begin{equation}
        \label{eq:goal}
        \boldsymbol{\theta}'\in\boldsymbol{\Theta}_{\boldsymbol{\mathcal{A}}} := \left\{\boldsymbol{\theta} \in \mathbb{R}^n\ \Big\vert\ \underset{i\in[n],\ i\neq i_0}{\textbf{max}}\theta_i \leq \theta_{i_0}\right\},
    \end{equation}
    where $i_0$ is the winner candidate desired by $\boldsymbol{\mathcal{A}}$. Then we will say that \textit{$\boldsymbol{\mathcal{A}}$ has a successful online manipulating strategy against $\boldsymbol{\mathcal{R}}$ by substituting $\boldsymbol{\mathcal{G}}(\boldsymbol{C})$ with $\boldsymbol{\mathcal{G}}(\boldsymbol{C}')$ through \textbf{sequential behavior}}. It is noteworthy that the goal in this paper implicitly requires sequential/online attack behavior, while \cite{9830042} needs the help of offline manipulation strategy. The differences between online and offline strategies are shown in the following parts.
    % It is noteworthy that the goals of the attacker in this paper and \cite{9830042} are similar. However, 
    % the knowledge and capability of the attacker in this paper are distinctively different from these of \cite{9830042} and we will show the differences in the subsequent parts of this section. 

    \vspace{2pt}
    \noindent\textbf{The Knowledge of Online Adversary. }
    Let 
    \begin{equation}
        \boldsymbol{C}(T)=\{c_1,c_2,\dots,c_T\}
    \end{equation}
    be a sub-sequence of $\boldsymbol{C}$ with its first $T$ pairwise comparisons. Without loss of generality, the number of pairwise comparisons in $\boldsymbol{C}$ will be increased by $1$ at each step as
    \begin{equation}
        \boldsymbol{C}(T) = [\boldsymbol{C}(T-1), c_{T}],\ T\in\mathbb{N}. 
    \end{equation}
    As a consequence, the knowledge of $\boldsymbol{\mathcal{A}}$ at $T$ step, denoted $\boldsymbol{C}_{\boldsymbol{\mathcal{A}}}(T)=[\boldsymbol{C}_{\boldsymbol{\mathcal{A}}}{(T-1)}, \varphi(c_{T})]$, contains two parts: 
    \begin{itemize}
        \item a subset of $\boldsymbol{C}(T-1)$ as  
        \begin{equation}
            \boldsymbol{C}_{\boldsymbol{\mathcal{A}}}{(T-1)}\subseteq\boldsymbol{C}{(T-1)}, 
        \end{equation}
        \item and the state of $c_{T}$:
        \begin{equation}
            \varphi(c_{T}) = \begin{cases}
            c_{T}, &\ \text{if}\ \boldsymbol{\mathcal{A}}\ \text{obtains}\ c_T,\\ 
            \varnothing, & \text{otherwise,}
            \end{cases}
        \end{equation}
        where $\varnothing$ indicates that no pairwise comparisons will enter the sequence.  
    \end{itemize}
    
    Based on the completeness of $\boldsymbol{C}_{\boldsymbol{\mathcal{A}}}{(T)}$, we consider the following two adversarial scenarios:
    \begin{enumerate}
        \item[ i)] If it holds that
        \begin{equation}
            \boldsymbol{C}_{\boldsymbol{\mathcal{A}}}{(T)} = \boldsymbol{C}{(T)},\ \forall\ T\in\mathbb{N}, 
        \end{equation}
        that is $\boldsymbol{C}_{\boldsymbol{\mathcal{A}}}{(T-1)}=\boldsymbol{C}{(T-1)}$ and $\varphi(c_{T})=c_{T}$, $\forall\ T\in\mathbb{N}$, we say that ${\boldsymbol{\mathcal{A}}}$ has the \textbf{\textit{complete knowledge}}.
        \item[ii)] If there exists a time step $T$ such that
        \begin{equation}
            \boldsymbol{C}_{\boldsymbol{\mathcal{A}}}{(T)} \subset \boldsymbol{C}{(T)},
        \end{equation}
        we say that $\boldsymbol{\mathcal{A}}$ has \textbf{\textit{incomplete knowledge}}. Limited by time and cost, the incomplete state will be held throughout the whole adversarial operation. 
    \end{enumerate}
    Special attention needs to be paid to the fact that the online manipulator in this paper lacks prior information of subsequent data $\boldsymbol{C}/\boldsymbol{C}^{(T)}$ at $T$ step. The offline manipulator of \cite{10.1109/TPAMI.2021.3087514,9830042}, on the other hand, doesn't need the prior information but requires the length of $\boldsymbol{C}$ to no longer grow, \textit{i.e.} there exist a step $T_0$ such that $|\boldsymbol{C}| = T_0$. Consequently, \textbf{\textit{the offline adversary in \cite{10.1109/TPAMI.2021.3087514,9830042} is a special case of the online manipulator, who is the online manipulator at the step that all pairwise comparisons have been collected}}. Such a distinction will affect the abilities of the offline and online attackers.

    \vspace{2pt}
    \noindent\textbf{The Capability of Online Adversary. }The above goal and knowledge empower the online attacker with completely divergent capabilities from those of the offline attacker. The online manipulator $\boldsymbol{\mathcal{A}}$ is able to insert arbitrary pairwise comparisons into the data stream. Then the perturbed data will replace to produce the comparison graph for rank aggregation. More specifically, the fabricated pairwise comparisons with the knowledge $\boldsymbol{C}_{\boldsymbol{\mathcal{A}}}{(T)}$ is   
    \begin{equation}
        \label{eq:cap}
        \boldsymbol{c}'{(\boldsymbol{C}_{\boldsymbol{\mathcal{A}}}{(T)})} = [c'(1),\dots,c'(a_T)],
    \end{equation}
    where $a_T$ is the maximum number of possible insertions at $T$ step. This sequence \eqref{eq:cap} reflects $\boldsymbol{\mathcal{A}}$'s capability. It is noticed that $\boldsymbol{\mathcal{A}}$ is unable to change the pairwise comparisons in $\boldsymbol{C}_{\boldsymbol{\mathcal{A}}}{(T)}\subseteq\boldsymbol{C}{(T)}$. However, in addition to insertion, the offline attackers of \cite{10.1109/TPAMI.2021.3087514,9830042} could delete or flip a pairwise comparison $c_t\in\boldsymbol{C}_{\boldsymbol{\mathcal{A}}}(T)\subseteq\boldsymbol{C}(T), \forall\ t\leq T$ even through $c_t$ is generated in the past or protected by the defense mechanisms. Therefore, \textbf{\textit{the online attacker is more restricted than its offline counterpart}}. The observed sequence of pairwise comparisons for $\boldsymbol{\mathcal{R}}$ at $T$ step is 
    \begin{equation}
        \label{eq:obv_seq}
        \boldsymbol{C}'(T) = (\boldsymbol{C}'(T-1), \boldsymbol{C}(T)/\boldsymbol{C}_{\boldsymbol{\mathcal{A}}}{(T)}, \boldsymbol{c}'{(\boldsymbol{C}_{\boldsymbol{\mathcal{A}}}{(T)})}).
        % \left\{\left[\left(\boldsymbol{C}(1), \boldsymbol{c}'{(\boldsymbol{C}_{\boldsymbol{\mathcal{A}}}{(1)})}\right),\dots, \right\}
        % \left\{\bigcup_{t=1}^{T_0}\boldsymbol{C}(t)/\boldsymbol{C}_{\boldsymbol{\mathcal{A}}}{(t)}\right\}\bigcup\left\{\bigcup_{t=1}^{T_0}\boldsymbol{c}'{(\boldsymbol{C}_{\boldsymbol{\mathcal{A}}}{(t)})}\right\},
    \end{equation}
    \eqref{eq:obv_seq} is a mixture of the collected data $\boldsymbol{C}'(T-1)$, the data inaccessible to attackers $\boldsymbol{C}(T)/\boldsymbol{C}_{\boldsymbol{\mathcal{A}}}{(T)}$, and the fabricated pairwise comparisons \eqref{eq:cap}.
    % where $T_0$ is the time to stop data collection. 
    % To fully reveal the potential risks in the collection of pairwise comparisons, we entitle $\boldsymbol{\mathcal{A}}$ with the capability of modifying the data stream. $\boldsymbol{\mathcal{A}}$ is able to insert arbitrary pairwise comparisons into the data stream before $\boldsymbol{\mathcal{S}}$ ends the action of sampling, which depends on the goal and knowledge of $\boldsymbol{\mathcal{A}}$. It means that 
    % \begin{itemize}
    %     \item $\boldsymbol{\mathcal{A}}$ utilizes the knowledge and goal to compute some pairwise comparisons and insert them into the data stream;
    %     \item $\boldsymbol{\mathcal{S}}$ decides to update the comparison graph with these data or not. 
    % \end{itemize}

    \subsection{Distributionally Robust Game between the Ranker and the Online Adversary}
    \label{sec:drg}

    With the above threat modelling, we can further understand the adversarial scenario from a game-theoretic perspective. When there exists $\boldsymbol{\mathcal{A}}$, the pairwise comparisons for $\boldsymbol{\mathcal{R}}$ come form two sources: the original data stream $\boldsymbol{C}$ and the fraud data $\boldsymbol{C}'/\boldsymbol{C}$. Due to the extreme difficulty of identifying the possible sources of pairwise comparisons, $\boldsymbol{\mathcal{R}}$ is only able to aggregate $\boldsymbol{C}'$ and obtain a ranking list $\boldsymbol{\theta}'$, which is different with the result $\boldsymbol{\theta}$ form $\boldsymbol{C}$. However, the existence of normal data stream $\boldsymbol{C}$ will alleviate the impact of $\boldsymbol{\mathcal{A}}$ on $\boldsymbol{\mathcal{R}}$ and try to keep $\boldsymbol{\theta}'$ away from $\boldsymbol{\Theta}_{\boldsymbol{\mathcal{A}}}$. With the help of defense and protection mechanisms, we believe that $\boldsymbol{\mathcal{R}}$ will select pairwise comparisons that will preserve the original result $\boldsymbol{\theta}$. At the same time, $\boldsymbol{\mathcal{A}}$ needs to induce $\boldsymbol{\mathcal{R}}$ with the interference of $\boldsymbol{C}$ and make $\boldsymbol{C}'$ sufficient to support $\boldsymbol{\theta}'\in\boldsymbol{\Theta}_{\boldsymbol{\mathcal{A}}}$. As a consequence, this adversarial scenario is a game between $\boldsymbol{\mathcal{R}}$ and $\boldsymbol{\mathcal{A}}$ who choose pairwise comparisons to produce the desired ranking results.

    To establish the adversarial game of the online adversary and the ranker, we first transfer the sequence of pairwise comparison at $T$ step $\boldsymbol{C}(T)=\{c_1,\dots,c_t,\dots,c_T\}$ into a comparison graph $\boldsymbol{\mathcal{G}}(\boldsymbol{C}(T))=\{\boldsymbol{V},\boldsymbol{E},\boldsymbol{w}(T)\}$. The vertex set $\boldsymbol{V}$ is the set of all candidates as $\boldsymbol{V} = [n]$ and the edge set $\boldsymbol{E}$ contains all directed edges between any pair of candidates as
    \begin{equation}
         \boldsymbol{E} = \{i\rightarrow j|i,\ j\in[n],\ i\neq j\}.
    \end{equation}
    Here $\boldsymbol{w}(T) = \{w_{1,2}(T),w_{1,3}(T),\dots,w_{n,n-1}(T)\}$ is the weights of $\boldsymbol{E}$ in $\boldsymbol{C}(T)$ as
    \begin{equation}
        \label{eq:sequence_to_weight}
        w_{i,j}(T) = \sum_{t=1}^{T}\ \mathbbm{I}[c_t = (i,j)]
    \end{equation}
    where $\mathbbm{I}[\cdot]$ is the Iverson bracket. 
    % When the sequence of sampled data $\boldsymbol{C}^{(t)}$ from $\boldsymbol{\mathcal{C}}$ transforms into the comparison graph $\boldsymbol{\mathcal{G}}^{(t)}=\{\boldsymbol{V},\boldsymbol{E},\boldsymbol{w}^{(t)}\}$ as \eqref{eq:sequence_to_weight}, we know that the weight $\boldsymbol{w}^{(t)}$ also comes from the data source $\boldsymbol{\mathcal{C}}$. 
    The weight $\boldsymbol{w}{(T)}$ represents how often a pairwise comparison occurs in $\boldsymbol{C}(T)$. 
    % Since the ranking result of the parametric models is scale-invariant, we change $\boldsymbol{w}{(T)}$ to a continuous variable as
    % \begin{equation}
    %     \boldsymbol{w}{(T)} \leftarrow \frac{1}{\|\boldsymbol{w}{(T)}\|_1}\cdot\boldsymbol{w}{(T)}.
    % \end{equation}
    Furthermore, $\boldsymbol{w}(T)$ can be treated as a random variable defined on the probability space $(\boldsymbol{\mathfrak{C}}, \boldsymbol{\mathcal{E}}, \mathbb{P})$:
    \begin{equation}
        \label{eq:distribution}
        \boldsymbol{w}(T)\sim\mathbb{P},
    \end{equation}
    where $\boldsymbol{\mathfrak{C}}$ is the sample space of all possible pairwise comparisons involving $n$ candidates
    \begin{equation}
        \boldsymbol{\mathfrak{C}} = \{(i,j)\ |\ i,j\in[n],\ i\neq j\},
    \end{equation}
    $\boldsymbol{\mathcal{E}}$ is the event space of all sequences with length $T$ and $\mathbb{P}$ is a probability function. Consequently, the data sequence $\boldsymbol{C}(T)$ is associated with data distribution $\mathbb{P}$ which describes the occurrence frequency of different pairwise comparisons. 

    A notable characteristic of the adversarial game in this paper is that the decision-making processes of $\boldsymbol{\mathcal{R}}$ and $\boldsymbol{\mathcal{A}}$ involve uncertainty: the true distribution of the mixed sequence (the observed weight) is unknown to both $\boldsymbol{\mathcal{R}}$ and $\boldsymbol{\mathcal{A}}$ during the whole procedure. All players bear the risk due to the uncertainty. Then the resulting Nash equilibrium may be different from the equilibrium with the true distribution. The uncertainty drives $\boldsymbol{\mathcal{R}}$ and $\boldsymbol{\mathcal{A}}$ to adopt conservative strategies. To be more specific, we consider the following Nash equilibrium problem: at any time step, each player needs to make decisions prior to the realization of uncertainty by minimizing their expected dis-utility with the most pessimistic situation:
    \begin{equation}
        \label{eq:DRG}
        \underset{\boldsymbol{a}_r\in\mathbb{Z}^{N}_+}{\ \ \textbf{\textit{min}}\phantom{g}}\ \underset{\mathbb{P}\in\boldsymbol{\mathcal{P}}_r}{\ \textbf{\textit{sup}}\phantom{g}}\ \mathbb{E}_{\boldsymbol{w}\sim\mathbb{P}}\Big[f_r(\boldsymbol{a}_r,\boldsymbol{a}_{-r},\boldsymbol{\pi}_r,\boldsymbol{w})\Big],
    \end{equation}
    where $r$ represents the $r^{\text{th}}$ player in the game. Such a distributionally robust game (\textbf{DRG}) models the interaction between the ranker and adversary. Each player in this game holds a continuous dis-utility function as
    \begin{equation}
        f_r:\mathbb{Z}^{N}_+\times\mathbb{Z}^{N}_+\times\mathbb{N}^n\times\mathbb{R}^{N}\rightarrow\mathbb{R},
    \end{equation}
    $N$ is the cardinality of $\boldsymbol{\mathcal{C}}$ as $N:=|\boldsymbol{\mathcal{C}}|=n(n-1)$. The action of $r$, $\boldsymbol{a}_{r}=\{a^r_{1,2},\dots,a^r_{n,n-1}\}\in\mathbb{Z}^{N}_+$, indicates the number of pairwise comparisons selecting by $r$, and $\boldsymbol{a}_{-r}$ represent the actions of $r$'s opponents. $\boldsymbol{\pi}_r$ is the desired ranking list of player $r$. Here the ``\textbf{\textit{maximum}}'' operation \textit{w.r.t} $\mathbb{P}$ means the player $r$ decides his/her optimal strategy on the worst expected value of $f_r$ from the set of distributions $\boldsymbol{\mathcal{P}}_r$ which is constructed from the partially observed information $\boldsymbol{w}_r$. \eqref{eq:DRG} is known as the distributionally robust game \cite{liu2018distributionally} in the literature. The solution of \eqref{eq:DRG} is named as distributionally robust Nash equilibrium. If any ambiguity set $\boldsymbol{\mathcal{P}}_r$ only contains a single distribution, \eqref{eq:DRG} collapses to a stochastic game problem \cite{chatterjee2004nash}. 
    It is noticed that the ranker $\boldsymbol{\mathcal{R}}$ is often unaware of the existence of $\boldsymbol{\mathcal{A}}$ and the corresponding action turns out to be
    \begin{equation}
        \label{eq:DRG_o}
        \underset{\boldsymbol{a}_r}{\textbf{\textit{min}}}\ \ \mathbb{E}_{\boldsymbol{w}\sim\mathbb{P}_0}\Big[f_r(\boldsymbol{a}_r,\boldsymbol{w},\boldsymbol{\pi}_r)\Big], 
    \end{equation}
    which means the ranker will focus on the original ranking list and choose the data with some sampling methods. Meanwhile the player of ${\boldsymbol{\mathcal{A}}}$ still needs to consider the most pessimistic situation as \eqref{eq:DRG}. 
    
    \begin{definition}[Distributionally Robust Nash Equilibrium]
    \label{def:DRNE}
    A tuple $(\boldsymbol{a}^*_1, \dots, \boldsymbol{a}^*_R)$ is a distributionally robust Nash equilibrium (\textbf{DRNE}) if
    \begin{equation}
        \label{eq:DRNE}
        \boldsymbol{a}^*_r \in\underset{\boldsymbol{a}_r}{\textbf{\textit{arg~min}}}\ \underset{\mathbb{P}\in\boldsymbol{\mathcal{P}}_r}{\textbf{\textit{sup}}}\ \ \mathbb{E}_{\boldsymbol{w}\sim\mathbb{P}}\Big[f_r(\boldsymbol{a}_r,\boldsymbol{a}^*_{-r},\boldsymbol{\pi}_r,\boldsymbol{w})\Big], 
    \end{equation}
    where
    \begin{equation}
        \boldsymbol{a}^*_{-r} = \sum_{s\neq r} \boldsymbol{a}^*_s.
    \end{equation}
    \end{definition}
    \noindent This definition shows that the \textbf{DRNE} is a solution of the corresponding \textbf{DRG} \eqref{eq:DRG}. Here we consider the case of $R=2$, say that the game between $\boldsymbol{\mathcal{R}}$ and ${\boldsymbol{\mathcal{A}}}$. In what follows, we investigate the existence of \textbf{DRNE} by the following theorem. The detailed proof can be found in the supplementary materials. 
    \begin{restatable}{theorem}{drne}
        \label{thm:DRNE}
        There exists a \textbf{DRNE} \eqref{eq:DRNE} if the following states hold for any $r=1,\dots, R$.
        \begin{enumerate}
            \item[(a)] Given $(\boldsymbol{a}_{-r},\boldsymbol{\pi}_r,\boldsymbol{w})$, $f_r(\cdot, \boldsymbol{a}_{-r},\boldsymbol{\pi}_r,\boldsymbol{w})$ is convex over $\mathbb{Z}^N_+$.
            \item[(b)] $\mathbb{E}_{\boldsymbol{w}\sim\mathbb{P}}\big[f_r(\boldsymbol{a}_{r},\boldsymbol{a}_{-r},\boldsymbol{\pi}_r,\boldsymbol{w})\big]$ has finite values for any $(\boldsymbol{a}_{r},\boldsymbol{a}_{-r})$, $\mathbb{P}\in\boldsymbol{\mathcal{P}}_r$ and $\boldsymbol{\pi}_r$ is a permutation of $[n]$.
            \item[(c)] $\boldsymbol{\mathcal{P}}_r$ has weakly compactness.
        \end{enumerate}
    \end{restatable}
    \noindent This result tells us that there exists at least one stable state for both $\boldsymbol{\mathcal{R}}$ and $\boldsymbol{\mathcal{A}}$ with the above conditions holding. The next key step toward executing manipulation is to identify an equilibrium state that is favorable to $\boldsymbol{\mathcal{A}}$ or not. When an equilibrium state is favorable to $\boldsymbol{\mathcal{A}}$, the perturbed data $\boldsymbol{C}'(T)$ could lead $\boldsymbol{\mathcal{R}}$ generating $\boldsymbol{\theta}'\in\boldsymbol{\Theta}_{\boldsymbol{\mathcal{A}}}$. 
    % If there exists such a data source $\boldsymbol{\mathcal{C}}'$, we say that $\boldsymbol{\mathcal{S}}$ is vulnerable to $\boldsymbol{\mathcal{A}}$ in the adversarial scenario.

\subsection{Successful Opportunity to Sequential Manipulation}
\label{sec:sampling}
Without prior knowledge of the original data distribution, analyzing the equilibrium of the proposed distributionally robust game is really challenging. Here we try to dissect the outcome after the adversarial game directly. At the end of the distributionally robust game between $\boldsymbol{\mathcal{R}}$ and $\boldsymbol{\mathcal{A}}$, we have a sequence of pairwise comparisons \eqref{eq:obv_seq} for construction of the comparison graph. To simulate the competitive results of \eqref{eq:DRNE}, we treat \eqref{eq:obv_seq} as a  stochastic process, the output of a sampling method $\boldsymbol{\mathcal{S}}$. The random nature of the sampling process replaces the uncertainty of the distributionally robust game. In addition, utilizing $\boldsymbol{\mathcal{S}}$ to analyze $\boldsymbol{C}'(T)$ makes our subsequent discussion more pertinent to actual confrontation scenarios. Here we show that the two classic sampling methods will become an accomplice of the online manipulator, who help to generate the stable sequences favoring to $\boldsymbol{\mathcal{A}}$’s goal.
% This sequence can be treated as the output of a sampling method $\boldsymbol{\mathcal{S}}$. However, the randomness of $\boldsymbol{\mathcal{S}}$ could generate irregular sequences. 

In the non-adversarial scenario, an $(\epsilon,\delta)$-\textit{representative} sampling method always suffices to take only a small number of random samples $\boldsymbol{C}$ in order to represent the data source $\boldsymbol{\mathcal{C}}$ truthfully \cite{Vapnik:71}. Even with the \textbf{\textit{aimless attacker}}\cite{10.1145/3375395.3387643} who still adopts the original data source $\boldsymbol{\mathcal{C}}$ as the source of perturbation, Bernoulli and the reservoir sampling methods could lead $\boldsymbol{\mathcal{R}}$ to generate the same ranking result. To be specific, the perturbed sequence $\boldsymbol{C}'_1$ and the original sequence $\boldsymbol{C}$ would produce different comparison graph weights $\boldsymbol{w}_{\boldsymbol{C}'_1}$ and $\boldsymbol{w}_{\boldsymbol{C}}$. However, they still obey the same distribution as  
\begin{equation}
    \boldsymbol{w}_{\boldsymbol{C}},\ \boldsymbol{w}_{\boldsymbol{C}'_1}\sim\mathbb{P}_0,
\end{equation}
where $\mathbb{P}_0$ is the distribution of original comparison graph weights. If the probability distribution of comparison graph's weight $\boldsymbol{w}$ is $\mathbb{P}_0$, it holds that
\begin{equation}
    \boldsymbol{\mathcal{R}}(\boldsymbol{w}_{\boldsymbol{C}}) = \boldsymbol{\mathcal{R}}(\boldsymbol{w}_{\boldsymbol{C}'}) =  \boldsymbol{\theta}_0.
\end{equation}
Consequently, $\boldsymbol{\mathcal{R}}$ generates the original ranking list $\boldsymbol{\pi}_{\boldsymbol{\theta}_0}$ even if the aimless attacker $\boldsymbol{\mathcal{A}}$ exists. This so-called ``\textit{adversarial robustness}''\cite{10.1145/3375395.3387643} is on side of the sampling methods. 

Unfortunately, the attacker will be sophisticated in the real confrontation scenario. He/she could construct new data sources instead of simply using $\boldsymbol{\mathcal{C}}$. For example, given any $\boldsymbol{\theta}'\in\boldsymbol{\Theta}_{\boldsymbol{\mathcal{A}}}$ like \eqref{eq:goal}, $\boldsymbol{\mathcal{A}}$ could generate the pairwise comparison through the \textbf{BTL} model: the larger the $\theta_i'/\theta_j'$, the higher the probability of generating pairwise comparison $i\succ j$. Such actions construct the adversarial data source $\boldsymbol{\mathcal{C}}_{\boldsymbol{\mathcal{A}}}$ whose underlying distribution is $\mathbb{P}_{\boldsymbol{\mathcal{A}}}$ which is distinct from the $\mathbb{P}_0$ as $\boldsymbol{\mathcal{R}}$ will produce the manipulated score:
\begin{equation}
    \label{eq:desired_result}
    \boldsymbol{\mathcal{R}}(\boldsymbol{w}) = \boldsymbol{\theta}',\ \boldsymbol{w}\sim\mathbb{P}_{\boldsymbol{\mathcal{A}}}.
\end{equation}
The distributionally robust game between $\boldsymbol{\mathcal{C}}$ and $\boldsymbol{\mathcal{C}}_{\boldsymbol{\mathcal{A}}}$ \eqref{eq:DRG} creates the mixed data source $\boldsymbol{\mathcal{C}}'$. Once the underlying distribution of $\boldsymbol{\mathcal{C}}'$ is consistent with $\mathbb{P}_{\boldsymbol{\mathcal{A}}}$, it holds that
\begin{equation}
    \boldsymbol{w}_{\boldsymbol{C}'_2}\sim\mathbb{P}_{\boldsymbol{\mathcal{A}}},
\end{equation}
where $\boldsymbol{C}'_2$ is the perturbed sequence form $\boldsymbol{\mathcal{C}}'$. There is no doubt that the sampled data $\boldsymbol{C}'_2$ will lead $\boldsymbol{\mathcal{R}}$ to obtain $\boldsymbol{\pi}_{\boldsymbol{\theta}'}$ as \eqref{eq:desired_result}. We call such an attacker the ``\textbf{\textit{purposeful adversary}}''. Then the $(\epsilon,\delta)$-representative sampling algorithms would fall into a trap. From the perspective of the purposeful adversary, the original ``representativeness'' turns into the ``{vulnerability}''. This ``{vulnerability}'' is the other side of the sampling methods like Bernoulli and reservoir. 

We introduce some definitions which will help to establish the vulnerability results of Bernoulli and reservoir sampling methods. The data stream from $\boldsymbol{\mathcal{C}}'$ is a mixture of two data streams from $\boldsymbol{\mathcal{C}}$ and $\boldsymbol{\mathcal{C}}_{\boldsymbol{\mathcal{A}}}$. The mixed data source $\boldsymbol{\mathcal{C}}'$ satisfies
    \begin{equation}
        \label{eq:mix_data_source}
        \boldsymbol{\mathcal{C}}' \subseteq \boldsymbol{\mathcal{C}}\cup \boldsymbol{\mathcal{C}}_{\boldsymbol{\mathcal{A}}},\ \boldsymbol{\mathcal{C}}'\cap\boldsymbol{\mathcal{C}} \neq \varnothing,\ \boldsymbol{\mathcal{C}}'\cap\boldsymbol{\mathcal{C}}_{\boldsymbol{\mathcal{A}}} \neq \varnothing.
    \end{equation}
Here we consider the following two types of mixtures, which correspond to different behaviors of players in the distributionally robust game. In fact, the dynamic stream comes from the distributionally robust game whose players execute \eqref{eq:DRG} and the static stream corresponds to players who executes \eqref{eq:DRG_o}. 
\begin{definition}[Static stream]
    Let $\boldsymbol{C}=\{c_t\}_{t=1}^{\infty}$ be a sequence from $\boldsymbol{\mathcal{C}}'$ which is a mixture of two data streams from $\boldsymbol{\mathcal{C}}$ and $\boldsymbol{\mathcal{C}}_{\boldsymbol{\mathcal{A}}}$. For any $c_t\in\boldsymbol{\mathcal{C}}'\cap\boldsymbol{\mathcal{C}}$, if the generation of $c_t$ is independent with $\{c_1, \dots, c_{t-1}\}$, we call $\boldsymbol{C}$ from $\boldsymbol{\mathcal{C}}'$ is a static stream.
\end{definition}
\begin{definition}[Dynamic stream]
    Let $\boldsymbol{C}=\{c_t\}_{t=1}^{\infty}$ be a sequence from $\boldsymbol{\mathcal{C}}'$ which is a mixture of two data streams from $\boldsymbol{\mathcal{C}}$ and $\boldsymbol{\mathcal{C}}_{\boldsymbol{\mathcal{A}}}$. For any $c_t\in\boldsymbol{\mathcal{C}}'\cap\boldsymbol{\mathcal{C}}$ and $c_t\notin\boldsymbol{\mathcal{C}}'\cap\boldsymbol{\mathcal{C}}_{\boldsymbol{\mathcal{A}}}$, if the generation of $c_t$ is dependent with $\{c_1, \dots, c_{t-1}\}$, we call $\boldsymbol{C}$ from $\boldsymbol{\mathcal{C}}'$ is a dynamic stream.
\end{definition}
The concept of $\epsilon$-approximation measures the similarity between two sequences from the same data source.
\begin{algorithm}[ht!]
    \SetAlgoLined
    \SetKwInOut{Input}{\ \ Input}
    \SetKwInOut{Output}{Output}
    \Input{the number of turns $T$, the sampling parameter $\varrho$, the true ranking list $\boldsymbol{\pi}_0$, the target ranking list $\boldsymbol{\pi}_{\boldsymbol{\mathcal{A}}}$, the dis-utility functions of the original and adversarial data sources $f$, $f_{\boldsymbol{\mathcal{A}}}$, the stopping time $S_0$ and the unweighted complete comparison graph $\boldsymbol{\mathcal{G}}=\{\boldsymbol{E},\boldsymbol{V}\}$.}

    {\bfseries Initialization: }let the stream, the sampled sequence and the comparison graph weights be empty: 
    \begin{equation*}
        \boldsymbol{C}=\varnothing,\ \boldsymbol{C}'=\varnothing,\ \boldsymbol{w}(0)=\boldsymbol{0}.
    \end{equation*}

    \For{$t=1$ {\bfseries to} $T$}
    {
        Action of the ranker $\boldsymbol{\mathcal{R}}$:
        \begin{equation*}
            \begin{aligned}
                \boldsymbol{a}\in\underset{\phantom{\boldsymbol{\mathcal{P}}}\boldsymbol{a}\phantom{\boldsymbol{\mathcal{P}}}}{\textbf{\textit{argmin}}}\underset{\mathbb{P}\in\boldsymbol{\mathcal{P}}(\boldsymbol{w}(t-1))}{\textbf{\textit{max\phantom{g}}}}\ \ \mathbb{E}_{\boldsymbol{w}\sim\mathbb{P}}\big[f(\boldsymbol{a}, \boldsymbol{w}, \boldsymbol{\pi}_0)\big].\\
            \end{aligned}
        \end{equation*}

        Update the data stream $\boldsymbol{C}$:
        \begin{equation*}
            \boldsymbol{C}\leftarrow \boldsymbol{a}.
        \end{equation*}

        Action of sampling method $\boldsymbol{\mathcal{S}}$:
        \begin{equation*}
            \boldsymbol{C}'\leftarrow\boldsymbol{\mathcal{S}}(\boldsymbol{C}, \varrho).
        \end{equation*}

        Update the weights:
        \begin{equation*}
            \boldsymbol{w}(t)\leftarrow\boldsymbol{C}'.
        \end{equation*}

        Let $s = 1$ and update the knowledge:
        \begin{equation*}
            \boldsymbol{w}_{\boldsymbol{\mathcal{A}}}(s) = \text{mask}(\boldsymbol{w}(t)).
        \end{equation*}

        \While{$s<S_0$ and $\boldsymbol{w}{(s)}\neq \boldsymbol{w}_{\boldsymbol{\mathcal{A}}}{(S_0)}$}
        {
            \vspace{5pt}
            Action of the online manipulator $\boldsymbol{\mathcal{A}}$:
            \begin{equation*}
                \begin{aligned}
                    & \ \ \ \ \ \ \ \ \ \ \ \ \ \ \ \ \ \ \ \ \ \ \ \ \ \ \ \boldsymbol{a}_{\boldsymbol{\mathcal{A}}}(s)\in\\
                    & \underset{\phantom{\boldsymbol{\mathcal{P}}}\boldsymbol{a}\phantom{\boldsymbol{\mathcal{P}}}}{\textbf{\textit{argmin}}}\ \underset{\mathbb{P}\in\boldsymbol{\mathcal{P}}(\boldsymbol{w}_{\boldsymbol{\mathcal{A}}(s)})}{\textbf{\textit{max\phantom{g}}}}\ \ \mathbb{E}_{\boldsymbol{w}\sim\mathbb{P}}\big[f_{\boldsymbol{\mathcal{A}}}(\boldsymbol{a}, \boldsymbol{w}, \boldsymbol{\pi}_{\boldsymbol{\mathcal{A}}})\big].
                \end{aligned}
            \end{equation*}     

            Update the data stream $\boldsymbol{C}$:
            \begin{equation*}
                \boldsymbol{C}\leftarrow \boldsymbol{a}_{\boldsymbol{\mathcal{A}}}(s).
            \end{equation*}     

            Action of sampling method $\boldsymbol{\mathcal{S}}$:
            \begin{equation*}
                \boldsymbol{C}'\leftarrow\boldsymbol{\mathcal{S}}(\boldsymbol{C}, \varrho).
            \end{equation*}

            Let $s \leftarrow s+1$ and Update the knowledge:
            \begin{equation*}
                \boldsymbol{w}_{\boldsymbol{\mathcal{A}}}(s)\leftarrow\boldsymbol{C}'.
            \end{equation*}
        }

        Update the weights:
        \begin{equation*}
            \boldsymbol{w}{(t)}\leftarrow\boldsymbol{C}'.
        \end{equation*}
    }
    \Output{$\boldsymbol{\mathcal{G}}^{(T)}=\{\boldsymbol{E},\boldsymbol{V},\boldsymbol{w}^{(T)}\}$.}
    \caption{Online Interaction between $\boldsymbol{\mathcal{A}}$ and $\boldsymbol{\mathcal{R}}$}
    \label{alg:ada_adv_game}
\end{algorithm}

\begin{definition}[$\epsilon$-approximation]
    \label{def:approx}
    A sequence $\boldsymbol{C}_1$ is an $\epsilon$-approximation of sequence $\boldsymbol{C}_0$ with respect to the data source $\boldsymbol{\mathcal{C}}$, if there exists an $\epsilon\in(0,1)$ such that  
    \begin{equation}
        \label{eq:approx}
        |d_{\boldsymbol{\mathcal{C}}}(\boldsymbol{C}_0)-d_{\boldsymbol{\mathcal{C}}}(\boldsymbol{C}_1)|\leq \epsilon,
    \end{equation}
    where $\boldsymbol{\mathcal{C}}$ is a data source, $d_{\boldsymbol{\mathcal{C}}}(\boldsymbol{C})$ is the \textit{density} of $\boldsymbol{\mathcal{C}}$ in the sequence $\boldsymbol{C}$, the fraction of pairwise comparisons in $\boldsymbol{\mathcal{C}}$ that are also in $\boldsymbol{C}$:
    \begin{equation}
        d_{\boldsymbol{\mathcal{C}}}(\boldsymbol{C})=\mathbb{P}\big(c\in\boldsymbol{\mathcal{C}}\ \big|\ c\in\boldsymbol{C}\big).
    \end{equation}
\end{definition}
\noindent This definition give us a similarity metric between two sequences with the density function. It is noteworthy that the lengths of $\boldsymbol{C}_0$ and $\boldsymbol{C}_1$ could be different, where the length of $\boldsymbol{C}_0$ could be infinite and $\boldsymbol{C}_1$ only has a limited number of elements. To portray the data source and analyze the vulnerability of sampling methods, we adopt the $(\epsilon,\delta)$-\textit{representativeness} to quantify the quality of a sampling method \textit{w.r.t} a data source. 
\begin{definition}[$(\epsilon,\delta)$-representativeness]
    A sampling method is called $(\epsilon,\delta)$-\textit{representative} if the sampled sequence of $\boldsymbol{C}_1$ is an $\epsilon$-approximation of the whole stream $\boldsymbol{C}_0$ with respect to $\boldsymbol{\mathcal{C}}$, with probability at least $1-\delta$.
\end{definition}

The following theoretical results show that as long as the sampling parameters satisfy certain conditions, $\boldsymbol{\mathcal{S}}$ must be $(\epsilon, \delta)$-representative with respect to $\boldsymbol{\mathcal{C}}'$.

\begin{restatable}{theorem}{vulnerability}
    \label{thm:vulnerability}
    Let $\boldsymbol{\mathcal{C}}'$ be a mixture of $\boldsymbol{\mathcal{C}}$ and $\boldsymbol{\mathcal{C}}_{\boldsymbol{\mathcal{A}}}$ satisfying \eqref{eq:mix_data_source}.
    \begin{enumerate}
        \item[i)] For any static stream $\boldsymbol{C}=\{c_t\}_{t=1}^{\infty}$ from $\boldsymbol{\mathcal{C}}'$, 
        \begin{itemize}
            \item if the parameter of Bernoulli sampling method satisfies 
            \begin{equation}
                \varrho\geq  c\cdot\frac{\textbf{\textit{log}}\ n(n-1)+\textbf{\textit{ln}}(1/\delta)}{\epsilon^2T}, 
            \end{equation}
            where $T$ is the number of sampling, the output of Bernoulli sampling is $(\epsilon,\delta)$-representative with respect to $\boldsymbol{\mathcal{C}}'$.
            \item if the parameter of reservoir sampling method satisfies 
            \begin{equation}
                \varrho\geq c\cdot\frac{\textbf{\textit{log}}\ n(n-1)+\textbf{\textit{ln}}(1/\delta)}{\epsilon^2},
            \end{equation}
            the output of reservoir sampling is $(\epsilon,\delta)$-representative with respect to $\boldsymbol{\mathcal{C}}'$.
        \end{itemize}
        \item[ii)] For any dynamic stream $\boldsymbol{C}=\{c_t\}_{t=1}^{\infty}$ from $\boldsymbol{\mathcal{C}}'$, 
        \begin{itemize}
            \item if the parameter of Bernoulli sampling method satisfies 
            \begin{equation}
                \varrho\geq 10\cdot\frac{\displaystyle\textbf{\textit{ln}}\ |\boldsymbol{\mathcal{C}}'|+\textbf{\textit{ln}}(4/\delta)}{\epsilon^2T},
            \end{equation}
            where $T$ is the number of sampling, the output of Bernoulli sampling is $(\epsilon,\delta)$-representative with respect to $\boldsymbol{\mathcal{C}}'$.
            \item if the parameter of reservoir sampling method satisfies
            \begin{equation}
                \varrho\geq 10\cdot\frac{\displaystyle\textbf{\textit{ln}}\ |\boldsymbol{\mathcal{C}}'|+\textbf{\textit{ln}}(2/\delta)}{\epsilon^2},
            \end{equation}
            the output of reservoir sampling is $(\epsilon,\delta)$-representative with respect to $\boldsymbol{\mathcal{C}}'$.
        \end{itemize}
    \end{enumerate}
\end{restatable}
Theorem \ref{thm:vulnerability} indicates that the mixed data source $\boldsymbol{\mathcal{C}}'$ could be a \textbf{DRNE} which will be a favor to $\boldsymbol{\mathcal{A}}$. When the underlying distribution of $\boldsymbol{\mathcal{C}}'$ is consistent with $\boldsymbol{\mathcal{C}}_{\boldsymbol{\mathcal{A}}}$, the Bernoulli and reservoir sampling methods always select the data that consisted with $\boldsymbol{\pi}(\boldsymbol{\theta}')$ with high probability. The detailed proof can be found in the supplementary materials. Now we formally define the online adversarial interaction between $\boldsymbol{\mathcal{R}}$ and $\boldsymbol{\mathcal{A}}$ discussed in this paper. 
\begin{itemize}
    \item The behavior of $\boldsymbol{\mathcal{R}}$ relies on the original data source $\boldsymbol{\mathcal{C}}$ whose distribution $\mathbb{P}_0$ would lead $\boldsymbol{\mathcal{R}}$ to generate $\boldsymbol{\theta}$. The pairwise comparisons from $\boldsymbol{\mathcal{C}}$ would be contrary to the attacker's goal as $\boldsymbol{\theta}\notin\boldsymbol{\Theta}_{\boldsymbol{\mathcal{A}}}$. The way that $\boldsymbol{\mathcal{C}}$ generates data can be active or passive, corresponding to dynamic and static streams, respectively. In fact, $\boldsymbol{\mathcal{C}}$ often takes a passive approach like \eqref{eq:DRG_o} and its dis-utility $f$ is dependent of the sampler $\boldsymbol{\mathcal{S}}$. When $\boldsymbol{\mathcal{C}}$ is active, we consider that the defense and protection mechanisms exist. It means that $\boldsymbol{\mathcal{C}}$ will help the ranker $\boldsymbol{\mathcal{R}}$ to take actions like \eqref{eq:DRG} and maintain $\boldsymbol{\pi}_{\boldsymbol{\theta}_0}$ as much as possible.
    \item In the proposed adversarial game, the online manipulator $\boldsymbol{\mathcal{A}}$ will try his/her best to creat data source $\boldsymbol{\mathcal{C}'}$ whose distribution $\mathbb{P}_{\boldsymbol{\mathcal{A}}}$ would induce $\boldsymbol{\mathcal{R}}$ to obtain $\boldsymbol{\theta}'$. The sequential strategy that $\boldsymbol{\mathcal{A}}$ employs along the way is dynamic. Based on the current partial information $\boldsymbol{w}_{\boldsymbol{\mathcal{A}}}(t)$, $\boldsymbol{\mathcal{A}}$ would like to choose the most helpful comparisons by $f_{\boldsymbol{\mathcal{A}}}$, which reduce the divergence between the potential aggregated result and the target ranking. $\boldsymbol{\mathcal{A}}$ needs to be aware of the existence of original ranking data which is always an obstruction of archiving his/her goal. Therefore, $\boldsymbol{\mathcal{A}}$ needs to take the most pessimistic action as \eqref{eq:DRG}.
\end{itemize}
All data sources are sampled with the same sampler $\boldsymbol{\mathcal{S}}$. The proposed online adversarial interaction is summarized in Algorithm \ref{alg:ada_adv_game}. The ranker $\boldsymbol{\mathcal{R}}$ will obtain the aggregated result with the output $\boldsymbol{\mathcal{G}}(T)=\{\boldsymbol{E},\boldsymbol{V},\boldsymbol{w}(T)\}$. The remaining question of executing manipulation is how to construct the data source $\boldsymbol{\mathcal{C}}'$ which owns $\mathbb{P}_{\boldsymbol{\mathcal{A}}}$ as the underlying distribution. We provide the details of dynamic attack strategy, say the adversarial pairwise comparison generation process, in the next section. 

\section{Adversarial Generation Process}
\label{sec:adv_gen_pro}
In Section \ref{sec:policy}, we propose two adversarial policies for adversary $\boldsymbol{\mathcal{A}}$ with complete knowledge and present the asymptotic optimality of these policies. Then, we provide the efficient optimization algorithm for incomplete information in Section \ref{sec:optimization}.

\subsection{Sequential Generation with Complete Knowledge}
\label{sec:policy}
The strategies of $\boldsymbol{\mathcal{A}}$ try to maximize the consistency between full order with his/her goal $\boldsymbol{\theta}'\in\boldsymbol{\Theta}_{\boldsymbol{\mathcal{A}}}$ in the online adversarial game. $\boldsymbol{\mathcal{A}}$ should choose the most destructive comparisons to inject based on the current partial information and stop when the ambiguity of ranking list falls below a certain level. The actions of adversary consist of two components: an adaptive generation rule and a stopping time. For the adaptive rule, we adopt probabilistic rules which contain the deterministic rules as the special cases. Let $\lambda_{i,j}$ denote the probability of generating $(i,j)$ and $\boldsymbol{\lambda}=[\lambda_{1,2},\dots,\lambda_{n,n-1}]\in\boldsymbol{\Delta}$ be the categorical distribution, where
\begin{equation}
    \boldsymbol{\Delta} = \left\{\ \boldsymbol{\lambda}\ \Bigg|\ \underset{(i,j)}{\sum}\ \lambda_{i,j} = 1, \lambda_{i,j}\geq 0\right\}
\end{equation}
is a probability simplex over $n(n-1)$ pairs. In each turn, the distributionally robust Nash equilibrium \eqref{eq:DRNE} (line $7$ in Algorithm \ref{alg:ada_adv_game}) decides $c_{\boldsymbol{\mathcal{A}}}$\footnote{Without lose of generality, we constraint the adversary insert only one pairwise comparison at any step $s$ in in Algorithm \ref{alg:ada_adv_game}. It means that the action $\boldsymbol{a}_{\boldsymbol{\mathcal{A}}}$ will be a one-hot vector which corresponds to $c_{\boldsymbol{\mathcal{A}}}$. } according to $\boldsymbol{\lambda}$, which depends on the goal $\boldsymbol{\theta}'$ and the knowledge $\boldsymbol{w}_{\boldsymbol{\mathcal{A}}}$. The generative rules in the online adversarial game constitute the following set:
\begin{equation}
    \boldsymbol{\Lambda} = \Big\{\boldsymbol{\lambda}^{(s)}\ \Big|\ \boldsymbol{\lambda}^{(s)}\in\boldsymbol{\Delta},\ s = 1,\ 2,\dots\Big\}.
\end{equation}
There is no doubt that the longer the stopping time \cite{doi:10.1287/moor.2020.1109}, the higher the possibility of achieving the manipulation. However, the adversary can't insert without limitations. A large amount of $\{c_{\boldsymbol{\mathcal{A}}}\}$ will alert the ranker $\boldsymbol{\mathcal{R}}$ thus lose the opportunity to attack. Consequently, we measure the quality of sequential manipulation via the generation cost and the ranking consistency. The risk associated with the stopping time $S$ is defined as
\begin{equation}
    \label{eq:cost_risk}
    \mathfrak{R}(S) = \chi\cdot S.
\end{equation}
Here the constant $\chi>0$ indicates the relative cost of inserting one $c_{\boldsymbol{\mathcal{A}}}$ into $\boldsymbol{C}$ (line 10 in Algorithm \ref{alg:ada_adv_game}). The choice of $\chi$ is associated with the difficulty of attack against the specific ranking system.

On the other hand, we adopt Kendall-$\tau$ distance to measure the risk of inconsistency: given a full ranking list $\boldsymbol{\pi}(\boldsymbol{\Lambda}, S)$ from the victim $\boldsymbol{\mathcal{R}}$ with $(\boldsymbol{\Lambda}, S)$, we convert $\boldsymbol{\pi}(\boldsymbol{\Lambda}, S)$ to the binary decisions set $\boldsymbol{R}_{\boldsymbol{\Lambda}, S}$ over pairs
\begin{equation}
    \label{eq:decision_variable}
    \boldsymbol{R}(\boldsymbol{\Lambda}, S)=\Big\{r_{i,j}\in\{0,1\}\ \Big|\ i,j\in[n],\ i\neq j\Big\}
\end{equation}
where 
\begin{equation}
    \label{eq:r_ij}
    r_{i,j} = \left\{
    \begin{array}{cl}
        1, & i\succ_{\boldsymbol{\pi}(\boldsymbol{\Lambda}, S)} j,\\
        0, & \text{otherwise},\\
    \end{array}
    \right.
\end{equation}
and $i\succ_{\boldsymbol{\pi}} j$ means that $i$ is located before $j$ in $\boldsymbol{\pi}$. The risk of inconsistency between $\boldsymbol{\pi}(\boldsymbol{\Lambda}, S)$ and the target ranking induced by $\boldsymbol{\theta}'=[\theta'_1,\dots,\theta'_n]$ is defined by
\begin{equation}
    \label{eq:kendall_risk}
    \begin{aligned}
        & &  &\ \ \mathfrak{R}(\boldsymbol{R}(\boldsymbol{\Lambda}, S))\\[3pt]
        & & =&\ \ \underset{(i,j)}{\sum}\ \mathbbm{I}[\theta'_i<\theta'_j]r_{i,j}+\mathbbm{I}[\theta'_i>\theta'_j](1-r_{i,j}).
    \end{aligned}
\end{equation}

In this paper, we consider the ranking algorithms tailored to the \textbf{BTL} model, say \textbf{HodgeRank} and \textbf{RankCentrality}. The ranking decision of these two victims is locating the candidates based on appropriate estimates of the latent preference scores in the full ranking list. Consequently, our proposed generation policy will depend on the maximum likelihood estimation (\textbf{MLE}) of \textbf{BTL} model. Given the ranker $\boldsymbol{\mathcal{R}}$ under attack, we analyze the combination of \eqref{eq:cost_risk} and \eqref{eq:kendall_risk} under the Bayesian decision framework, in which the manipulated preference score of the victim is assumed to be random and follows a prior distribution $\rho_{\boldsymbol{\theta}'}(\boldsymbol{\theta})$ which is specified by the adversary. The Bayesian risk associated with the victim $\boldsymbol{\mathcal{R}}$ is defined as 
\begin{equation}
    \label{eq:bayesian_risk}
    \mathfrak{R}(\boldsymbol{\Lambda}, S) = \mathbb{E}[\mathfrak{R}(S)+\mathfrak{R}(\boldsymbol{R}(\boldsymbol{\Lambda}, S))],
\end{equation}
where the expectation $\mathbb{E}[\cdot]$ is taken with respect to the adaptive generation rule $\boldsymbol{\Lambda}$ and the stopping time $S$. The adversary $\boldsymbol{\mathcal{A}}$ hopes to execute the optimal policy $(\boldsymbol{\Lambda}^*, S^*)$ which will lead to the minimal risk $\mathfrak{R}^*$
\begin{equation}
    \label{eq:mini_risk}
    \mathfrak{R}^* = \underset{\boldsymbol{\Lambda},S}{\textbf{\textit{inf}}}\ \mathfrak{R}(\boldsymbol{\Lambda}, S).
\end{equation}
For any given cost $\chi$, the value of $\mathfrak{R}^*$ represents the effect of manipulation: a small $\mathfrak{R}^*$ indicates $\boldsymbol{\mathcal{A}}$ would be close to his/her purpose and vice versa. However, obtaining the analytical form of $(\boldsymbol{\Lambda}^*, S^*)$ is typically infeasible. We turn to the asymptotic optimality \cite{10.1214/aoms/1177706205} which is the other well-known evaluation of sequential decision. A policy $(\boldsymbol{\Lambda}, S)$ for $\boldsymbol{\mathcal{R}}$ is said to be asymptotically optimal if
\begin{equation}
    \label{eq:asy_opt}
    \underset{\chi\rightarrow 0}{\textbf{\textit{inf}}}\ \frac{\mathfrak{R}(\boldsymbol{\Lambda}, S)}{\mathfrak{R}^*} = 1.
\end{equation}

By the above definition, we know that the asymptotically optimal policy could work when the relative cost $\chi$ converges to $0$. Although $\chi$ cannot be ignored, the relative cost is negligible compared to the huge profit from a successful manipulation. The adversary $\boldsymbol{\mathcal{A}}$ can do whatever it takes to manipulate the ranking results. Therefore, the asymptotically optimal policy is still important for $\boldsymbol{\mathcal{A}}$. Now we pay attention to the inner loop of Algorithm \ref{alg:ada_adv_game} (from line 8 to 13). Suppose the log-likelihood function of the \textbf{BTL} model with a comparison graph $\boldsymbol{\mathcal{G}} = \{\boldsymbol{V},\boldsymbol{E},\boldsymbol{w}_{\mathcal{A}}(S)\}$ is 
\begin{equation}
    \label{eq:log_likelihood}
    L\bigg(\boldsymbol{\theta}, \boldsymbol{w}_{\mathcal{A}}(S)\bigg) = \sum_{(i,j)}w_{i,j}(S)\cdot\textbf{\textit{log}}~g_{i,j}(\boldsymbol{\theta}),
\end{equation}
where $g_{i,j}(\boldsymbol{\theta})$ is the probability mass function of $i\succ j$ with $\boldsymbol{\theta}$ and $\boldsymbol{w}_{\boldsymbol{\mathcal{A}}}(S)$ represent the complete knowledge. The corresponding \textbf{MLE} with adversarial goal $\rho_{\boldsymbol{\theta}'}$ is 
\begin{equation}
    \label{eq:MLE}
    \boldsymbol{\hat{\theta}}_S = \underset{\boldsymbol{\theta}\in\textbf{\textit{Supp}}(\rho_{\boldsymbol{\theta}'})}{\textbf{\textit{arg min}}}\ -L\bigg(\boldsymbol{\theta}, \boldsymbol{w}_{\mathcal{A}}(S)\bigg),
\end{equation}
where $\textbf{\textit{Supp}}(\rho_{\boldsymbol{\theta}'})$ is the support of the prior probability density function $\rho_{\boldsymbol{\theta}'}(\boldsymbol{\theta})$. 

\noindent\textbf{Stopping Time.} 
Based on the generalized likelihood ratio statistic \cite{fan2001generalized}, we leverage two types of stopping time to decide the number of inserted pairwise comparisons with the complete knowledge:
\begin{equation}
    \label{eq:stopping_time}
    \begin{aligned}
        & S_1 &=&\ \ \textbf{\textit{inf}}\left\{\ S>0 \ \ \Bigg|\ \ \sum_{(i,j)}e^{-|\Delta_{i,j} L_{S}|}\leq e^{-z_{\alpha}(\chi)}\right\}\\[5pt]
        & S_2 &=&\ \ \textbf{\textit{inf}}\left\{\ S>0 \ \ \Bigg|\ \ \ \underset{(i,j)}{\textbf{\textit{min}}}\ \ |\Delta_{i,j} L_{S}|\ \geq\ z_{\alpha}(\chi)\ \right\},
    \end{aligned}
\end{equation}
where $z_{\alpha}(\cdot)$ is a monotone function with $\alpha\in(0,1)$
\begin{equation}
    \label{eq:time_threshd}
    z_{\alpha}(\chi) = |\textbf{\textit{log}}(\chi)|\cdot\big(1+|\textbf{\textit{log}}(\chi)|^{-\alpha}\big). 
\end{equation}
Here $\Delta_{i,j} L_S$ measures the difference between $\theta_i\geq\theta_j$ and $\theta_i\leq\theta_j$ in \eqref{eq:log_likelihood}: 
\begin{equation}
    \begin{aligned}
        & \Delta_{i,j} L_{S} &=&\ \ \underset{\boldsymbol{\theta}\in\boldsymbol{\Theta}_{i,j}}{\textbf{\textit{min}}\ }-L\bigg(\boldsymbol{\theta}, \boldsymbol{w}_{\mathcal{A}}(S)\bigg)\\
        & & & \ \ -\underset{\boldsymbol{\theta}\in\boldsymbol{\Theta}_{j,i}}{\textbf{\textit{min}}\ }-L\bigg(\boldsymbol{\theta}, \boldsymbol{w}_{\mathcal{A}}(S)\bigg),
    \end{aligned}
\end{equation}
where 
\begin{equation}
    \label{eq:theta_ij}
    \boldsymbol{\Theta}_{i,j} = \big\{\boldsymbol{\ \theta}\in\mathbb{R}^n_+\ |\ \theta_i\geq\theta_j\ \big\}\cap\textbf{\textit{Supp}}(\rho_{\boldsymbol{\theta}'}).
\end{equation}
% and $\overline{\textbf{\textit{Supp}}(\rho_{\boldsymbol{\theta}'})}$ denotes the closure of $\textbf{\textit{Supp}}(\rho_{\boldsymbol{\theta}'})$.
Generally speaking, the proposed criteria \eqref{eq:stopping_time} will stop the generation process when the likelihood \eqref{eq:log_likelihood} can decide $\theta_i\geq\theta_j$ or vice versa.

\noindent\textbf{Generation Rule.}
Next we discuss the probabilistic generation rule for sequential manipulation. Inspired by the existing sequential design for rank aggregation \cite{doi:10.1287/moor.2021.1209}, selecting the desired $\boldsymbol{\lambda}^{(S)}$ equals to maximize the consistency between $\boldsymbol{\hat{\theta}}_S$ and the goal $\boldsymbol{\theta}'$. Such consistency could be measured by the minimum of the mutual information between $g_{i,j}(\boldsymbol{\hat{\theta}}_S)$ and any other $g_{i,j}(\boldsymbol{\tilde{\theta}})$ when $(i,j)$ is generated according to $\boldsymbol{\lambda}^{(S)}$:
\begin{equation}
    \label{opt:inner_opt}
    \begin{aligned}
        & &\underset{\boldsymbol{\tilde{\theta}}\in\textbf{\textit{Supp}}(\rho_{\boldsymbol{\theta}'})}{\textbf{\textit{min}}}&\ \ \sum_{(i,j)} \lambda^{(S)}_{i,j}\cdot g_{i,j}(\boldsymbol{\hat{\theta}}_S)\cdot\textbf{\textit{log}}\frac{g_{i,j}(\boldsymbol{\hat{\theta}}_S)}{g_{i,j}(\boldsymbol{\tilde{\theta}})},\\[3pt]
        & &\textbf{\textit{subject to}}\ &\ \ \boldsymbol{\pi}(\boldsymbol{\hat{\theta}}_S)\neq\boldsymbol{\pi}(\boldsymbol{\tilde{\theta}}).
    \end{aligned}
\end{equation}
It is noteworthy that \eqref{opt:inner_opt} also minimizes the drift of log-likelihood ratio statistics between two distributions of pairwise comparisons specified by $\boldsymbol{\hat{\theta}}_S$ and $\boldsymbol{\tilde{\theta}}$ under the \textbf{BTL} model and the probabilistic generation $\boldsymbol{\lambda}^{(S)}$. The smaller the minimum value of \eqref{opt:inner_opt} corresponding to the given $\boldsymbol{\lambda}^{(S)}$, the higher the consistency between $\boldsymbol{\hat{\theta}}_S$ and the goal $\boldsymbol{\theta}'$. Then we select a generation rule $\boldsymbol{\lambda}^{(S)}$ to maximize the consistency measured by \eqref{opt:inner_opt}:
\begin{equation}
    \label{opt:generation_rule}
    \begin{aligned}
        & &\underset{\ \boldsymbol{\lambda}\in\boldsymbol{\Delta}\phantom{\tilde{1}}}{\textbf{\textit{max}}}\ \underset{\boldsymbol{\tilde{\theta}}\in\textbf{\textit{Supp}}(\rho_{\boldsymbol{\theta}'})}{\textbf{\textit{min}}}&\ \ \sum_{(i,j)} \lambda_{i,j}\cdot g_{i,j}(\boldsymbol{\hat{\theta}}_S)\cdot\textbf{\textit{log}}\frac{g_{i,j}(\boldsymbol{\hat{\theta}}_S)}{g_{i,j}(\boldsymbol{\tilde{\theta}})},\\[3pt]
        & &\textbf{\textit{subject to}}\ \ \ \ \ \ \ &\ \ \boldsymbol{\pi}(\boldsymbol{\hat{\theta}}_S)\neq\boldsymbol{\pi}(\boldsymbol{\tilde{\theta}}).   
    \end{aligned}
\end{equation}
We discuss the detailed optimization approach for solving \eqref{opt:generation_rule} in the following part. With the balance between the exploration and exploitation for the generation procedure controlled by \eqref{eq:stopping_time} and \eqref{opt:generation_rule}, we provide the asymptotic optimality guarantee of the proposed policy \eqref{eq:stopping_time} and \eqref{opt:generation_rule} in the supplementary materials.

\subsection{Robust Optimization with Incomplete Knowledge}
\label{sec:optimization}
By the adversarial policy \eqref{eq:stopping_time} and \eqref{opt:generation_rule} with complete knowledge, the adversary $\boldsymbol{\mathcal{A}}$ could insert pairwise comparisons to manipulate the rank aggregation results in the sequential way. However, complete knowledge assumption could be not realistic in the actual confrontation scenarios. To dissect the vulnerability as much as possible, we develop a distributionally robust formulation against the uncertainty of knowledge. 

Notice that the log-likelihood function $L$ \eqref{eq:log_likelihood} is a scale-free function \textit{w.r.t} the weights of a comparison graph. The \textbf{MLE} \eqref{eq:MLE} would be invariant when we map $\boldsymbol{w}_{\boldsymbol{\mathcal{A}}}$ into a probabilistic simplex and replace the discrete variable $\boldsymbol{w}_{\boldsymbol{\mathcal{A}}}$\footnote{We omit the indices of stopping time $S$ when the context is clear.} with a continuous variable $\boldsymbol{p}=[p_{1,2},\dots,p_{n,n-1}]\in\mathbb{R}^{n(n-1)}_+$:
\begin{equation}
    \label{eq:prob_ver}
    \boldsymbol{p} = \frac{1}{M}\cdot\boldsymbol{w}_{\boldsymbol{\mathcal{A}}},\ \boldsymbol{p}^\top\boldsymbol{1}=1
\end{equation}
where $\boldsymbol{1}$ is a n-dimension vector whose elements are $1$ and $M$ is the total number of observed pairwise comparisons by $\boldsymbol{\mathcal{A}}$:
\begin{equation}
    M = \sum_{(i,j)}w_{i,j}.
\end{equation}
In fact, $\boldsymbol{p}$ is drawn from a distribution $\mathbb{P}$: 
\begin{equation}
    \label{eq:empricial_distribution}
    \mathbb{P}= \frac{1}{n(n-1)}\ \underset{(i,j)}{\sum}\ \delta(p_{i,j}).
\end{equation}
where $\delta(p_{i,j})$ is the Dirac measure concentrated at $p_{i,j}$. What is more, we can portray the difference between $\boldsymbol{w}_{\boldsymbol{\mathcal{S}}}{(t)}$ and $\boldsymbol{w}_{\boldsymbol{\mathcal{A}}}$ (line $7$ in Algorithm \ref{alg:ada_adv_game}) by the distance between distributions. Such treatments introduce an uncertainty set of $\mathbb{P}$ which contains the probability distributions around $\mathbb{P}$:
\begin{equation}
    \label{eq:uncert_set}
    \boldsymbol{\mathfrak{U}}^{\gamma}(\mathbb{P}) = \left\{\ \mathbb{Q}\ \Big\vert\ \mathcal{W}_1(\mathbb{P},\ \mathbb{Q})\leq\gamma\ \right\},
\end{equation}
where $\mathcal{W}_1(\cdot,\cdot)$ is the $1$-Wasserstein distance \cite{frohmader20211} as the discrepancy measure. The definition of $1$-Wasserstein distance and related properties can be found in the supplementary materials. With the help of $\boldsymbol{\mathfrak{U}}^{\gamma}(\mathbb{P})$, we execute a conservative strategy to estimate the parameter of \textbf{BTL} model with incomplete knowledge. Instead of $\boldsymbol{p}$, we choose the other random variable $\boldsymbol{q}$ as the weight in \eqref{eq:log_likelihood}. The distribution of $\boldsymbol{q}$ belongs to $\boldsymbol{\mathfrak{U}}^{\gamma}(\mathbb{P})$ and $\boldsymbol{q}$ conducts the worst expected value of $L$. Such a conservative strategy can alleviate the uncertainty generated by incomplete knowledge in the sequential decisions of manipulation policy. Then the relative ranking score with incomplete is estimated by solving the following distributionally robust optimization (\textbf{DRO}) problem:
\begin{equation}
    \label{eq:DRE}
    \underset{\ \boldsymbol{\theta}\in\textbf{\textit{Supp}}(\rho_{\boldsymbol{\theta}'})}{\ \textbf{\textit{max\phantom{p}}}}\ \underset{\mathbb{Q}\in\boldsymbol{\mathfrak{U}}^{\gamma}(\mathbb{P})}{\textbf{\textit{sup}}}\ \mathbb{E}_{\boldsymbol{q}\sim\mathbb{Q}}\left[L(\boldsymbol{\theta},\boldsymbol{q})\right],
\end{equation}
where $L(\boldsymbol{\theta},\boldsymbol{q})$ replaces the incomplete knowledge $\boldsymbol{w}_{\boldsymbol{\mathcal{A}}}$ with the random variable $\boldsymbol{q}\sim\mathbb{Q}$ in \eqref{eq:log_likelihood}. The supreme operation \textit{w.r.t.} $\mathbb{Q}$ means that the estimation of the latent preference score is based on the worst expected value of $L$ from the set of distributions $\boldsymbol{\mathfrak{U}}^{\gamma}(\mathbb{P})$. 

Next, we specify the formulation of $\textbf{\textit{Supp}}(\rho_{\boldsymbol{\theta}'})$. Without a lost of generality, we assume the estimated and the desired scores belong to a probability simplex. Given the desired relative ranking score $\boldsymbol{\theta}'$, we hope that the estimation from \eqref{eq:DRE} is in a neighborhood of $\boldsymbol{\theta}'$, namely, the distance $d:\mathbb{R}^{n}\times\mathbb{R}^{n}\rightarrow\mathbb{R}_+$ between the estimation from \eqref{eq:DRE} and $\boldsymbol{\theta}_{\mathcal{A}}$ would be sufficiently small:
\begin{equation}
    \label{eq:support_set}
    \textbf{\textit{Supp}}(\rho_{\boldsymbol{\theta}'}) = \left\{\ \boldsymbol{\theta}\in\mathbb{R}^n_+\ \Big|\ \|\boldsymbol{\theta}-\boldsymbol{\theta}'\|_2^2\leq\beta,\ \ \boldsymbol{\theta}^\top\boldsymbol{1} = 1\ \right\}.
\end{equation}

\begin{restatable}{theorem}{reformulation}
    \label{thm:reformulation}
    Suppose that $\boldsymbol{p}$ is drawn from the empirical distribution $\mathbb{P}$ \eqref{eq:empricial_distribution} and $\boldsymbol{q}$ is drawn from $\mathbb{Q}\in\boldsymbol{\mathfrak{U}}^{\gamma}(\mathbb{P})$ \eqref{eq:uncert_set}. If the distance between $p_{i,j}$ and $q_{i,j}$ is chosen as 
    \begin{equation}
        \label{eq:metric}
        d(p_{ij},q_{ij})=\big|\ p_{ij}-q_{ij}\ \big|.
    \end{equation}
    Then, the \textbf{DRO} problem \eqref{eq:DRE} has an equivalent form:
    \begin{equation}
        \label{opt:dual_problem}
        \begin{aligned}
            \underset{\boldsymbol{\theta}\in\textbf{\textit{Supp}}(\rho_{\boldsymbol{\theta}'})}{\textbf{\textit{max}}}\ h(\boldsymbol{\theta}),
        \end{aligned}
    \end{equation}
    where
    \begin{equation}
        h(\boldsymbol{\theta})=\sqrt{\gamma\ }\underset{(i,j)}{\sum}\textbf{\textit{log}}~g_{i,j}(\boldsymbol{\theta})+\underset{(i,j)}{\sum}p_{i,j}\textbf{\textit{log}}~g_{i,j}(\boldsymbol{\theta}).
    \end{equation}
    Moreover, if the comparison model is \textbf{BTL} model, we have
    \begin{equation}
        \label{eq:robust_ref}
        \begin{aligned}
             & h(\boldsymbol{\theta})&=&\ \ \sqrt{\gamma\ }\cdot\underset{(i,j)}{\sum}\textbf{\textit{log}}(1+\textbf{\textit{exp}}(\theta_j-\theta_i))\\[5pt]
             & & & \phantom{\gamma\underset{(i,j)}{\sum}}+\underset{(i,j)}{\sum}\ p_{i,j}\ \textbf{\textit{log}}(1+\textbf{\textit{exp}}(\theta_j-\theta_i)).
        \end{aligned}
    \end{equation}
\end{restatable}

From the above theoretical results, we conduct the adversarial policy for the incomplete knowledge. The stopping time \eqref{eq:stopping_time} turns to be
\begin{equation}
    \label{eq:robsut_stop}
    \begin{aligned}
        & S'_1 &=&\ \ \textbf{\textit{inf}}\left\{\ S>0 \ \ \Bigg|\ \ \sum_{(i,j)}e^{-|\Delta'_{i,j} L_{S}|}\leq e^{-z_{\alpha}(\chi)}\right\}\\[5pt]
        & S'_2 &=&\ \ \textbf{\textit{inf}}\left\{\ S>0 \ \ \Bigg|\ \ \ \underset{(i,j)}{\textbf{\textit{min}}}\ \ |\Delta'_{i,j} L_{S}|\ \geq\ z_{\alpha}(\chi)\ \right\},
    \end{aligned}
\end{equation}
where
\begin{equation}
    \begin{aligned}
        & \Delta'_{i,j} L_{S} &=&\ \ \underset{\ \boldsymbol{\theta}\in\boldsymbol{\Theta}_{i,j}}{\ \textbf{\textit{max\phantom{p}}}}\ \underset{\mathbb{Q}\in\boldsymbol{\mathfrak{U}}^{\gamma}(\mathbb{P})}{\textbf{\textit{sup}}}\ \mathbb{E}_{\boldsymbol{q}\sim\mathbb{Q}}\left[L(\boldsymbol{\theta},\boldsymbol{q})\right]\\[3pt]
        & & &\ \ -\underset{\ \boldsymbol{\theta}\in\boldsymbol{\Theta}_{j,i}}{\ \textbf{\textit{max\phantom{p}}}}\ \underset{\mathbb{Q}\in\boldsymbol{\mathfrak{U}}^{\gamma}(\mathbb{P})}{\textbf{\textit{sup}}}\ \mathbb{E}_{\boldsymbol{q}\sim\mathbb{Q}}\left[L(\boldsymbol{\theta},\boldsymbol{q})\right]\\[3pt]
        & &=&\ \ \underset{\ \boldsymbol{\theta}\in\boldsymbol{\Theta}_{i,j}}{\ \textbf{\textit{max\phantom{p}}}} h(\boldsymbol{\theta}) - \underset{\ \boldsymbol{\theta}\in\boldsymbol{\Theta}_{j,i}}{\ \textbf{\textit{max\phantom{p}}}} h(\boldsymbol{\theta}).
    \end{aligned}
\end{equation}
Now we discuss the generation rule with the robust estimation by \eqref{opt:dual_problem}. Let $\boldsymbol{\bar{\theta}}_S$ be a solution of \eqref{opt:dual_problem} with stopping time \eqref{eq:robsut_stop}. We obtain the generation rule with incomplete knowledge by replacing $\boldsymbol{\hat{\theta}}_S$ with $\boldsymbol{\bar{\theta}}_S$ in \eqref{opt:generation_rule}:
\begin{equation}
    \label{opt:robust_generation_rule}
    \begin{aligned}
        & &\underset{\ \boldsymbol{\lambda}\in\boldsymbol{\Delta}\phantom{\tilde{1}}}{\textbf{\textit{max}}}\ \underset{\boldsymbol{\tilde{\theta}}\in\textbf{\textit{Supp}}(\rho_{\boldsymbol{\theta}'})}{\textbf{\textit{min}}}&\ \ \sum_{(i,j)} \lambda_{i,j}\cdot g_{i,j}(\boldsymbol{\bar{\theta}}_S)\cdot\textbf{\textit{log}}\frac{g_{i,j}(\boldsymbol{\bar{\theta}}_S)}{g_{i,j}(\boldsymbol{\tilde{\theta}})},\\[3pt]
        & &\textbf{\textit{subject to}}\ \ \ \ \ \ \ &\ \ \boldsymbol{\pi}(\boldsymbol{\bar{\theta}}_S)\neq\boldsymbol{\pi}(\boldsymbol{\tilde{\theta}}).   
    \end{aligned}
\end{equation}
Consider the inner problem
\begin{equation}
    \label{opt:mirror_inner}
    \begin{aligned}
        & &\ \underset{\boldsymbol{\tilde{\theta}}\in\textbf{\textit{Supp}}(\rho_{\boldsymbol{\theta}'})}{\textbf{\textit{min}}}&\ \ \sum_{(i,j)} \lambda_{i,j}\cdot g_{i,j}(\boldsymbol{\bar{\theta}}_S)\cdot\textbf{\textit{log}}\frac{g_{i,j}(\boldsymbol{\bar{\theta}}_S)}{g_{i,j}(\boldsymbol{\tilde{\theta}})},\\[3pt]
        & &\textbf{\textit{subject to}}&\ \ \boldsymbol{\pi}(\boldsymbol{\bar{\theta}}_S)\neq\boldsymbol{\pi}(\boldsymbol{\tilde{\theta}}),
    \end{aligned}
\end{equation}
we know that the objective function is smooth and convex \textit{w.r.t} $\boldsymbol{\tilde{\theta}}$ for the \textbf{BTL} model $g_{i,j}(\cdot)$. Moreover, the flexible set 
\begin{equation}
    \left\{\boldsymbol{\tilde{\theta}}\in\textbf{\textit{Supp}}(\rho_{\boldsymbol{\theta}'})\ \Big|\ \boldsymbol{\pi}(\boldsymbol{\bar{\theta}}_S)\neq\boldsymbol{\pi}(\boldsymbol{\tilde{\theta}})\right\}
\end{equation}
could be re-written as a union of convex sets which contain at most $2*n$ linear equalities and inequalities like
\begin{equation}
    \left\{\boldsymbol{\tilde{\theta}}\in\textbf{\textit{Supp}}(\rho_{\boldsymbol{\theta}'})\ \Big|\ \tilde{\theta}_i<\bar{\theta}_{\pi_{\bar{\theta}}(i'-1)}, \tilde{\theta}_i=\bar{\theta}_{\pi_{\bar{\theta}}(i')}\right\},
\end{equation}
where $\boldsymbol{\tilde{\theta}} = [\tilde{\theta}_1,\dots,\tilde{\theta}_n]$, and $\bar{\theta}_{\pi_{\bar{\theta}}(i')}$ indicates the preference score of candidates $\pi_{\bar{\theta}}(i)$ whose position in $\boldsymbol{\pi}_{\boldsymbol{\bar{\theta}}}$ is $i'$. Therefore, the inner problem is the convex problem which can be solved efficiently by the standard numerical solvers. Then we analyze the outer problem:
\begin{equation}
    \underset{\boldsymbol{\lambda}\in\boldsymbol{\Delta}}{\textbf{\textit{min}}}\ F(\boldsymbol{\lambda}),\ \ F(\boldsymbol{\lambda}) = \underset{\begin{matrix}\scriptstyle\boldsymbol{\tilde{\theta}}\in\textbf{\textit{Supp}}(\rho_{\boldsymbol{\theta}'})\\\scriptstyle\boldsymbol{\pi}(\boldsymbol{\bar{\theta}}_S)\neq\boldsymbol{\pi}(\boldsymbol{\tilde{\theta}})\end{matrix}}{\textbf{\textit{max}}} \phi(\boldsymbol{\lambda}, \boldsymbol{\tilde{\theta}})
\end{equation}
where
\begin{equation}
    \phi(\boldsymbol{\lambda}, \boldsymbol{\tilde{\theta}}) = -\sum_{(i,j)} \lambda_{i,j}\cdot g_{i,j}(\boldsymbol{\bar{\theta}}_S)\cdot\textbf{\textit{log}}\frac{g_{i,j}(\boldsymbol{\bar{\theta}}_S)}{g_{i,j}(\boldsymbol{\tilde{\theta}})}.
\end{equation}
It is noteworthy that $\phi(\boldsymbol{\lambda},\boldsymbol{\tilde{\theta}})$ is a continuous and bounded function. Furthermore, $\phi(\boldsymbol{\lambda},\boldsymbol{\tilde{\theta}})$ is convex \textit{w.r.t} $\boldsymbol{\lambda}$ for any $\boldsymbol{\tilde{\theta}}$ and $\textbf{\textit{Supp}}(\rho_{\boldsymbol{\theta}'})$ is a convex set. By Danskin Theorem \cite{Bertsekas/99}, $F(\boldsymbol{\lambda})$ is a convex function \textit{w.r.t} $\boldsymbol{\lambda}$ and the min-max optimization problem \eqref{opt:robust_generation_rule} can be solved efficiently using the mirror descent algorithm \cite{DBLP:journals/orl/BeckT03}. The corresponding solution process is summarized as Algorithm \ref{alg:seq_rank_policy}. We elaborate the steps of Algorithm \ref{alg:seq_rank_policy} in the supplementary materials. 

\begin{algorithm}[ht]
    \SetAlgoLined
    \SetKwInOut{Input}{\ \ Input}
    \SetKwInOut{Output}{Output}
    \Input{the probability mass function $g$, the incomplete knowledge $\boldsymbol{p}$, the uncertainty radius $\gamma$.}

    Obtain the robust estimation based on the partial observation by solving \eqref{opt:dual_problem}:
    \begin{equation*}
        \boldsymbol{\theta} = \textbf{\textit{RobustEstimation}}(g,\boldsymbol{p},\gamma).
    \end{equation*}

    Solve the generation rule $\boldsymbol{\lambda}$ via the min-max problem \eqref{opt:robust_generation_rule}:
    \begin{equation*}
        \boldsymbol{\lambda} = \textbf{\textit{MirrorDescent}}(g, \boldsymbol{\theta}).
    \end{equation*}

    % Balance the exploration and the exploitation

    Select pairwise comparison $c$ according to the categorical distribution $\boldsymbol{\lambda}$.

    \Output{a pairwise comparison $c$.}
    \caption{Adversarial Generation}
    \label{alg:seq_rank_policy}
\end{algorithm}

\section{Experiments}
\label{sec:experiment}
In this section, three examples are exhibited with both simulated and real-world data to illustrate the validity of the proposed online attack strategy against the Bernoulli method for rank aggregation like \textbf{\textit{HodgeRank}} \cite{Jiang2011} and \textbf{\textit{RankCentrality}} \cite{DBLP:journals/ior/NegahbanOS17}. The first example is with simulated data while the latter two exploit real-world datasets involved in election and crowdsourcing.

\subsection{General Setting}
We treat the online manipulation against the rank aggregation as the interaction between the original and adversarial data source like Algorithm \ref{alg:ada_adv_game}. In each turn of this game, the original and adversarial data sources generate the pairwise comparisons separately. The generation process of the original data source is always a black box for the adversary and the action of the original data source (line $3$ of Algorithm \ref{alg:ada_adv_game}) will be replaced by a random procedure. Based on the analysis in Sec. \ref{sec:sampling}, the sampling method could be $(\epsilon,\delta)-$ representative \textit{w.r.t} the mixed data source $\boldsymbol{\mathcal{C}'}$ when the sampling parameters are sufficiently large. In most cases, the sampling methods for rank aggregation satisfy these conditions. Even if the sampling methods reject the samples generated by the adversary, he/she can still try repeatedly until such samples affect the final aggregation result. As a consequence, we don't consider the effects of sampling methods (Line $5$ and $11$ of Algorithm \ref{alg:ada_adv_game}) in the experimental studies. It is noteworthy that the attacker only has incomplete knowledge in the adversarial game, say that he/she only observes partial weight of the comparison graph (Line $7$ of Algorithm \ref{alg:ada_adv_game}). The actions of the attacker is the adversarial generation process with the target ranking list $\boldsymbol{\pi}'$, which is discussed in Sec. \ref{sec:adv_gen_pro} summarized as Algorithm \ref{alg:seq_rank_policy}. In each turn, the adversarial generation process will insert $S_0$ pairwise comparisons into the mixed data stream, where $S_0$ is the stooping time. We establish the asymptotic optimality of the proposed stopping time $S_1, S_2$ \eqref{eq:stopping_time} under the complete knowledge condition. With the incomplete knowledge, we set the stopping time $S_0$ empirically. At the end of the adversarial game, we finish the collection of pairwise comparisons and obtain the weighted comparison graph $\boldsymbol{\mathcal{G}}^{(T)}$ (the output of Algorithm \ref{alg:ada_adv_game}). Then the rank aggregation methods leverage $\boldsymbol{\mathcal{G}}^{(T)}$ to create the ranking list $\boldsymbol{\pi}''$. We evaluate the similarity between $\boldsymbol{\pi}'$ and $\boldsymbol{\pi}''$. The more the two orders are similar, the more the manipulation is successful.

\begin{figure*}[ht]
    \centering
    \begin{subfigure}[b]{0.24\textwidth}
        \centering
        \includegraphics[width=\textwidth]{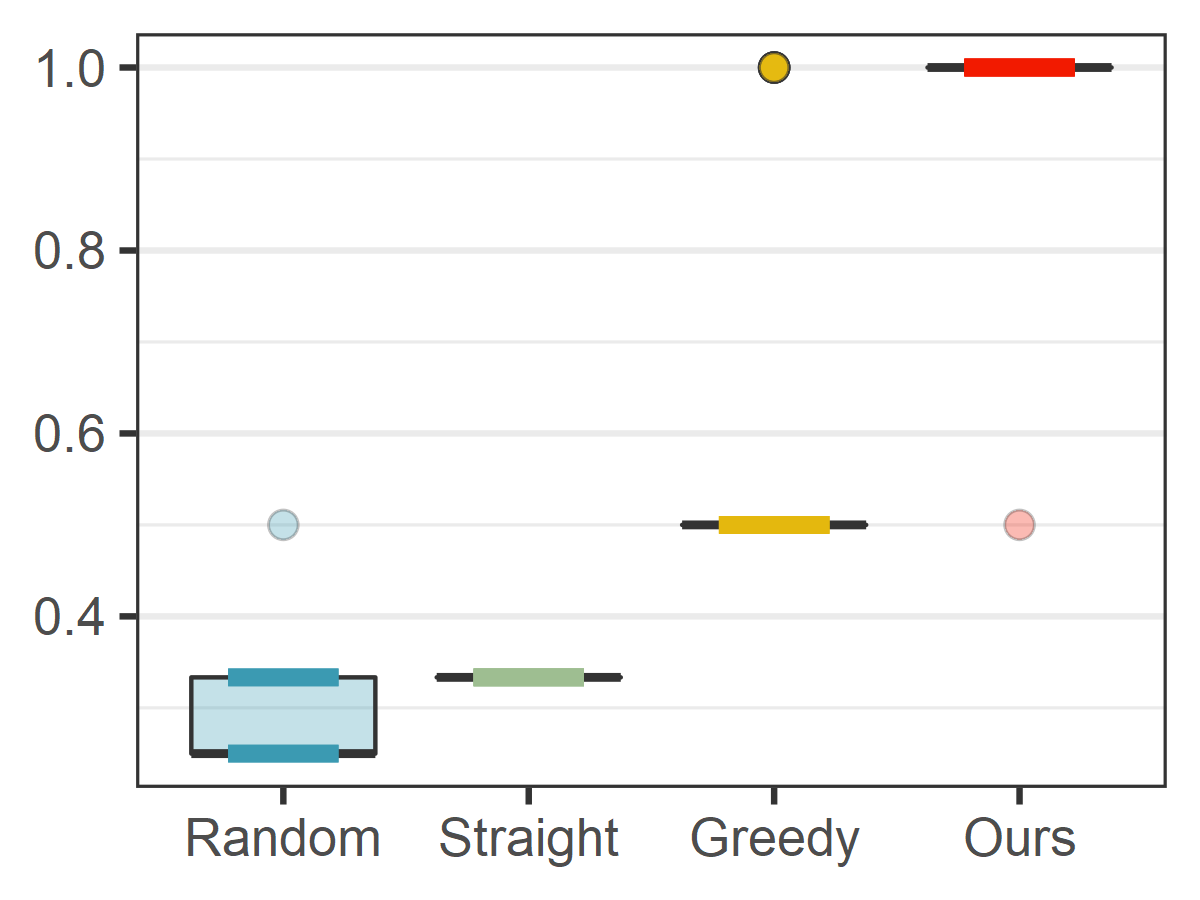}
        \caption{R. Rank of \textbf{HodgeRank}}
    \end{subfigure}
    \hfill
    \begin{subfigure}[b]{0.24\textwidth}
        \centering
        \includegraphics[width=\textwidth]{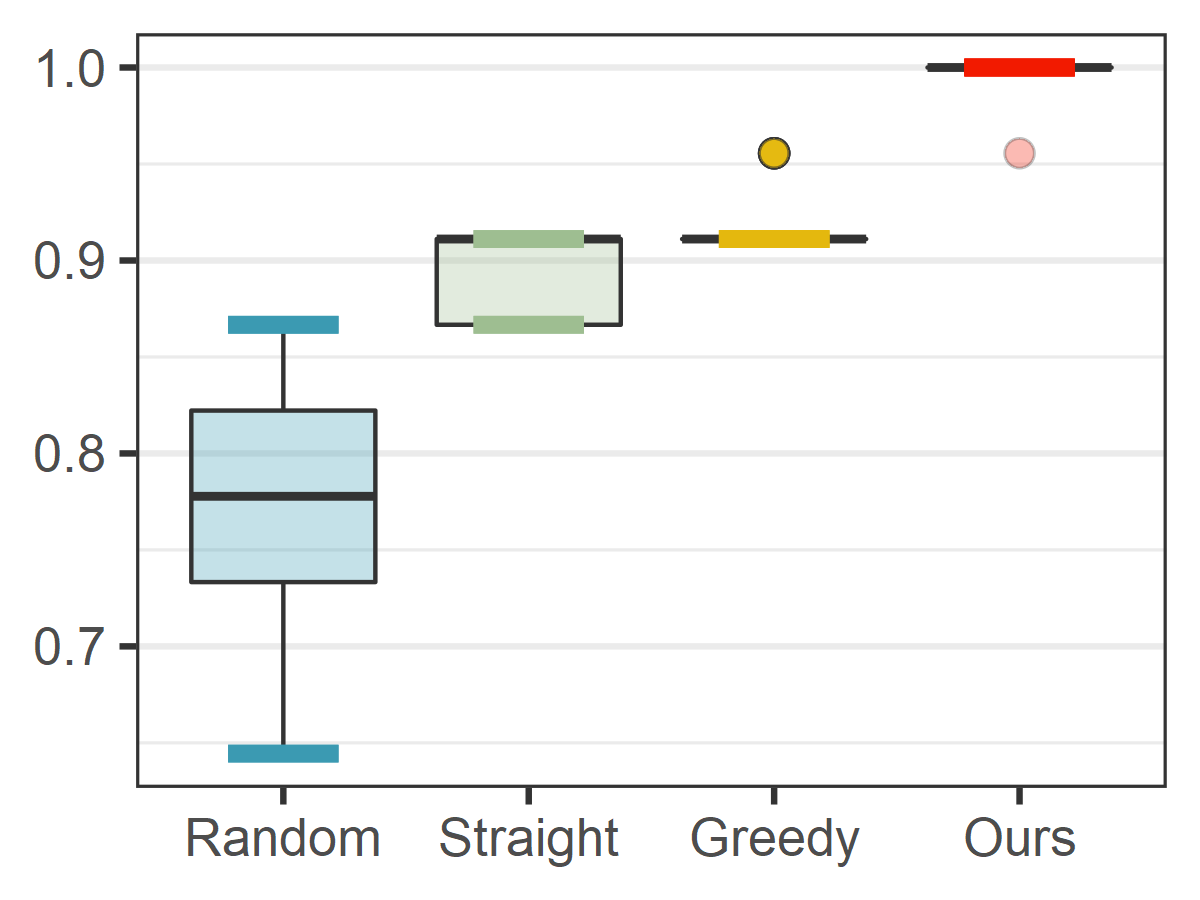}
        \caption{K. $\tau$ of \textbf{HodgeRank}}
    \end{subfigure}
    \hfill
    \begin{subfigure}[b]{0.24\textwidth}
        \centering
        \includegraphics[width=\textwidth]{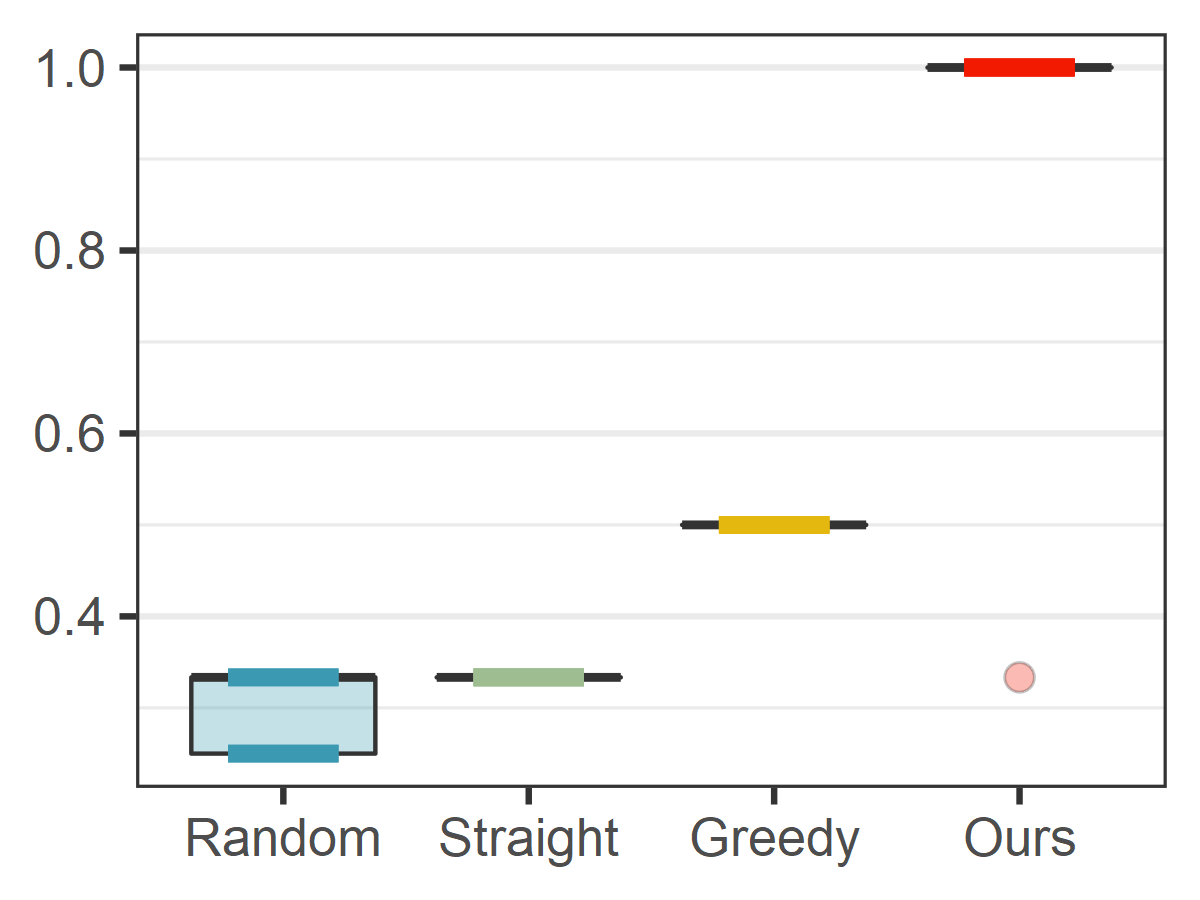}
        \caption{R. Rank of \textbf{RankCentrality}}
    \end{subfigure}
    \hfill
    \begin{subfigure}[b]{0.24\textwidth}
        \centering
        \includegraphics[width=\textwidth]{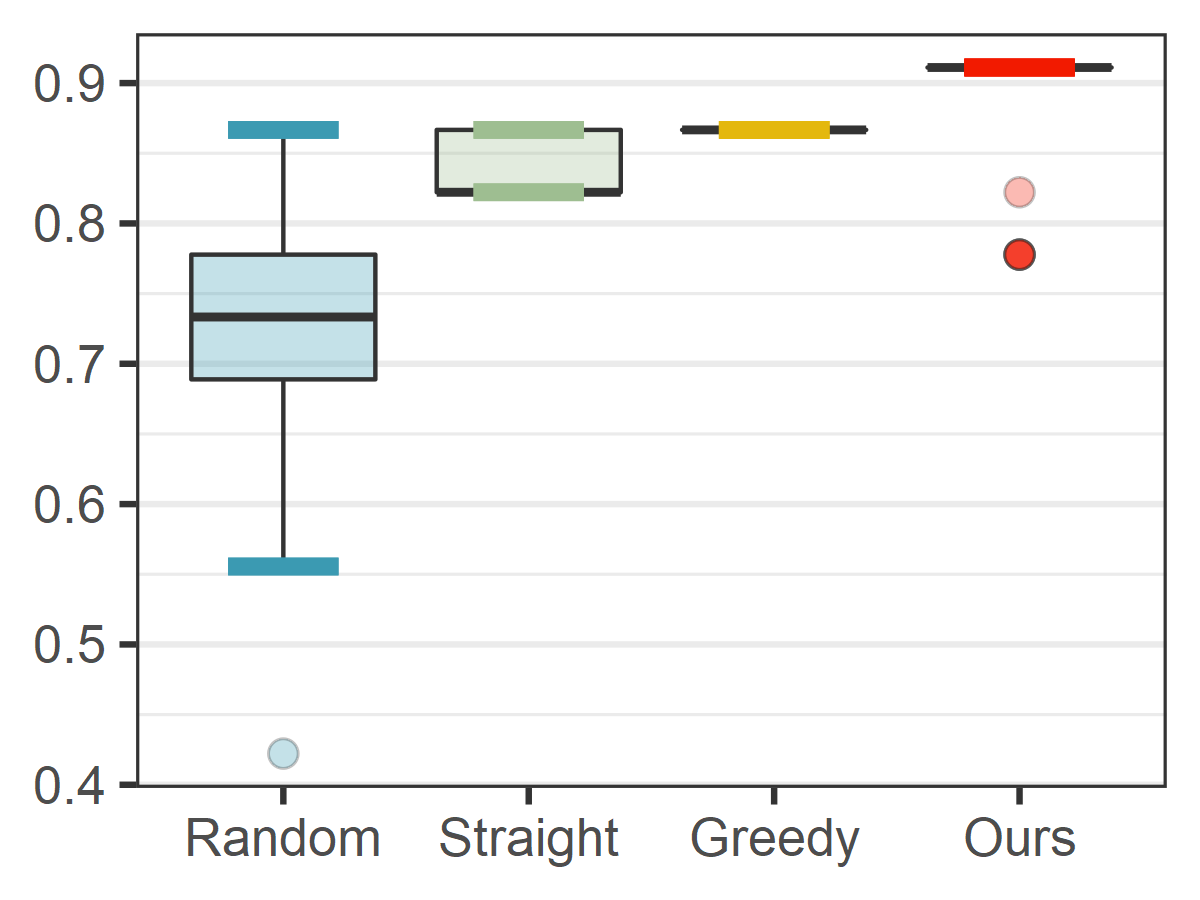}
        \caption{K. $\tau$ of \textbf{RankCentrality}}
    \end{subfigure}
    \caption{Comparative results of different sequential manipulation methods against \textbf{HodgeRank} and \textbf{RankCentrality} on simulated data. The box plot illustrates the results of $50$ trials with different data sequences which will make \textbf{HodgeRank} and \textbf{RankCentrality} generate $\boldsymbol{\pi}_0 =(10,9,8,7,6,5,4,3,2,1)$. The target list of the adversary is $\boldsymbol{\pi}'=(8,9,10,7,5,6,4,3,2,1)$. The proposed method provides a stable manipulation in the form of sequential action. All metrics of the proposed method will be $1$ with rare outliers. Meanwhile the three competitors fail to manipulate \textbf{HodgeRank} and \textbf{RankCentrality} with sequential actions. The `\textbf{Greedy}' perturbation only focuses on the top-$1$ candidate but can't guarantee the designation of a winner. The result of `\textbf{Straightforward}' strategy is inferior to the proposed method when the number of actions is the same.}
    \label{fig:simu_hodge_metric}
\end{figure*}

\begin{figure*}[ht]
    \centering
    \begin{subfigure}[b]{0.24\textwidth}
        \centering
        \includegraphics[width=\textwidth]{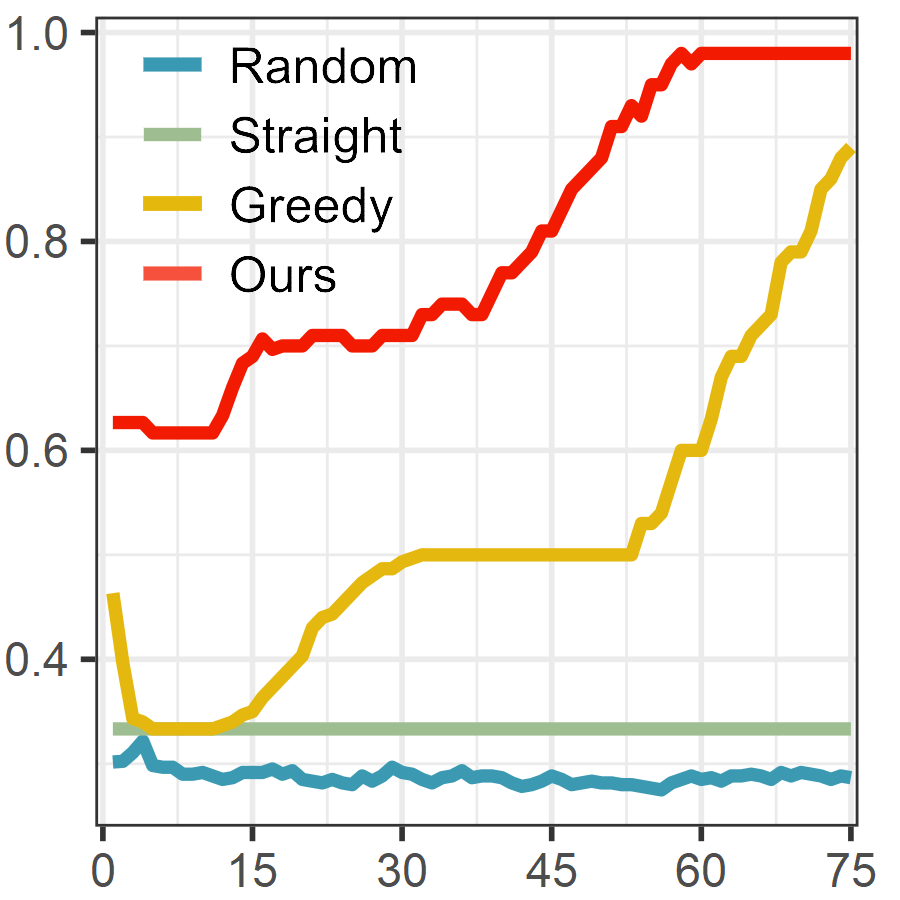}
        \caption{R. Rank of \textbf{HodgeRank}}
    \end{subfigure}
    \hfill
    \begin{subfigure}[b]{0.24\textwidth}
        \centering
        \includegraphics[width=\textwidth]{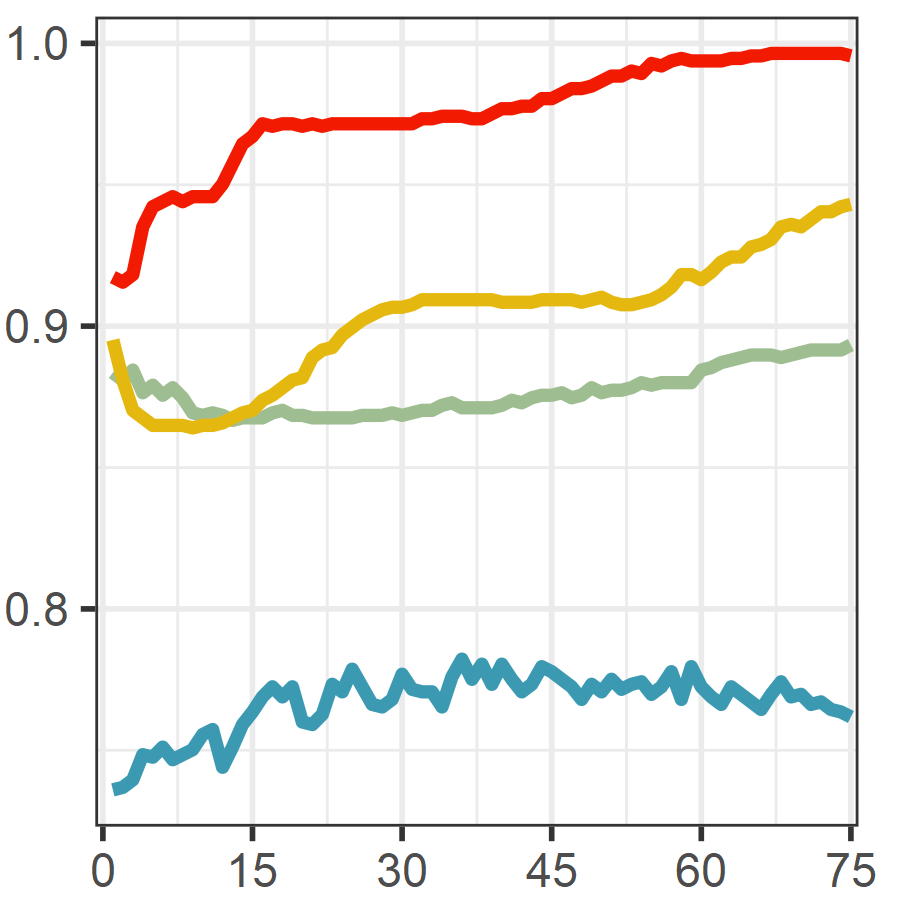}
        \caption{K. $\tau$ of \textbf{HodgeRank}}
    \end{subfigure}
    \hfill
    \begin{subfigure}[b]{0.24\textwidth}
        \centering
        \includegraphics[width=\textwidth]{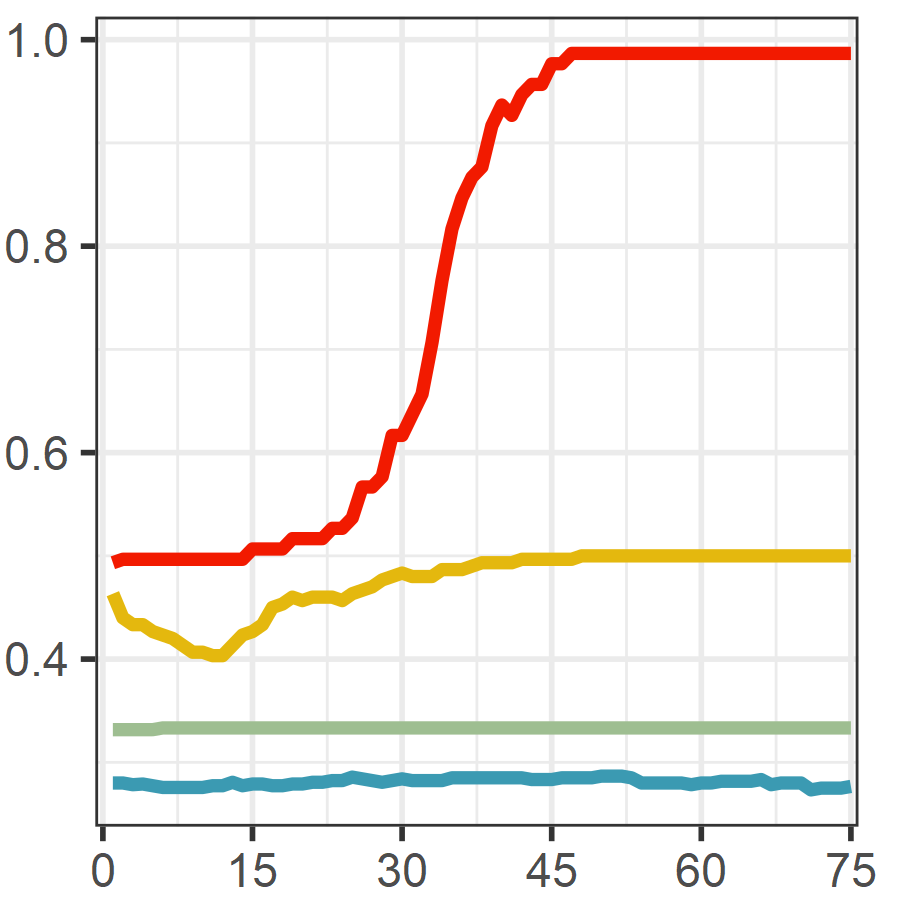}
        \caption{R. Rank of \textbf{RankCentrality}}
    \end{subfigure}
    \hfill
    \begin{subfigure}[b]{0.24\textwidth}
        \centering
        \includegraphics[width=\textwidth]{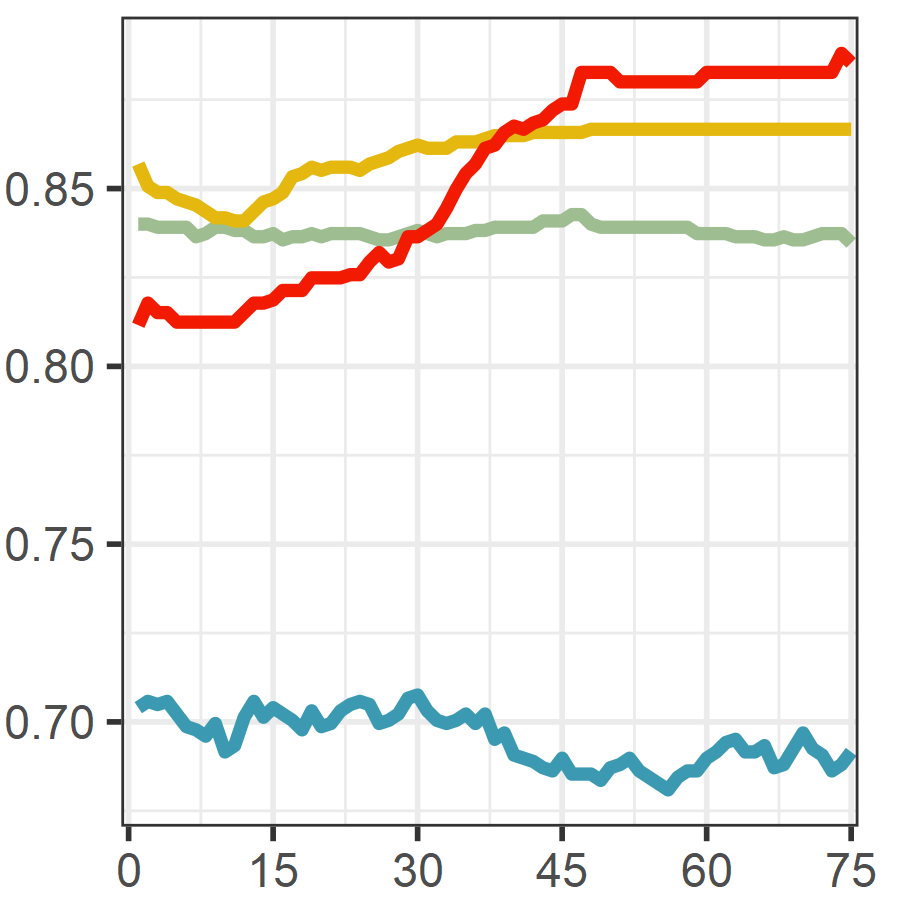}
        \caption{K. $\tau$ of \textbf{RankCentrality}}
    \end{subfigure}
    \caption{Change of evaluation metrics on simulated data for different sequential manipulation methods against \textbf{HodgeRank} and \textbf{RankCentrality}. The horizontal axis lists the turns of game. When the interaction proceeds, the proposed method is able to generate malicious pairwise comparisons with incomplete knowledge and manipulate the victim, whose aggregated results are consistent with the attacker's target. }
    \label{fig:simu_hodge_turn}
\end{figure*}

\subsection{Evaluation Metrics}

Here we adopt \textit{Reciprocal rank} and \textit{Kendall ${\tau}$ coefficient} for evaluating the correlation between the target ranking list $\boldsymbol{\pi}'$ and the final aggregated result $\boldsymbol{\pi}''$. These measurements can be divided into two categories. The reciprocal rank metric reflects whether the first candidates of two ranking lists are the same. The Kendall $\tau$ coefficient considers the consistency of every pairwise comparison in the two ranking lists. 

\vspace{2pt}
\noindent\textbf{\textit{Reciprocal Rank (R. Rank).}} The reciprocal rank is a statistic measure for evaluating any process that produces an order list of possible responses to a series of queries, ordered by the probability of correctness or the ranking scores. The reciprocal rank of a ranking list is the multiplicative inverse of the leading object's position in the new order list. 
% Let $\boldsymbol{\pi}_1$ be a given ranking list and $\boldsymbol{\pi}_2$ be the aggregation result of the victim with the original data or the manipulated data. 
The $\textbf{\textit{R. Rank}}$ between $\boldsymbol{\pi}'$ and $\boldsymbol{\pi}''$ is defined as 
\begin{equation}
    \textbf{\textit{{R. Rank}}}(\boldsymbol{\pi}',\boldsymbol{\pi}'') = \frac{1}{\ (\boldsymbol{\pi}'')^{-1}\big[\boldsymbol{\pi}'(1)\big]\ },
\end{equation}
where $\boldsymbol{\pi}(i)$ refers to the item index which lies in the $i$-th position of ranking list $\boldsymbol{\pi}$, and $\boldsymbol{\pi}^{-1}[i]$ indicates the position of ranking list $\boldsymbol{\pi}$ which belongs to the item $i$. If it holds that $\boldsymbol{\pi}'(1) = \boldsymbol{\pi}''(1)$, we have $\textbf{\textit{R.r}}(\boldsymbol{\pi}',\boldsymbol{\pi}'')$ archives its maximum value $1$. Lager reciprocal rank value indicates a better manipulation result. 

\vspace{2pt}
\noindent\textbf{\textit{Kendall $\boldsymbol{\tau}$ Coefficient (K. $\boldsymbol{\tau}$).}} The Kendall rank correlation coefficient evaluates the degree of similarity between two ranking lists given the same objects. This coefficient depends upon the number of inversions of pairs of objects which would be needed to transform one rank order into the other. The definition of Kendall $\tau$ coefficient is the normalization of \eqref{eq:kendall_risk}. Lager Kendall-$\tau$ value indicates a better purposeful attack result. If $d_{\tau}(\boldsymbol{\pi}',\ \boldsymbol{\pi}'')=1$, we have $\boldsymbol{\pi}'=\boldsymbol{\pi}''$.

% \vspace{2pt}
% \noindent\textbf{\textit{Precision at $K$ (P$@K$).}} Precision at $K$ is the proportion of the top-$K$ objects in the aggregated order $\boldsymbol{\pi}''$ that are consistent with the target ranking $\boldsymbol{\pi}'$. In this case, the precision and recall will be the same. So we do not report the recall and F score.  

% \vspace{2pt}
% \noindent\textbf{\textit{Average Precision at $K$ (AP$@K$).}} Average precision at K is a weighted average of the precision. If the top objects in the aggregated ranking list are consistent with the desired one, they will contribute more than the tail objects in this metric. 

% \vspace{2pt}
% \noindent\textbf{\textit{Normalized Discounted Cumulative Gain at K (NDCG$@K$).}} Using a graded relevance scale of objects in ranking result, discounted cumulative gain (DCG) measures the usefulness, or gain, of the objects based on its position in the order list when recovering to the true ranking. The gain is accumulated from the top to the bottom, with the gain of each result discounted at lower ranks. Compared to DCG, NDCG will be normalized by the ideal DCG.

\begin{figure*}[h]
    \centering
    \begin{subfigure}[b]{0.49\textwidth}
        \centering
        \includegraphics[width=\textwidth]{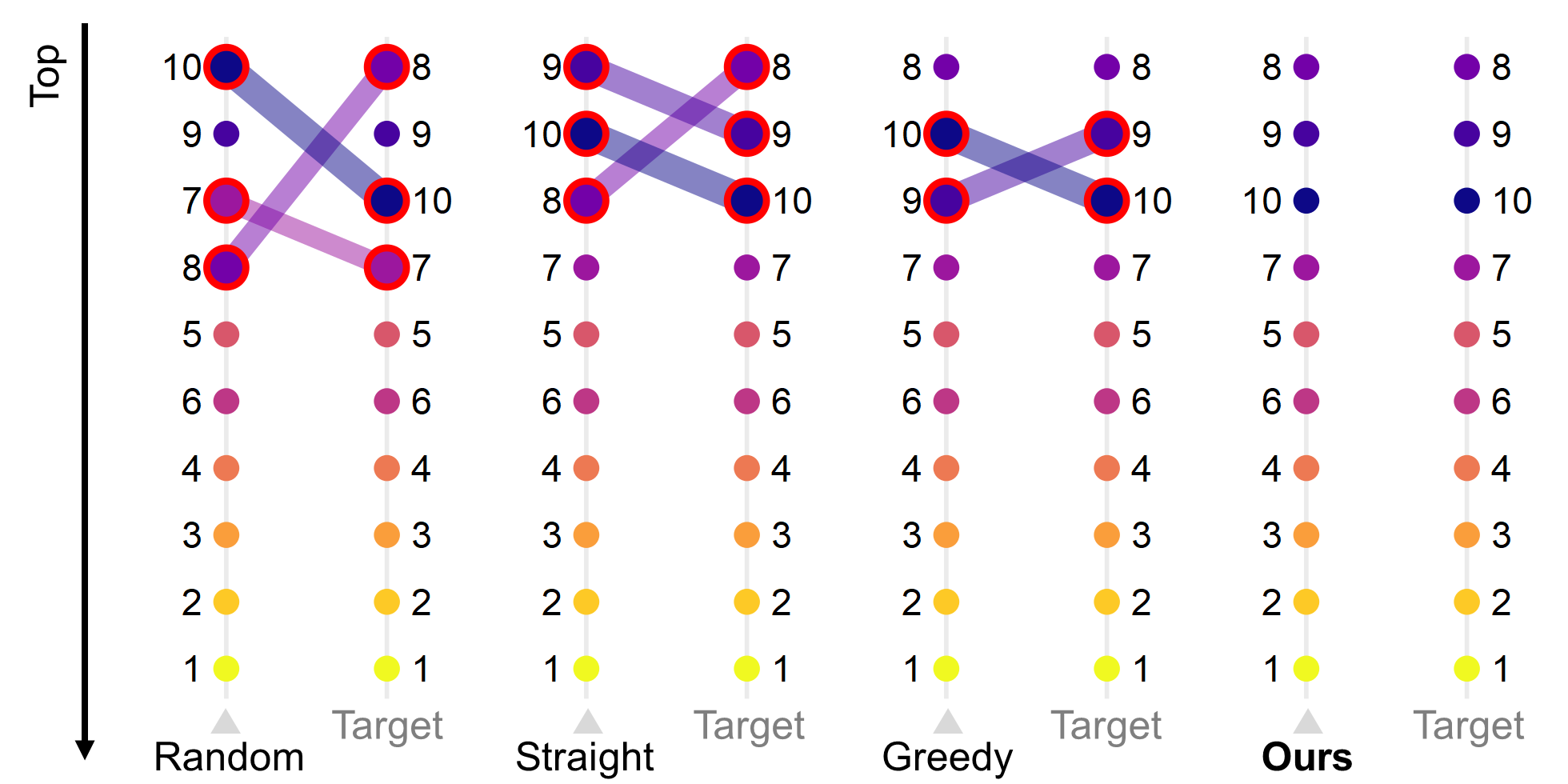}
        \caption{\textbf{HodgeRank}}
        \label{fig:hodge_ranking}
    \end{subfigure}
    \hfill
    \begin{subfigure}[b]{0.49\textwidth}
        \centering
        \includegraphics[width=\textwidth]{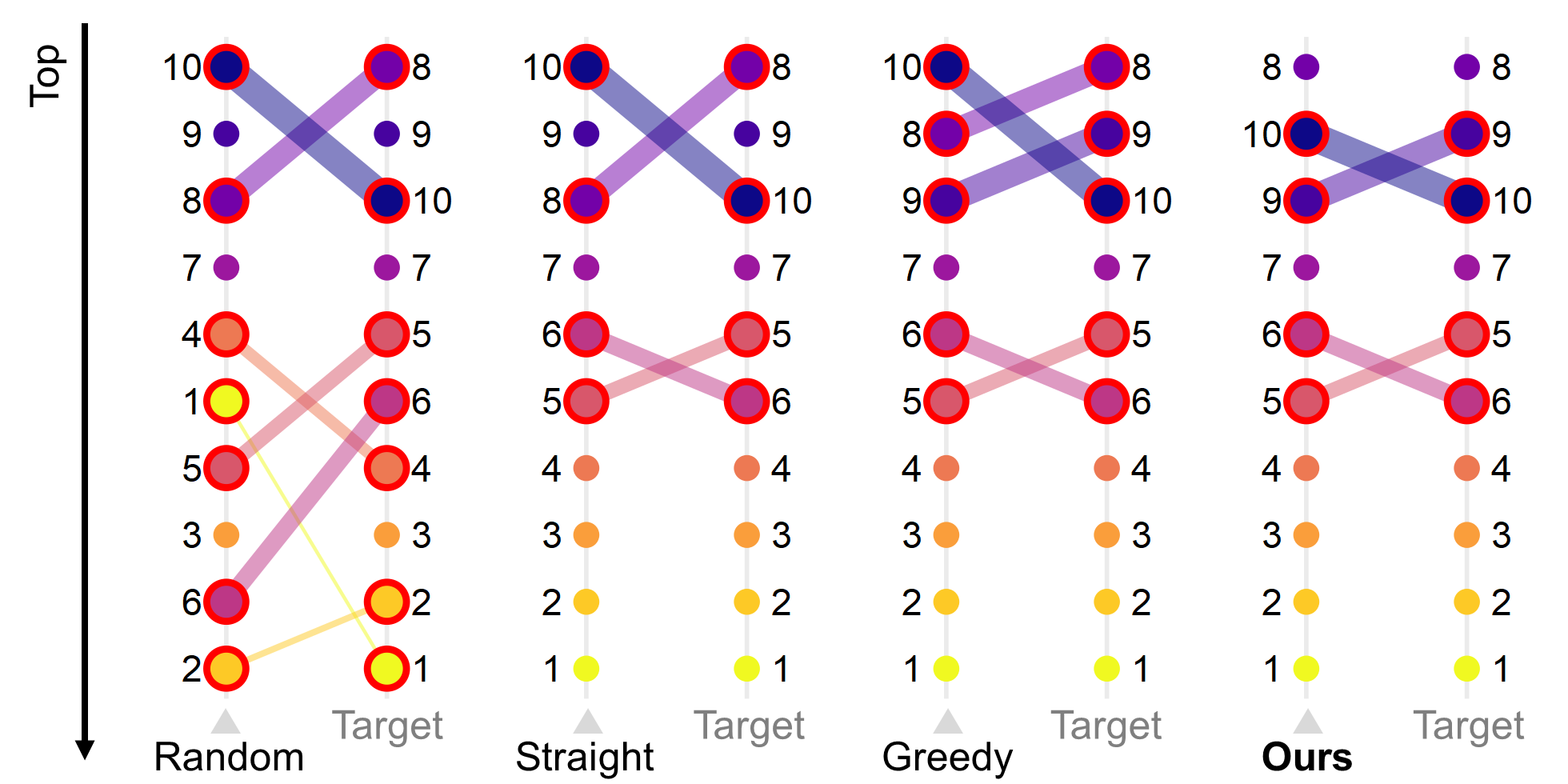}
        \caption{\textbf{RankCentrality}}
        \label{fig:spr_ranking}
    \end{subfigure}    
    \caption{The victims' aggregation results of different manipulation methods on simulated data. The original ranking list is $\boldsymbol{\pi}_0=[10,9,8,7,6,5,4,3,2,1]$ and the target one is $\boldsymbol{\pi}'=[8,9,10,7,5,6,4,3,2,1]$ (refer to \textcolor{Gray}{Target} in the figure). The dark dots represent the candidates with large ID and vice versa. If a candidate is not in the same position in both ranking lists, there will exist an intersection and the width of the line represent the degree of inconsistent influence on the aggregation results. We mark the inconsistent dots with red circles. When there exist multiple intersections and the key locations (top-$3$) are marked by red circles, the attacker has failed to achieve his/her goal. (a) The victim is \textbf{HodgeRank}. The proposed method accomplishes the manipulation of the complete ranking list (no intersection or red circle). (b) The victim is \textbf{RankCentrality}.}
    \label{fig:ranking}
\end{figure*}

\subsection{Competitors}
To the best of our knowledge, the proposed method is the first overture to online manipulation strategies against the pairwise ranking algorithms. We compare the proposed methods with the following three competitors: the random strategy (referred to as `\textbf{\textit{Random}}'), the greedy strategy (referred to as `\textbf{\textit{Greedy}}') and the straightforward strategy (referred to as `\textbf{\textit{Straight}}'). 

\vspace{2pt}
\noindent \textbf{\textit{Random}} perturbation involves a random data source which generates any pairwise comparisons with the same probability. This method does not rely on any information of the desired ranking $\boldsymbol{\pi}'$. We conjecture that the purposeless behavior of the random perturbation would not sculpt the desired results out of the mixed data stream $\boldsymbol{C}'$. However, this strategy is still the evidence to prove the necessity of sophisticated attacker in the online manipulation against the rank aggregation methods.

\vspace{2pt}
\noindent \textbf{\textit{Greedy}} manipulation generates the mixed data stream  $\boldsymbol{C}'$ in a greedy way with the help of the target ranking list $\boldsymbol{\pi}'$. Specifically, this method only insert $\boldsymbol{\pi}'(1)\succ\boldsymbol{\pi}'(j),j=2,\dots,n$ with the same probability. The greedy manipulation could designate the leading position of the aggregated list. Compared with the proposed method, the greedy method lack the manipulation ability of full ranking list.  

\vspace{2pt}
\noindent \textbf{\textit{Straightforward}} strategy implements the adversarial data source with the so-called ``\textbf{\textit{Matthew Principle}}'': increasing the number of the pairwise comparisons which are consistent with the desired ranking list. There exist $n(n-1)/2$ pairwise comparisons which are consistent with $\boldsymbol{\pi}'$ and they have the equal chance to insert into the mixed data stream $\boldsymbol{C}'$. Obviously, this strategy could archive the goal with sufficient turns in the adversarial game between two data source. Compared with the proposed method, the straightforward method will waste some opportunities and need more actions to archive the goal.

\subsection{Simulated Study}
\noindent\textbf{Description. }We validate the proposed sequential manipulation strategies against \textbf{HodgeRank} and \textbf{RankCentrality} on simulated data. We generate the data stream as follows. First, we build a complete graph $\boldsymbol{\mathcal{G}}=(\boldsymbol{V},\boldsymbol{E})$ where $\boldsymbol{V}=[n]$. Then the latent preference score is assigned to each candidate/vertex of $\boldsymbol{V}$ and the true ranking is obtained by these scores. Setting $n=10$ and the true ranking is $\boldsymbol{\pi}_0 =(10,9,8,7,6,5,4,3,2,1)$. Next, we randomly sample $744$ pairwise comparisons from $\boldsymbol{V}\times\boldsymbol{V}$ based on their preference score. Notice that the samples could contain the comparisons which are inconsistent with the true ranking. We regard these samples from the original data source (line 3 of Algorithm \ref{alg:ada_adv_game}) and construct $50$ sequences with different orders. Each sequence will be a trail for the adversary. The goal of adversary is to make \textbf{HodgeRank} and \textbf{RankCentrality} produce $\boldsymbol{\pi}' =(8,9,10,7,5,6,4,3,2,1)$. The number of turns in the adversarial game ($T=75$ in Algorithm \ref{alg:ada_adv_game}) is $10$ percent of the length of the complete sequence. In each turn, the sample from the original data source has an $80\%$ chance of being observed by the adversary (line $7$ of Algorithm \ref{alg:ada_adv_game}). If his/her knowledge is not updated, the attacker will not take any action and wait for another sample from the original data source. Moreover, the attackers could insert $S_0 = 5$ pairwise comparisons to construct the comparison graph in each turn (line $8$ of Algorithm \ref{alg:ada_adv_game}). 

% \usepackage{multirow}

\begin{figure*}
    \centering
    \begin{subfigure}[b]{0.24\textwidth}
        \centering
        \includegraphics[width=\textwidth]{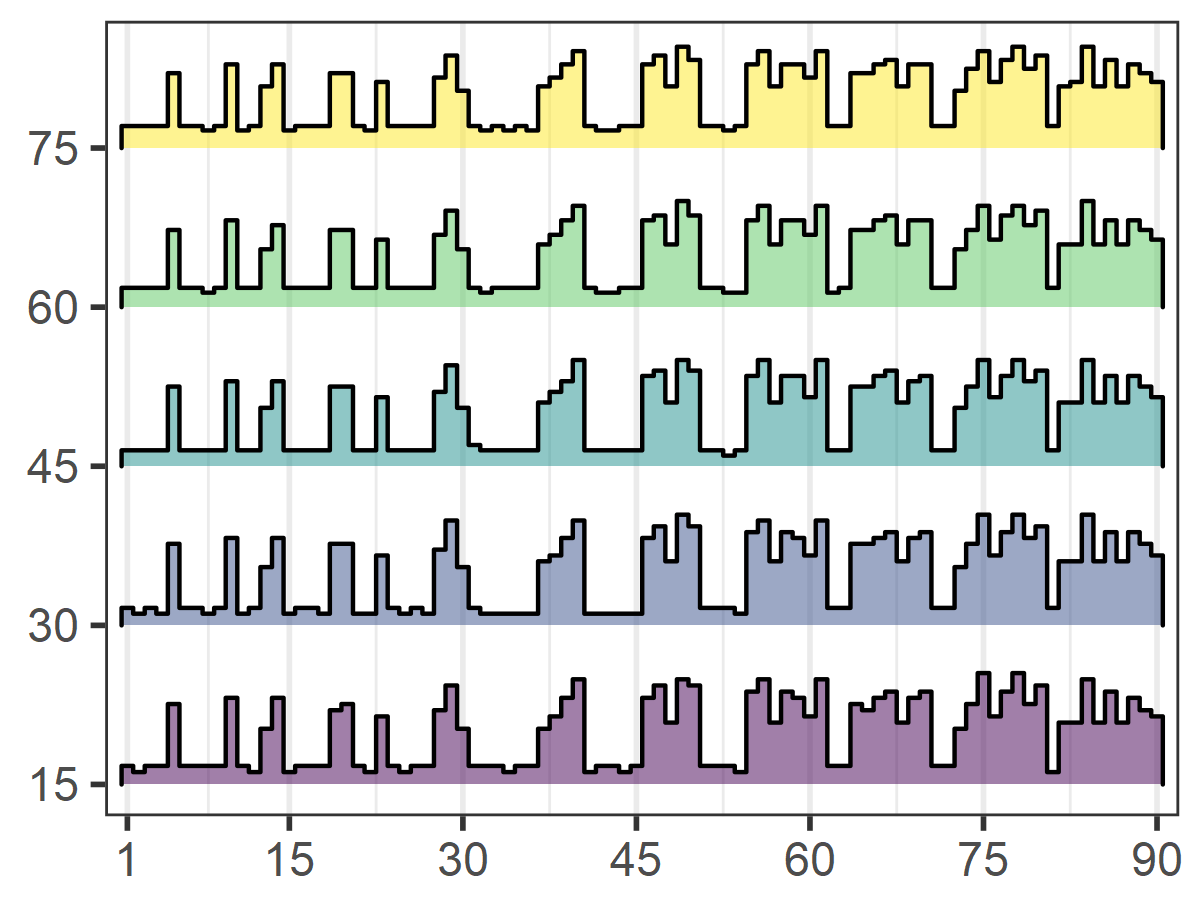}
        \caption{\footnotesize Random v.s. \textbf{HodgeRank}}
    \end{subfigure}
    \hfill
    \begin{subfigure}[b]{0.24\textwidth}
        \centering
        \includegraphics[width=\textwidth]{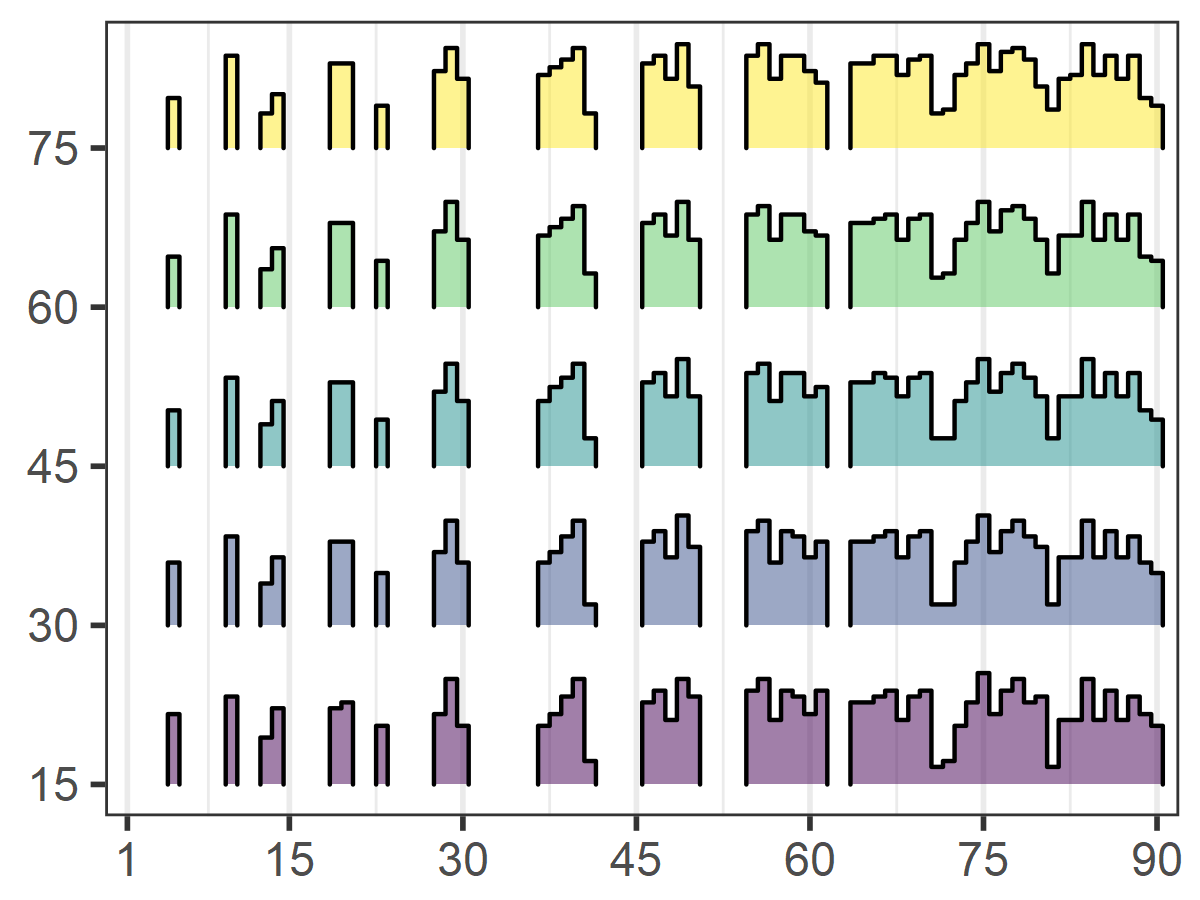}
        \caption{\footnotesize Straight v.s. \textbf{HodgeRank}}
    \end{subfigure}
    \hfill
    \begin{subfigure}[b]{0.24\textwidth}
        \centering
        \includegraphics[width=\textwidth]{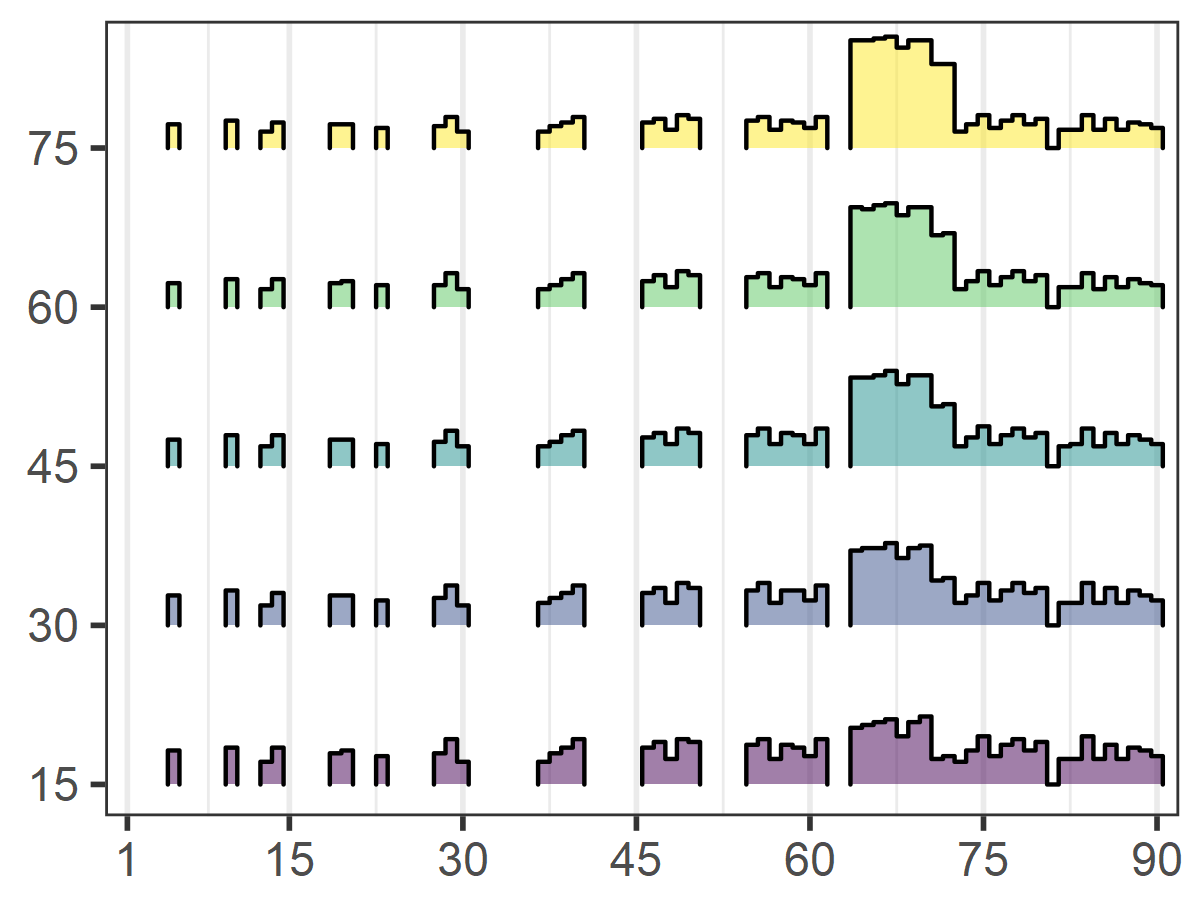}
        \caption{\footnotesize Greedy v.s. \textbf{HodgeRank}}
    \end{subfigure}
    \hfill
    \begin{subfigure}[b]{0.24\textwidth}
        \centering
        \includegraphics[width=\textwidth]{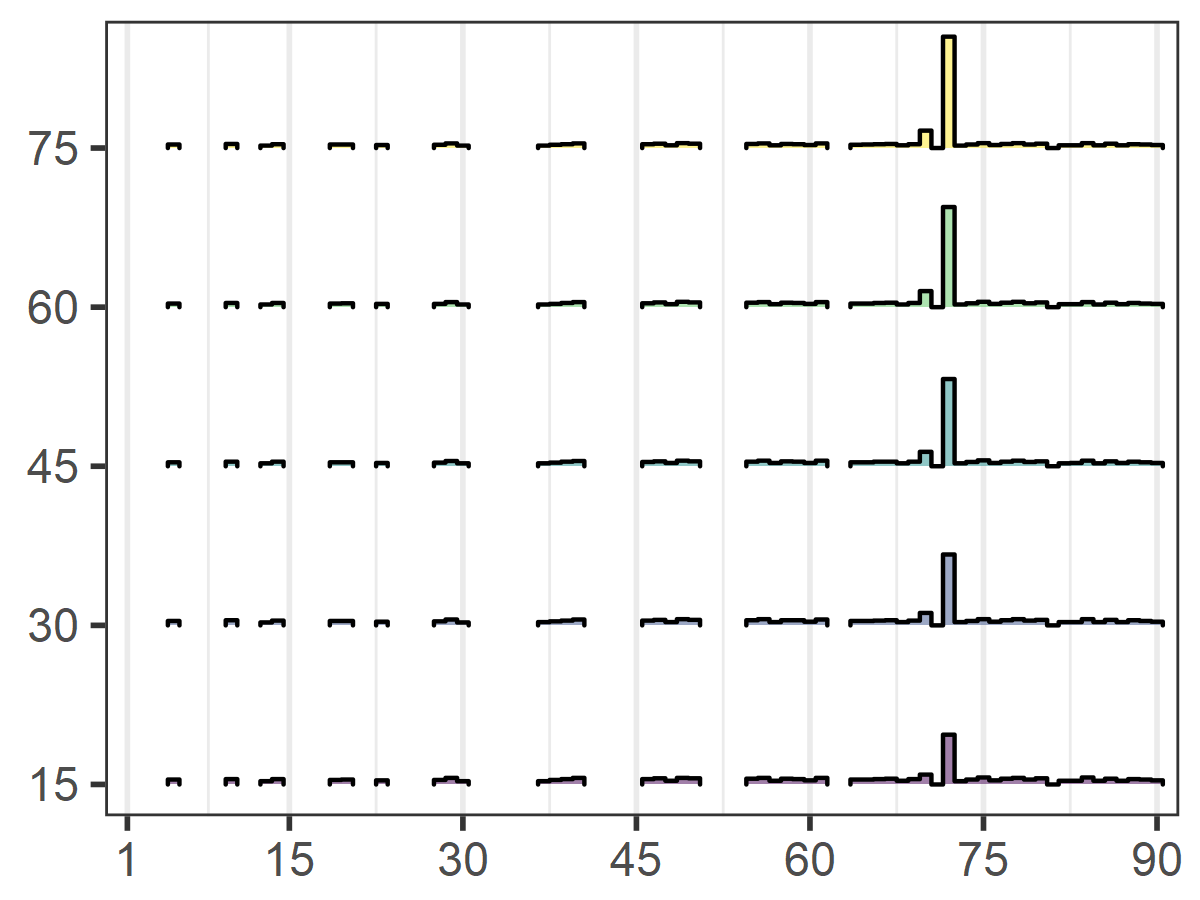}
        \caption{\footnotesize Ours v.s. \textbf{HodgeRank}}
    \end{subfigure}
    \hfill
    \begin{subfigure}[b]{0.24\textwidth}
        \centering
        \includegraphics[width=\textwidth]{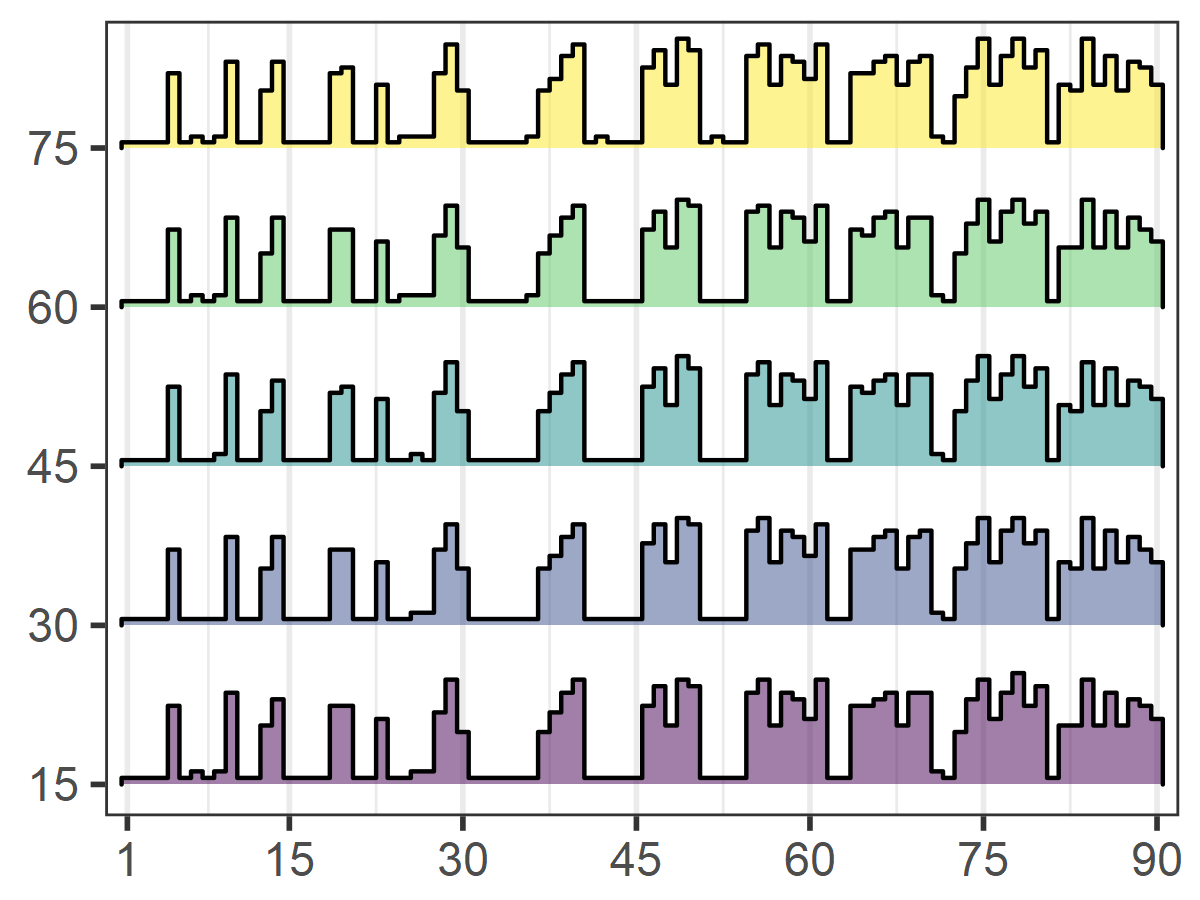}
        \caption{\footnotesize Random v.s. \textbf{RankCentrality}}
    \end{subfigure}
    \hfill
    \begin{subfigure}[b]{0.24\textwidth}
        \centering
        \includegraphics[width=\textwidth]{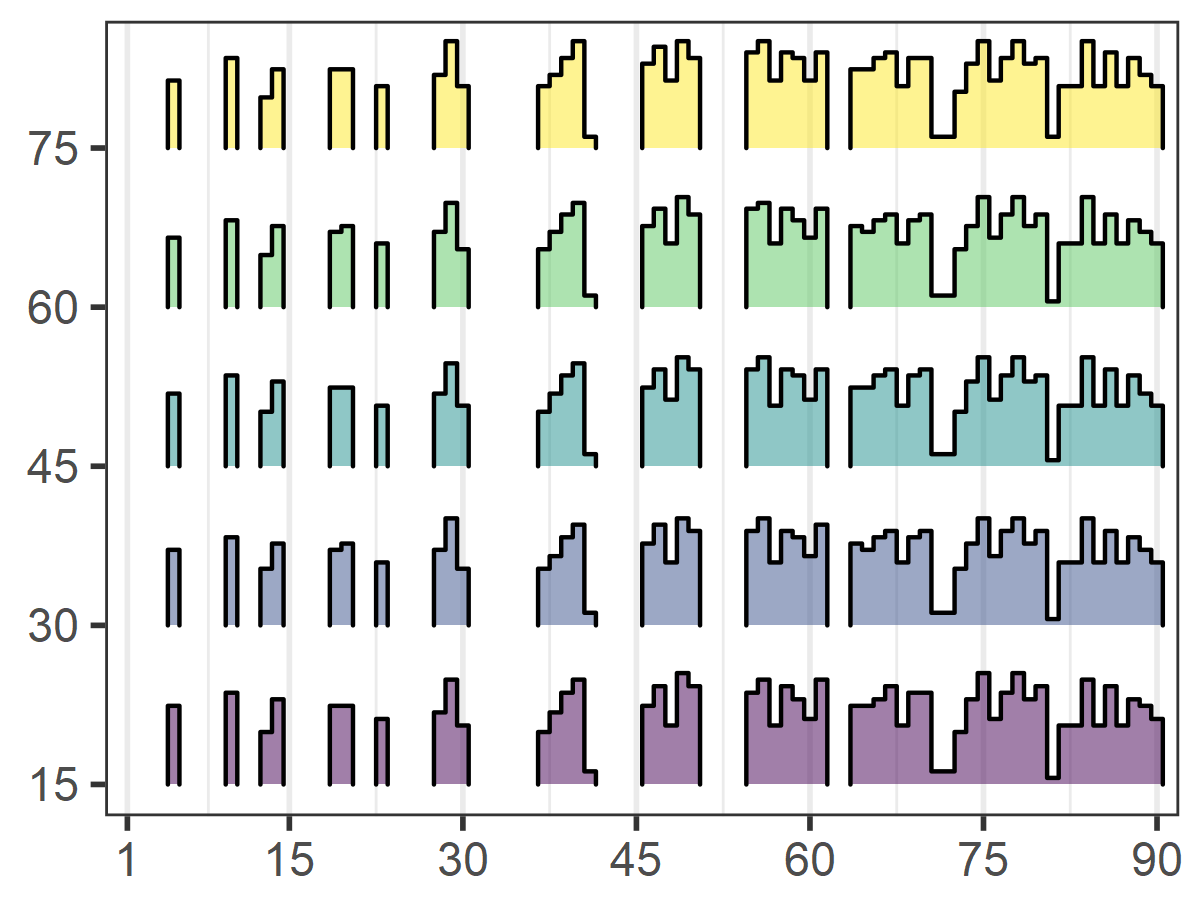}
        \caption{\footnotesize Straight v.s. \textbf{RankCentrality}}
    \end{subfigure}
    \hfill
    \begin{subfigure}[b]{0.24\textwidth}
        \centering
        \includegraphics[width=\textwidth]{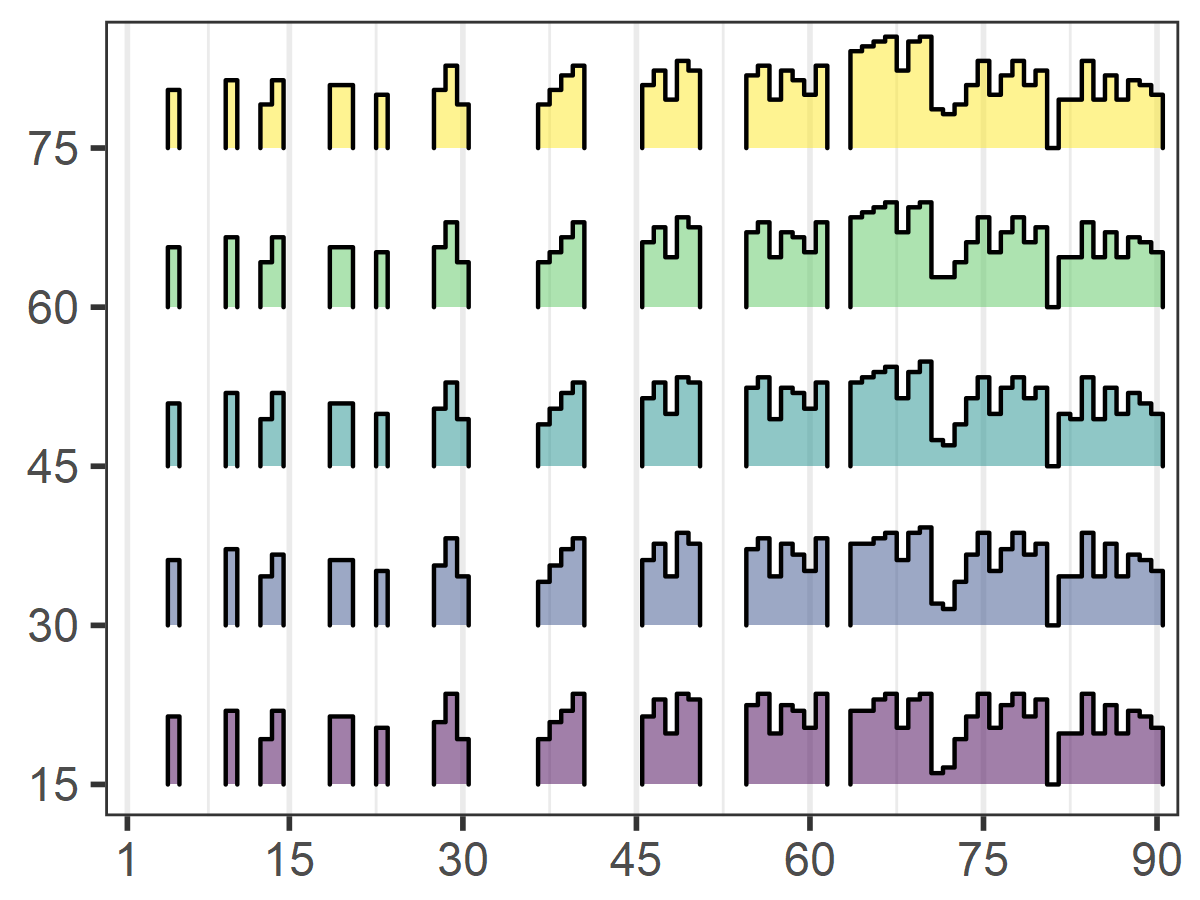}
        \caption{\footnotesize Greedy v.s. \textbf{RankCentrality}}
    \end{subfigure}
    \hfill
    \begin{subfigure}[b]{0.24\textwidth}
        \centering
        \includegraphics[width=\textwidth]{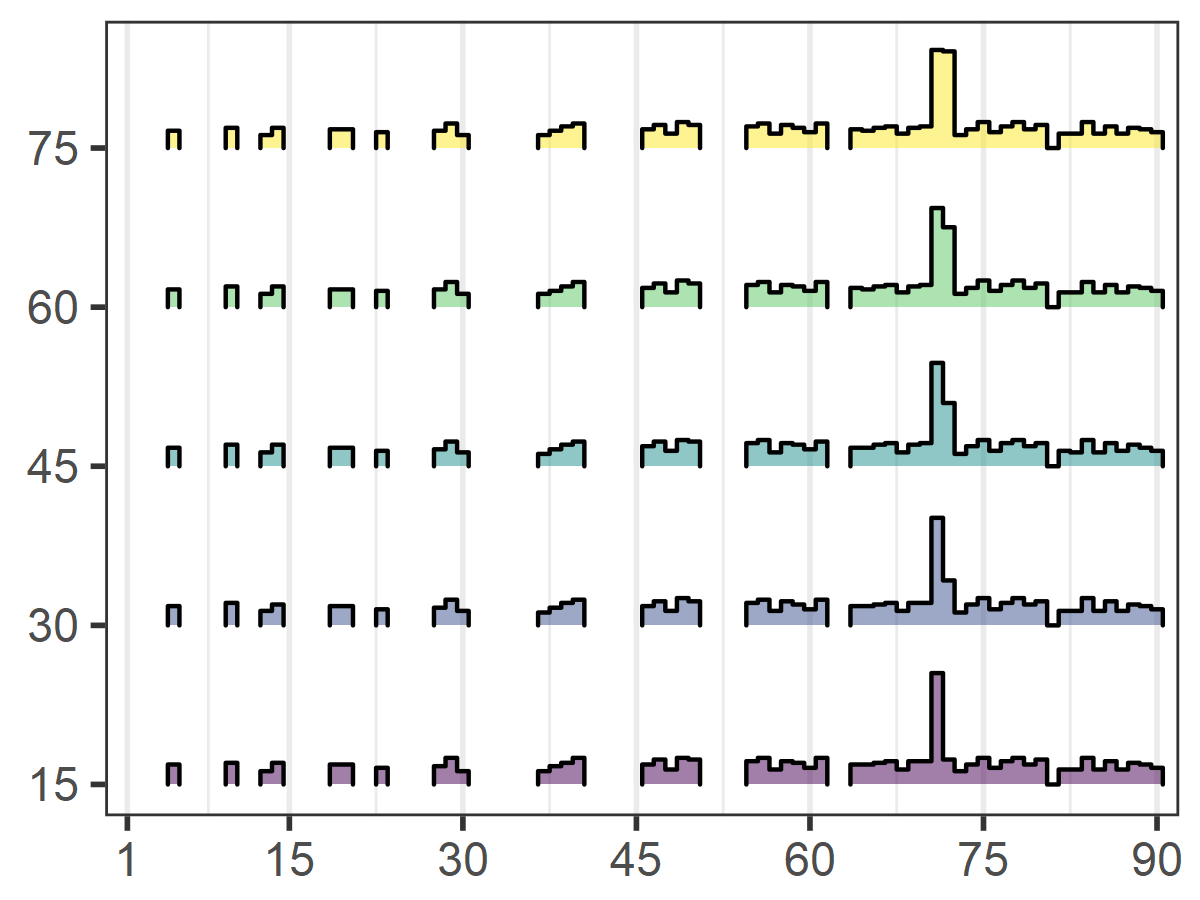}
        \caption{\footnotesize Our v.s. \textbf{RankCentrality}}
    \end{subfigure}
    % \caption{Comparative results of different sequential manipulation methods against \textbf{HodgeRank} on simulated data. The target list is $\boldsymbol{\pi}'=\boldsymbol{\pi}' =(8,9,10,7,5,6,4,3,2,1)$. The box plot illustrates the results of $50$ trials with different data sequence with the same original order $\boldsymbol{\pi}_0 =(10,9,8,7,6,5,4,3,2,1)$. The proposed methods will get the \textbf{R-Rank} to be $1$ and \textbf{Kendall}-$\tau$ close to $1$. The first row (a-f) displays the results of \textbf{HodgeRank}. The second row (g-l) is \textbf{RankCentrality} with the \textbf{reversible} stochastic transition matrix. The last row (m-r) represents \textbf{RankCentrality} with \textbf{irreversible} stochastic transition matrix. The odd columns (column $1$, $3$, $5$) show the \textbf{R-Rank} values and the even columns (column $2$, $4$, $6$) are the results of \textbf{Kendall}-$\tau$. The first two columns exhibit the confrontation scenario with \textbf{complete information and perfect feedback}. The second two columns reveal the results of adversary with \textbf{complete information and imperfect feedback}. The third two columns indicate the \textbf{incomplete information and perfect feedback}. The red frames (o \& p) are failure cases where the proposed methods could not make $\boldsymbol{\pi}_{\boldsymbol{\hat{\theta}}}(1) = \boldsymbol{\pi}^2_{\mathcal{A}}(1) = \boldsymbol{\pi}_{\boldsymbol{\bar{\theta}}}(2)$.}
    \caption{Data distribution generated by different methods on simulated data. The vertical axis lists the number of rounds in the adversarial games and the horizontal axis displays all possible pairwise comparisons. For the same victim, all results are based on the same observed data. The original ranking list is $\boldsymbol{\pi}_0=[10,9,8,7,6,5,4,3,2,1]$ and the target one is $\boldsymbol{\pi}'=[8,9,10,7,5,6,4,3,2,1]$. }
    \label{fig:simu_data_distribution}
\end{figure*}

\vspace{0.25cm}
\noindent\textbf{Comparative Results. }The results of each method against \textbf{HodgeRank} on simulated data are reported in Fig. \ref{fig:simu_hodge_metric} (a)-(b). Our method exhibits higher success rate due to the parsimonious mechanism and consistently outperforms all the competitors by a significant margin. Concerning the performances of the three competitors, we can easily find that: 
\begin{itemize}
    \item Due to the existence of non-modifiable data, the blind attack strategies cannot even interfere with the aggregation results of victim with limited actions. The random perturbation (`\textbf{Random}') can't boost the position of candidate $8$ in all aggregated results. Although the mean of Kendall $\tau$ coefficient with $50$ trials is $0.77$, the degree of consistency between \textbf{Random} and $\boldsymbol{\pi}_0$ remains higher than that between \textbf{Random} and $\boldsymbol{\pi}'$. These arguments could be justified by the visualization of ranking lists in Figure \ref{fig:ranking}. 
    \item The greedy manipulation (`\textbf{Greedy}') can have an impact on the winner of the final ranking lists. However, this method failed to consistently manipulate the aggregation results over a specified number of actions. The interquartile range of reciprocal rank is a large interval and the median is $0.5$ when the adversary executes \textbf{Greedy} against \textbf{HodgeRank}. As all the insertions of \textbf{Greedy} are consistent with the target list, the Kendall $\tau$ coefficient of \textbf{Greedy} is higher than that of \textbf{Random}. This phenomenon does not imply that \textbf{Greedy} had complete control over the victim's result as its generation mechanism only guarantees the desired winner could beat the other candidates. 
    \item In principle, the straightforward strategy (`\textbf{Straight}') has potential to manipulate the complete aggregation results of victim. The efficiency of the method is a concern as it ignores the existing data. We suspect that the method will only work under the conditions of the so-called ``flooding attack'', where $T$ and $S_0$ will be sufficiently large. 
\end{itemize}
We report the results of each method against \textbf{RankCentrality} on simulated data in Fig. \ref{fig:simu_hodge_metric} (c)-(d). The best performance and the median of our method consistently surpass all the competitors. The Kendall $\tau$ coefficient of the proposed method is not $1$. We speculate that this phenomenon comes from the challenge posed by controlling the spectral structure of comparison graph. However, the proposed method is the only one which is able to designate the winner of the aggregation results. Furthermore, we show the specific behavior of the proposed method in every turn of the adversarial game in Fig. \ref{fig:simu_hodge_turn}. Despite the existence of unknown data, all metrics of the proposed method grow and eventually remain stable. This phenomenon implies that the proposed method could obtain an equilibrium state which favors to the adversary even if the details of the victims are not involved. The ranking lists of different methods are shown in Fig. \ref{fig:ranking}. We illustrate every pair of the manipulation result and the target ranking list. All results are based on the same observed data sequence. The proposed method could designate the winner of the aggregation result ($8$ is the top-$1$ candidate in our results) and keep a high correlation with the target ranking list (there only exist two intersections in our result). 

\begin{table*}[!ht]
    \centering
    \caption{Numeric results of different attack methods on crowdsourcing data. The best results are highlighted with \textbf{bold} text. }
    \begin{tabular}{c|cccccc|cccccc}
    \hline
    \multirow{3}{*}{Methods} & \multicolumn{6}{c|}{\textbf{HodgeRank}}                                                       & \multicolumn{6}{c}{\textbf{RankCentrality}}                                                 \\
             & \multicolumn{2}{c}{$(20,29,8,13)$} & \multicolumn{2}{c}{$(8,29,20,13)$} & \multicolumn{2}{c|}{$(13,29,20,8)$} & \multicolumn{2}{c}{$((20,29,8,13)$} & \multicolumn{2}{c}{$(8,29,20,13)$} & \multicolumn{2}{c}{$(13,29,20,8)$} \\
             & R. Rank       & K. $\tau$     & R. Rank       & K. $\tau$     & R. Rank       & K. $\tau$     & R. Rank       & K. $\tau$     & R. Rank       & K. $\tau$     & R. Rank       & K. $\tau$     \\ \hline
    Random   & 0.50          & 0.37          & 0.50          & 0.35          & 0.20          & 0.25          & 0.33          & 0.39          & 0.33          & 0.33          & 0.25          & 0.14          \\
    Straight & 0.50          & \textbf{0.64} & 0.50          & \textbf{0.59} & 0.25          & \textbf{0.52} & 0.50          & 0.44          & 0.33          & 0.35          & 0.25          & 0.23          \\
    Greedy   & \textbf{1.00} & 0.52          & \textbf{1.00} & 0.49          & 0.50          & 0.38          & \textbf{1.00} & 0.55          & 0.33          & 0.55          & 0.25          & 0.24          \\
    Ours     & \textbf{1.00} & 0.55          & \textbf{1.00} & 0.54          & \textbf{1.00} & 0.44          & \textbf{1.00} & \textbf{0.67} & \textbf{1.00} & \textbf{0.55} & \textbf{1.00} & \textbf{0.36} \\ \hline
    \end{tabular}
    \label{tab:age}   
\end{table*}

\begin{table*}[!ht]
    \centering
    \caption{Numeric results of different attack methods on Dublin election data. The best results are highlighted with \textbf{bold} text. }
    \begin{tabular}{c|cccccc|cccccc}
    \hline
    \multirow{3}{*}{Methods} & \multicolumn{6}{c|}{\textbf{HodgeRank}}                                                 & \multicolumn{6}{c}{\textbf{RankCentrality}}                                                 \\
             & \multicolumn{2}{c}{$(2,4,13,5)$} & \multicolumn{2}{c}{$(13,4,2,5)$} & \multicolumn{2}{c|}{$(5,4,2,13)$} & \multicolumn{2}{c}{$(2,4,13,5)$} & \multicolumn{2}{c}{$(13,4,2,5)$} & \multicolumn{2}{c}{$(5,4,2,13)$} \\
             & R. Rank       & K. $\tau$     & R. Rank       & K. $\tau$     & R. Rank       & K. $\tau$     & R. Rank       & K. $\tau$     & R. Rank       & K. $\tau$     & R. Rank       & K. $\tau$     \\ \hline
    Random   & 0.50          & 0.93          & 0.50          & 0.91          & 0.25          & 0.58          & 0.50          & 0.93          & 0.33          & 0.91          & 0.25          & 0.58          \\
    Straight & 0.50          & 0.97          & 0.50          & 0.93          & 0.25          & 0.58          & 0.50          & 0.93          & 0.33          & 0.91          & 0.25          & 0.58          \\
    Greedy   & \textbf{1.00} & 0.98          & \textbf{1.00} & 0.96          & 0.50          & 0.87          & \textbf{1.00} & 0.93          & 0.50          & 0.91          & 0.33          & 0.49          \\
    Ours     & \textbf{1.00} & \textbf{1.00} & \textbf{1.00} & \textbf{1.00} & \textbf{1.00} & \textbf{1.00} & \textbf{1.00} & \textbf{1.00} & \textbf{1.00} & \textbf{1.00} & \textbf{1.00} & \textbf{1.00} \\ \hline
    \end{tabular}
    \label{tab:dublin}
\end{table*}

The data distributions generated by different methods are illustrated in Fig. \ref{fig:simu_data_distribution}. With the help of the data distribution, we can better understand the reasons why the proposed method can achieve sequential manipulation with the interference of the original data source. The victim of the first row is \textbf{HodgeRank} and the second one is \textbf{RankCentrality}. Here the horizontal axis lists all possible pairwise comparisons. The vertical axis list $5$ representative turns in the adversarial games. We index them as follows: No.$(i-1)*9+1$ to $i*9$ are the comparisons $\{(i,j)\ |\ j\in[10],j\neq i\}$. The proposed method makes efficient manipulation with specific purposes. We observe that the bins which represent the desired winner defeating the other candidates are higher after the attack procedure, especially No. $72$ ($8\succ 10$) and No. $71$ ($8\succ 9$). Such behavior ensures that the aggregation results of the victims are resistant to the original data source. The `\textbf{Greedy}' method also increases the number of No. $64$ to No. $72$. However, it is not sufficient to guarantee that `\textbf{Greedy}' could manipulate the victims' ranking lists. The `\textbf{Straight}' method disperses its power and fails to promote the position of candidate $8$. The `\textbf{Random}' method uniformly generates all kinds of pairwise comparisons but it does not help to achieve the goal. 

\subsection{Crowdsourcing}
\textbf{Description. }$30$ images from the human age dataset FGNET\footnote{\url{https://yanweifu.github.io/FG_NET_data/}} are annotated by a group of volunteer users on a crowdsourcing platform\footnote{\url{http://www.chinacrowds.com/}}. The ground-truth age ranking is known to us. The annotator is presented with two images and given a binary choice of which one is older. Totally, we obtain $8,017$ pairwise comparisons from $94$ annotators. The top-$4$ candidates of the true ranking is $(29,20,8,13)$. The goals of adversary are to make \textbf{HodgeRank} and \textbf{RankCentrality} produce $(20,29,8,13)$, $(8,29,20,13)$ and $(13,29,20,8)$ as the top-$4$ candidates. The rest part of the whole ranking list remains unchanged. The number of turns in the adversarial game is $5\%$ of the length of the complete sequence. In each turn, the sample from the original data source has a $90\%$ chance of being observed by the adversary. If his/her knowledge is not updated, the attacker will not take any action and wait for another sample from the original data source. Moreover, the attackers could insert $S_0 = 10$ pairwise comparisons to construct the comparison graph in each turn. 

% \usepackage{multirow}
% \begin{table*}[]
%     \centering
%     \begin{tabular}{c|cc|cc|cc}
%     \hline
%     \multirow{2}{*}{Methods} & \multicolumn{2}{c|}{$(20,29,8,13)$} & \multicolumn{2}{c|}{$(8,29,20,13)$} & \multicolumn{2}{c}{$(13,29,20,8)$}  \\
%                              & R. Rank          & K. $\tau$      & R. Rank          & K. $\tau$     & R. Rank       & K. $\tau$ \\ \hline
%     Random                   & 0.50             & 0.37           & 0.33             & 0.35          & 0.20          & 0.25      \\
%     Straight                 & 0.50             & \textbf{0.64}  & 0.33             & \textbf{0.59} & 0.25          & \textbf{0.52}      \\
%     Greedy                   & \textbf{1.00}    & 0.52           & \textbf{1.00}    & 0.49          & 0.50          & 0.38      \\
%     Ours                     & \textbf{1.00}    & 0.55           & \textbf{1.00}    & 0.54          & \textbf{1.00} & 0.44      \\ \hline
%     \end{tabular}
%     \caption{Numeric results of different attack methods on Crowdsourcing data.}
% \end{table*}

% \usepackage{multirow}

\vspace{0.25cm}
\noindent\textbf{Comparative Results. } It is worth mentioning that this real-world data has a high percentage of outliers (about $20\%$ of all comparisons conflict with the correct age ranking). The proposed methods against \textbf{HodgeRank} and \textbf{RankCentrality} still show promise manipulation as Table \ref{tab:age}. It is more challenging to change $(29,20,8,13)$ to $(8,29,20,13)$ than to $(20,29,8,13)$. Consequently, the values of Kendall $\tau$ coefficient will decrease when the difficulty of the manipulation increases. 

\subsection{Election}
\textbf{Description. }The Dublin election data set\footnote{\url{http://www.preflib.org/data/election/irish/}} contains a complete record of votes for elections held in county Meath, Dublin, Ireland in 2002. This set contains $64,081$ votes over $14$ candidates. These votes could be a complete or partial list of the candidate set. The ground-truth ranking of $14$ candidates is based on their obtained first preference votes\footnote{\url{https://electionsireland.org/result.cfm?election=2002&cons=178&sort=first}}. The five candidates who receive the most first preference votes will be the winner of the election. The top-$4$ of is $\boldsymbol{\pi}_0 =(4,2,13,5)$. Then these votes are converted into the pairwise comparisons. The total number of the comparisons is $652,817$. The goals of adversary are to make \textbf{HodgeRank} and \textbf{RankCentrality} produce $(2,4,13,5)$, $(13,4,2,5)$ and $(5,4,2,13)$ as the top-$4$ candidates. The number of turns in the adversarial game is $1\%$ of the length of the complete sequence. In each turn, the sample from the original data source has an $80\%$ chance of being observed by the adversary. If his/her knowledge is not updated, the attacker will not take any action and wait for another sample from the original data source. Moreover the attackers could insert $S_0 = 5$ pairwise comparisons to construct the comparison graph in each turn. 

% Here we only allow the adversary to insert $65,281$ pairwise comparisons at most and these data will be inserted at the tails of the whole data collection procedure. 

% \usepackage{multirow}
% \begin{table*}[]
%     \centering
%     \begin{tabular}{c|cc|cc|cc}
%     \hline
%     \multirow{2}{*}{Methods} & \multicolumn{2}{c|}{$(2,4,13,5)$} & \multicolumn{2}{c|}{$(13,4,2,5)$} & \multicolumn{2}{c}{$(5,4,2,13)$}  \\
%                              & R. Rank & K. $\tau$ & R. Rank   & K. $\tau$ & R. Rank   & K. $\tau$ \\ \hline
%     Random                   & 0.50    & 0.93      & 0.50      & 0.91      & 0.25      & 0.58      \\
%     Straight                 & 0.50    & 0.97      & 0.50      & 0.93      & 0.25      & 0.58      \\
%     Greedy                   & \textbf{1.00}    & 0.98      & \textbf{1.00}      & 0.96      & 0.50      & 0.87      \\
%     Ours                     & \textbf{1.00}    & \textbf{1.00}      & \textbf{1.00}      & \textbf{1.00}      & \textbf{1.00}      & \textbf{1.00}      \\ \hline
%     \end{tabular}
%     \caption{Numeric results of different attack methods on Dublin election data.}
% \end{table*}

\vspace{0.25cm}
\noindent\textbf{Comparative Results. }It is worth noting that the election result is not obtained by pairwise ranking aggregation. However, the ordered list aggregated from induced comparisons still shows a positive correlation with the actual election result. Once the attackers generate a successful manipulation strategy against the ballots collection process, this attack plan could be adopted to manipulate the election in the real world. Consequently, the proposed sequential strategy is still able to manipulate the election results. The aggregation results of \textbf{HodgeRank} and \textbf{RankCentrality} are still manipulated by the proposed method, see Table \ref{tab:dublin}. 
% Different from the manipulation or strategic voting problem in social choice \cite{brandt2016handbook}, the proposed attack framework could break the barrier of computational complexity \cite{DBLP:journals/amai/Walsh11,DBLP:conf/atal/VaishM0B16,DBLP:journals/toct/HemaspaandraHM20}. 

\section{Conclusion}
\label{sec:conclusion}
In this paper, we establish the first study of sequential manipulation in the context of ranking aggregation with pairwise comparisons to the best of our knowledge. We find that the data collection process is the Achilles' heel of the rank aggregation. The sequential attack problem is formulated as a distributionally robust game between two players, the online manipulator and the ranker who possesses the original data `source'. Furthermore, we introduce the sampling algorithms to analyze the properties of the underlying distributionally Nash equilibrium. Like the two sides of a coin, we prove that the representation ability of sampling methods could turn into the vulnerability when the mixed data source supports the goal of an adversary. With the help of Bayesian decision theory, we develop the manipulation policy with complete knowledge, which achieves the asymptotic optimality. Then a distributionally robust generation rule is proposed to resist the uncertainty of the observed sequence. Our empirical studies show that the proposed sequential manipulation methods could achieve the attacker's goal in the sense that the leading candidate of the aggregated ranking list is the designated one by the adversary. 
% or


% \section{Related Work}
% \label{sec:related_work}

% \subsection{Poisoning Attack in Machine Learning}
% The vulnerability and their security ramifications of machine learning algorithms have long been studied. Much of these
%  work has focused on the regression and classification. The Attacks come in two forms: poisoning the training data of a learner, and attacking an already-learned model.

% \subsection{Robust Game and the Equilibrium}

% \subsection{Distributionally Robust Optimization}

% \appendices

% % you can choose not to have a title for an appendix
% % if you want by leaving the argument blank
% \section{}

% % use section* for acknowledgment
% \ifCLASSOPTIONcompsoc
%   % The Computer Society usually uses the plural form
%   \section*{Acknowledgments}
% \else
%   % regular IEEE prefers the singular form
% \fi


\ifCLASSOPTIONcaptionsoff
  \newpage
\fi



%\bibliographystyle{IEEEtran}

% % \bibitem{IEEEhowto:kopka}
% % H.~Kopka and P.~W. Daly, \emph{A Guide to \LaTeX}, 3rd~ed.\hskip 1em plus
% %   0.5em minus 0.4em\relax Harlow, England: Addison-Wesley, 1999.

% \end{thebibliography}
% \balance
\bibliographystyle{plain}
\bibliography{sample}

\begin{thebibliography}{10}

\bibitem{DBLP:conf/icml/0001AKP20}
Arpit Agarwal, Shivani Agarwal, Sanjeev Khanna, and Prathamesh Patil.
\newblock Rank aggregation from pairwise comparisons in the presence of
  adversarial corruptions.
\newblock In {\em International Conference on Machine Learning,}, pages 85--95,
  2020.

\bibitem{arrow2012social}
K.J. Arrow and E.S. Maskin.
\newblock {\em Social Choice and Individual Values: Third Edition}.
\newblock Yale University Press, 2012.

\bibitem{10.2748/tmj/1178243286}
Kazuoki Azuma.
\newblock {Weighted sums of certain dependent random variables}.
\newblock {\em Tohoku Mathematical Journal}, 19(3):357--367, 1967.

\bibitem{badgeley2015hybrid}
Marcus~A Badgeley, Stuart~C Sealfon, and Maria~D Chikina.
\newblock Hybrid bayesian-rank integration approach improves the predictive
  power of genomic dataset aggregation.
\newblock {\em Bioinformatics}, 31(2):209--215, 2015.

\bibitem{bank1982non}
Bernd Bank, J{\"u}rgen Guddat, Diethard Klatte, Bernd Kummer, and Klaus Tammer.
\newblock {\em Non-linear Parametric Optimization}.
\newblock Springer, 1982.

\bibitem{bartholdi1989voting}
John Bartholdi, Craig~A Tovey, and Michael~A Trick.
\newblock Voting schemes for which it can be difficult to tell who won the
  election.
\newblock {\em Social Choice and Welfare}, 6:157--165, 1989.

\bibitem{DBLP:journals/orl/BeckT03}
Amir Beck and Marc Teboulle.
\newblock Mirror descent and non-linear projected subgradient methods for
  convex optimization.
\newblock {\em Operations Research Letters}, 31(3):167--175, 2003.

\bibitem{10.1145/3375395.3387643}
Omri Ben-Eliezer and Eylon Yogev.
\newblock The adversarial robustness of sampling.
\newblock In {\em ACM SIGMOD-SIGACT-SIGAI Symposium on Principles of Database
  Systems}, page 49–62, 2020.

\bibitem{Bertsekas/99}
D.P. Bertsekas.
\newblock {\em Nonlinear Programming}.
\newblock Athena Scientific, 1999.

\bibitem{doi:10.1287/moor.2018.0936}
Jose Blanchet and Karthyek Murthy.
\newblock Quantifying distributional model risk via optimal transport.
\newblock {\em Mathematics of Operations Research}, 44(2):565--600, 2019.

\bibitem{pmlr-v162-bong22a}
Heejong Bong and Alessandro Rinaldo.
\newblock Generalized results for the existence and consistency of the {MLE} in
  the bradley-terry-luce model.
\newblock In {\em International Conference on Machine Learning}, pages
  2160--2177, 2022.

\bibitem{bradley1952rank}
Ralph~Allan Bradley and Milton~E Terry.
\newblock Rank analysis of incomplete block designs: I. the method of paired
  comparisons.
\newblock {\em Biometrika}, 39(3):324--345, 1952.

\bibitem{chatterjee2004nash}
Krishnendu Chatterjee, Rupak Majumdar, and Marcin Jurdzi{\'n}ski.
\newblock On nash equilibria in stochastic games.
\newblock In {\em International Workshop on Computer Science Logic}, pages
  26--40, 2004.

\bibitem{10.1214/22-AOS2175}
Pinhan Chen, Chao Gao, and Anderson~Y. Zhang.
\newblock {Optimal full ranking from pairwise comparisons}.
\newblock {\em The Annals of Statistics}, 50(3):1775 -- 1805, 2022.

\bibitem{10.1214/21-AOS2166}
Pinhan Chen, Chao Gao, and Anderson~Y. Zhang.
\newblock {Partial Recovery for Top-k Ranking: Optimality of MLE and
  Sub-optimality of the Spectral Method}.
\newblock {\em The Annals of Statistics}, 50(3):1618 -- 1652, 2022.

\bibitem{doi:10.1287/moor.2021.1209}
Xi~Chen, Yunxiao Chen, and Xiaoou Li.
\newblock Asymptotically optimal sequential design for rank aggregation.
\newblock {\em Mathematics of Operations Research}, 47(3):2310--2332, 2022.

\bibitem{chen2019spectral}
Yuxin Chen, Jianqing Fan, Cong Ma, and Kaizheng Wang.
\newblock Spectral method and regularized mle are both optimal for top-k
  ranking.
\newblock {\em Annals of statistics}, 47(4):2204, 2019.

\bibitem{10.1214/aoms/1177706205}
Herman Chernoff.
\newblock Sequential design of experiments.
\newblock {\em The Annals of Mathematical Statistics}, 30(3):755--770, 1959.

\bibitem{im/1175266369}
Fan Chung and Linyuan Lu.
\newblock {Concentration inequalities and martingale inequalities: a survey}.
\newblock {\em Internet Mathematics}, 3(1):79 -- 127, 2006.

\bibitem{DBLP:journals/aim/DasguptaC19}
Prithviraj Dasgupta and Joseph~B. Collins.
\newblock A survey of game theoretic approaches for adversarial machine
  learning in cybersecurity tasks.
\newblock {\em {AI} Mag.}, 40(2):31--43, 2019.

\bibitem{DBLP:conf/icml/DuchiSSC08}
John~C. Duchi, Shai Shalev{-}Shwartz, Yoram Singer, and Tushar Chandra.
\newblock Efficient projections onto the $\ell_1$-ball for learning in high
  dimensions.
\newblock In {\em International Conference on Machine Learning}, pages
  272--279, 2008.

\bibitem{fan2001generalized}
Jianqing Fan, Chunming Zhang, and Jian Zhang.
\newblock Generalized likelihood ratio statistics and wilks phenomenon.
\newblock {\em The Annals of Statistics}, 29(1):153--193, 2001.

\bibitem{10.1214/aop/1176996452}
David~A. Freedman.
\newblock On tail probabilities for martingales.
\newblock {\em The Annals of Probability}, 3(1):100--118, 1975.

\bibitem{frohmader20211}
Andrew Frohmader and Hans Volkmer.
\newblock 1-wasserstein distance on the standard simplex.
\newblock {\em Algebraic Statistics}, 12(1):43--56, 2021.

\bibitem{doi:10.1287/moor.2022.1275}
Rui Gao and Anton Kleywegt.
\newblock Distributionally robust stochastic optimization with wasserstein
  distance.
\newblock {\em Mathematics of Operations Research}, 49(2):1--59, 2023.

\bibitem{doi:10.1287/moor.2020.1109}
Alexander Goldenshluger and Assaf Zeevi.
\newblock Optimal stopping of a random sequence with unknown distribution.
\newblock {\em Mathematics of Operations Research}, 47(1):29--49, 2022.

\bibitem{hiriart2013convex}
Jean-Baptiste Hiriart-Urruty and Claude Lemar{\'e}chal.
\newblock {\em Convex Analysis and Minimization Algorithms I: Fundamentals},
  volume 305.
\newblock Springer, 2013.

\bibitem{10.2307/2282952}
Wassily Hoeffding.
\newblock Probability inequalities for sums of bounded random variables.
\newblock {\em Journal of the American Statistical Association},
  58(301):13--30, 1963.

\bibitem{Jiang2011}
Xiaoye Jiang, Lek-Heng Lim, Yuan Yao, and Yinyu Ye.
\newblock Statistical ranking and combinatorial hodge theory.
\newblock {\em Mathematical Programming}, 127(1):203--244, 2011.

\bibitem{kakutani1941generalization}
Shizuo Kakutani.
\newblock A generalization of brouwer’s fixed point theorem.
\newblock {\em Duke Mathematical Journal}, 8(3):457--459, 1941.

\bibitem{keener1993perron}
James~P Keener.
\newblock The perron--frobenius theorem and the ranking of football teams.
\newblock {\em SIAM Review}, 35(1):80--93, 1993.

\bibitem{lerman2021robust}
Gilad Lerman and Yunpeng Shi.
\newblock Robust group synchronization via cycle-edge message passing.
\newblock {\em Foundations of Computational Mathematics}, 22:1665--1741, 2022.

\bibitem{pmlr-v180-li22g}
Wanshan Li, Shamindra Shrotriya, and Alessandro Rinaldo.
\newblock $\ell_{\infty}$-bounds of the mle in the btl model under general
  comparison graphs.
\newblock In {\em International Conference on Uncertainty in Artificial
  Intelligence}, pages 1178--1187, 2022.

\bibitem{LI2001516}
Yi~Li, Philip~M. Long, and Aravind Srinivasan.
\newblock Improved bounds on the sample complexity of learning.
\newblock {\em Journal of Computer and System Sciences}, 62(3):516--527, 2001.

\bibitem{liu2018distributionally}
Yongchao Liu, Huifu Xu, Shu-Jung~Sunny Yang, and Jin Zhang.
\newblock Distributionally robust equilibrium for continuous games: Nash and
  stackelberg models.
\newblock {\em European Journal of Operational Research}, 265(2):631--643,
  2018.

\bibitem{doi:10.1287/opre.2022.2313}
Yue Liu, Ethan~X. Fang, and Junwei Lu.
\newblock Lagrangian inference for ranking problems.
\newblock {\em Operations Research}, 71(1):202--223, 2023.

\bibitem{10.1109/TPAMI.2021.3087514}
Ke~Ma, Qianqian Xu, Jinshan Zeng, Xiaochun Cao, and Qingming Huang.
\newblock Poisoning attack against estimating from pairwise comparisons.
\newblock {\em IEEE Transactions on Pattern Analysis and Machine Intelligence},
  44(10):6393--6408, 2022.

\bibitem{9830042}
Ke~Ma, Qianqian Xu, Jinshan Zeng, Guorong Li, Xiaochun Cao, and Qingming Huang.
\newblock A tale of hodgerank and spectral method: Target attack against rank
  aggregation is the fixed point of adversarial game.
\newblock {\em {IEEE} Transactions on Pattern Analysis and Machine
  Intelligence}, pages 1--18, 2022.

\bibitem{McDiarmid1998}
Colin McDiarmid.
\newblock {\em Concentration}, pages 195--248.
\newblock Springer Berlin Heidelberg, 1998.

\bibitem{monge1781memoire}
Gaspard Monge.
\newblock M{\'e}moire sur la th{\'e}orie des d{\'e}blais et des remblais.
\newblock {\em Histoire de l'Acad{\'e}mie Royale des Sciences de Paris}, pages
  666--704, 1781.

\bibitem{DBLP:journals/ior/NegahbanOS17}
Sahand Negahban, Sewoong Oh, and Devavrat Shah.
\newblock Rank centrality: Ranking from pairwise comparisons.
\newblock {\em Operation Research}, 65(1):266--287, 2017.

\bibitem{peyre2019computational}
Gabriel Peyr{\'e}, Marco Cuturi, et~al.
\newblock Computational optimal transport.
\newblock {\em Foundations and Trends{\textregistered} in Machine Learning},
  11(5-6):355--607, 2019.

\bibitem{DBLP:conf/kdd/RadlinskiJ07}
Filip Radlinski and Thorsten Joachims.
\newblock Active exploration for learning rankings from click-through data.
\newblock In {\em {ACM} International Conference on Knowledge Discovery and
  Data Mining}, pages 570--579, 2007.

\bibitem{10.2307/1911749}
J.~B. Rosen.
\newblock Existence and uniqueness of equilibrium points for concave n-person
  games.
\newblock {\em Econometrica}, 33(3):520--534, 1965.

\bibitem{saari2000mathematics}
Donald~G Saari.
\newblock The mathematics of voting: Democratic symmetry.
\newblock {\em Economist}, 83, 2000.

\bibitem{skrondal2003multilevel}
Anders Skrondal and Sophia Rabe-Hesketh.
\newblock Multilevel logistic regression for polytomous data and rankings.
\newblock {\em Psychometrika}, 68:267--287, 2003.

\bibitem{10.1214/aop/1176988847}
M.~Talagrand.
\newblock {Sharper Bounds for Gaussian and Empirical Processes}.
\newblock {\em The Annals of Probability}, 22(1):28--76, 1994.

\bibitem{Vapnik:71}
V.~N. Vapnik and A.~Y. Chervonenkis.
\newblock On the uniform convergence of relative frequencies of events to their
  probabilities.
\newblock {\em Theory of Probability and its Applications}, 16(2):264--280,
  1971.

\bibitem{villani2008optimal}
C{\'e}dric Villani.
\newblock {\em Optimal Transport: Old and New}.
\newblock Springer, 2008.

\bibitem{DBLP:conf/aistats/0001SR20}
Jingyan Wang, Nihar~B. Shah, and R.~Ravi.
\newblock Stretching the effectiveness of {MLE} from accuracy to bias for
  pairwise comparisons.
\newblock In {\em International Conference on Artificial Intelligence and
  Statistics}, pages 66--76, 2020.

\bibitem{DBLP:journals/pami/WeiGY23}
Xingxing Wei, Ying Guo, and Jie Yu.
\newblock Adversarial sticker: {A} stealthy attack method in the physical
  world.
\newblock {\em {IEEE} Transactions on Pattern Analysis and Machine
  Intelligence}, 45(3):2711--2725, 2023.

\bibitem{DBLP:journals/pami/WeiGYZ23}
Xingxing Wei, Ying Guo, Jie Yu, and Bo~Zhang.
\newblock Simultaneously optimizing perturbations and positions for black-box
  adversarial patch attacks.
\newblock {\em {IEEE} Transactions on Pattern Analysis and Machine
  Intelligence}, 45(7):9041--9054, 2023.

\bibitem{DBLP:journals/pami/WeiWY23}
Xingxing Wei, Songping Wang, and Huanqian Yan.
\newblock Efficient robustness assessment via adversarial spatial-temporal
  focus on videos.
\newblock {\em {IEEE} Transactions on Pattern Analysis and Machine
  Intelligence}, 45(9):10898--10912, 2023.

\end{thebibliography}
% \end{IEEEbiography}
% % insert where needed to balance the two columns on the last page with
% %\newpage
%\vfill
\begin{IEEEbiography}[{\includegraphics[width=1in,height=1.25in,clip,keepaspectratio]{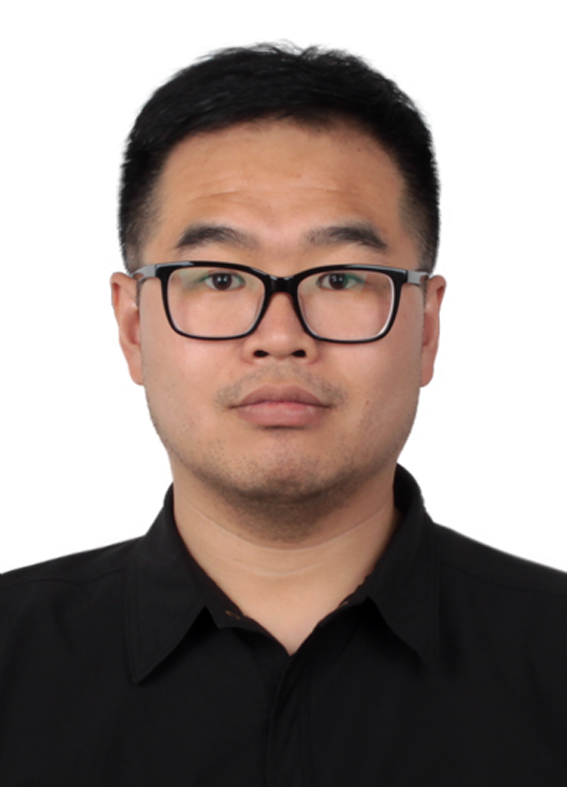}}]
    {Ke Ma} is an associate professor with the School of Electronic, Electrical and Communication Engineering, University of Chinese Academy of Sciences (UCAS), Beijing, China. He received the B.S. degree in mathematics from Tianjin University in 2009, M.E. degree in software engineering from Beihang University (BUAA) in 2013, and the Ph.D. degree in computer science from the Key Laboratory of Information Security (SKLOIS), Institute of Information Engineering (IIE), Chinese Academy of Sciences (CAS), in 2019. His research interests include rank aggregation and algorithmic game theory.
\end{IEEEbiography}

\begin{IEEEbiography}[{\includegraphics[width=1in,height=1.25in,clip,keepaspectratio]{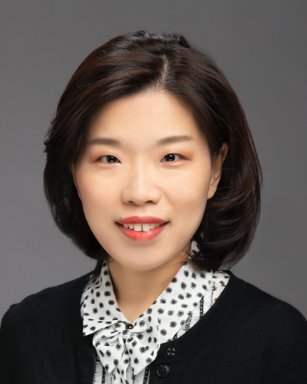}}]
    {Qianqian Xu} received the B.S. degree in computer science from China University of Mining and Technology in 2007 and the Ph.D. degree in computer science from University of Chinese Academy of Sciences in 2013. She is currently a Professor with the Institute of Computing Technology, Chinese Academy of Sciences, Beijing, China. Her research interests include statistical machine learning, with applications in multimedia and computer vision. She has authored or coauthored 70+ academic papers in prestigious international journals and conferences (including T-PAMI, IJCV, T-IP, NeurIPS, ICML, CVPR, AAAI, etc). Moreover, she serves as an associate editor of IEEE Transactions on Circuits and Systems for Video Technology, IEEE Transactions on Multimedia, and ACM Transactions on Multimedia Computing, Communications, and Applications.
\end{IEEEbiography}

\begin{IEEEbiography}[{\includegraphics[width=1in,height=1.25in,clip,keepaspectratio]{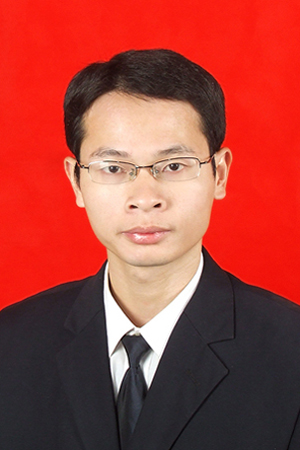}}]
    {Jinshan Zeng} received the Ph.D. degree in applied mathematics from Xi'an Jiaotong University, Xi'an, China, in 2015. He is currently a professor with the School of Computer and Information Engineering, Jiangxi Normal University, Nanchang, China. His research interests include non-convex optimization and machine learning.
\end{IEEEbiography}

\begin{IEEEbiography}[{\includegraphics[width=1in,height=1.25in,clip,keepaspectratio]{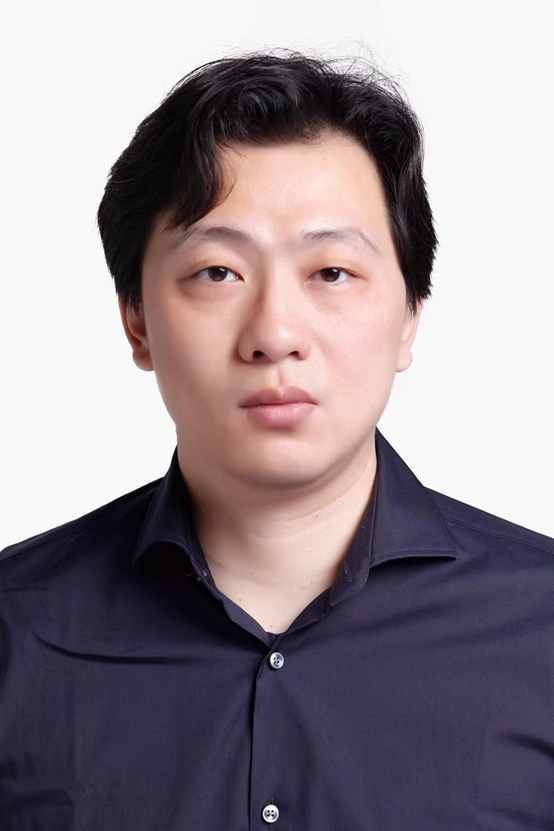}}]
    {Wei Liu} is currently a Distinguished Scientist of Tencent and the Director of Ads Multimedia AI at Tencent Data Platform. Prior to that, he has been a research staff member of IBM T. J. Watson Research Center, USA. Dr. Liu has long been devoted to fundamental research and technology development in core fields of AI, including deep learning, machine learning, computer vision, pattern recognition, information retrieval, big data, etc. To date, he has published extensively in these fields with more than 270 peer-reviewed technical papers, and also issued 23 US patents. He currently serves on the editorial boards of IEEE TPAMI, TNNLS, IEEE Intelligent Systems, and Transactions on Machine Learning Research. He is an Area Chair of top-tier computer science and AI conferences, e.g., NeurIPS, ICML, IEEE CVPR, IEEE ICCV, IJCAI, and AAAI. Dr. Liu is a Fellow of the IEEE, IAPR, AAIA, IMA, RSA, and BCS, and an Elected Member of the ISI.
\end{IEEEbiography}

\begin{IEEEbiography}[{\includegraphics[width=1in,height=1.25in,clip,keepaspectratio]{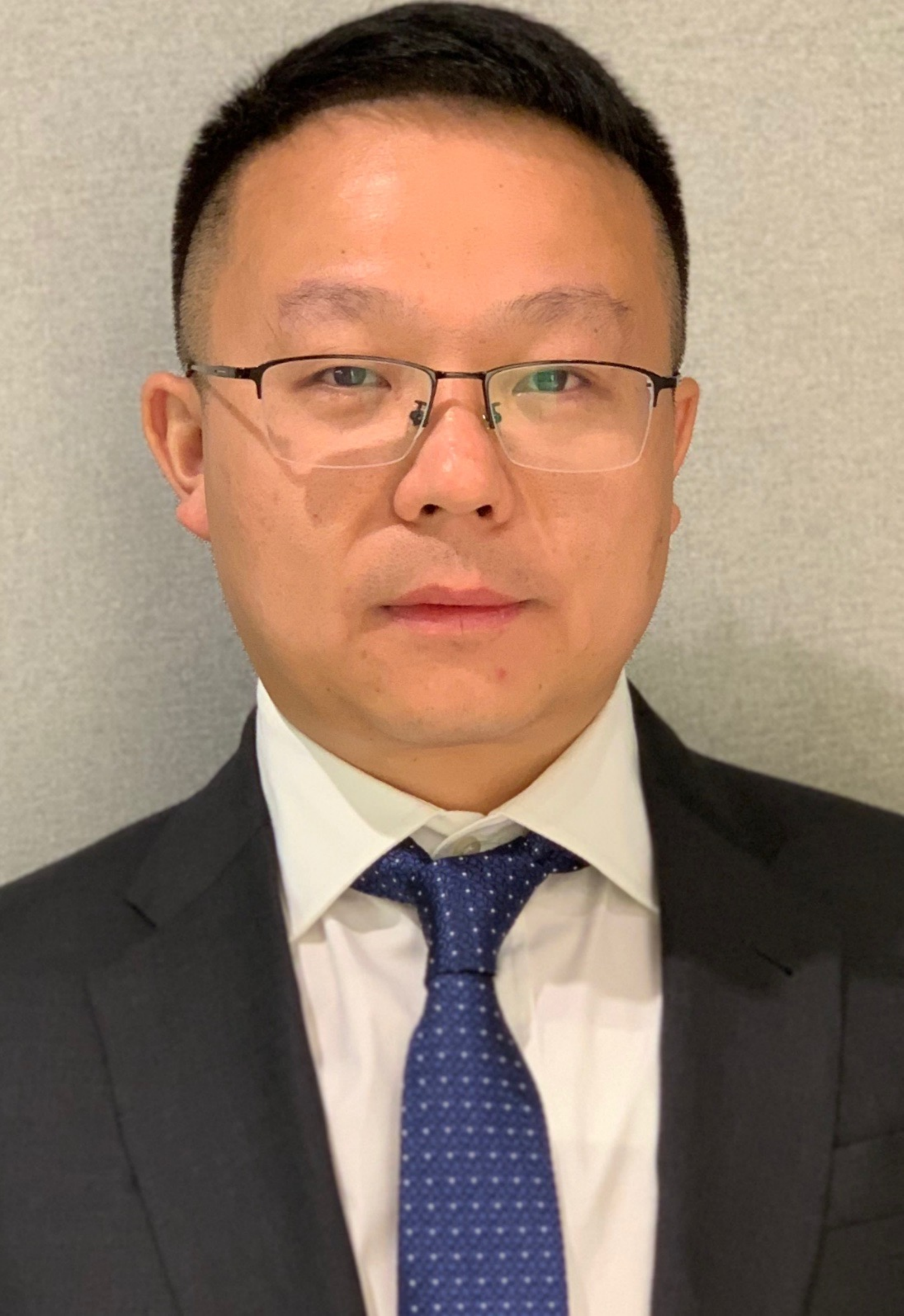}}]
    {Xiaochun Cao} is a Professor of School of Cyber Science and Technology, Shenzhen Campus of Sun Yat-sen University. He received the B.E. and M.E. degrees both in computer science from Beihang University (BUAA), China, and the Ph.D. degree in computer science from the University of Central Florida, USA, with his dissertation nominated for the university level Outstanding Dissertation Award. After graduation, he spent about three years at ObjectVideo Inc. as a Research Scientist. From 2008 to 2012, he was a professor at Tianjin University. Before joining SYSU, he was a professor at Institute of Information Engineering, Chinese Academy of Sciences. He has authored and coauthored over 200 journal and conference papers. In 2004 and 2010, he was the recipients of the Piero Zamperoni best student paper award at the International Conference on Pattern Recognition. He is on the editorial boards of IEEE Trans. on Image Processing and IEEE Trans. on Multimedia, and was on the editorial board of IEEE Trans. on Circuits and Systems for Video Technology.
\end{IEEEbiography}

\begin{IEEEbiography}[{\includegraphics[width=1in,height=1.25in,clip,keepaspectratio]{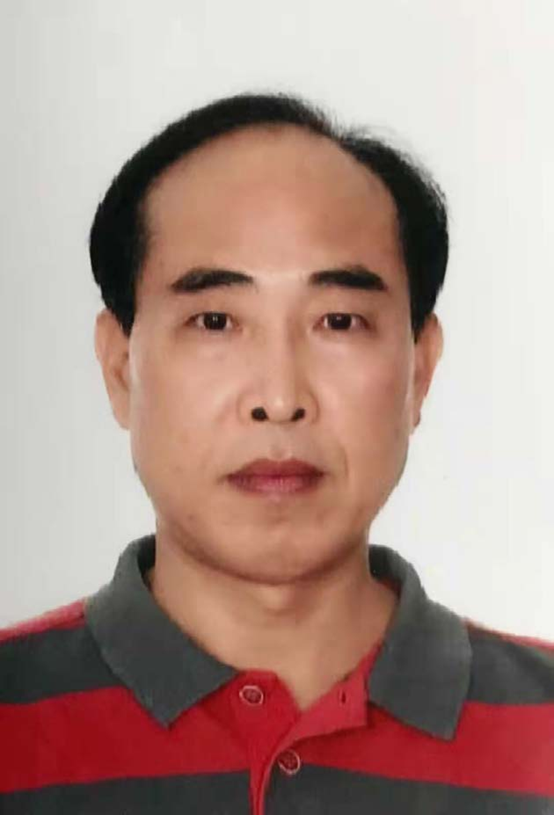}}]
    {Yingfei Sun} received the Ph.D. degree in applied mathematics from the Beijing Institute of Technology, in 1999. He is currently a Full Professor with the School of Electronic, Electrical and Communication Engineering, University of Chinese Academy of Sciences. His current research interests include machine learning and pattern recognition.
\end{IEEEbiography}

\begin{IEEEbiography}[{\includegraphics[width=1in,height=1.25in,clip,keepaspectratio]{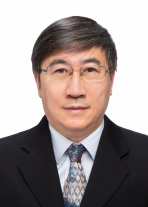}}]
    {Qingming Huang} is a chair professor in the University of Chinese Academy of Sciences and an adjunct research professor in the Institute of Computing Technology, Chinese Academy of Sciences. He graduated with a Bachelor degree in Computer Science in 1988 and Ph.D. degree in Computer Engineering in 1994, both from Harbin Institute of Technology, China. His research areas include multimedia computing, image processing, computer vision and pattern recognition. He has authored or coauthored more than 400 academic papers in prestigious international journals and top-level international conferences. He was the associate editor of IEEE Trans. on CSVT and Acta Automatica Sinica, and the reviewer of various international journals including IEEE Trans. on PAMI, IEEE Trans. on Image Processing, IEEE Trans. on Multimedia, etc. He is a Fellow of IEEE and has served as general chair, program chair, track chair and TPC member for various conferences, including ACM Multimedia, CVPR, ICCV, ICME, ICMR, PCM, BigMM, PSIVT, etc.
\end{IEEEbiography}

\newpage
\onecolumn
% Supplementary Materials

% \setcounter{section}{0}
% \appendix
\appendixpage

\startcontents[sections]
\printcontents[sections]{l}{1}{\setcounter{tocdepth}{2}}

\vspace{2em}

The first part is about two ranking algorithms tailored to the \textbf{BTL} model, say \textbf{Hodgerank} \cite{Jiang2011} and the spectral ranking algorithm \cite{DBLP:journals/ior/NegahbanOS17}. 

The second part is the proof details of Theorem \ref{thm:DRNE}, which states the existence of a distributionally robust Nash equilibrium. This result tells us that there exists at least one stable state for both ranker and attacker. To prove the existence of the distributionally robust Nash equilibrium, we need a proposition from \cite{liu2018distributionally}, which shows that the distributionally robust Nash equilibrium is a global minimizer of the reformulation of \eqref{eq:DRG}. This reformulation is also well known for the deterministic Nash equilibrium problem \cite{10.2307/1911749}. Based on this proposition, we show the existence results of \textbf{DRNE} for the adversarial game. This result is an extension of the famous Kakutani’s fixed point theorem \cite{kakutani1941generalization}. 

The third part is the proof details of Theorem \ref{thm:vulnerability}. We need some important lemmas. Lemma \ref{lemma:martingale_concentration} is a martingale concentration inequality which can deal with the case that the maximum value $M$ of $|x_{t+1}-x_{t}|$ is large, but the maximum is rarely attained (making the variance much smaller than $M^2$) \cite{10.1214/aop/1176996452,McDiarmid1998,im/1175266369}. The following two lemmas assert that for any given subset $\boldsymbol{S}_{\mathcal{A}}$ of the universe $\boldsymbol{\mathcal{C}}$, the fraction of elements from $\boldsymbol{S}_{\mathcal{A}}$ within the sample typically does not differ by much from the corresponding fraction among the whole stream. Lemma \ref{lemma:main_lemma_Bernoulli} corresponds to the Bernoulli sampling. Lemma \ref{lemma:main_lemma_Reservoir} corresponds to the reservoir sampling. Theorem \ref{thm:vulnerability} shows the vulnerability of the Bernoulli and reservoir sampling methods. When the sampling parameters of these two methods fulfill the conditions, the data for the ranker will be $(\epsilon,\delta)$-representative with respect to the sequence fabricated by the adversary. The adversary are able to obtain the desired ranking list with these pairwise comparisons. 

The forth part proves the asymptotic optimality of the proposed adversarial policy \eqref{eq:stopping_time} and \eqref{opt:generation_rule} with complete knowledge. First we show some important assumptions. Considering Assumption \ref{assmp:support}-\ref{assmp:positive_density}, Theorem \ref{thm:asymptotic2expectation} establishes a lower bound on the minimal Bayesian risk with the help of Lemma \ref{lemma:martingale}-\ref{lemma:7}. Theorem \ref{thm:asymptotic_kendall_tau} provides the asymptotic upper bounds for the expected Kendall tau of the proposed manipulation policy with complete information with the help of Lemma \ref{lemma:8} and \ref{lemma:9}. Theorem \ref{thm:6} shows the asymptotic optimality of the expected stopping time of the proposed manipulation policy with the help of Lemma \ref{lemma:10} and \ref{lemma:11}. These three theorems are sufficient to show the asymptotic optimality of the proposed stopping time \eqref{eq:stopping_time} and the generation rule solved by \eqref{opt:generation_rule} with complete knowledge as the identifiability of \textbf{BTL} model by adopting the \textbf{MLE} to obtain the preference score. 

The fifth part gives us the solution of \eqref{eq:DRE}, which gives birth to the stopping time \eqref{eq:robsut_stop} and generation rule \eqref{opt:robust_generation_rule} for adversary with incomplete knowledge. Proposition \ref{prop:duality} shows that the strong duality ensures that the inner supremum of \eqref{eq:DRE} admits a reformulation which is a simple, univariate optimization problem. Theorem \ref{thm:reformulation} gives the equivalent form of \eqref{eq:DRE}, which can be solved efficiently using the mirror descent algorithm. The detailed solving process is Algorithm \ref{alg:robust_estimation}, \ref{alg:simplex_proj} and \ref{alg:mirror_descent}.

\newpage
\section*{HodgeRank and RankCentrality}
\addcontentsline{toc}{section}{HodgeRank and RankCentrality}

\subsection*{HodgeRank}

\noindent The \textbf{Hodgerank} method discussed in \cite{Jiang2011} consists in finding the relative ranking score by solving the following least-squares problem:
\begin{equation}
    \label{opt:Hodgerank}
    \begin{aligned}
        % & & &\ \ \underset{\boldsymbol{\theta}\in\mathbb{R}^n}{\textbf{\textit{minimize}}}\ \ \frac{\ 1\ }{\ 2\ }\underset{(i,j)\in\boldsymbol{E},\ \boldsymbol{u}\in\boldsymbol{U}}{\sum}\ \big(y^{\boldsymbol{u}}_{ij}-\theta_j+\theta_i\big)^2\\[5pt]
        % & &=&\ \ \underset{\boldsymbol{\theta}\in\mathbb{R}^n}{\textbf{\textit{minimize}}}\ \ \frac{\ 1\ }{\ 2\ }\underset{(i,j)\in\boldsymbol{E}}{\sum}\ w^*_{ij}\big(y_{ij}-\theta_j+\theta_i\big)^2,\\
        % & &=&\ \ \underset{\boldsymbol{\theta}\in\mathbb{R}^n}{\textbf{\textit{minimize}}}\ \frac{\ 1\ }{\ 2\ }\ \big\|\boldsymbol{y} - \boldsymbol{A}\boldsymbol{\theta}\big\|^2_{2,\ \boldsymbol{w}^*}
        \underset{\boldsymbol{\theta}\in\mathbb{R}^n}{\textbf{\textit{minimize}}}\ \ \frac{\ 1\ }{\ 2\ }\underset{(i,j)\in\boldsymbol{E}}{\sum}\ w^*_{ij}\big(y_{ij}-\theta_j+\theta_i\big)^2
    \end{aligned}
\end{equation}
where $\boldsymbol{y}=[y_{12},y_{13},\dots,y_{n,n-1}]^\top$ represents the directions of edges. As $\boldsymbol{\mathcal{G}}$ is a complete graph, we set $y_{ij} = 1$ which indicates a direct edge from node $i$ to $j$. Based on combinational Hodge theory \cite{Jiang2011}, the minimal norm solution of \eqref{opt:Hodgerank} is simply given as
\begin{equation}
    \label{eq:HodgeRank}
    \boldsymbol{\bar{\theta}} = -\mathcal{L}^{\dagger}_0\cdot\textbf{div}(\boldsymbol{y}),
\end{equation}
where $\mathcal{L}^{\dagger}_0$ is the Moore-Penrose pseudo-inverse of $\mathcal{L}_0$, and the divergence operator $\textbf{div}$ is defined as 
\begin{equation}
    [\textbf{div}(\boldsymbol{y})](i) = \underset{j:(i,j)\in\boldsymbol{E}}{\sum}\ w^*_{ij}y_{ij},\ \forall\ i\in[n].
\end{equation}

\subsection*{Rank Centrality}

\noindent The spectral ranking algorithm, or \textbf{RankCentrality} \cite{DBLP:journals/ior/NegahbanOS17}, is motivated by the connection between the pairwise comparisons $\boldsymbol{w}^*$ and a random walk over a directed graph $\boldsymbol{\mathcal{G}}$. The spectral method constructs a random walk on $\boldsymbol{\mathcal{G}}$ where at each time, the random walk is likely to go from vertex $i$ to vertex $j$ if items $i$ and $j$ were ever compared; and if so, the likelihood of going from $i$ to $j$ depends on how often $i$ lost to $j$. That is, the random walk is more likely to move to a neighbor who is more probable to “wins”. How frequently this walk visits a particular node in the long run, or equivalently the stationary distribution, is the score of the corresponding item. Thus, effectively this algorithm captures the preference of the given item versus all the others, not just immediate neighbors: the global effect induced by the transitivity of comparisons is captured through the stationary distribution. 

A random walk can be represented by a time-independent transition matrix $\boldsymbol{P} = \{P_{i,j}\}_{1\leq i,j\leq n}\in\mathbb{R}^{n\times n}_+$, where $P_{i,j} = \mathbb{P}(X_{t+1}=j|X_{t}=i)$ and $X_t$ represents the state of the process (arriving node) at time $t$. By definition, the entries of a transition matrix are nonnegative and satisfy
\begin{equation}
    P_{i,j}+P_{j,i}=1,\ \ \forall\ i,\ j\in[n],\ i\neq j.
\end{equation}
One way to define a valid transition matrix $\boldsymbol{P}$ of a random walk on $\boldsymbol{\mathcal{G}}$ is to scale all the edge weights by the maximum out-degree of a node, noted as $d_{\text{max}}$. This re-scaling ensures that each row-sum is at most one. Finally, to ensure that each row-sum is exactly one, the spectral method adds a self-loop to each node of $\boldsymbol{V}$. Concretely, the transition matrix $\boldsymbol{P}^*$ is converted from the pairwise comparison data $\boldsymbol{w}^*$ in such a way that
\begin{equation}
    \label{eq:empirical_matrix}
    P^*_{i,j} =
    \begin{dcases}
        \ \ \ \ \ \ \ \ \ \ \ \ \ \ \ \ \ \ \ \ \ \ \ \ \ \ \ \ \ \ \ \ \ \ \ 0, &\text{if}\ i\neq j,\ w^*_{ij}+w^*_{ji}=0,\\[3.5pt]
        \ \ \ \ \ \ \ \ \ \ \ \ \frac{1}{d_{\text{max}}}\ \frac{w^*_{ij}}{w^*_{ji}+w^*_{ij}}, &\text{if}\ i\neq j,\ w^*_{ij}+w^*_{ji}\neq 0,\\[2.5pt]
        1 - \frac{1}{d_{\text{max}}}\underset{k\neq i}{\sum}\ \frac{w^*_{ik}}{w^*_{ik}+w^*_{ki}}, &\text{otherwise}.
    \end{dcases}
\end{equation}
Rank centrality estimates the probability distribution obtained by applying matrix $\boldsymbol{P}^*$ repeatedly starting from any initial condition. Precisely, let $\theta_t(i)=\mathbb{P}(X_t=i)$ denote the distribution of the random walk at time $t$ with $\boldsymbol{\theta}_0=\{\theta_0(i)\}\in\mathbb{R}^n_+$ as an arbitrary starting distribution on $[n]$. Then the random walk holds
\begin{equation}
    \label{eq:transition_matrix_condition}
    \boldsymbol{\theta}^{\top}_{t+1} = \boldsymbol{\theta}^{\top}_t\boldsymbol{P}^*.
\end{equation}wwwwww
One expects the stationary distribution of the sample version $\boldsymbol{P}^*$ to form a good estimate of true relative ranking score\footnote{\tiny The original paper assumes that the true relative scores are generated from the logistic pairwise comparison model, \textit{e.g.} Bradley-Terry-Luce (BTL) model, multi-nominal logit (MNL) and Plackett-Luce (PL) model.}, provided the sample size is sufficiently large. When the transition matrix has a unique left eigenvector $\boldsymbol{\theta}^*$ related to the largest eigenvalue, then starting from any initial distribution $\boldsymbol{\theta}_0$, the limiting distribution $\boldsymbol{\theta}_{t+1}$ is unique. This stationary distribution $\underset{t\rightarrow\infty}{\lim}\boldsymbol{\theta}_{t}$ is the top left eigenvector of $\boldsymbol{P}^*$ as 
\begin{equation}
    \label{eq:detailed_balance}
    \underset{t\rightarrow\infty}{\lim}\boldsymbol{\theta}_{t} = \boldsymbol{\bar{\theta}}\ \ \text{and}\ \ \boldsymbol{\bar{\theta}}^\top=\boldsymbol{\bar{\theta}}^\top\boldsymbol{P}^*,
\end{equation}
which only involves a simple eigenvector computation.

\newpage

\section*{Proof of Theorem \ref{thm:DRNE}}
\label{sec:proof_thm_1}
\addcontentsline{toc}{section}{Proof of Theorem \ref{thm:DRNE}\nameref{sec:proof_thm_1}}
We clarify the definition of weakly compactness of $\boldsymbol{\mathcal{P}}$.
\begin{definition}
    A set of probability distribution (measures) $\boldsymbol{\mathcal{P}}$ is said to be weakly compact if every sequence $\{\mathbb{P}_N\}\subset\boldsymbol{\mathcal{P}}$ contains a sub-sequence $\{\mathbb{P}_{N_0}\}$ and a point $\mathbb{P}_0$ such that $\{\mathbb{P}_{N_0}\}$ converges to $\mathbb{P}_0$ weakly.
\end{definition}

\ul{To prove the existence of the distributionally robust Nash equilibrium, we need the following proposition} \cite{liu2018distributionally}, \ul{which shows that the distributionally robust Nash equilibrium is a global minimizer of the reformulation of} \eqref{eq:DRG}. This reformulation is also well known for the deterministic Nash equilibrium problem \cite{10.2307/1911749}. 
\begin{proposition}
    \label{prop:1}
    Let
    \begin{equation}
        \label{eq:drg_ref}
        \phi(\boldsymbol{A}',\boldsymbol{A}) = \sum_{r=1}^R \underset{\mathbb{P}\in\boldsymbol{\mathcal{P}}_r}{\ \textbf{\textit{max}}\phantom{g}}\ \mathbb{E}_{\boldsymbol{w}\sim\mathbb{P}}\Big[f_r(\boldsymbol{a}'_{r},\boldsymbol{a}_{-r},\boldsymbol{w})\Big],
    \end{equation}
    where $\boldsymbol{A}' = [\boldsymbol{a}'_1,\dots,\boldsymbol{a}'_R]$. Under the conditions of Theorem \ref{thm:DRNE}, $\boldsymbol{A}^*$ is a distributionally robust Nash equilibrium as \eqref{eq:DRNE} if and only if 
    \begin{equation}
        \label{eq:phi_global}
        \boldsymbol{A}^* \in \underset{\boldsymbol{A}'}{\textbf{\textit{arg min}}}\ \phi(\boldsymbol{A}',\boldsymbol{A}^*).
    \end{equation}
\end{proposition}

\begin{proof}
    \begin{enumerate}
        \item The `\textbf{if}' part. If $\boldsymbol{A}^*$ is not an equilibrium state, there exists at least one $\boldsymbol{a}'_{r_0}$ which satisfies
        \begin{equation}
            \label{eq:not_eqm}
            \begin{aligned}
                & & &\ \ \underset{\mathbb{P}\in\boldsymbol{\mathcal{P}}_{r_0}}{\ \textbf{\textit{max}}\phantom{g}\ }\ \mathbb{E}_{\boldsymbol{w}\sim\mathbb{P}}\Big[f_{r_0}(\boldsymbol{a}'_{r_0},\boldsymbol{a}^*_{-{r_0}},\boldsymbol{w})\Big]\\[5pt]
                & &<&\ \  \underset{\mathbb{P}\in\boldsymbol{\mathcal{P}}_{r_0}}{\ \textbf{\textit{max}}\phantom{g}\ }\ \mathbb{E}_{\boldsymbol{w}\sim\mathbb{P}}\Big[f_{r_0}(\boldsymbol{a}^*_{r_0},\boldsymbol{a}^*_{-{r_0}},\boldsymbol{w})\Big].
            \end{aligned}
        \end{equation}
        Let $\boldsymbol{A}''$ be
        \begin{equation}
            \boldsymbol{A}'' = [\boldsymbol{a}^*_1,\dots,\boldsymbol{a}^*_{r_0-1}, \boldsymbol{a}'_{r_0}, \boldsymbol{a}^*_{r_0+1},\dots,\boldsymbol{a}^*_{R}]
        \end{equation}
        and we could conduct the following contradiction by \eqref{eq:not_eqm}:
        \begin{equation}
            \phi(\boldsymbol{A}'', \boldsymbol{A}^*)<\phi(\boldsymbol{A}^*, \boldsymbol{A}^*).
        \end{equation}
        \item The `\textbf{only if}' part. By the definition of distributionally robust Nash equilibrium, for any $\boldsymbol{A}'$, it holds that
        \begin{equation}
            \begin{aligned}
                & & &\ \ \underset{\mathbb{P}\in\boldsymbol{\mathcal{P}}_{r}}{\ \textbf{\textit{max}}\phantom{g}\ }\ \mathbb{E}_{\boldsymbol{w}\sim\mathbb{P}}\Big[f_{r}(\boldsymbol{a}'_{r},\boldsymbol{a}^*_{-{r}},\boldsymbol{w})\Big]\\[5pt]
                & &\geq&\ \  \underset{\mathbb{P}\in\boldsymbol{\mathcal{P}}_{r}}{\ \textbf{\textit{max}}\phantom{g}\ }\ \mathbb{E}_{\boldsymbol{w}\sim\mathbb{P}}\Big[f_{r}(\boldsymbol{a}^*_{r},\boldsymbol{a}^*_{-{r}},\boldsymbol{w})\Big],\ \ \forall\ r\in[R].
            \end{aligned}
        \end{equation}
        Summing the two sides of the above inequality, we have 
        \begin{equation}
            \begin{aligned}
                & \phi(\boldsymbol{A}',\boldsymbol{A}^*) &=&\ \ \sum_{r=1}^R\underset{\mathbb{P}\in\boldsymbol{\mathcal{P}}_{r}}{\ \textbf{\textit{max}}\phantom{g}\ }\ \mathbb{E}_{\boldsymbol{w}\sim\mathbb{P}}\Big[f_{r}(\boldsymbol{a}'_{r},\boldsymbol{a}^*_{-{r}},\boldsymbol{w})\Big]\\[5pt]
                & &\geq&\ \ \sum_{r=1}^R\underset{\mathbb{P}\in\boldsymbol{\mathcal{P}}_{r}}{\ \textbf{\textit{max}}\phantom{g}\ }\ \mathbb{E}_{\boldsymbol{w}\sim\mathbb{P}}\Big[f_{r}(\boldsymbol{a}^*_{r},\boldsymbol{a}^*_{-{r}},\boldsymbol{w})\Big]\\[7.5pt]
                & &=&\ \ \phi(\boldsymbol{A}^*,\boldsymbol{A}^*)
            \end{aligned}
        \end{equation}
        which implies $\boldsymbol{A}^*$ is a global optimal solution as \eqref{eq:phi_global}.
    \end{enumerate}
\end{proof}

\ul{Based on the Proposition} \ref{prop:1}, \ul{we show the existence results of \textbf{DRNE} for the adversarial game. This results is an extension of the famous Kakutani’s fixed point theorem} \cite{kakutani1941generalization}. 

\drne*

\begin{proof}
    We know that the `\textit{supremum}' operator will preserve the convexity. Moreover, with the weakly compactness of $\boldsymbol{\mathcal{P}}_r$, the `\textit{supremum}' operator also preserve the continuity. As a consequence, for any given $\boldsymbol{a}_{-r}$, $\mathbb{E}_{\boldsymbol{w}\sim\mathbb{P}}[f_r(\cdot,\boldsymbol{a}_{-r},\boldsymbol{w})]$ is continuous and convex for every $\mathbb{P}\in\boldsymbol{\mathcal{P}}_r,r = 1,\dots, R$. By the definition of $\phi$ \eqref{eq:drg_ref}, for any given $\boldsymbol{A}$, $\phi(\boldsymbol{A}', \boldsymbol{A})$ is continuous and convex \textit{w.r.t.} any $\boldsymbol{A}'$. Besides, the existence of an optimal solution to
    \begin{equation}
        \underset{\boldsymbol{A}'}{\textbf{\textit{min}}}\ \phi(\boldsymbol{A}',\boldsymbol{A})
    \end{equation}
    with given $\boldsymbol{a}$ is guaranteed by assuming that $\mathbb{E}_{\boldsymbol{w}\sim\mathbb{P}}[f_r(\cdot,\boldsymbol{a}_{-r},\boldsymbol{w})]$ only has finite values.

    The remaining part is to show the existence of $\boldsymbol{A}^*$ which satisfies
    \begin{equation}
        \boldsymbol{A}^* \in \underset{\boldsymbol{A}'}{\textbf{\textit{arg min}}}\ \phi(\boldsymbol{A}',\boldsymbol{A}^*).
    \end{equation}
    Let $\Gamma(\boldsymbol{A})$ be the solution set of $\textbf{\textit{min}}\ \phi(\boldsymbol{A}',\boldsymbol{A})$ with given $\boldsymbol{A}$. By the convexity of $\phi(\cdot,\boldsymbol{A})$, $\Gamma(\boldsymbol{A})$ is a convex set. $\Gamma(\boldsymbol{A})$ is also a closed set: for any $\{\boldsymbol{A}_t\}\rightarrow\boldsymbol{\bar{A}}$ as $t\rightarrow\infty$ and $\boldsymbol{A}'_t\in\Gamma(\boldsymbol{A}_t)$ with $\{\boldsymbol{A}'_t\}\rightarrow\boldsymbol{\bar{A}}'$, it holds that
    \begin{equation}
        \boldsymbol{\bar{A}}'\in\Gamma(\boldsymbol{\bar{A}}).
    \end{equation}
    By \cite{bank1982non}, $\Gamma(\boldsymbol{A})$ is upper semi-continuous on $\mathbb{R}^{N\times R}$. With the well-known Kakutani’s fixed point theorem \cite{kakutani1941generalization}, there exists $\boldsymbol{A}^*$ such that
    \begin{equation}
        \boldsymbol{A}^* \in \underset{\boldsymbol{A}'}{\textbf{\textit{arg min}}}\ \phi(\boldsymbol{A}',\boldsymbol{A}^*).
    \end{equation}
    With Proposition \ref{prop:1}, we know $\boldsymbol{A}^*$ is a distributionally robust Nash equilibrium.
\end{proof}

\newpage

\section*{Proof of Theorem \ref{thm:vulnerability}}
\label{sec:proof_thm_2}
\addcontentsline{toc}{section}{Proof of Theorem \ref{thm:vulnerability}\nameref{sec:proof_thm_2}}

In this section, we prove the main technical result for Bernoulli sampling. \ul{First we need the follow well-known results of martingale inequalities}. 

\begin{definition}
    A martingale is a sequence $\boldsymbol{X}=(x_1,\dots,x_T)$ of random variables with finite means, such that
    \begin{equation}
        \mathbb{E}\big[x_{t+1}|x_1,\dots,x_{t}\big] = x_{t},\ \ \forall\ t\in[T].
    \end{equation}
\end{definition}

The most basic and well-known concentration result of martingale is the Azuma’s (or Hoeffding’s) inequality, which asserts that martingales with bounded differences $|x_{t+1}-x_{t}|$ are well-concentrated around their mean. \ul{However, we need the other concentration inequality which can deal with the case that the maximum value $M$ of $|x_{t+1}-x_{t}|$ is large, but the maximum is rarely attained (making the variance much smaller than $M^2$)} \cite{10.1214/aop/1176996452,McDiarmid1998,im/1175266369}. \ul{The martingales that we investigate in this paper depict this behavior.} 

\begin{lemma}
    \label{lemma:martingale_concentration}
    Let $\boldsymbol{X}=(x_1,\dots,x_T)$ be a martingale and the variance of $x_{t+1}$ given $x_1,\dots,x_t$ be bounded
    \begin{equation}
        \textbf{\textit{Var}}(x_{t+1}|x_1,\dots,x_t)\leq\sigma^2_{t+1},\ \ \forall\ t\in[T],
    \end{equation}
    where $\sigma_{t}\geq 0$. If there exist a constant $M$ such that 
    \begin{equation}
        |x_{t+1}-x_{t}|\leq M,
    \end{equation}
    we have
    \begin{equation}
        \mathbb{P}\big(\boldsymbol{X}-\mathbb{E}[\boldsymbol{X}]\geq\lambda\big)\leq\textbf{\textit{exp}}\left(-\frac{\lambda^2}{\sum_{t=1}^{T}\sigma^2_t+\frac{\lambda M}{3}}\right),
    \end{equation}
    where $\mathbb{E}[\boldsymbol{X}]$ is defined as
    \begin{equation}
        \mathbb{E}[\boldsymbol{X}] = \sum_{t=1}^{T}\mathbb{E}[x_t].
    \end{equation}
    Particularly,
    \begin{equation}
        \mathbb{P}\Big(\big|\boldsymbol{X}-\mathbb{E}[\boldsymbol{X}]\big|\geq\lambda\Big)\leq2\cdot\textbf{\textit{exp}}\left(-\frac{\lambda^2}{\sum_{t=1}^{T}\sigma^2_t+\frac{\lambda M}{3}}\right).
    \end{equation}
\end{lemma}

\ul{The following lemmas assert that for any given subset $\boldsymbol{S}_{\mathcal{A}}$ of the universe $\boldsymbol{\mathcal{C}}$, the fraction of elements from $\boldsymbol{S}_{\mathcal{A}}$ within the sample typically does not differ by much from the corresponding fraction among the whole stream.} \ul{Lemma} \ref{lemma:main_lemma_Bernoulli} \ul{corresponds to the Bernoulli sampling}. \ul{Lemma} \ref{lemma:main_lemma_Reservoir} \ul{corresponds to the reservoir sampling.} 

\begin{lemma}
    \label{lemma:main_lemma_Bernoulli}
    % Given $\epsilon,\delta\in(0,1)$ and the adversarial set system $(\boldsymbol{\mathcal{C}},\mathcal{S}_{\mathcal{A}})$, let $\boldsymbol{C}=(c_1,c_2,\dots,c_T)$ be the whole data sequence with length $T$ and $\boldsymbol{C}'$ is a sample set chosed by the Bernoulli method which could be modified by the adversary. 
    For any dynamic stream $\boldsymbol{C}=\{c_t\}_{t=1}^{\infty}$ from $\boldsymbol{\mathcal{C}}'$, if the parameter of Bernoulli method $\varrho$ holds that
    \begin{equation}
        \label{eq:sampling_para}
        \varrho \geq 10\cdot\frac{\textbf{\textit{ln}}(4/\delta)}{\epsilon^2T}, 
    \end{equation}
    we have 
    \begin{equation}
        \label{eq:lemma_result}
        \mathbb{P}(|d_{\boldsymbol{\mathcal{C}'}}(\boldsymbol{C})-d_{\boldsymbol{\mathcal{C}}'}(\boldsymbol{C}')|\geq\epsilon)\leq \delta,
    \end{equation}
    where $\boldsymbol{C}'$ is a sequence which is sampled from $\boldsymbol{C}$ by the Bernoulli method.
\end{lemma}

\begin{proof}
    At any given time step $t\in[T]$ along the sampling process, let $\boldsymbol{C}_t=(c_1,\dots,c_t)$ be the sequence of pairwise comparisons submitted to the Bernoulli method until time $t$, and $\boldsymbol{C}_t'\subseteq\boldsymbol{C}_t$ be a sub-sequence of the sampled pairwise comparisons from $\boldsymbol{C}_t$. Note that $\boldsymbol{C}_T=\boldsymbol{C}$ and $\boldsymbol{C}'_T=\boldsymbol{C}'$, and hence, to prove the lemma, we need to show that 
    \begin{equation*}
        |d_{\boldsymbol{\mathcal{C}'}}(\boldsymbol{C}_T)-d_{\boldsymbol{\mathcal{C}'}}(\boldsymbol{C}_T')|\leq\epsilon.
    \end{equation*}
    Given a $\boldsymbol{\mathcal{C}'}$, we define the random variables
    \begin{equation}
        \begin{aligned}
            & A_t({\boldsymbol{\mathcal{C}'}}) &=&\ \ \frac{\ t}{T}\cdot d_{\boldsymbol{\mathcal{C}'}}(\boldsymbol{C}_t) = \frac{|\boldsymbol{\mathcal{C}'}\cap\boldsymbol{C}_t|}{T},\\[7.5pt]
            & B_t({\boldsymbol{\mathcal{C}'}}) &=&\ \ \frac{|\boldsymbol{\mathcal{C}'}\cap\boldsymbol{C}'_t|}{\varrho T},\\[7.5pt]
            & Z_t({\boldsymbol{\mathcal{C}'}}) &=&\ \ B_t({\boldsymbol{\mathcal{C}'}}) - A_t({\boldsymbol{\mathcal{C}'}}),
        \end{aligned}
    \end{equation}
    where the intersection between a set $\boldsymbol{\mathcal{C}'}$ and a sequence $\boldsymbol{C}_t$ is the sub-sequence of $\boldsymbol{C}_t$ consisting of all pairwise comparisons (repetitions are allowed) that also belong to $\boldsymbol{\mathcal{C}'}$. Next we show that $(Z_0({\boldsymbol{\mathcal{C}'}}), \dots, Z_T({\boldsymbol{\mathcal{C}'}}))$ is a martingale. Suppose that the $\boldsymbol{C}'_{t-1}$ is fixed and thence the values of $Z_0({\boldsymbol{\mathcal{C}'}}), \dots, Z_{t-1}({\boldsymbol{\mathcal{C}'}})$ are fixed. Now a new pairwise comparison $c_t$ is ready to submit. Note that $c_t$ may be either the original data or the comparison perturbed by the adversary. 

    If $c_t\notin\boldsymbol{\mathcal{C}'}$, we have
    \begin{equation}
        \label{eq:notin_v}
        \begin{aligned}
            & A_t({\boldsymbol{\mathcal{C}'}}) &=&\ \ A_{t-1}({\boldsymbol{\mathcal{C}'}}),\\[5pt]
            & B_t({\boldsymbol{\mathcal{C}'}}) &=&\ \ B_{t-1}({\boldsymbol{\mathcal{C}'}}),\\[5pt]
            & Z_t({\boldsymbol{\mathcal{C}'}}) &=&\ \ Z_{t-1}({\boldsymbol{\mathcal{C}'}}),
        \end{aligned}
    \end{equation}
    and it holds that
    \begin{equation}
        \label{eq:notin_mean}
        \mathbb{E}\big[\ Z_t({\boldsymbol{\mathcal{C}'}})\ \big|\ Z_0({\boldsymbol{\mathcal{C}'}}), \dots, Z_{t-1}({\boldsymbol{\mathcal{C}'}}),\ c_t\notin\boldsymbol{\mathcal{C}'}\ \big] = Z_{t-1}({\boldsymbol{\mathcal{C}'}})
    \end{equation}

    When $c_t\in\boldsymbol{\mathcal{C}'}$, we have
    \begin{equation}
        \begin{aligned}
            & A_t({\boldsymbol{\mathcal{C}'}}) &=&\ \ A_{t-1}({\boldsymbol{\mathcal{C}'}}) + \frac{1}{T},\\[7.5pt]
            & B_t({\boldsymbol{\mathcal{C}'}}) &=&\ \ \begin{cases}
                                                            \displaystyle B_{t-1}({\boldsymbol{\mathcal{C}'}}), &\ \text{if}\ c_t\ \text{is not sampled},\\[5pt] 
                                                            \displaystyle B_{t-1}({\boldsymbol{\mathcal{C}'}}) + \frac{1}{\varrho T},& \text{otherwise.}
                                                         \end{cases}\\[5pt]
            & Z_t({\boldsymbol{\mathcal{C}'}}) &=&\ \  \begin{cases}
                                                            \displaystyle Z_{t-1}({\boldsymbol{\mathcal{C}'}})-\frac{1}{T}, &\ \text{if}\ c_t\ \text{is not sampled},\\[7.5pt] 
                                                            \displaystyle Z_{t-1}({\boldsymbol{\mathcal{C}'}}) + \frac{1}{\varrho T} - \frac{1}{T},&\ \text{otherwise.}
                                                         \end{cases}
        \end{aligned}
    \end{equation}
    Recall that each pairwise comparison is sampled by the Bernoulli method, independently, with probability $\varrho$. Therefore, we have that
    \begin{equation}
        \begin{aligned}
            & & &\ \ \mathbb{E}\big[\ Z_t({\boldsymbol{\mathcal{C}'}})\ \big|\ Z_0({\boldsymbol{\mathcal{C}'}}), \dots, Z_{t-1}({\boldsymbol{\mathcal{C}'}}),\ c_t\in\boldsymbol{\mathcal{C}'}\ \big]\\[5pt]
            & &=&\ \ Z_{t-1}({\boldsymbol{\mathcal{C}'}}) + \varrho\cdot\left(\frac{1}{\varrho T}-\frac{1}{T}\right) + (1-\varrho) \cdot \left(-\frac{1}{T}\right)\\[5pt]
            & &=&\ \ Z_{t-1}({\boldsymbol{\mathcal{C}'}}).
        \end{aligned}
    \end{equation}
    Combine the two cases, we know that $(Z_0({\boldsymbol{\mathcal{C}'}}), \dots, Z_T({\boldsymbol{\mathcal{C}'}}))$ is a martingale. 

    Next we will show that the variance of $Z_t({\boldsymbol{\mathcal{C}'}})$ conditioned on $Z_0({\boldsymbol{\mathcal{C}'}}), \dots, Z_{t-1}({\boldsymbol{\mathcal{C}'}})$ is bounded by $1/\varrho T^2$. If $c_t\notin\boldsymbol{\mathcal{C}'}$ and with simple calculation from \eqref{eq:notin_v} and \eqref{eq:notin_mean}, the variance of $Z_t({\boldsymbol{\mathcal{C}'}})$ given $Z_0({\boldsymbol{\mathcal{C}'}}), \dots, Z_{t-1}({\boldsymbol{\mathcal{C}'}})$ equals to zeros as 
    \begin{equation}
        \textbf{\textit{Var}}(\ Z_t({\boldsymbol{\mathcal{C}'}})\ \big|\ Z_0({\boldsymbol{\mathcal{C}'}}), \dots, Z_{t-1}({\boldsymbol{\mathcal{C}'}}),\ c_t\notin\boldsymbol{\mathcal{C}'}\ ) = 0.
    \end{equation}
    When $c_t\in\boldsymbol{\mathcal{C}'}$, we have
    \begin{equation}
        \begin{aligned}
            & & &\ \ \textbf{\textit{Var}}(\ Z_t({\boldsymbol{\mathcal{C}'}})\ \big|\ Z_0({\boldsymbol{\mathcal{C}'}}), \dots, Z_{t-1}({\boldsymbol{\mathcal{C}'}}),\ c_t\in\boldsymbol{\mathcal{C}'}\ )\\[5pt]
            & &=&\ \ (1-\varrho)\cdot\left(\frac{1}{T}\right)^2 + \varrho\cdot\left(\frac{1}{\varrho T}-\frac{1}{T}\right)^2\\[5pt]
            & &=&\ \ \frac{1}{T^2}\cdot\left(\frac{1}{\varrho}-1\right)\leq \frac{1}{\varrho T^2}. 
        \end{aligned}
    \end{equation}
    Combine the two cases, we know that the variance of $Z_t({\boldsymbol{\mathcal{C}'}})$ conditioned on $Z_0({\boldsymbol{\mathcal{C}'}}), \dots, Z_{t-1}({\boldsymbol{\mathcal{C}'}})$ is bounded by $1/\varrho T^2$. 

    It always holds that
    \begin{equation}
        |Z_t({\boldsymbol{\mathcal{C}'}})-Z_{t-1}({\boldsymbol{\mathcal{C}'}})|\leq \textbf{\textit{max}}\left\{\frac{1}{T},\frac{1}{\varrho T}-\frac{1}{T}\right\}\leq\frac{1}{\varrho T}.
    \end{equation}

    At last, we complete the proof of this lemma by proving the following two inequalities for any $\varrho$ satisfying the condition \eqref{eq:sampling_para}.
    \begin{subequations}
        \begin{eqnarray}
            \mathbb{P}\left(\ \big|A_{T}(\boldsymbol{\mathcal{C}'})\ -B_{T}(\boldsymbol{\mathcal{C}'})\big|\geq\frac{\epsilon}{2}\ \right)\leq\frac{\delta}{2}\label{subeq1}\\[7.5pt]
            \mathbb{P}\left(\ \big|B_{T}(\boldsymbol{\mathcal{C}'})-d_{\boldsymbol{\mathcal{C}'}}(\boldsymbol{C}_T')\big|\geq\frac{\epsilon}{2}\ \right)\leq\frac{\delta}{2}\label{subeq2}
        \end{eqnarray}
    \end{subequations}
    We can choose $\lambda = \epsilon/2$, $\sigma^2_t=1/\varrho T^2$ and $M=1/\varrho T$ and apply Lemma \ref{lemma:martingale_concentration} on $(Z_0({\boldsymbol{\mathcal{C}'}}), \dots, Z_T({\boldsymbol{\mathcal{C}'}}))$. As $Z_{0}(\boldsymbol{\mathcal{C}'})=0$, we have 
    \begin{equation}
        |A_{T}(\boldsymbol{\mathcal{C}'})-B_{T}(\boldsymbol{\mathcal{C}'})|=|Z_{T}(\boldsymbol{\mathcal{C}'})-Z_{0}(\boldsymbol{\mathcal{C}'})|,
    \end{equation}
    and \eqref{subeq1} will holds that 
    \begin{equation}
        \label{eq:ineq_1}
        \begin{aligned}
            & &    &\ \ \mathbb{P}\left(\ \big|A_{T}(\boldsymbol{\mathcal{C}'})-B_{T}(\boldsymbol{\mathcal{C}'})\big|\geq\frac{\epsilon}{2}\ \right)\\[5pt]
            & &\leq&\ \ 2\cdot\textbf{\textit{exp}}\left(-\frac{(\epsilon/2)^2}{2T\cdot(1/\varrho T^2)+\epsilon/(6\varrho T)}\right)\\[5pt]
            & &<   &\ \ 2\cdot\textbf{\textit{exp}}\left(-\frac{\epsilon^2\varrho T}{9}\right).
        \end{aligned}
    \end{equation}
    When
    \begin{equation}
        \varrho \geq 9\cdot\frac{\textbf{\textit{ln}}(\delta/4)}{\epsilon^2T},
    \end{equation}
    the \eqref{eq:ineq_1} is further bounded by $\delta/2$. 

    To prove \eqref{subeq2}, we observe that 
    \begin{equation}
        B_{T}(\boldsymbol{\mathcal{C}'}) = d_{\boldsymbol{\mathcal{C}'}}(\boldsymbol{C}'_{T})\cdot\frac{|\boldsymbol{C}'_{T}|}{\varrho T},
    \end{equation}
    and each pairwise comparison is selected by Bernoulli method with probability $\varrho$, independently of other pairwise comparison. The size of $\boldsymbol{C}'_{t}$ equals to the expectation of binaomial distribution $\textbf{\textit{Binaomial}}(T, \varrho)$, regardless of the adversary's strategy. Applying the Chernoff inequality with $\delta=\epsilon/2$, we have
    \begin{equation}
        \begin{aligned}
            & & &\ \ \mathbb{P}\left(\big||\boldsymbol{C}'_{T}|-\varrho T\big|\geq \frac{\epsilon\varrho T}{2}\right)\\[5pt]
            & &\leq&\ \ 2\cdot\textbf{\textit{exp}}\left(-\frac{(\epsilon/2)^2\varrho T}{2+\epsilon/3}\right)\\[5pt]
            & &<&\ \ 2\cdot\textbf{\textit{exp}}\left(-\frac{\epsilon^2\varrho T}{10}\right).
        \end{aligned}
    \end{equation}
    When
    \begin{equation} 
        \varrho \geq 10\cdot\frac{\textbf{\textit{ln}}(\delta/4)}{\epsilon^2T},
    \end{equation}
    the above inequality is further bounded by $\delta/2$. Conditioning on this event ($||\boldsymbol{C}'_{t}|-\varrho T|\geq \epsilon\varrho T/2$), we have
    \begin{equation}
        \begin{aligned}
            & & &\ \ \big|d_{\boldsymbol{\mathcal{C}'}}(\boldsymbol{C}'_{T})-B_{T}(\boldsymbol{\mathcal{C}'})\big|\\[5pt]
            & &=&\ \ \left|1-\frac{|\boldsymbol{C}'_{T}|}{\varrho T}\right|\cdot d_{\boldsymbol{\mathcal{C}'}}(\boldsymbol{C}'_{T})\\[5pt]
            & &\leq&\ \ \left|1-\frac{|\boldsymbol{C}'_{T}|}{\varrho T}\right|\leq\frac{\epsilon}{2},
        \end{aligned}
    \end{equation}
    where the first inequality follows the fact $d_{\boldsymbol{\mathcal{C}'}}(\boldsymbol{C}'_{T})$ is always bounded by $1$, and the second inequality follows form the condition $||\boldsymbol{C}'_{t}|-\varrho T|\geq \epsilon\varrho T/2$. We complete the proof of \eqref{subeq2}. 

    Indeed, taking a union bound over \eqref{subeq1} and \eqref{subeq2}, applying the triangle inequality and observing that $A_{T}(\boldsymbol{\mathcal{C}'})=d_{\boldsymbol{\mathcal{C}'}}(\boldsymbol{C}'_{T})$, we obtain the desired conclusion \eqref{eq:lemma_result}. 
\end{proof}

\begin{lemma}
    \label{lemma:main_lemma_Reservoir}
    For any dynamic stream $\boldsymbol{C}=\{c_t\}_{t=1}^{\infty}$ from $\boldsymbol{\mathcal{C}}'$, if the parameter of reservoir method $\varrho$ holds that
    \begin{equation}
        \label{eq:reservoir_para}
        \varrho \geq 2\cdot\frac{\textbf{\textit{ln}}(2/\delta)}{\epsilon^2}, 
    \end{equation}
    we have 
    \begin{equation}
        \mathbb{P}(|d_{\boldsymbol{\mathcal{C}'}}(\boldsymbol{C})-d_{\boldsymbol{\mathcal{C}'}}(\boldsymbol{C}')|\geq\epsilon)\leq \delta,
    \end{equation}
     where $\boldsymbol{C}'$ is a sequence which is sampled from $\boldsymbol{C}$ by the reservoir method.
\end{lemma}

\begin{proof}
    Generally speaking, the proof of this lemma goes along the same lines as Lemma \ref{lemma:main_lemma_Bernoulli}, except that adopting the other martingale. Specifically, we define
    \begin{equation}
        \begin{aligned}
            & A_t({\boldsymbol{\mathcal{C}'}}) &=&\ \ t\cdot d_{\boldsymbol{\mathcal{C}'}}(\boldsymbol{C}_t) = |\boldsymbol{\mathcal{C}'}\cap\boldsymbol{C}_t|,\\[7.5pt]
            & B_t({\boldsymbol{\mathcal{C}'}}) &=&\ \ t\cdot d_{\boldsymbol{\mathcal{C}'}}(\boldsymbol{C}'_t)= \frac{t}{\varrho}|\boldsymbol{\mathcal{C}'}\cap\boldsymbol{C}'_t|,\\[7.5pt]
            & Z_t({\boldsymbol{\mathcal{C}'}}) &=&\ \ B_t({\boldsymbol{\mathcal{C}'}}) - A_t({\boldsymbol{\mathcal{C}'}}),
        \end{aligned}
    \end{equation}
    for $t\in(\varrho, T]$. When $t<\varrho$, we define
    \begin{equation}
        A_t({\boldsymbol{\mathcal{C}'}}) = B_t({\boldsymbol{\mathcal{C}'}}) = |\boldsymbol{\mathcal{C}'}\cap\boldsymbol{C}_t|.
    \end{equation}
    The next step is similar to Lemma \ref{lemma:main_lemma_Bernoulli}. We first show that $(Z_0({\boldsymbol{\mathcal{C}'}}),\dots,Z_T({\boldsymbol{\mathcal{C}'}}))$ is a martingale. Notice that $(Z_0({\boldsymbol{\mathcal{C}'}}),\dots,Z_T({\boldsymbol{\mathcal{C}'}}))$ is obviously a martingale for $t\leq\varrho$. When $t>\varrho$,  the $\boldsymbol{C}'_{t-1}$ is fixed and thence the values of $Z_0({\boldsymbol{\mathcal{C}'}}), \dots, Z_{t-1}({\boldsymbol{\mathcal{C}'}})$ are fixed. Let $c_t$ be the next pairwise comparison for the reservoir sampling method, which could be either from the original data source or the adversarial data source controlled by the adversary. It is easy to check that
    \begin{equation}
        A_t({\boldsymbol{\mathcal{C}'}}) = 
        \left\{
        \begin{matrix}
        A_{t-1}({\boldsymbol{\mathcal{C}'}}), & c_t\notin{\boldsymbol{\mathcal{C}'}},\\ 
        A_{t-1}({\boldsymbol{\mathcal{C}'}})+1. & c_t\in{\boldsymbol{\mathcal{C}'}}.
        \end{matrix}
        \right.
    \end{equation}
    For $B_t({\boldsymbol{\mathcal{C}'}})$, we consider the three factors:
    \begin{itemize}
        \item[i)] is $c_t\in{\boldsymbol{\mathcal{C}'}}$ or not?
        \item[ii)] is $c_t$ sampled or not?
        \item[iii)] conditioning on $c_t$ being sampled, does it replace an element $r_t$ from ${\boldsymbol{\mathcal{C}'}}$ in the sample, or $r_t$ not in ${\boldsymbol{\mathcal{C}'}}$?
    \end{itemize}
    \vspace{0.5cm}
    \begin{itemize}
        \item[\textbf{Case 1.}] When $c_t\notin{\boldsymbol{\mathcal{C}'}}$ is either not sampled, or sampled but with $r_t\notin{\boldsymbol{\mathcal{C}'}}$, the pairwise comparisons from $\boldsymbol{\mathcal{C}'}$ are neither added nor removed into the cache of the reservoir method. Consequently, we have
        \begin{equation}
            \boldsymbol{\mathcal{C}'}\cap\boldsymbol{C}_t = \boldsymbol{\mathcal{C}'}\cap\boldsymbol{C}_{t-1}.
        \end{equation}
        By the definition of $B_t({\boldsymbol{\mathcal{C}'}})$
        \begin{equation}
            \begin{aligned}
                & B_t({\boldsymbol{\mathcal{C}'}}) &=&\ \ \frac{t}{\varrho}\cdot|\boldsymbol{\mathcal{C}'}\cap\boldsymbol{C}_t|\\
                & &=&\ \ \frac{t-1}{\varrho}\cdot|\boldsymbol{\mathcal{C}'}\cap\boldsymbol{C}_{t-1}|+\frac{1}{\varrho}\cdot|\boldsymbol{\mathcal{C}'}\cap\boldsymbol{C}_{t-1}|\\
                & &=&\ \ B_{t-1}({\boldsymbol{\mathcal{C}'}})+d_{\boldsymbol{\mathcal{C}'}}(\boldsymbol{C}_{t-1}),
            \end{aligned}
        \end{equation}
        where the third equality stands since the sampled pairwise comparisons of $\boldsymbol{C}_{t-1}$ is $|\boldsymbol{C}_{t-1}|=\varrho$ when $t>\varrho$. Therefore conditioned on $c_t\notin{\boldsymbol{\mathcal{C}'}}$, the expectation of $B_t({\boldsymbol{\mathcal{C}'}})$ is
        \begin{equation}
            \begin{aligned}
                & & &\ \ \mathbb{E}[B_t({\boldsymbol{\mathcal{C}'}})|c_t\notin{\boldsymbol{\mathcal{C}'}}]\\[5pt]
                & &=&\ \  \left(1-\frac{\varrho}{t}d_{\boldsymbol{\mathcal{C}'}}(\boldsymbol{C}_{t-1})\right)\cdot\Big(B_{t-1}({\boldsymbol{\mathcal{C}'}})+d_{\boldsymbol{\mathcal{C}'}}(\boldsymbol{C}_{t-1})\Big)+\frac{\varrho}{t}d_{\boldsymbol{\mathcal{C}'}}(\boldsymbol{C}_{t-1})\left(B_{t-1}({\boldsymbol{\mathcal{C}'}})+d_{\boldsymbol{\mathcal{C}'}}(\boldsymbol{C}_{t-1})-\frac{t}{\varrho}\right)\\[5pt]
                & &=&\ \ B_{t-1}({\boldsymbol{\mathcal{C}'}}).
            \end{aligned}
        \end{equation}
        Moreover, $A_{t}({\boldsymbol{\mathcal{C}'}})=A_{t-1}({\boldsymbol{\mathcal{C}'}})$ when $c_t\notin{\boldsymbol{\mathcal{C}'}}$ and we deduce that
        \begin{equation}
            \mathbb{E}\big[\ Z_t({\boldsymbol{\mathcal{C}'}})\ \big|\ Z_0({\boldsymbol{\mathcal{C}'}}), \dots, Z_{t-1}({\boldsymbol{\mathcal{C}'}}),\ c_t\notin\boldsymbol{\mathcal{C}'}\ \big] = Z_{t-1}({\boldsymbol{\mathcal{C}'}}).
        \end{equation}
        \item[\textbf{Case 2.}]Now $c_t\in{\boldsymbol{\mathcal{C}'}}$. When $c_t$ is neither added nor removed into the cache of the reservoir method, we have $|\boldsymbol{C}_{t}|= |\boldsymbol{C}_{t-1}|$ and $B_t({\boldsymbol{\mathcal{C}'}})=B_{t-1}({\boldsymbol{\mathcal{C}'}})+d_{\boldsymbol{\mathcal{C}'}}(\boldsymbol{C}_{t-1})$. If $c_t$ into the cache and the replaced element $r_t\notin{\boldsymbol{\mathcal{C}'}}$, which has probability $(\varrho/t)\cdot(1-d_{\boldsymbol{\mathcal{C}'}}(\boldsymbol{C}_{t-1}))$, we have
        \begin{equation}
            \begin{aligned}
                & |\boldsymbol{\mathcal{C}'}\cap\boldsymbol{C}_{t}| &=&\ \ |\boldsymbol{\mathcal{C}'}\cap\boldsymbol{C}_{t-1}| + 1\\
                & B_t({\boldsymbol{\mathcal{C}'}}) &=&\ \ \frac{t}{\varrho}\cdot|\boldsymbol{\mathcal{C}'}\cap\boldsymbol{C}_{t}|\\
                & &=&\ \ \frac{t}{\varrho}\cdot|\boldsymbol{\mathcal{C}'}\cap\boldsymbol{C}_{t-1}| + \frac{t}{\varrho}\\
                & &=&\ \ B_{t-1}({\boldsymbol{\mathcal{C}'}}) + d_{\boldsymbol{\mathcal{C}'}}(\boldsymbol{C}_{t-1}) + \frac{t}{\varrho}.
            \end{aligned}
        \end{equation}
        Then the expectation of $B_t({\boldsymbol{\mathcal{C}'}})$ conditioned on $c_t\in{\boldsymbol{\mathcal{C}'}}$ is
        \begin{equation}
            \begin{aligned}
                & & &\ \ \mathbb{E}[B_t({\boldsymbol{\mathcal{C}'}})|c_t\in{\boldsymbol{\mathcal{C}'}}]\\[5pt]
                & &=&\ \  B_{t-1}({\boldsymbol{\mathcal{C}'}})+d_{\boldsymbol{\mathcal{C}'}}(\boldsymbol{C}_{t-1})+\left(\frac{\varrho}{t}\cdot(1-d_{\boldsymbol{\mathcal{C}'}}(\boldsymbol{C}_{t-1})\right)\cdot\frac{t}{\varrho}\\[5pt]
                & &=&\ \ B_{t-1}({\boldsymbol{\mathcal{C}'}})+1.
            \end{aligned}
        \end{equation}
        Furthermore, with the definition of $A_t({\boldsymbol{\mathcal{C}'}})$ when $c_t\in{\boldsymbol{\mathcal{C}'}}$, we know that
        \begin{equation}
            \mathbb{E}\big[\ Z_t({\boldsymbol{\mathcal{C}'}})\ \big|\ Z_0({\boldsymbol{\mathcal{C}'}}), \dots, Z_{t-1}({\boldsymbol{\mathcal{C}'}}),\ c_t\in\boldsymbol{\mathcal{C}'}\ \big] = Z_{t-1}({\boldsymbol{\mathcal{C}'}}).
        \end{equation}
        The analysis of the above two cases implies that $(Z_0({\boldsymbol{\mathcal{C}'}}), \dots, Z_{t}({\boldsymbol{\mathcal{C}'}}))$ is indeed a martingale.

        The second part of proof is to obtain the bounds on the difference $|Z_t({\boldsymbol{\mathcal{C}'}})-Z_{t-1}({\boldsymbol{\mathcal{C}'}})|$ and the variance of $Z_t({\boldsymbol{\mathcal{C}'}})$ given $Z_0({\boldsymbol{\mathcal{C}'}}), \dots, Z_{t-1}({\boldsymbol{\mathcal{C}'}})$. With the above analysis, we know that
        \begin{equation}
            \begin{aligned}
                & A_t({\boldsymbol{\mathcal{C}'}}) &=&\ \ \left\{\begin{matrix}
                    A_{t-1}({\boldsymbol{\mathcal{C}'}}), & c_t\notin\boldsymbol{\mathcal{C}'},\\[5pt]
                    A_{t-1}({\boldsymbol{\mathcal{C}'}})+1, & c_t\in\boldsymbol{\mathcal{C}'}.
                \end{matrix} \right.\\[5pt]
                & B_t({\boldsymbol{\mathcal{C}'}}) &\in&\ \ \left\{\begin{matrix}
                    \displaystyle\left[B_{t-1}({\boldsymbol{\mathcal{C}'}}),\ B_{t-1}({\boldsymbol{\mathcal{C}'}})+1\right], & c_t\notin\boldsymbol{\mathcal{C}'},\\[5pt]
                    \displaystyle\left[B_{t-1}({\boldsymbol{\mathcal{C}'}}),\ B_{t-1}({\boldsymbol{\mathcal{C}'}})+1+\frac{t}{\varrho}\right], & c_t\in\boldsymbol{\mathcal{C}'}.
                \end{matrix}\right.
            \end{aligned}
        \end{equation}
        By the definition of $Z_{t}({\boldsymbol{\mathcal{C}'}})$, we conclude that
        \begin{equation}
            |Z_t({\boldsymbol{\mathcal{C}'}})-Z_{t-1}({\boldsymbol{\mathcal{C}'}})|\leq\frac{t}{\varrho}.
        \end{equation}
    \end{itemize}
    We next bound the variance of $Z_t({\boldsymbol{\mathcal{C}'}})$ conditioned on $Z_0({\boldsymbol{\mathcal{C}'}}), \dots, Z_{t-1}({\boldsymbol{\mathcal{C}'}})$ and $d_{\boldsymbol{\mathcal{C}'}}(\boldsymbol{C}_{t-1})$. When $c_t\notin\boldsymbol{\mathcal{C}'}$, with probability $(\varrho/t)\cdot d_{\boldsymbol{\mathcal{C}'}}(\boldsymbol{C}_{t-1})$, it holds that
    \begin{equation}
        \mathbb{E}[Z_t({\boldsymbol{\mathcal{C}'}})]- Z_t({\boldsymbol{\mathcal{C}'}}) = \frac{t}{\varrho}-d_{\boldsymbol{\mathcal{C}'}}(\boldsymbol{C}_{t-1}).
    \end{equation}
    Otherwise, with probability $1-(\varrho/t)\cdot d_{\boldsymbol{\mathcal{C}'}}(\boldsymbol{C}_{t-1})$, we have
    \begin{equation}
        Z_t({\boldsymbol{\mathcal{C}'}})-\mathbb{E}[Z_t({\boldsymbol{\mathcal{C}'}})] = d_{\boldsymbol{\mathcal{C}'}}(\boldsymbol{C}_{t-1}).   
    \end{equation}
    Therefore,
    \begin{equation}
        \label{eq:var_bound_1}
        \begin{aligned}
            & & &\ \ \textbf{\textit{Var}}(\ Z_t({\boldsymbol{\mathcal{C}'}})\ \big|\ Z_0({\boldsymbol{\mathcal{C}'}}), \dots, Z_{t-1}({\boldsymbol{\mathcal{C}'}}),\ c_t\notin\boldsymbol{\mathcal{C}'},d_{\boldsymbol{\mathcal{C}'}}(\boldsymbol{C}_{t-1})\ )\\
            & &=&\ \ \frac{\varrho}{t}\cdot d_{\boldsymbol{\mathcal{C}'}}(\boldsymbol{C}_{t-1})\cdot\left(\frac{t}{\varrho}-d_{\boldsymbol{\mathcal{C}'}}(\boldsymbol{C}_{t-1})\right)^2+\left(1-\frac{\varrho}{t}\cdot d_{\boldsymbol{\mathcal{C}'}}(\boldsymbol{C}_{t-1})\right)\cdot d^2_{\boldsymbol{\mathcal{C}'}}(\boldsymbol{C}_{t-1})\\
            & &=&\ \ \frac{t}{\varrho}\cdot d_{\boldsymbol{\mathcal{C}'}}(\boldsymbol{C}_{t-1})-d^2_{\boldsymbol{\mathcal{C}'}}(\boldsymbol{C}_{t-1})\\
            & &\leq&\ \ \frac{t}{\varrho}.
        \end{aligned}
    \end{equation}
    When $c_t\in\boldsymbol{\mathcal{C}'}$, with probability $(\varrho/t)\cdot d_{\boldsymbol{\mathcal{C}'}}(\boldsymbol{C}_{t-1})$, it holds that
    \begin{equation}
        Z_t({\boldsymbol{\mathcal{C}'}})-\mathbb{E}[Z_t({\boldsymbol{\mathcal{C}'}})] = \frac{t}{\varrho}+d_{\boldsymbol{\mathcal{C}'}}(\boldsymbol{C}_{t-1})-1.
    \end{equation}
    Otherwise, with probability $1-(\varrho/t)\cdot d_{\boldsymbol{\mathcal{C}'}}(\boldsymbol{C}_{t-1})$, we have
    \begin{equation}
        \mathbb{E}[Z_t({\boldsymbol{\mathcal{C}'}})]-Z_t({\boldsymbol{\mathcal{C}'}}) = 1-d_{\boldsymbol{\mathcal{C}'}}(\boldsymbol{C}_{t-1}).   
    \end{equation}
    Thus,
    \begin{equation}
        \label{eq:var_bound_2}
        \begin{aligned}
            & & &\ \ \textbf{\textit{Var}}(\ Z_t({\boldsymbol{\mathcal{C}'}})\ \big|\ Z_0({\boldsymbol{\mathcal{C}'}}), \dots, Z_{t-1}({\boldsymbol{\mathcal{C}'}}),\ c_t\in\boldsymbol{\mathcal{C}'},d_{\boldsymbol{\mathcal{C}'}}(\boldsymbol{C}_{t-1})\ )\\
            & &=&\ \ \frac{\varrho}{t}\cdot d_{\boldsymbol{\mathcal{C}'}}(\boldsymbol{C}_{t-1})\cdot\left(\frac{t}{\varrho}-d_{\boldsymbol{\mathcal{C}'}}(\boldsymbol{C}_{t-1})\right)^2+\left(1-\frac{\varrho}{t}\cdot d_{\boldsymbol{\mathcal{C}'}}(\boldsymbol{C}_{t-1})\right)\cdot d^2_{\boldsymbol{\mathcal{C}'}}(\boldsymbol{C}_{t-1})\\
            & &=&\ \ \frac{t}{\varrho}\cdot d_{\boldsymbol{\mathcal{C}'}}(\boldsymbol{C}_{t-1})-d^2_{\boldsymbol{\mathcal{C}'}}(\boldsymbol{C}_{t-1})\\
            & &\leq&\ \ \frac{t}{\varrho}.
        \end{aligned}
    \end{equation} 
    \eqref{eq:var_bound_1} and \eqref{eq:var_bound_2} indicate that the conditional variance of $Z_t({\boldsymbol{\mathcal{C}'}})$ is bounded by $t/\varrho$. Moreover, the bound remains intact when we remove the condition on $d_{\boldsymbol{\mathcal{C}'}}(\boldsymbol{C}_{t-1})$. 

    Now we come to the conclusion of the whole lemma. Observe that
    \begin{equation}
        \begin{aligned}
            & & &\ \ \mathbb{P}(|d_{\boldsymbol{\mathcal{C}'}}(\boldsymbol{C})-d_{\boldsymbol{\mathcal{C}'}}(\boldsymbol{C}')|\geq\epsilon)\\[5pt]
            & &=&\ \ \mathbb{P}(|B_T(\boldsymbol{\mathcal{C}'})-A_T(\boldsymbol{\mathcal{C}'})|\geq\epsilon\cdot T)\\[5pt]
            & &=&\ \ \mathbb{P}(|Z_T(\boldsymbol{\mathcal{C}'})-Z_0(\boldsymbol{\mathcal{C}'})|\geq\epsilon\cdot T).
        \end{aligned}
    \end{equation}
    Then we apply Lemma \ref{lemma:martingale_concentration} on the martingale $Z(\boldsymbol{\mathcal{C}'})=(Z_0(\boldsymbol{\mathcal{C}'}),\dots,Z_T(\boldsymbol{\mathcal{C}'}))$ with $\lambda = \epsilon T$, $\sigma^2_t = t/\varrho$ for $t\geq\varrho$ and $\sigma^2_t =0$ for $t\geq\varrho$, and $M = T/\varrho$
    \begin{equation}
        \begin{aligned}
            & & &\ \ \mathbb{P}(|Z_T(\boldsymbol{\mathcal{C}'})-Z_0(\boldsymbol{\mathcal{C}'})|\geq\epsilon\cdot T)\\[5pt]
            & &\leq&\ \ 2\textbf{\textit{exp}}\left(-\frac{\displaystyle \lambda^2}{\displaystyle2\sum_{t=1}^T\sigma^2_t+\frac{\lambda M}{3}}\right)\\[5pt]
            & &=&\ \ 2\textbf{\textit{exp}}\left(-\frac{\epsilon^2T^2}{\displaystyle2\sum_{t=1}^T\frac{t}{\varrho}+\frac{\epsilon T^2 }{3\varrho}}\right)\\[5pt]
            & &=&\ \ 2\textbf{\textit{exp}}\left(-\frac{\epsilon^2T^2\varrho}{\displaystyle T(T+1)+\frac{\epsilon T^2 }{3}}\right)\\[7.5pt]
            & &\leq&\ \ 2\textbf{\textit{exp}}\left(-\frac{\epsilon^2T^2\varrho}{2T^2}\right) = 2\textbf{\textit{exp}}\left(-\frac{\epsilon^2\varrho}{2}\right).
        \end{aligned}
    \end{equation}
    Therefore, it suffices to hold 
    \begin{equation}
        \varrho\geq \frac{2}{\epsilon^2}\textbf{\textit{ln}}\ \left(\frac{2}{\delta}\right)
    \end{equation}
    and get the desired result. 
\end{proof}

\vulnerability*

\begin{proof}
    For $i)$, the results of static case have been discussed by \cite{10.1214/aop/1176988847,LI2001516,Vapnik:71}. Notice that $\textbf{\textit{log}}\ n(n-1)$ is the VC-dimension of $\boldsymbol{\mathcal{C}'}$. For $ii)$, we start with the Bernoulli sampling method. For any dynamic stream $\boldsymbol{C}$ form $\boldsymbol{\mathcal{C}'}$, we apply the first part of Lemma \ref{lemma:main_lemma_Bernoulli} with $\epsilon$ and $|\boldsymbol{\mathcal{C}'}|$
    \begin{equation}
        P\Big(\big|d_{\boldsymbol{\mathcal{C}'}}(\boldsymbol{C})-d_{\boldsymbol{\mathcal{C}'}}(\boldsymbol{C}')\big|\geq\epsilon\Big)\leq \frac{\delta}{|\boldsymbol{\mathcal{C}'}|},
    \end{equation}
    where $\boldsymbol{C}'$ is the sampled sequence by the Bernoulli sampling method. In the event
    \begin{equation}
        \big|d_{\boldsymbol{\mathcal{C}'}}(\boldsymbol{C})-d_{\boldsymbol{\mathcal{C}'}}(\boldsymbol{C}')\big|\leq\epsilon,\ \ \forall\ \boldsymbol{C}\in\boldsymbol{\mathcal{C}'},
    \end{equation}
    by definition we know that $\boldsymbol{C}'$ is an $\epsilon$-approximation of $\boldsymbol{C}$. Taking a union bound over all $\boldsymbol{C'}\in\boldsymbol{\mathcal{C}'}$, we conclude that the probability of this event \textbf{\textit{not}} to hold is bounded by
    \begin{equation}
        \frac{\delta}{|\boldsymbol{\mathcal{C}'}|}\cdot|\boldsymbol{\mathcal{C}'}| = \delta,
    \end{equation}
    meaning that Bernoulli method with $\varrho$ as above is $(\epsilon, \delta)$-representative. 

    The proof for reservoir method is identical, except that we apple Lemma \ref{lemma:main_lemma_Reservoir}. 
\end{proof}

% \begin{theorem}
%     \label{thm:vul}
%     Let $\boldsymbol{C}=(c_1,\dots,c_T)$ be a sequence from the universe $\boldsymbol{\mathcal{C}}$ and $T$ is the length of the sequence. Suppose that $(\boldsymbol{\mathcal{C}},\mathcal{S})$ is the original set system and $(\boldsymbol{\mathcal{C}},\mathcal{S}_{\mathcal{A}})$ is the the adversarial set system where $\mathcal{S}_{\mathcal{A}} \neq \mathcal{S}$. Given $\epsilon,\delta\in(0,1)$, if the adversary replaces $(\boldsymbol{\mathcal{C}},\mathcal{S})$ with $(\boldsymbol{\mathcal{C}},\mathcal{S}_{\mathcal{A}})$ and the parameter of Bernoulli method satisfies 
%     \begin{equation}
%         \label{eq:bernoulli_param}
%         \varrho\geq 10\cdot\frac{\displaystyle\textbf{\textit{ln}}|\mathcal{S}_{\mathcal{A}}|+\textbf{\textit{ln}}(4/\delta)}{\epsilon^2T},
%     \end{equation}
%     the output $\boldsymbol{C}'$ of Bernoulli sampling is $(\epsilon,\delta)$-representative with respect to $(\boldsymbol{\mathcal{C}},\mathcal{S}_{\mathcal{A}})$.
%     % an $\epsilon$-approximation of the whole stream $\boldsymbol{C}$ with respect to $(\boldsymbol{\mathcal{C}},\mathcal{S}_{\mathcal{A}})$, with probability at least $1-\delta$.}
% \end{theorem}

\newpage

\section*{Asymptotic Optimality of the Proposed Policy with Complete Knowledge}
\label{sec:asyp_opt_compl}
\addcontentsline{toc}{section}{Asymptotic Optimality of the Proposed Policy with Complete Knowledge\nameref{sec:asyp_opt_compl}}
\ul{We discuss the asymptotic optimality of the proposed stooping time and generation rule with the complete knowledge. The assumption for the theoretical analysis is described}. First there exist some regularity conditions on the prior distribution $\rho_{\boldsymbol{\theta}'}$. Without loss of generality, we could set $\theta'_1 = 0$ and the unknown model parameter satisfies $\boldsymbol{\theta}=[\theta_2,\dots,\theta_n]\in\mathbb{R}^{n-1}$. The following assumptions have been applied in the sequential design for rank aggregation \cite{doi:10.1287/moor.2021.1209}. It is noteworthy that the following assumptions are mainly for the theoretical analysis. The proposed sequential manipulation method does not depend on the condition of $\textbf{\textit{Supp}}(\rho_{\boldsymbol{\theta}'})$ or the probability mass function $g$.

\begin{assumption}
    \label{assmp:support}
    The support 
    \begin{equation}
        \textbf{\textit{Supp}}(\rho_{\boldsymbol{\theta}'}) = \overline{\left\{\ \boldsymbol{\theta}\in\mathbb{R}^{n-1}\ \Bigg|\ \rho_{\boldsymbol{\theta}'}(\boldsymbol{\theta})>0\ \right\}}
    \end{equation}
    is a compact set, where $\overline{\{\cdot\}}$ denotes the closeure of $\{\cdot\}$. Besides, for any full ranking list $\boldsymbol{\pi}_0$ which does not include the candidates belong to $\theta'_1$, 
    \begin{equation}
         \left\{\left\{\ \boldsymbol{\theta}\in\mathbb{R}^{n-1}\ \Bigg|\ \boldsymbol{\pi}(\boldsymbol{\theta}) = \boldsymbol{\pi}_0\right\}\bigcap\textbf{\textit{Supp}}(\rho_{\boldsymbol{\theta}'})\right\}^{\circ} \neq \varnothing,
     \end{equation} 
     where $\{\cdot\}^{\circ}$ represents the interior of $\{\cdot\}$. 
\end{assumption}
This assumption assigns the bounded support $\textbf{\textit{Supp}}(\rho_{\boldsymbol{\theta}'})$ to the prior distribution $\rho_{\boldsymbol{\theta}'}$. The second part tell us that the support $\textbf{\textit{Supp}}(\rho_{\boldsymbol{\theta}'})$ contains a non-empty interior for every full ranking list. 

\begin{assumption}
    \label{assmp:compat}
    For all $s>0$, there exists a constant $\delta>0$ such that
    \begin{equation}
        \mathcal{L}\left(\boldsymbol{B}(\boldsymbol{\theta},\epsilon)\cap\textbf{\textit{Supp}}(\rho_{\boldsymbol{\theta}'})\right)\geq\textbf{\textit{min}}\{\delta\epsilon^{n-1}, 1\},
    \end{equation}
    where $\boldsymbol{B}(\boldsymbol{\theta},\epsilon)$ denotes the open ball centered at $\boldsymbol{\theta}$ with radius $\epsilon$, and $\mathcal{L}(\cdot)$ denotes the Lebesgue measure. 
\end{assumption}

This assumption keeps the support $\textbf{\textit{Supp}}(\rho_{\boldsymbol{\theta}'})$ be non-singular.

\begin{assumption}
    \label{assmp:continuous}
    The log probability mass function $\textbf{\textit{log}}\ g_{ij}(\boldsymbol{\theta})$ is uniform continuous differentiable \textit{w.r.t.} $\boldsymbol{\theta}$ for all pairwise comparisons $(i,j)$, that is 
    \begin{equation}
        % \underset{\boldsymbol{\theta}\in\textbf{\textit{Supp}}(\rho_{\boldsymbol{\theta}'}),\ (i,j)\in\boldsymbol{\mathfrak{C}}}{\textbf{\textit{sup}}}\ \left\|\nabla_{\boldsymbol{\theta}}\ \textbf{\textit{log}}\ g_{ij}(\boldsymbol{\theta})\right\|<\infty
        \underset{\begin{matrix}\scriptstyle\boldsymbol{\theta}\in\textbf{\textit{Supp}}(\rho_{\boldsymbol{\theta}'}),\\\scriptstyle(i,j)\in\boldsymbol{\mathfrak{C}}\end{matrix}}{\textbf{\textit{sup}}}\ \left\|\nabla_{\boldsymbol{\theta}}\ \textbf{\textit{log}}\ g_{ij}(\boldsymbol{\theta})\right\|<\infty.
    \end{equation}
\end{assumption}

This assumption needs the smoothness of the likelihood function. The \textbf{BTl} model satisfies this assumption.

\begin{assumption}
    The probability mass function $g$ satisfies:
    \label{assup:no_tie}
    \begin{equation}
        \underset{\begin{matrix}\scriptstyle\boldsymbol{\theta},\boldsymbol{\tilde{\theta}}\in\textbf{\textit{Supp}}(\rho_{\boldsymbol{\theta}'}),\\ \scriptstyle \boldsymbol{\pi}(\boldsymbol{\theta})\neq\boldsymbol{\pi}(\boldsymbol{\tilde{\theta}})\end{matrix}}{\textbf{\textit{min}}}\ \underset{(i,j)}{\textbf{\textit{max}}}\ \ g_{i,j}(\boldsymbol{\theta})\cdot\textbf{\textit{log}}\frac{g_{i,j}(\boldsymbol{\theta})}{g_{i,j}(\boldsymbol{\tilde{\theta}})} > 0.
    \end{equation}
\end{assumption}

This assumption requires the distinguishability between any pair of candidates in $\textbf{\textit{Supp}}(\rho_{\boldsymbol{\theta}'})$, \textit{e.g.} there does not exist tie between any pair of candidates. Assumption \ref{assup:no_tie} is a standard assumption in sequential hypothesis testing, which is known as the ``indifference zone'' assumption. The ``indifference zone'' condition tell us that the null and alternative hypotheses are separated in the sense that the Kullback-Leibler divergence between the two hypotheses is positive. Here the assumption excludes the case that the true preference score is in between the two hypotheses ($i\succ j$ and $j\succ i$). Furthermore, it means that selecting $i\succ j$ or $j\succ i$ (the null and alternative hypothesis) will be different.

\begin{assumption}
    \label{assmp:positive_density}
    The prior distribution $\rho_{\boldsymbol{\theta}'}$ satisfies
    \begin{equation}
        \begin{aligned}
            & \underset{\boldsymbol{\theta}\in\{\textbf{\textit{Supp}}(\rho_{\boldsymbol{\theta}'})\}^{\circ}}{\textbf{\textit{inf}}}\ \ \rho_{\boldsymbol{\theta}'}(\boldsymbol{\theta})>0,\\[5pt]
            & \underset{\ \ \boldsymbol{\theta}\in\textbf{\textit{Supp}}(\rho_{\boldsymbol{\theta}'})\ \ }{\textbf{\textit{sup}}}\ \ \rho_{\boldsymbol{\theta}'}(\boldsymbol{\theta})<0.
        \end{aligned}
    \end{equation}
\end{assumption}
This assumption requires that the density function of prior distribution $\rho_{\boldsymbol{\theta}'}$ is positive over $\textbf{\textit{Supp}}(\rho_{\boldsymbol{\theta}'})$. 

\ul{Considering the Assumption} \ref{assmp:support}-\ref{assmp:positive_density}, \ul{the following theorem establishes a lower bound on the minimal Bayesian risk} \eqref{eq:mini_risk}
\begin{equation}
    \mathfrak{R}^* = \underset{\boldsymbol{\Lambda},S}{\textbf{\textit{inf}}}\ \mathfrak{R}(\boldsymbol{\Lambda}, S). \tag{\ref{eq:mini_risk}}
\end{equation}

\begin{theorem}
    \label{thm:asymptotic2expectation}
    If the Assumption \ref{assmp:support}-\ref{assmp:positive_density} hold, we have 
    \begin{equation}
        \underset{\chi\rightarrow 0}{\textbf{\textit{lim inf}}}\ \ \frac{\mathfrak{R}^*}{\chi\mathbb{E}\left[\tau_{\chi}(\Theta)\right]}\geq 1,
    \end{equation}
    where
    \begin{equation}
        \tau_{\chi}(\boldsymbol{\theta}) = |\textbf{\textit{log}}\ \chi|\cdot\left(\underset{\boldsymbol{\lambda}\in\boldsymbol{\Delta}\phantom{\boldsymbol{\tilde{\theta}}}}{\textbf{\textit{max\phantom{f}}}}\underset{\begin{matrix}\scriptstyle\boldsymbol{\tilde{\theta}}\in\textbf{\textit{Supp}}(\rho_{\boldsymbol{\theta}'}),\\ \scriptstyle \boldsymbol{\pi}(\boldsymbol{\theta})\neq\boldsymbol{\pi}(\boldsymbol{\tilde{\theta}})\end{matrix}}{\textbf{\textit{min}}}\ \ \sum_{(i,j)}\lambda_{i,j}g_{i,j}(\boldsymbol{\theta})\cdot\textbf{\textit{log}}\frac{g_{i,j}(\boldsymbol{\theta})}{g_{i,j}(\boldsymbol{\tilde{\theta}})}\right)^{-1},
    \end{equation}
    and
    \begin{equation}
        \mathbb{E}\left[\tau_{\chi}(\Theta)\right] = \int_{\textbf{\textit{Supp}}(\rho_{\boldsymbol{\theta}'})}\tau_{\chi}(\boldsymbol{\theta})\rho_{\boldsymbol{\theta}'}(\boldsymbol{\theta})d\boldsymbol{\theta}.
    \end{equation}
    To avoid confusion, we write $\Theta$ when $\boldsymbol{\theta}$ is viewed as a random variable.
\end{theorem}

\begin{proof}
    For any manipulation policy $(\boldsymbol{\Lambda}, S)$ and a prior probability density function $\rho_{\boldsymbol{\theta}'}$, there only exist two cases:
    \begin{itemize}
        \item 
        \begin{equation}
            \boldsymbol{E}[\mathfrak{R}(\boldsymbol{R}(\boldsymbol{\Lambda}, S))]\geq \chi|\textbf{\textit{log}}\ \chi|^2.
        \end{equation}
        \item
        \begin{equation}
            \boldsymbol{E}[\mathfrak{R}(\boldsymbol{R}(\boldsymbol{\Lambda}, S))]< \chi|\textbf{\textit{log}}\ \chi|^2.
        \end{equation} 
    \end{itemize}
    For the first case, the Bayesian risk
    \begin{equation}
        \mathfrak{R}(\boldsymbol{\Lambda}, S) = \mathbb{E}[\mathfrak{R}(S)+\mathfrak{R}(\boldsymbol{R}(\boldsymbol{\Lambda}, S))] \tag{\ref{eq:bayesian_risk}}
    \end{equation}
    satisfies
    \begin{equation}
        \mathfrak{R}(\boldsymbol{\Lambda}, S) \geq \chi|\textbf{\textit{log}}\ \chi|^2 \geq (1+o(1))\chi\mathbb{E}\left[\tau_{\chi}(\Theta)\right].
    \end{equation}
    For the second case, it is easy to see
    \begin{equation}
         \mathfrak{R}(\boldsymbol{\Lambda}, S) \geq \mathbb{E}[\mathfrak{R}(S)]=\chi\mathbb{E}[S].
    \end{equation}
    Therefore, proving the results equals to show that
    \begin{equation}
        \underset{\chi\rightarrow 0}{\textbf{\textit{lim inf}}}\ \ \frac{\chi\mathbb{E}[S]}{\chi\mathbb{E}\left[\tau_{\chi}(\Theta)\right]}\geq 1. 
    \end{equation}
    In other words, for any $\delta>0$, there exists a $\chi_0>0$ such taht when $\chi<\chi_0$, 
    \begin{equation}
        \label{eq:thm_result}
        \mathbb{E}[S]\geq(1-\delta)\mathbb{E}\left[\tau_{\chi}(\Theta)\right].
    \end{equation}
    For each $\delta>0$, we set 
    \begin{equation}
        \tau_{\chi,\delta}(\boldsymbol{\theta}) = \left(1-\frac{2}{3}\delta\right)\tau_{\chi}(\boldsymbol{\theta}),
    \end{equation}
    then 
    \begin{equation}
        \begin{aligned}
            & \mathbb{E}[S] &\geq&\ \ \ \mathbb{E}[S|S>\tau_{\chi,\delta}(\Theta)]\\[5pt]
            & &\geq&\ \ \ \int_{\textbf{\textit{Supp}}(\rho_{\boldsymbol{\theta}'})}\rho_{\boldsymbol{\theta}'}(\boldsymbol{\theta})\tau_{\chi,\delta}(\boldsymbol{\theta})\mathbb{P}(S>\tau_{\chi,\delta}(\Theta)|\Theta=\boldsymbol{\theta})d\boldsymbol{\theta}\\[5pt]
            & &=&\ \ \mathbb{E}\left[\tau_{\chi,\delta}(\Theta)\right]-\int_{\textbf{\textit{Supp}}(\rho_{\boldsymbol{\theta}'})}\rho_{\boldsymbol{\theta}'}(\boldsymbol{\theta})\tau_{\chi,\delta}(\boldsymbol{\theta})\mathbb{P}(S\leq\tau_{\chi,\delta}(\Theta)|\Theta=\boldsymbol{\theta})d\boldsymbol{\theta}\\[7.5pt]
            & &\geq&\ \ \mathbb{E}\left[\tau_{\chi,\delta}({\Theta})\right] - \tau^{\textbf{\textit{max}}}_{\chi,\delta}\cdot\mathbb{P}(S\leq\tau_{\chi,\delta}(\Theta)),
        \end{aligned}
    \end{equation}
where $\tau^{\textbf{\textit{max}}}_{\chi,\delta}$ is defined as
\begin{equation}
    \tau^{\textbf{\textit{max}}}_{\chi,\delta} = \underset{\boldsymbol{\theta}\in\textbf{\textit{Supp}}(\rho_{\boldsymbol{\theta}'})}{\textbf{\textit{max}}}\ \tau_{\chi,\delta}(\boldsymbol{\theta}). 
\end{equation}
According to Assumption \ref{assup:no_tie}, we know that
\begin{equation}
    \tau^{\textbf{\textit{max}}}_{\chi,\delta} = O(|\textbf{\textit{log}}\ \chi|) = O(\mathbb{E}\left[\tau_{\chi}({\Theta})\right]).
\end{equation}
Furthermore, to prove \eqref{eq:thm_result}, it is sufficient to show
\begin{equation}
   \mathbb{P}(S\leq\tau_{\chi,\delta}(\Theta))=o(1). 
\end{equation}
The next step is to establish an upper bound for $\mathbb{P}(S\leq\tau_{\chi,\delta}(\Theta))$. Given a full ranking list $\boldsymbol{\pi}_0$, we write 
\begin{equation}
    \boldsymbol{\Theta}_{\boldsymbol{\pi}_0} = \{\boldsymbol{\theta}|\boldsymbol{\pi}(\boldsymbol{\theta})=\boldsymbol{\pi}_0\}
\end{equation}
as all preference score which will generate the full ranking $\boldsymbol{\pi}_0$. Then 
\begin{equation}
    \label{eq:129}
    \begin{aligned}
       & \mathbb{P}(S\leq\tau_{\chi,\delta}(\Theta)) &=&\ \ \sum_{\boldsymbol{\pi}_0}\ \mathbb{P}(S\leq\tau_{\chi,\delta}(\Theta),\Theta\in\boldsymbol{\Theta}_{\boldsymbol{\pi}_0})\\
       & &=&\ \ O(1)\cdot\underset{\boldsymbol{\pi}_0}{\textbf{\textit{max}}}\ \mathbb{P}(S\leq\tau_{\chi,\delta}(\Theta),\Theta\in\boldsymbol{\Theta}_{\boldsymbol{\pi}_0}).
    \end{aligned}
\end{equation}
Then we conduct an upper bound of $\mathbb{P}(S\leq\tau_{\chi,\delta}(\Theta),\Theta\in\boldsymbol{\Theta}_{\boldsymbol{\pi}_0})$ for any $\boldsymbol{\pi}_0$. Define an event $\boldsymbol{E}_{\boldsymbol{\pi}_0}$: 
\begin{equation}
    \boldsymbol{E}_{\boldsymbol{\pi}_0} = \left\{\frac{\mathbb{P}(\Theta\in\boldsymbol{\Theta}_{\boldsymbol{\pi}_0}|\mathcal{F}_S)}{\underset{(i,j):\boldsymbol{\Theta}_{i,j}\cap\boldsymbol{\Theta}_{\boldsymbol{\pi}_0}=\varnothing}{\textbf{\textit{max}}}\ \mathbb{P}(\Theta\in\boldsymbol{\Theta}_{i,j}|\mathcal{F}_S)}>\frac{\chi^{\frac{\delta}{10}}}{\epsilon}\right\},
\end{equation}
where $\mathcal{F}_S=\sigma(c_1,\dots,c_S)$ denotes the $\sigma$-algebra generated by $c_1,\dots,c_S$, $\epsilon>0$ is a constant and the definition of $\boldsymbol{\Theta}_{i,j}$ is
\begin{equation}
    \boldsymbol{\Theta}_{i,j} = \big\{\boldsymbol{\ \theta}\in\mathbb{R}^n_+\ |\ \theta_i\geq\theta_j\ \big\}\cap\textbf{\textit{Supp}}(\rho_{\boldsymbol{\theta}'}).\tag{\ref{eq:theta_ij}}
\end{equation}
For each $\mathbb{P}(S\leq\tau_{\chi,\delta}(\Theta),\Theta\in\boldsymbol{\Theta}_{\boldsymbol{\pi}_0})$, we have
\begin{equation}
    \label{eq:131}
    \begin{aligned}
        & & &\ \ \mathbb{P}(S\leq\tau_{\chi,\delta}(\Theta),\Theta\in\boldsymbol{\Theta}_{\boldsymbol{\pi}_0})\\[5pt]
        & &=&\ \ \mathbb{P}(S\leq\tau_{\chi,\delta}(\Theta),\Theta\in\boldsymbol{\Theta}_{\boldsymbol{\pi}_0}, \boldsymbol{E}_{\boldsymbol{\pi}_0}) + \mathbb{P}(S\leq\tau_{\chi,\delta}(\Theta),\Theta\in\boldsymbol{\Theta}_{\boldsymbol{\pi}_0}, \boldsymbol{E}^c_{\boldsymbol{\pi}_0})
    \end{aligned}
\end{equation}
where $\boldsymbol{E}^c_{\boldsymbol{\pi}_0}$ is the complement of $\boldsymbol{E}_{\boldsymbol{\pi}_0}$. Then $\mathbb{P}(S\leq\tau_{\chi,\delta}(\Theta),\Theta\in\boldsymbol{\Theta}_{\boldsymbol{\pi}_0})$ can be further bounded by
\begin{equation}
    \label{eq:132}
    \begin{aligned}
        & & &\ \ \mathbb{P}(S\leq\tau_{\chi,\delta}(\Theta),\Theta\in\boldsymbol{\Theta}_{\boldsymbol{\pi}_0})\\[5pt]
        & &\leq&\ \ \mathbb{P}(S\leq\tau_{\chi,\delta}(\Theta),\Theta\in\boldsymbol{\Theta}_{\boldsymbol{\pi}_0}, \boldsymbol{E}_{\boldsymbol{\pi}_0}) + \mathbb{P}(\boldsymbol{\theta}\in\boldsymbol{\Theta}_{\boldsymbol{\pi}_0}, \boldsymbol{E}^c_{\boldsymbol{\pi}_0}).
    \end{aligned}
\end{equation}
The first term $\mathbb{P}(S\leq\tau_{\chi,\delta}(\Theta),\Theta\in\boldsymbol{\Theta}_{\boldsymbol{\pi}_0}, \boldsymbol{E}_{\boldsymbol{\pi}_0})$ equals to 
\begin{equation}
    \label{eq:133}
    \mathbb{P}(S\leq\tau_{\chi,\delta}(\Theta),\Theta\in\boldsymbol{\Theta}_{\boldsymbol{\pi}_0}, \boldsymbol{E}_{\boldsymbol{\pi}_0}) = \int_{\boldsymbol{\Theta}_{\boldsymbol{\pi}_0}} \mathbb{P}(S\leq\tau_{\chi,\delta}(\boldsymbol{\theta}),\boldsymbol{E}_{\boldsymbol{\pi}_0}|\Theta=\boldsymbol{\theta})\rho_{\boldsymbol{\theta}'}(\boldsymbol{\theta})d\boldsymbol{\theta}.
\end{equation}
Notice that the intersection of $\{S\leq\tau_{\chi,\delta}(\boldsymbol{\theta})\}$ and $\boldsymbol{E}_{\boldsymbol{\pi}_0}$ will be
\begin{equation}
    \{S\leq\tau_{\chi,\delta}(\boldsymbol{\theta})\}\cap\boldsymbol{E}_{\boldsymbol{\pi}_0}\subset\left\{\underset{1\leq S\leq \tau_{\chi,\delta}(\boldsymbol{\theta})}{\textbf{\textit{max}}}\frac{\mathbb{P}(\Theta\in\boldsymbol{\Theta}_{\boldsymbol{\pi}_0}|\mathcal{F}_S)}{\underset{(i,j):\boldsymbol{\Theta}_{i,j}\cap\boldsymbol{\Theta}_{\boldsymbol{\pi}_0}=\varnothing}{\textbf{\textit{max}}}\ \mathbb{P}(\Theta\in\boldsymbol{\Theta}_{i,j}|\mathcal{F}_S)}>\frac{\chi^{\frac{\delta}{10}}}{\epsilon}\right\}.
\end{equation}
Then
\begin{equation}
    \label{eq:event_prob}
    \mathbb{P}(S\leq\tau_{\chi,\delta}(\boldsymbol{\theta}),\boldsymbol{E}_{\boldsymbol{\pi}_0}|\Theta=\boldsymbol{\theta})\leq\mathbb{P}\left(\underset{1\leq S\leq \tau_{\chi,\delta}(\boldsymbol{\theta})}{\textbf{\textit{max}}}\frac{\mathbb{P}(\Theta\in\boldsymbol{\Theta}_{\boldsymbol{\pi}_0}|\mathcal{F}_S)}{\underset{(i,j):\boldsymbol{\Theta}_{i,j}\cap\boldsymbol{\Theta}_{\boldsymbol{\pi}_0}=\varnothing}{\textbf{\textit{max}}}\ \mathbb{P}(\Theta\in\boldsymbol{\Theta}_{i,j}|\mathcal{F}_S)}>\frac{\chi^{\frac{\delta}{10}}}{\epsilon}\Bigg|\Theta = \boldsymbol{\theta}\right).
\end{equation}
\ul{The next step is to transform the right-side of the inequality to a probability by a martingale parameterized by $\boldsymbol{\theta}$. The corresponding results are represented by the Lemma 3 in} \cite{doi:10.1287/moor.2021.1209}. 
\begin{lemma}
    \label{lemma:martingale}
    Suppose that $\mathcal{M}_S(\boldsymbol{\theta}'')$ is a martingale \textit{w.r.t} the filtration $\{\mathcal{F}_S:S\geq 1\}$ and probability measure $\mathbb{P}(\cdot|\Theta=\boldsymbol{\theta})$ for any $\boldsymbol{\theta}''\in\boldsymbol{\Theta}_{\boldsymbol{\pi}_0}$,
    \begin{equation}
         \mathcal{M}_S(\boldsymbol{\theta}'') = \ell_S(\boldsymbol{\theta}'')-\ell_S(\boldsymbol{\theta}^*_S)-\sum_{s=1}^S\sum_{(i,j)}\lambda_{i,j}^{(s)}g_{i,j}(\boldsymbol{\theta})\cdot\textbf{\textit{log}}\frac{g_{i,j}(\boldsymbol{\theta})}{g_{i,j}(\boldsymbol{\theta}^*_S)}+\sum_{s=1}^S\sum_{(i,j)}\lambda_{i,j}^{(s)}g_{i,j}(\boldsymbol{\theta})\cdot\textbf{\textit{log}}\frac{g_{i,j}(\boldsymbol{\theta})}{g_{i,j}(\boldsymbol{\theta}'')},
     \end{equation}
     where 
     \begin{equation}
        \label{eq:log-likelihood}
        \ell_S(\boldsymbol{\theta}) = \textbf{\textit{log}}\ \prod_{s=1}^S g_{ij}(\boldsymbol{\theta}),
     \end{equation}
     and
     \begin{equation}
         \boldsymbol{\theta}^*_S\in\underset{\begin{matrix}\scriptstyle\boldsymbol{\tilde{\theta}}\in\textbf{\textit{Supp}}(\rho_{\boldsymbol{\theta}'}),\\ \scriptstyle \boldsymbol{\pi}(\boldsymbol{\theta})\neq\boldsymbol{\pi}(\boldsymbol{\tilde{\theta}})\end{matrix}}{\textbf{\textit{arg\ min}}}\ \sum_{s=1}^S\sum_{(i,j)}\lambda^{(s)}_{i,j}g_{i,j}(\boldsymbol{\theta})\textbf{\textit{log}}\frac{g_{i,j}(\boldsymbol{\theta})}{g_{i,j}(\tilde{\boldsymbol{\theta}})}.
     \end{equation}
     There exists a $\chi_0>0$ such that when $\chi$ satisfies $0<\chi<\chi_0$, it holds that
     \begin{equation}
        \label{eq:lemma_3_result}
         \begin{aligned}
             & & &\ \ \mathbb{P}\left(\underset{1\leq S\leq \tau_{\chi,\delta}(\boldsymbol{\theta})}{\textbf{\textit{max}}}\frac{\mathbb{P}(\Theta\in\boldsymbol{\Theta}_{\boldsymbol{\pi}_0}|\mathcal{F}_S)}{\underset{(i,j):\boldsymbol{\Theta}_{i,j}\cap\boldsymbol{\Theta}_{\boldsymbol{\pi}_0}=\varnothing}{\textbf{\textit{max}}}\ \mathbb{P}(\Theta\in\boldsymbol{\Theta}_{i,j}|\mathcal{F}_S)}>\frac{\chi^{\frac{\delta}{10}}}{\epsilon}\ \Bigg|\ \Theta = \boldsymbol{\theta}\right)\\
             & &\leq&\ \ \mathbb{P}\left(\underset{\begin{matrix}\scriptstyle 1\leq S\leq \tau_{\chi,\delta}(\boldsymbol{\theta}),\\ \scriptstyle \boldsymbol{\theta}''\in\boldsymbol{\Theta}_{\boldsymbol{\pi}_0}\end{matrix}}{\textbf{\textit{max}}}\mathcal{M}_S(\boldsymbol{\theta}'')\geq\frac{\delta}{2}|\textbf{\textit{log}}\ \chi|\ \Bigg|\ \Theta = \boldsymbol{\theta}\right).
         \end{aligned}
     \end{equation}
\end{lemma}
With Lemma \ref{lemma:martingale}, we know that establishing the upper bound of \eqref{eq:event_prob} equals to find an upper bound of the right hand side of \eqref{eq:lemma_3_result}. It is noteworthy that the right hand side of \eqref{eq:lemma_3_result} is a stochastic process, indexed by $\boldsymbol{\theta}''$ and $S$, going beyond a certain level ($\delta|\textbf{\textit{log}}\ \chi|/2$). 

\ul{To handle the level-crossing probabilities, we introduce the Azuma-Hoeffding inequality} \cite{10.2748/tmj/1178243286,10.2307/2282952} \ul{and derive the level-crossing probability by aggregating marginal tail bounds of a random field by Lemma} \ref{lemma:Azuma_Hoeffding} \ul{and} \ref{lemma:aggre_marginal_tails}.

\begin{lemma}[Azuma-Hoeffding Inequality]
    \label{lemma:Azuma_Hoeffding}
    Let $\mathcal{M}_S$ be a martingale \textit{w.r.t.} the filtration $\{\mathcal{F}_S:S\geq 1\}$ and $\Delta \mathcal{M}_S= \mathcal{M}_S-\mathcal{M}_{S-1}$. Suppose that $\Delta \mathcal{M}_S\in[a_S,b_S]$ where $a_S$ and $b_S$ are deterministic constants, for each $S>0$, we have
    \begin{equation}
        \mathbb{P}\left(\underset{1\leq s\leq S}{\textbf{\textit{max}}}\ \mathcal{M}_s\geq S_0\right)\leq\textbf{\textit{exp}}\left(-\frac{2S_0^2}{\sum_{s=1}^{S}(b_s-a_s)^2}\right).
    \end{equation}
\end{lemma}

\begin{lemma}
    \label{lemma:aggre_marginal_tails}
    Let $\{\zeta(\boldsymbol{\theta}):\boldsymbol{\theta}\in\boldsymbol{\Theta}\}$ be a random filed over a compact set $\boldsymbol{\Theta}\subset\mathbb{R}^n$ that satisfies Assumption \ref{assmp:compat}, where $\zeta(\cdot)$ has a continuous sample path almost surely under a probability measure $\mathbb{P}$. Moreover, $\zeta(\cdot)$ has a Lipschitz-continuous sample path in the sense that there exists a constant $\kappa$ such that for all $\boldsymbol{\theta},\boldsymbol{\theta}'\in\boldsymbol{\Theta}$, 
    \begin{equation}
        \big|\zeta(\boldsymbol{\theta})-\zeta(\boldsymbol{\theta}')\big|\leq\kappa\|\boldsymbol{\theta}-\boldsymbol{\theta}'\|
    \end{equation}
    almost surely under $\mathbb{P}$. We define 
    \begin{equation}
        \beta(\boldsymbol{\theta}, b) = \mathbb{P}\big(\zeta(\boldsymbol{\theta})\geq b\big),
    \end{equation}
    and for all $\gamma>0$, it holds that 
    \begin{equation}
        \mathbb{P}\left(\underset{\boldsymbol{\theta}\in\boldsymbol{\Theta}}{\textbf{\textit{max}}}\ \zeta(\boldsymbol{\theta})\geq b\right)\leq\frac{\kappa^{n-1}_L}{\gamma^{n-1}\delta_b}\times\int_{\boldsymbol{\Theta}}\beta(\boldsymbol{\theta}, b-\gamma)d\boldsymbol{\theta},
    \end{equation}
    where $\delta_b$ is the constant in Assumption \ref{assmp:compat}.
\end{lemma}

With Lemma \ref{lemma:Azuma_Hoeffding}, we set 
\begin{equation}
    \begin{aligned}
        & \mathbb{P} &=&\ \ \mathbb{P}(\cdot|\Theta=\boldsymbol{\theta}),\\[5pt]
        & S   &=&\ \ \tau_{\chi,\delta}(\boldsymbol{\theta}),\\[5pt]
        & S_0 &=&\ \ \frac{\delta}{2}|\textbf{\textit{log}}\ \chi| - 1,\\[5pt]
        & \mathcal{M}_{S} &=&\ \  \mathcal{M}_{S}(\boldsymbol{\theta}''),\\[5pt]
        & a_S\ \ =\ \ b_S&=&\ \ 2\ \underset{\begin{matrix}\scriptstyle\boldsymbol{\theta}\in\textbf{\textit{Supp}}(\rho_{\boldsymbol{\theta}'}),\\\scriptstyle(i,j)\in\mathfrak{C}\end{matrix}}{\textbf{\textit{max}}}|\textbf{\textit{log}}\ g_{i,j}(\boldsymbol{\theta})|,
    \end{aligned}
\end{equation}
and for each $\boldsymbol{\theta}''$, we have 
\begin{equation}
    \mathbb{P}\left(\underset{1\leq S\leq \tau_{\chi,\delta}(\boldsymbol{\theta})}{\textbf{\textit{max}}}\ \mathcal{M}_{S}(\boldsymbol{\theta}'')\geq \frac{\delta}{2}|\textbf{\textit{log}}\ \chi| - 1\ \Bigg|\ \Theta = \boldsymbol{\theta}\right)\leq\textbf{\textit{exp}}\left(-\frac{2\left( \frac{\delta}{2}|\textbf{\textit{log}}\ \chi| - 1\right)^2}{\tau_{\chi,\delta}(\boldsymbol{\theta})\cdot a_1^2}\right).
\end{equation}
According to Assumption \ref{assmp:support} and \ref{assmp:continuous}, $a_1<\infty$ stands and
\begin{equation}
    \label{eq:146}
     \mathbb{P}\left(\underset{1\leq S\leq \tau_{\chi,\delta}(\boldsymbol{\theta})}{\textbf{\textit{max}}}\ \mathcal{M}_{S}(\boldsymbol{\theta}'')\geq \frac{\delta}{2}|\textbf{\textit{log}}\ \chi| - 1\ \Bigg|\ \Theta = \boldsymbol{\theta}\right)\leq\textbf{\textit{exp}}\left(-\Omega(\delta^2|\textbf{\textit{log}}\ \chi|)\right),
\end{equation} 
where $O(\cdot)$ is the infinitesimal of the same order. Notice that 
\begin{equation}
    \begin{aligned}
        & & &\ \ \underset{1\leq S\leq \tau_{\chi,\delta}(\boldsymbol{\theta})}{\textbf{\textit{max}}}\ \mathcal{M}_{S}(\boldsymbol{\theta}'') - \underset{1\leq S\leq \tau_{\chi,\delta}(\boldsymbol{\theta})}{\textbf{\textit{max}}}\ \mathcal{M}_{S}(\boldsymbol{\theta}''')\\[5pt]
        & &\leq&\ \ \underset{1\leq S\leq \tau_{\chi,\delta}(\boldsymbol{\theta})}{\textbf{\textit{max}}}\ |\mathcal{M}_{S}(\boldsymbol{\theta}'')-\mathcal{M}_{S}(\boldsymbol{\theta}''')|\\[5pt]
        & &\leq&\ \ \tau_{\chi,\delta}(\boldsymbol{\theta})\cdot\kappa_0\cdot\|\boldsymbol{\theta}''-\boldsymbol{\theta}'''\|,\ \ \ \forall\ \ \boldsymbol{\theta}'',\boldsymbol{\theta}'''\in\boldsymbol{\Theta}_{\boldsymbol{\pi}_0},
    \end{aligned}
\end{equation}
where 
\begin{equation}
    \kappa_0 = 4\ \underset{\begin{matrix}\scriptstyle\boldsymbol{\theta}\in\textbf{\textit{Supp}}(\rho_{\boldsymbol{\theta}'}),\\\scriptstyle(i,j)\in\mathfrak{C}\end{matrix}}{\textbf{\textit{max}}}\ |\nabla \textbf{\textit{log}}\ g_{ij}(\boldsymbol{\theta})|<\infty
\end{equation}
is the Lipschitz constant of $\mathcal{M}_1(\boldsymbol{\theta}'')$. Consequently, $\mathcal{M}_{S}(\boldsymbol{\theta}')$ is a Lipschitz-continuous random field \textit{w.r.t.} $\boldsymbol{\theta}''\in\boldsymbol{\Theta}_{\boldsymbol{\pi}_0}$. Then adding the constraint $\boldsymbol{\theta}''\in\boldsymbol{\Theta}_{\boldsymbol{\pi}_0}$ into \eqref{eq:146} and adopting Lemma \ref{lemma:aggre_marginal_tails} give
\begin{equation}
    \label{eq:149}
    \begin{aligned}
        & & &\ \ \mathbb{P}\left(\underset{\begin{matrix}\scriptstyle1\leq S\leq \tau_{\chi,\delta}(\boldsymbol{\theta}),\\\scriptstyle\boldsymbol{\theta}''\in\boldsymbol{\Theta}_{\boldsymbol{\pi}_0}\end{matrix}}{\textbf{\textit{max}}}\ \mathcal{M}_{S}(\boldsymbol{\theta}'')\geq \frac{\delta}{2}|\textbf{\textit{log}}\ \chi| - 1\ \Bigg|\ \Theta = \boldsymbol{\theta}\right)\\[5pt]
        & &\leq&\ \ \textbf{\textit{exp}}\left(-\Omega(\delta^2|\textbf{\textit{log}}\ \chi|)\right)\cdot\mathcal{L}(\boldsymbol{\Theta}_{\boldsymbol{\pi}_0})\cdot\frac{\tau_{\chi,\delta}(\boldsymbol{\theta})^{n-1}\kappa^{n-1}_0}{\delta_b}\\[5pt]
        & &=&\ \ \textbf{\textit{exp}}\left(-\Omega(\delta^2|\textbf{\textit{log}}\ \chi|)\right)\cdot O(|\textbf{\textit{log}}\ \chi|^{n-1}).
    \end{aligned}
\end{equation}

Combine \eqref{eq:133}, \eqref{eq:event_prob}, \eqref{eq:lemma_3_result} and \eqref{eq:149}, we have
\begin{equation}
    \mathbb{P}(S\leq\tau_{\chi,\delta}(\Theta),\Theta\in\boldsymbol{\Theta}_{\boldsymbol{\pi}_0}, \boldsymbol{E}_{\boldsymbol{\pi}_0})\leq\textbf{\textit{exp}}\left(-\Omega(\delta^2|\textbf{\textit{log}}\ \chi|)\right)\cdot O(|\textbf{\textit{log}}\ \chi|^{n-1}).
\end{equation}

\ul{we have established a upper bound for first term in the right-hand side of} \eqref{eq:132}. \ul{To bound the second term, we adopt the following lemma} \cite{doi:10.1287/moor.2021.1209}.
\begin{lemma}
\label{lemma:7}
For every full ranking involving $n$ candidates $\boldsymbol{\pi}_0$ generated by $(\boldsymbol{\Lambda}, S)$, the corresponding risk is 
\begin{equation}
    \mathfrak{R}_{\Theta}(\boldsymbol{R}(\boldsymbol{\Lambda}, S))=\underset{(i,j)\in\boldsymbol{\mathfrak{C}}}{\sum}\mathbbm{I}[\theta'_i<\theta'_j]r_{i,j}+\mathbbm{I}[\theta'_i>\theta'_j](1-r_{i,j}). \tag{\ref{eq:kendall_risk}}
\end{equation}
If 
\begin{equation}
    \mathbb{E}[\mathfrak{R}(\boldsymbol{R}(\boldsymbol{\Lambda}, S))]\leq \epsilon,
\end{equation}
we have
\begin{equation}
    \mathbb{P}(\Theta\in\boldsymbol{\Theta}_{\boldsymbol{\pi}_0},\boldsymbol{E}^c_{\boldsymbol{\pi}_0})\leq\left(1+\frac{\chi^{\frac{\delta}{10}}}{\epsilon}\right)\epsilon.
\end{equation}
\end{lemma}
Consequently, \eqref{eq:132} will be bounded by
\begin{equation}
    \mathbb{P}(S\leq\tau_{\chi,\delta}(\Theta),\Theta\in\boldsymbol{\Theta}_{\boldsymbol{\pi}_0})\leq\textbf{\textit{exp}}\left(-\Omega(\delta^2|\textbf{\textit{log}}\ \chi|)\right)\cdot O(|\textbf{\textit{log}}\ \chi|^{n-1})+\left(1+\frac{\chi^{\frac{\delta}{10}}}{\epsilon}\right)\epsilon.
\end{equation}
Back to \eqref{eq:129}, it holds that
\begin{equation}
    \mathbb{P}(S\leq\tau_{\chi,\delta}(\Theta)\leq O(1)\times\left\{\textbf{\textit{exp}}\left(-\Omega(\delta^2|\textbf{\textit{log}}\ \chi|)\right)\cdot O(|\textbf{\textit{log}}\ \chi|^{n-1})+\left(1+\frac{\chi^{\frac{\delta}{10}}}{\epsilon}\right)\epsilon\right\}.
\end{equation}
Therefore, when $\chi\rightarrow0$
\begin{equation}
    \mathbb{P}(S\leq\tau_{\chi,\delta}(\Theta) = o(1).
\end{equation}
We finish the proof.
\end{proof}

By the definition of asymptotic optimality of manipulation policy $(\boldsymbol{\Lambda}, S)$: 
\begin{equation}
    \underset{\chi\rightarrow 0}{\textbf{\textit{inf}}}\ \frac{\mathfrak{R}(\boldsymbol{\Lambda}, S)}{\mathfrak{R}^*} = 1, \tag{\ref{eq:asy_opt}}
\end{equation}
we know that $(\boldsymbol{\Lambda}, S)$ is asymptotic optimal when $\chi\rightarrow0$ if 
\begin{equation}
    \mathfrak{R}(\boldsymbol{\Lambda}, S)=(1+o(1))\cdot\mathfrak{R}^*. 
\end{equation}
Theorem \ref{thm:asymptotic2expectation} tells us that 
\begin{equation}
    \mathfrak{R}(\boldsymbol{\Lambda}, S)=(1+o(1))\cdot\mathfrak{R}^* = (1+o(1))\chi\mathbb{E}_{\boldsymbol{\theta}\sim\textbf{\textit{Supp}}(\rho_{\boldsymbol{\theta}'})}\left[\tau_{\chi}(\boldsymbol{\theta})\right]
\end{equation}
as $\chi\rightarrow0$ is sufficient to show the asymptotic optimality of $(\boldsymbol{\Lambda}, S)$. To prove the asymptotic optimality with complete knowledge of the proposed stopping time \eqref{eq:stopping_time} and the generation rule solved by \eqref{opt:generation_rule}, we require the identifiability of model to adopt the \textbf{MLE} and obtain the preference score. It is worth noting that the \textbf{BTL} model could satisfy this requirement. 

\ul{The following theorem provides the asymptotic upper bounds for the expected Kendall tau of the proposed manipulation policy with complete information.}

\begin{theorem}
    \label{thm:asymptotic_kendall_tau}
    Consider the proposed stopping time \eqref{eq:stopping_time} and the generation rule solved by \eqref{opt:generation_rule} with complete knowledge, when the exploration probability $p\in(0,1)$ will be 
    \begin{equation}
        p\propto |\textbf{\textit{log}} \chi|^{-\frac{1}{2}+\delta_0}
    \end{equation}
    for some $\delta_0\in\big(0,\frac{1}{2}\big)$. Then we have
    \begin{equation}
        \mathbb{E}[\mathfrak{R}(\boldsymbol{R}(\boldsymbol{\Lambda}, S))]=O(\chi).
    \end{equation}
\end{theorem}

\begin{proof}
    We first discuss the second stopping time $S_2$. By the generation rule solved through \eqref{opt:generation_rule}, the expected Kendall tau at $S_2$ is
    \begin{equation}
        \begin{aligned}
            & \mathbb{E}[\mathfrak{R}(\boldsymbol{R}(\boldsymbol{\Lambda}, S_2))] &=&\ \ \mathbb{E}\underset{(i,j)\in\boldsymbol{\mathfrak{C}}}{\sum}\mathbbm{I}[\theta'_i<\theta'_j]r_{i,j}\\[5pt]
            & &=&\ \ \int_{\textbf{\textit{Supp}}(\rho_{\boldsymbol{\theta}'})}\underset{(i,j):\boldsymbol{\theta}\in\boldsymbol{\Theta}_{i,j}}{\sum}\mathbb{P}\left\{\underset{\boldsymbol{\theta}_1\in\boldsymbol{\Theta}_{i,j}}{\textbf{\textit{sup}}}\ \ell_{S_2}(\boldsymbol{\theta}_1)>\underset{\boldsymbol{\theta}_2\in\boldsymbol{\Theta}_{j,i}}{\textbf{\textit{sup}}}\ \ell_{S_2}(\boldsymbol{\theta}_2)\ \Bigg|\ \Theta=\boldsymbol{\theta}\right\}\rho_{\boldsymbol{\theta}'}(\boldsymbol{\theta})d\boldsymbol{\theta}\\[5pt]
            & &=&\ \ \int_{\textbf{\textit{Supp}}(\rho_{\boldsymbol{\theta}'})}\underset{\boldsymbol{\theta}\in\boldsymbol{\Theta}_{i,j}}{\sum}\mathbb{P}\left\{\underset{\boldsymbol{\theta}\in\boldsymbol{\Theta}_{i,j}}{\textbf{\textit{sup}}}\ \ell_{S_2}(\boldsymbol{\theta}_1)-\underset{\boldsymbol{\theta}\in\boldsymbol{\Theta}_{j,i}}{\textbf{\textit{sup}}}\ \ell_{S_2}(\boldsymbol{\theta}_2)>z_{\alpha}(\chi)\ \Bigg|\ \Theta=\boldsymbol{\theta}\right\}\rho_{\boldsymbol{\theta}'}(\boldsymbol{\theta})d\boldsymbol{\theta}
        \end{aligned}
    \end{equation}
    where $\boldsymbol{R}(\boldsymbol{\Lambda}, S)$, $r_{i,j}$, $\boldsymbol{\Theta}_{i,j}$, $z_{\alpha}(\chi)$, $\ell_{S}(\boldsymbol{\theta})$ are defined as \eqref{eq:decision_variable}, \eqref{eq:r_ij}, \eqref{eq:theta_ij}, \eqref{eq:time_threshd} and \eqref{eq:log-likelihood} correspondingly. The above equation can be further bounded by 
    \begin{equation}
        \label{eq:161}
        \begin{aligned}
            & & &\ \ \mathbb{E}[\mathfrak{R}(\boldsymbol{R}(\boldsymbol{\Lambda}, S_2))]\\[5pt] 
            & &\leq&\ \ \frac{n(n-1)}{2}\cdot\mathcal{L}\big(\textbf{\textit{Supp}}(\rho_{\boldsymbol{\theta}'})\big)\times\underset{\boldsymbol{\theta}\in\textbf{\textit{Supp}}(\rho_{\boldsymbol{\theta}'})}{\textbf{\textit{sup}}}\ \rho_{\boldsymbol{\theta}'}(\boldsymbol{\theta})\\[5pt]
            & & &\ \ \times\underset{\boldsymbol{\theta}\in\textbf{\textit{Supp}}(\rho_{\boldsymbol{\theta}'})}{\textbf{\textit{sup}}}\ \underset{(i,j):\boldsymbol{\theta}\in\boldsymbol{\Theta}_{i,j}}{\textbf{\textit{max\phantom{p}}}}\ \mathbb{P}\left\{\underset{\boldsymbol{\theta}''\in\boldsymbol{\Theta}_{i,j}}{\textbf{\textit{sup}}}\ \ell_{S_2}(\boldsymbol{\theta}'')-\ell_{S_2}(\boldsymbol{\theta})>z_{\alpha}(\chi)\ \Bigg|\ \Theta=\boldsymbol{\theta}\right\}.
        \end{aligned}
    \end{equation}
    Here the inequality holds as 
    \begin{equation}
        \underset{\boldsymbol{\theta}_1\in\boldsymbol{\Theta}_{i,j}}{\textbf{\textit{sup}}} \ell_{S_2}(\boldsymbol{\theta}_1)\geq \ell_{S_2}(\boldsymbol{\theta}_2),\ \ \text{if}\ \ \boldsymbol{\theta}_2\notin\boldsymbol{\Theta}_{i,j}
    \end{equation}
    and
    \begin{equation}
        \underset{\boldsymbol{\theta}\in\textbf{\textit{Supp}}(\rho_{\boldsymbol{\theta}'})}{\textbf{\textit{sup}}}\ \rho_{\boldsymbol{\theta}'}(\boldsymbol{\theta})<\infty
    \end{equation}
    by Assumption \eqref{assmp:positive_density}. The last part of \eqref{eq:161} can be decomposed as
    \begin{equation}
        \label{eq:164}
        \begin{aligned}
            & & &\ \ \mathbb{P}\left\{\underset{\boldsymbol{\theta}''\in\boldsymbol{\Theta}_{i,j}}{\textbf{\textit{sup}}}\ \ell_{S_2}(\boldsymbol{\theta}'')-\ell_{S_2}(\boldsymbol{\theta})>z_{\alpha}(\chi)\ \Bigg|\ \Theta=\boldsymbol{\theta}\right\}\\[5pt]
            & &\leq&\ \ \mathbb{P}\left\{\underset{\boldsymbol{\theta}''\in\boldsymbol{\Theta}_{i,j}}{\textbf{\textit{sup}}}\ \ell_{S_2}(\boldsymbol{\theta}'')-\ell_{S_2}(\boldsymbol{\theta})>z_{\alpha}(\chi)\ \text{and}\ S_2\leq\tau\ \Bigg|\ \Theta=\boldsymbol{\theta}\right\} + \mathbb{P}\left\{S_2\geq\tau\ \Bigg|\ \Theta=\boldsymbol{\theta}\right\}.
        \end{aligned}
    \end{equation}
    For the first term of the above right-hand side, we introduce $S_2\wedge\tau = \textbf{\textit{min}}(S_2,\tau)$ and write
    \begin{equation}
        \begin{aligned}
            & & &\ \ \mathbb{P}\left\{\underset{\boldsymbol{\theta}''\in\boldsymbol{\Theta}_{i,j}}{\textbf{\textit{sup}}}\ \ell_{S_2}(\boldsymbol{\theta}'')-\ell_{S_2}(\boldsymbol{\theta})>z_{\alpha}(\chi)\ \text{and}\ S_2\leq\tau\ \Bigg|\ \Theta=\boldsymbol{\theta}\right\}\\[5pt]
            & &\leq&\ \ \mathbb{P}\left\{\underset{\boldsymbol{\theta}''\in\boldsymbol{\Theta}_{i,j}}{\textbf{\textit{sup}}}\ \ell_{S_2\wedge\tau}(\boldsymbol{\theta}'')-\ell_{S_2\wedge\tau}(\boldsymbol{\theta})>z_{\alpha}(\chi)\ \Bigg|\ \Theta=\boldsymbol{\theta}\right\}
        \end{aligned}
    \end{equation}
    Let $\eta(\boldsymbol{\theta}'')$ be a random field
    \begin{equation}
        \eta(\boldsymbol{\theta}'') = \ell_{S_2\wedge\tau}(\boldsymbol{\theta}'')-\ell_{S_2\wedge\tau}(\boldsymbol{\theta}),\ \ \forall\ \ \boldsymbol{\theta}''\in\boldsymbol{\Theta}_{i,j}.
    \end{equation}
    \ul{The marginal tail probability of $\eta(\boldsymbol{\theta}'')$ can be obtained by the following lemma} \cite{doi:10.1287/moor.2021.1209}. 
    \begin{lemma}
        \label{lemma:8}
        For all $\boldsymbol{\theta}''\neq\boldsymbol{\theta}$ and constant $L>0$, we have
        \begin{equation}
            \mathbb{P}\left\{\underset{\boldsymbol{\theta}''\in\boldsymbol{\Theta}_{i,j}}{\textbf{\textit{sup}}}\ \ell_{S\wedge\tau}(\boldsymbol{\theta}'')-\ell_{S\wedge\tau}(\boldsymbol{\theta})>L\ \Bigg|\ \Theta=\boldsymbol{\theta}\right\}\leq\textbf{\textit{exp}}(-L).
        \end{equation}
    \end{lemma}
    We can take $L = z_{\alpha}(\chi) - 1$ and obtain
    \begin{equation}
        \mathbb{P}\left\{\eta(\boldsymbol{\theta}'')>z_{\alpha}(\chi)\ \Bigg|\ \Theta=\boldsymbol{\theta}\right\}\leq\textbf{\textit{exp}}(-z_{\alpha}(\chi)+1).
    \end{equation}
    Moreover, $\eta(\boldsymbol{\theta}'')$ is a Lipschitz-continuous function as 
    \begin{equation}
        \eta(\boldsymbol{\theta}'')-\eta(\boldsymbol{\theta}''')\leq|\ell_{S_2\wedge\tau}(\boldsymbol{\theta}'')-\ell_{S_2\wedge\tau}(\boldsymbol{\theta}''')|\leq\tau\cdot\kappa_0\|\boldsymbol{\theta}''-\boldsymbol{\theta}'''\|.
    \end{equation}
    Combining Lemma \ref{lemma:aggre_marginal_tails} and \ref{lemma:7}, we arrive at
    \begin{equation}
        \label{eq:170}
        \mathbb{P}\left\{\underset{\boldsymbol{\theta}''\in\boldsymbol{\Theta}_{i,j}}{\textbf{\textit{sup}}}\ \eta(\boldsymbol{\theta}'')>z_{\alpha}(\chi)\ \Bigg|\ \Theta=\boldsymbol{\theta}\right\}\leq O\Big(\tau^{n-1}\textbf{\textit{exp}}\big(-z_{\alpha}(\chi)\big)\Big).
    \end{equation}
    \ul{To bound} $\mathbb{P}\left\{S_2\geq\tau\ \Bigg|\ \Theta=\boldsymbol{\theta}\right\}$, \ul{we need the following lemma.}
    \begin{lemma}
        \label{lemma:9}
        When $\tau$ satisfies 
        \begin{equation}
            \tau = \Omega(|\textbf{\textit{log}}\ \chi|^3),
        \end{equation}
        we have
        \begin{equation}
            \mathbb{P}\left\{S_2\geq\tau\ \Bigg|\ \Theta=\boldsymbol{\theta}\right\}\leq \chi^2. 
        \end{equation}
        For the stopping time $S_1$, we have the similar result:
        \begin{equation}
            \mathbb{P}\left\{S_2\geq\tau\ \Bigg|\ \Theta=\boldsymbol{\theta}\right\}\leq \chi^2,\ \ \text{if}\ \tau = \Omega(|\textbf{\textit{log}}\ \chi|^3). 
        \end{equation}
    \end{lemma}
    Combining \eqref{eq:164} with \eqref{eq:170} and Lemma \ref{lemma:8}, we have
    \begin{equation}
        \begin{aligned}
            & & &\ \ \mathbb{P}\left\{\underset{\boldsymbol{\theta}''\in\boldsymbol{\Theta}_{i,j}}{\textbf{\textit{sup}}}\ \ell_{S_2}(\boldsymbol{\theta}'')-\ell_{S_2}(\boldsymbol{\theta})>z_{\alpha}(\chi)\ \Bigg|\ \Theta=\boldsymbol{\theta}\right\}\\[5pt]
            & &\leq&\ \ O(\chi^2) + O\Big(\tau^{n-1}\textbf{\textit{exp}}\big(-z_{\alpha}(\chi)\big)\Big)\\[5pt]
            & &=&\ \ O(\chi^2) + O\Big(\textbf{\textit{exp}}(-|\textbf{\textit{log}}\ \chi|-|\textbf{\textit{log}}\ \chi|^{1-\alpha}+(n-1)\textbf{\textit{log}}\ \tau)\Big)\\[5pt]
            & &=&\ \ O(\chi^2) + O\Big(\chi\textbf{\textit{exp}}(-|\textbf{\textit{log}}\ \chi|^{1-\alpha}+3(n-1)\textbf{\textit{log}}\ |\textbf{\textit{log}}\ \chi|)\Big)\\[5pt]
            & &=&\ \ o(\chi)
        \end{aligned}
    \end{equation}
    and finish the analysis of $S_2$. 

    For $S_1$, we have
    \begin{equation}
        \begin{aligned}
            & & &\ \ \underset{(i,j):1\leq i<j\leq n}{\textbf{\textit{max}}}\ \textbf{\textit{exp}}\left\{\underset{}{\textbf{\textit{min}}}\left(\underset{\boldsymbol{\theta}_1\in\boldsymbol{\Theta}_{i,j}}{\textbf{\textit{sup}}}\ \ell_{S_1}(\boldsymbol{\theta}_1)-\underset{\boldsymbol{\theta}_2\in\textbf{\textit{Supp}}(\rho_{\boldsymbol{\theta}'})}{\textbf{\textit{sup}}}\ \ell_{S_1}(\boldsymbol{\theta}_2),\underset{\boldsymbol{\theta}_1\in\boldsymbol{\Theta}_{j,i}}{\textbf{\textit{sup}}}\ \ell_{S_1}(\boldsymbol{\theta}_1)-\underset{\boldsymbol{\theta}_2\in\textbf{\textit{Supp}}(\rho_{\boldsymbol{\theta}'})}{\textbf{\textit{sup}}}\ \ell_{S_1}(\boldsymbol{\theta}_2)\right)\right\}\\[5pt]
            & &\leq&\ \ \underset{(i,j):1\leq i<j\leq n}{\sum}\textbf{\textit{exp}}\left\{\underset{}{\textbf{\textit{min}}}\left(\underset{\boldsymbol{\theta}_1\in\boldsymbol{\Theta}_{i,j}}{\textbf{\textit{sup}}}\ \ell_{S_1}(\boldsymbol{\theta}_1)-\underset{\boldsymbol{\theta}_2\in\textbf{\textit{Supp}}(\rho_{\boldsymbol{\theta}'})}{\textbf{\textit{sup}}}\ \ell_{S_1}(\boldsymbol{\theta}_2),\underset{\boldsymbol{\theta}_1\in\boldsymbol{\Theta}_{j,i}}{\textbf{\textit{sup}}}\ \ell_{S_1}(\boldsymbol{\theta}_1)-\underset{\boldsymbol{\theta}_2\in\textbf{\textit{Supp}}(\rho_{\boldsymbol{\theta}'})}{\textbf{\textit{sup}}}\ \ell_{S_1}(\boldsymbol{\theta}_2)\right)\right\}\\[5pt]
            & &\leq&\ \ \textbf{\textit{exp}}(-z_{\alpha}(\chi)).
        \end{aligned}
    \end{equation}
    We take logarithm for the both sides of the inequality 
    \begin{equation}
        \underset{}{\textbf{\textit{min}}}\left(\underset{\boldsymbol{\theta}_2\in\textbf{\textit{Supp}}(\rho_{\boldsymbol{\theta}'})}{\textbf{\textit{sup}}}\ \ell_{S_1}(\boldsymbol{\theta}_2) - \underset{}{\textbf{\textit{min}}}\left(\underset{\boldsymbol{\theta}_1\in\boldsymbol{\Theta}_{i,j}}{\textbf{\textit{sup}}}\ \ell_{S_1}(\boldsymbol{\theta}_1),\underset{\boldsymbol{\theta}_1\in\boldsymbol{\Theta}_{j,i}}{\textbf{\textit{sup}}}\ \ell_{S_1}(\boldsymbol{\theta}_1)\right)\right)\geq z_{\alpha}(\chi).
    \end{equation}
    Then the expected Kendall tau at $S_1$ will be bounded by
    \begin{equation}
        \begin{aligned}
            & & &\ \ \mathbb{E}[\mathfrak{R}(\boldsymbol{R}(\boldsymbol{\Lambda}, S_1))]\\[5pt] 
            & &\leq&\ \ \frac{n(n-1)}{2}\cdot\mathcal{L}\big(\textbf{\textit{Supp}}(\rho_{\boldsymbol{\theta}'})\big)\times\underset{\boldsymbol{\theta}\in\textbf{\textit{Supp}}(\rho_{\boldsymbol{\theta}'})}{\textbf{\textit{sup}}}\ \rho_{\boldsymbol{\theta}'}(\boldsymbol{\theta})\\[5pt]
            & & &\ \ \times\underset{\boldsymbol{\theta}\in\textbf{\textit{Supp}}(\rho_{\boldsymbol{\theta}'})}{\textbf{\textit{sup}}}\ \underset{(i,j):\boldsymbol{\theta}\in\boldsymbol{\Theta}_{i,j}}{\textbf{\textit{max\phantom{p}}}}\ \mathbb{P}\left\{\underset{\boldsymbol{\theta}''\in\boldsymbol{\Theta}_{i,j}}{\textbf{\textit{sup}}}\ \ell_{S_1}(\boldsymbol{\theta}'')-\ell_{S_1}(\boldsymbol{\theta})>z_{\alpha}(\chi)\ \Bigg|\ \Theta=\boldsymbol{\theta}\right\}.
        \end{aligned}
    \end{equation}
    The rest of the proof is similar to that for $S_2$. 
\end{proof}

\ul{Now we arrive at the asymptotic optimality of the expected stopping time of the proposed manipulation policy.} 

\begin{theorem}
    \label{thm:6}
    Consider the proposed stopping time \eqref{eq:stopping_time} and the generation rule solved by \eqref{opt:generation_rule} with complete knowledge, when the exploration probability $p\in(0,1)$ will be 
    \begin{equation}
        p\propto |\textbf{\textit{log}} \chi|^{-\frac{1}{2}+\delta_0}
    \end{equation}
    for some $\delta_0\in\big(0,\frac{1}{2}\big)$. Then we have
    \begin{equation}
        \underset{\chi\rightarrow 0}{\textbf{\textit{lim sup}}}\ \frac{\mathbb{E}[S]}{\mathbb{E}\left[\tau_{\chi}(\Theta)\right]}\leq 1.
    \end{equation}
\end{theorem}

\begin{proof}
    First we establish a upper bound for the expectation of a stopping time $S$, the numerator $\mathbb{E}[S]$.
    \begin{equation}
        \label{eq:180}
        \begin{aligned}
            & \mathbb{E}[S] &=&\ \ \sum_{m=0}^{\infty}\mathbb{E}\left[S\ \big| \ m(1+\delta)\tau_{\chi}(\Theta)\leq S < (m+1)(1+\delta)\tau_{\chi}(\Theta)\right]\\[5pt]
            & &\leq&\ \ (1+\delta)\mathbb{E}[\tau_{\chi}(\Theta)] + \sum_{m=1}^{\infty}\mathbb{E}\left[S\ \big| \ m(1+\delta)\tau_{\chi}(\Theta)\leq S < (m+1)(1+\delta)\tau_{\chi}(\Theta)\right]\\[7.5pt]
            & &\leq&\ \ (1+\delta)\mathbb{E}[\tau_{\chi}(\Theta)] + (1+\delta)\cdot\underset{\boldsymbol{\theta}\in\textbf{\textit{Supp}}(\rho_{\boldsymbol{\theta}'})}{\textbf{\textit{max}}}\ \tau_{\chi}(\boldsymbol{\theta})\cdot\\[5pt]
            & & &\ \ \phantom{(1+\delta)\mathbb{E}[\tau_{\chi}(\Theta)]}\sum_{m=1}^{\infty}(m+1)\mathbb{P}(m(1+\delta)\tau_{\chi}(\Theta)\leq S < (m+1)(1+\delta)\tau_{\chi}(\Theta))\\[7.5pt]
            & &\leq&\ \ (1+\delta)\mathbb{E}[\tau_{\chi}(\Theta)] + (1+\delta)\cdot\underset{\boldsymbol{\theta}\in\textbf{\textit{Supp}}(\rho_{\boldsymbol{\theta}'})}{\textbf{\textit{max}}}\ \tau_{\chi}(\boldsymbol{\theta})\cdot\\[7.5pt]
            & & &\ \ \phantom{(1+\delta)\mathbb{E}[\tau_{\chi}(\Theta)]}\sum_{m=1}^{\infty}(m+1)\underset{\boldsymbol{\theta}\in\textbf{\textit{Supp}}(\rho_{\boldsymbol{\theta}'})}{\textbf{\textit{max}}}\mathbb{P}\left(m(1+\delta)\tau_{\chi}(\Theta)\leq S < (m+1)(1+\delta)\tau_{\chi}(\Theta)\ \Big|\ \Theta=\boldsymbol{\theta}\right).
        \end{aligned}
    \end{equation}
    We start with stopping time $S_2$. For $m\geq1$,
    \begin{equation}
        \label{eq:181}
        \begin{aligned}
            & & &\ \ \mathbb{P}\left(m(1+\delta)\tau_{\chi}(\Theta)\leq S_2 < (m+1)(1+\delta)\tau_{\chi}(\Theta)\ \Big|\ \Theta=\boldsymbol{\theta}\right)\\[5pt]
            & &\leq&\ \ \mathbb{P}\left(m(1+\delta)\tau_{\chi}(\Theta)\leq S_2 < (m+1)(1+\delta)\tau_{\chi}(\Theta), \underset{\begin{matrix}\scriptstyle S/(1+\delta)\tau_{\chi}(\Theta)\\\scriptstyle\in[\delta_2m,m+1]\end{matrix}}{\textbf{\textit{max}}}\ \|\boldsymbol{\hat{\theta}}_S-\boldsymbol{\theta}\|\leq\ |\textbf{\textit{log}}\ \chi|^{-\delta_1}\ \Bigg|\ \Theta=\boldsymbol{\theta}\right)\\[5pt]
            & & &\ \ + \mathbb{P}\left(\underset{\begin{matrix}\scriptstyle S/(1+\delta)\tau_{\chi}(\Theta)\\\scriptstyle\in[\delta_2m,m+1]\end{matrix}}{\textbf{\textit{max}}}\ \|\boldsymbol{\hat{\theta}}_S-\boldsymbol{\theta}\|\geq\ |\textbf{\textit{log}}\ \chi|^{-\delta_1}\ \Bigg|\ \Theta=\boldsymbol{\theta}\right),
        \end{aligned}
    \end{equation}
    where $\boldsymbol{\hat{\theta}}_S$ is the \textbf{MLE}
    \begin{equation}
        \boldsymbol{\hat{\theta}}_S = \underset{\boldsymbol{\theta}\in\textbf{\textit{Supp}}(\rho_{\boldsymbol{\theta}'})}{\textbf{\textit{arg sup}}}\ L\bigg(\boldsymbol{\theta}, \boldsymbol{w}^{(S)}_{\mathcal{A}}\bigg), \tag{\ref{eq:MLE}}
    \end{equation}
    and the $\delta_1$ and $\delta_2$ is two constants related to exploration probability $p\propto |\textbf{\textit{log}} \chi|^{-\frac{1}{2}+\delta_0}$:
    \begin{equation}
        \delta_1 = \frac{\delta_0}{8},\ \ \delta_2 = |\textbf{\textit{log}}\ \chi|^{-\frac{\delta_0}{2}}.
    \end{equation}
    \ul{The following lemma show a upper bound of the second term in the preceding display.}
    \begin{lemma}
        \label{lemma:10}
        Suppose $\boldsymbol{\lambda}^{(S)}=\{\lambda^{(S)}_{i,j}\}$ is a generation rule at $S$ step and $\{\epsilon_{\boldsymbol{\lambda}, T_1, T_2}\}$, $\{\delta_{T_1, T_2}\}$ is two sequences of real numbers such that
        \begin{equation}
            \begin{aligned}
                & & & \underset{\begin{matrix}\scriptstyle T_1\leq S\leq T_2,\\\scriptstyle(i,j)\in\mathfrak{C}\end{matrix}}{\textbf{\textit{min}}}\ \lambda^{(S)}_{i,j} \geq \epsilon_{\boldsymbol{\lambda}, T_1, T_2},\\[3pt]
                & & & \underset{T_1\rightarrow\infty}{\textbf{\textit{lim}}}\ T_1\cdot\epsilon_{\boldsymbol{\lambda}, T_1, T_2}\cdot\delta^2_{T_1, T_2} = \infty.
            \end{aligned}
        \end{equation}
        Then, it holds that
        \begin{equation}
            \mathbb{P}\left(\underset{T_1\leq S\leq T_2}{\textbf{\textit{max}}}\ \|\boldsymbol{\hat{\theta}}_S-\boldsymbol{\theta}\|\geq \delta_{T_1, T_2}\ \Bigg|\ \Theta = \boldsymbol{\theta}\right)\leq\textbf{\textit{exp}}\Big(-\Omega\big(T_1\epsilon^2_{\boldsymbol{\lambda}, T_1, T_2}\delta^4_{T_1, T_2}\big)\Big)\times O\big(T_2^n\big).
        \end{equation}
    \end{lemma}
    When
    \begin{equation}
        T_1 = m(1+\delta)\delta_2\tau_{\chi}(\boldsymbol{\theta}),\ \ T_2 = m(1+\delta)\tau_{\chi}(\boldsymbol{\theta}),\ \ \epsilon_{\boldsymbol{\lambda}, T_1, T_2} = \Omega\Big(|\textbf{\textit{log}} \chi|^{-\frac{1}{2}+\delta_0}\Big),\ \ \delta^2_{T_1, T_2} = |\textbf{\textit{log}} \chi|^{-\delta_1}, 
    \end{equation}
    we have
    \begin{equation}
        \begin{aligned}
            & & &\ \ \mathbb{P}\left(\underset{\begin{matrix}\scriptstyle S/(1+\delta)\tau_{\chi}(\Theta)\\\scriptstyle\in[\delta_2m,m+1)\end{matrix}}{\textbf{\textit{max}}}\ \|\boldsymbol{\hat{\theta}}_S-\boldsymbol{\theta}\|\geq\ |\textbf{\textit{log}}\ \chi|^{-\delta_1}\ \Bigg|\ \Theta=\boldsymbol{\theta}\right)\\[5pt]
            & &\leq&\ \ \textbf{\textit{exp}}\left(-\Omega\Big(m(1+\delta)\delta_2\tau_{\chi}(\boldsymbol{\theta})|\textbf{\textit{log}} \chi|^{-4\delta_1}|\textbf{\textit{log}} \chi|^{-1+2\delta_0}\Big)\right)\times O\Big(m^{n-1}|\textbf{\textit{log}} \chi|^{n-1}\Big)\\[5pt]
            & &=&\ \ \textbf{\textit{exp}}\left(-\Omega\Big(m|\textbf{\textit{log}} \chi|^{2\delta_0-4\delta_1}\delta_2\Big)\right)\times O\Big(m^{n-1}|\textbf{\textit{log}} \chi|^{n-1}\Big)\\[5pt]
            & &=&\ \ \textbf{\textit{exp}}\left(-\Omega\Big(m|\textbf{\textit{log}} \chi|^{\delta_0}\Big)\right)\times O\Big(m^{n-1}|\textbf{\textit{log}} \chi|^{n-1}\Big).
        \end{aligned}
    \end{equation}
    Now we analyze the first term on the right-hand side of \eqref{eq:181}. For $m\geq 1$, $S_2>m(1+\delta)\tau_{\chi}(\boldsymbol{\theta})$ implies that there exists a pairwise comparison $(i,j)$ such that
    \begin{equation}
        \left|\underset{\boldsymbol{\theta}_1\in\boldsymbol{\Theta}_{i,j}}{\textbf{\textit{sup}}}\ \ell_S(\boldsymbol{\theta}_1)-\underset{\boldsymbol{\theta}_2\in\boldsymbol{\Theta}_{j,i}}{\textbf{\textit{sup}}}\ \ell_S(\boldsymbol{\theta}_2)\right|\leq z_{\alpha}(\chi),
    \end{equation}
    where $S=m(1+\delta)\tau_{\chi}(\boldsymbol{\theta})$. Without loss of generality, let $\boldsymbol{\theta}\in\boldsymbol{\Theta}_{i,j}$ and $S_2>m(1+\delta)\tau_{\chi}(\boldsymbol{\theta})$ further shows that
    \begin{equation}
        \ell_S(\boldsymbol{\theta})-\underset{\boldsymbol{\theta}_2\in\boldsymbol{\Theta}_{j,i}}{\textbf{\textit{sup}}}\ \ell_S(\boldsymbol{\theta}_2)\leq z_{\alpha}(\chi).
    \end{equation}
    Consequently, the first term on the right-hand side of \eqref{eq:181} will be bounded
    \begin{equation}
        \begin{aligned}
            & & &\ \ \mathbb{P}\left(m(1+\delta)\tau_{\chi}(\Theta)\leq S_2 < (m+1)(1+\delta)\tau_{\chi}(\Theta), \underset{\begin{matrix}\scriptstyle S/(1+\delta)\tau_{\chi}(\Theta)\\\scriptstyle\in[\delta_2m,m+1)\end{matrix}}{\textbf{\textit{max}}}\ \|\boldsymbol{\hat{\theta}}_S-\boldsymbol{\theta}\|\leq\ |\textbf{\textit{log}}\ \chi|^{-\delta_1}\ \Bigg|\ \Theta=\boldsymbol{\theta}\right)\\[5pt]
            & &\leq&\ \ \mathbb{P}\left(\ell_S(\boldsymbol{\theta})-\underset{\boldsymbol{\theta}_2\in\boldsymbol{\Theta}_{j,i}}{\textbf{\textit{sup}}}\ \ell_S(\boldsymbol{\theta}_2)\leq z_{\alpha}(\chi), \underset{\begin{matrix}\scriptstyle S/(1+\delta)\tau_{\chi}(\Theta)\\\scriptstyle\in[\delta_2m,m+1)\end{matrix}}{\textbf{\textit{max}}}\ \|\boldsymbol{\hat{\theta}}_S-\boldsymbol{\theta}\|\leq\ |\textbf{\textit{log}}\ \chi|^{-\delta_1}\ \Bigg|\ \Theta=\boldsymbol{\theta}\right).
        \end{aligned}
    \end{equation}
    \ul{The preceding display can be bounded by the following lemma.}
    \begin{lemma}
        \label{lemma:11}
        Suppose that the generation rule $\boldsymbol{\lambda}^*(\boldsymbol{\hat{\theta}}_S)$ is solved by \eqref{opt:generation_rule}:
        \begin{equation}
            \boldsymbol{\lambda}^*(\boldsymbol{\hat{\theta}}_S)\in\underset{\ \boldsymbol{\lambda}\in\boldsymbol{\Delta}\phantom{\tilde{1}}}{\textbf{\textit{arg\ max}}}\ \underset{\begin{matrix}\scriptstyle\boldsymbol{\tilde{\theta}}\in\textbf{\textit{Supp}}(\rho_{\boldsymbol{\theta}'})\\\scriptstyle\boldsymbol{\pi}(\boldsymbol{\hat{\theta}}_S)\neq\boldsymbol{\pi}(\boldsymbol{\tilde{\theta}})\end{matrix}}{\textbf{\textit{min\phantom{g}}}}\ \sum_{(i,j)} \lambda_{i,j}\cdot g_{i,j}(\boldsymbol{\hat{\theta}}_S)\cdot\textbf{\textit{log}}\frac{g_{i,j}(\boldsymbol{\hat{\theta}}_S)}{g_{i,j}(\boldsymbol{\tilde{\theta}})}. 
            \tag{\ref{opt:generation_rule}}
        \end{equation}
        If $\boldsymbol{\lambda}^*(\boldsymbol{\hat{\theta}}_S)$ is adopted with probability $1-o(1)$ uniformly for $S\in[m(1+\delta)\delta_2\tau_{\chi}(\boldsymbol{\theta}), m(1+\delta)\tau_{\chi}(\boldsymbol{\theta})]$, we have
        \begin{equation}
            \begin{aligned}
                & & &\ \ \mathbb{P}\left(\ell_S(\boldsymbol{\theta})-\underset{\boldsymbol{\theta}_2\in\boldsymbol{\Theta}_{j,i}}{\textbf{\textit{sup}}}\ \ell_S(\boldsymbol{\theta}_2)\leq z_{\alpha}(\chi), \underset{\begin{matrix}\scriptstyle S/(1+\delta)\tau_{\chi}(\Theta)\\\scriptstyle\in[\delta_2m,m+1)\end{matrix}}{\textbf{\textit{max}}}\ \|\boldsymbol{\hat{\theta}}_S-\boldsymbol{\theta}\|\leq\ |\textbf{\textit{log}}\ \chi|^{-\delta_1}\ \Bigg|\ \Theta=\boldsymbol{\theta}\right)\\[5pt]
                & &\leq&\ \ \textbf{\textit{exp}}\Big(-\Omega(m|\textbf{\textit{log}}\ \chi|)\Big)\times O\Big(|\textbf{\textit{log}}\ \chi|^{n-1}m^{n-1}\Big),
            \end{aligned}
        \end{equation}
        where $S = m(1+\delta)\tau_{\chi}(\boldsymbol{\theta})$.
    \end{lemma}
    Combining the results of Lemma \ref{lemma:9} and Lemma \ref{lemma:10}, \eqref{eq:181} will be bounded as
    \begin{equation}
        \begin{aligned}
            & & &\ \ \mathbb{P}\left(m(1+\delta)\tau_{\chi}(\Theta)\leq S_2 < (m+1)(1+\delta)\tau_{\chi}(\Theta)\ \Big|\ \Theta=\boldsymbol{\theta}\right)\\[5pt]
            & &\leq&\ \ \left(\textbf{\textit{exp}}\Big(-\Omega(m|\textbf{\textit{log}}\ \chi|)\Big)+\textbf{\textit{exp}}\Big(-\Omega(m|\textbf{\textit{log}}\ \chi|^{\delta_0})\Big)\right)\times O\Big(|\textbf{\textit{log}}\ \chi|^{n-1}m^{n-1}\Big).
        \end{aligned}
    \end{equation}
    Aggregating the preceding display with \eqref{eq:180}, the expectation of the stopping time $S_2$ will be bounded
    \begin{equation}
        \begin{aligned}
            & \mathbb{E}[S_2]&\leq&\ \ (1+\delta)\mathbb{E}[\tau_{\chi}(\Theta)] + (1+\delta)\cdot\underset{\boldsymbol{\theta}\in\textbf{\textit{Supp}}(\rho_{\boldsymbol{\theta}'})}{\textbf{\textit{max}}}\ \tau_{\chi}(\boldsymbol{\theta})\cdot\\[5pt]
            & & &\ \ \phantom{(1+\delta)}\sum_{m=1}^{\infty}(m+1)\underset{\boldsymbol{\theta}\in\textbf{\textit{Supp}}(\rho_{\boldsymbol{\theta}'})}{\textbf{\textit{max}}}\mathbb{P}\left(m(1+\delta)\tau_{\chi}(\Theta)\leq S < (m+1)(1+\delta)\tau_{\chi}(\Theta)\ \Big|\ \Theta=\boldsymbol{\theta}\right)\\[5pt]
            & &\leq&\ \ (1+\delta)\mathbb{E}[\tau_{\chi}(\Theta)] + O(|\textbf{\textit{log}}\ \chi|)\times\\[5pt]
            & & & \ \ \phantom{(1+\delta)}\sum_{m=1}^{\infty}(m+1)\left(\textbf{\textit{exp}}\Big(-\Omega(m|\textbf{\textit{log}}\ \chi|)\Big)+\textbf{\textit{exp}}\Big(-\Omega(m|\textbf{\textit{log}}\ \chi|^{\delta_0})\Big)\right)\times O\Big(|\textbf{\textit{log}}\ \chi|^{n-1}m^{n-1}\Big)\\[5pt]
            & &\leq&\ \ (1+\delta)\mathbb{E}[\tau_{\chi}(\Theta)] + o(|\textbf{\textit{log}}\ \chi|). 
        \end{aligned}
    \end{equation}
    The preceding display is the desired conclusion for the asymptotic optimality of $S_2$. 

    Next we proceed to the case of $S_1$. Notice that the event $S_1>S$ implies that
    \begin{equation}
        \sum_{(i,j)}\textbf{\textit{exp}}\left(\textbf{\textit{min}}\left(\underset{\boldsymbol{\theta}_1\in\boldsymbol{\Theta}_{i,j}}{\textbf{\textit{sup}}}\ \ell_{S}(\boldsymbol{\theta}_1)-\underset{\boldsymbol{\theta}_2\in\textbf{\textit{Supp}}(\boldsymbol{\theta}')}{\textbf{\textit{sup}}}\ \ell_{S}(\boldsymbol{\theta}_2),\underset{\boldsymbol{\theta}_1\in\boldsymbol{\Theta}_{j,i}}{\textbf{\textit{sup}}}\ \ell_{S}(\boldsymbol{\theta}_1)-\underset{\boldsymbol{\theta}_2\in\textbf{\textit{Supp}}(\boldsymbol{\theta}')}{\textbf{\textit{sup}}}\ \ell_{S}(\boldsymbol{\theta}_2)\right)\right)> \textbf{\textit{exp}}(-z_{\alpha}(\chi))
    \end{equation}
    which further conducts
    \begin{equation}
        \begin{aligned}
            & & &\ \ n(n-1)\cdot\underset{(i,j)}{\textbf{\textit{max}}}\ \textbf{\textit{exp}}\left(\textbf{\textit{min}}\left(\underset{\boldsymbol{\theta}_1\in\boldsymbol{\Theta}_{i,j}}{\textbf{\textit{sup}}}\ \ell_{S}(\boldsymbol{\theta}_1)-\underset{\boldsymbol{\theta}_2\in\textbf{\textit{Supp}}(\boldsymbol{\theta}')}{\textbf{\textit{sup}}}\ \ell_{S}(\boldsymbol{\theta}_2),\underset{\boldsymbol{\theta}_1\in\boldsymbol{\Theta}_{j,i}}{\textbf{\textit{sup}}}\ \ell_{S}(\boldsymbol{\theta}_1)-\underset{\boldsymbol{\theta}_2\in\textbf{\textit{Supp}}(\boldsymbol{\theta}')}{\textbf{\textit{sup}}}\ \ell_{S}(\boldsymbol{\theta}_2)\right)\right)\\[5pt]
            & &>&\ \ \textbf{\textit{exp}}(-z_{\alpha}(\chi)).
        \end{aligned}
    \end{equation}
    It means that there exist a pairwise comparison $(i,j)$ such that
    \begin{equation}
        \left|\underset{\boldsymbol{\theta}_1\in\boldsymbol{\Theta}_{i,j}}{\textbf{\textit{sup}}}\ \ell_{S}(\boldsymbol{\theta}_1)-\underset{\boldsymbol{\theta}_2\in\in\boldsymbol{\Theta}_{j,i}}{\textbf{\textit{sup}}}\ \ell_{S}(\boldsymbol{\theta}_2)\right|\leq z_{\alpha}(\chi) + \textbf{\textit{log}}\ n(n-1).
    \end{equation}
    Then the analysis process of $S_1$ is similar to $S_2$ by replacing $z_{\alpha}(\chi)$ with $z_{\alpha}(\chi) + \textbf{\textit{log}}\ n(n-1)$. 
\end{proof}
\newpage

\section*{Proof of Theorem \ref{thm:reformulation}}
\label{sec:proof_thm_3}
\addcontentsline{toc}{section}{Proof of Theorem \ref{thm:reformulation}\nameref{sec:proof_thm_3}}

\ul{The following proposition} \cite{doi:10.1287/moor.2018.0936} \ul{shows the strong duality for the Wasserstein \textbf{DRO} problem that we investigate in this paper}:
\begin{equation}
    \underset{\ \boldsymbol{\theta}\in\textbf{\textit{Supp}}(\rho_{\boldsymbol{\theta}'})}{\ \textbf{\textit{max\phantom{p}}}}\ \underset{\mathbb{Q}\in\boldsymbol{\mathfrak{U}}^{\gamma}(\mathbb{P})}{\textbf{\textit{sup}}}\ \mathbb{E}_{\boldsymbol{q}\sim\mathbb{Q}}\left[L(\boldsymbol{\theta},\boldsymbol{q})\right], \tag{\ref{eq:DRE}}
\end{equation}
This strong duality result ensures that the inner supremum admits a reformulation which is a simple, univariate optimization problem. Note that there exists the other strong duality result of Wasserstein DRO \cite{doi:10.1287/moor.2022.1275}.
\begin{proposition}
    \label{prop:duality}
    Let $d:\mathbb{R}\times\mathbb{R}\rightarrow[0,\infty]$ be a lower semi-continuous cost function satisfying $d(p,\ q)=0$ whenever $p = q$. For $\lambda\geq 0$ and loss function $\ell$ 
    \begin{equation}
        \ell(\boldsymbol{\theta},p_{i,j}) = p_{i,j}\cdot\textbf{\textit{log}}~g_{i,j}(\boldsymbol{\theta}),
    \end{equation}
    where $f_{(i,j)}(\boldsymbol{\theta})$ is the probabilistic mass function of pairwise comparison $(i,j)$, we define
    \begin{equation}
        \label{eq:psi_1}
        \psi_{\lambda,\ell}(\boldsymbol{\theta},q_{i,j}) :=\underset{q_{i,j}\in\mathbb{R}_+}{\textbf{\textit{sup}}}\ \Big\{\ \ell(\boldsymbol{\theta},q_{i,j})-\lambda\cdot d(p_{i,j},q_{i,j})\ \Big\}.
    \end{equation}
    Then it holds that
    \begin{equation}
        \label{eq:duality}
        \underset{\mathbb{Q}\in\boldsymbol{\mathfrak{U}}^{\gamma}(\mathbb{P})}{\textbf{\textit{sup}}}\ \ \mathbb{E}_{\boldsymbol{q}\sim\mathbb{Q}}\big[L\big(\boldsymbol{\theta},\boldsymbol{q}\big)\big] = \underset{\lambda\geq 0}{\ \textbf{\textit{min\phantom{p}}}}\left\{\ \lambda\gamma+\sum_{(i,j)}\psi_{\lambda,\ell}(\boldsymbol{\theta},q_{i,j})\ \right\}
    \end{equation}
\end{proposition}

\ul{With the strong duality, we have the following result.}
  
\reformulation*

\begin{proof}
    Let $\Delta_{ij} = q_{ij}-p_{ij}$ where $\boldsymbol{q}=(q_{1,2},\dots,q_{n,n-1})\in\mathbb{R}^N_+$. We define $\psi_{\lambda,\ell}(\boldsymbol{\theta})$ as
    \begin{equation}
        \label{eq:psi_3}
        \begin{aligned}
            & \psi_{\lambda,\ell}(\boldsymbol{\theta}) &=&\ \ \underset{\boldsymbol{q}}{\textbf{\textit{sup}}}\ \underset{(i,j)}{\sum}\ \Big\{\ell(\boldsymbol{\theta},\ q_{ij})-\lambda\big[d(p_{i,j},\ q_{i,j})\big]^2\Big\}\\[5pt]
            & &=&\ \ \underset{\boldsymbol{q}}{\textbf{\textit{sup}}}\ \underset{(i,j)}{\sum}\ \left\{q_{ij}\cdot\textbf{\textit{log}}~g_{(i,j)}(\boldsymbol{\theta})-\lambda\big|p_{ij}-q_{ij}\big|^2\right\}\\[5pt]
            & &=& \ \ \underset{(i,j)}{\sum}\ \underset{\Delta_{ij}\in\mathbb{R}}{\textbf{\textit{sup}}}\ \Big(\Delta_{ij} b_{ij}-\lambda \Delta_{ij}^2 + p_{ij}b_{ij}\Big),
        \end{aligned}
    \end{equation}
    where 
    \begin{equation}
        b_{i,j} = \textbf{\textit{log}}~g_{i,j}(\boldsymbol{\theta}),
    \end{equation}
    and the third equality holds due to $\psi_{\lambda, \ell}(\boldsymbol{\theta})$ is a decomposable function. Expanding \eqref{eq:psi_3}, we can simplify $\psi_{\lambda,\ell}(\boldsymbol{\theta})$ as below:
    \begin{equation}
        \begin{aligned}
            & \psi_{\lambda,\ell}(\boldsymbol{\theta})&=&\ \ \big\langle\boldsymbol{p},\boldsymbol{b}\big\rangle + \underset{(i,j)}{\sum}\ \underset{\Delta_{ij}\in\mathbb{R}}{\textbf{\textit{sup}}}\ \big(\Delta_{ij} b_{ij}-\lambda \Delta_{ij}^2\big)\\[7.5pt]
            & &=&\ \ 
            \left\{
            \begin{array}{lc}
                \displaystyle\langle\boldsymbol{p},\ \boldsymbol{b}\rangle+\frac{1}{4\lambda }\|\boldsymbol{b}\|^2_2,\ &\ \text{if }\ \lambda>0,\\[10pt]
                \infty,\ &\ \text{if }\ \lambda=0.
            \end{array}
            \right. 
        \end{aligned}
    \end{equation}
    Next, we investigate the duality of \eqref{eq:DRE} with Proposition \ref{prop:duality}. As $\psi_{\lambda,\ell}(\boldsymbol{\theta},\ q_{i,j}) = \infty$ when $\lambda=0$, the dual formulation of the supremum in \eqref{eq:DRE} would be
    \begin{equation}
        \label{eq:supmre_gamma}
        \begin{aligned}
            & & &\ \ \ \ \underset{\mathbb{Q}\in\boldsymbol{\mathfrak{U}}^{\gamma}(\mathbb{P})}{\textbf{\textit{sup}}}\ \mathbb{E}_{\boldsymbol{q}\sim\mathbb{Q}}\left[L(\boldsymbol{\theta},\boldsymbol{q})\right]\\[2pt]
            & & &=\ \ \underset{\lambda\geq0}{\textbf{\textit{min}}}\ \ \Bigg\{\lambda\gamma+\psi_{\lambda,\ell}(\boldsymbol{\theta})\Bigg\}\\[5pt]
            & & &=\ \ \underset{\lambda>0}{\textbf{\textit{min}}}\ \ \Bigg\{\lambda\gamma+\langle\boldsymbol{p},\ \boldsymbol{b}\rangle+\frac{1}{4\lambda}\|\boldsymbol{b}\|^2_2\Bigg\}.
        \end{aligned}
    \end{equation}
    By the definition of $\boldsymbol{b}$, we know that 
    \begin{equation}
        L(\boldsymbol{\theta},\ \boldsymbol{p}) = \big\langle\boldsymbol{p},\ \boldsymbol{b}\big\rangle
    \end{equation}
    Moreover, notice that the right hand side of \eqref{eq:supmre_gamma} is a convex function which approaches infinity when $\lambda\rightarrow\infty$, the global optimal of it can be obtained uniquely via the first order optimality condition as
    \begin{equation}
        \frac{\partial}{\partial \lambda} \Bigg\{\ \lambda\gamma+\langle\boldsymbol{p},\ \boldsymbol{b}\rangle+\frac{1}{4\lambda}\|\boldsymbol{b}\|^2_2\ \Bigg\} = 0,
    \end{equation}
    and the optimal dual variable is
    \begin{equation}
        \lambda^*_{\gamma} = \frac{\|\boldsymbol{b}\|_2}{2\sqrt{\gamma}}.
    \end{equation}
    Substituting $\lambda^*_{\gamma}$ and $\boldsymbol{b}$ into \eqref{eq:supmre_gamma}, we have
    \begin{equation}
        \begin{aligned}
            & & &\ \ \underset{\mathbb{Q}\in\boldsymbol{\mathfrak{U}}^{\gamma}(\mathbb{P})}{\textbf{\textit{sup}}}\ \mathbb{E}_{\boldsymbol{q}\sim\mathbb{Q}}\left[L(\boldsymbol{\theta},\boldsymbol{q})\right]\\[7.5pt]
            & &=&\ \ \sqrt{\gamma}\cdot\|\boldsymbol{b}\|_2+\langle\boldsymbol{p},\ \boldsymbol{b}\rangle\\[9pt]
            & &=&\ \ \sqrt{\ \gamma\underset{(i,j)}{\sum}[\textbf{\textit{log}}~g_{i,j}(\boldsymbol{\theta})]^2}+\underset{(i,j)}{\sum}p_{i,j}\textbf{\textit{log}}~g_{i,j}(\boldsymbol{\theta}).
            % & &=&\ \ \sqrt{\ \gamma\underset{(i,j)}{\sum}\big[\textbf{\textit{log}}(1+\textbf{\textit{exp}}(\theta_j-\theta_i))\big]^2}+\underset{(i,j)}{\sum}\ p_{i,j}\ \textbf{\textit{log}}(1+\textbf{\textit{exp}}(\theta_j-\theta_i)).
        \end{aligned}
    \end{equation}
    When we sepicify the probability mass function of $(i,j)$ as
    \begin{equation*}
        g_{i,j}(\boldsymbol{\theta}) = \frac{e^{\theta_i}}{e^{\theta_i}+e^{\theta_j}},
    \end{equation*}
    we can obtain the formulation of $h$ for the \textbf{BTL} model.
\end{proof}

\section*{Details of Algorithm \ref{alg:robust_estimation} and \ref{alg:simplex_proj}}
\begin{algorithm}[ht]
    \caption{Robust Estimation}
    \label{alg:robust_estimation}
    \SetAlgoLined
    \SetKwInOut{Input}{\ \ Input}
    \SetKwInOut{Output}{Output}
    \Input{the probability mass function $g$, the support set $\textbf{\textit{Supp}}(\rho_{\boldsymbol{\theta}'})$, the incomplete knowledge $\boldsymbol{S}$ and the solution accuracy $\epsilon$.}
    Initialization:
    \begin{equation*}
        \begin{aligned}
            & \mu_{\text{min}}=0,\\
            & \mu_{\text{max}} = \mu_{\infty} = \textbf{\textit{max}}\left\{m\cdot\|\boldsymbol{z}\|_{\infty},\sqrt{\frac{m}{2\beta}}\cdot\|\boldsymbol{z}\|_2\right\},
        \end{aligned}
    \end{equation*}
    where $\boldsymbol{z} = [z_{1,2},\dots,z_{n,n-1}]$
    \begin{equation*}
        z_{ij} = \underset{\boldsymbol{\theta}\in\textbf{\textit{Supp}}(\rho_{\boldsymbol{\theta}'})}{\textbf{\textit{max}}}\ -\textbf{\textit{log}}~g_{i,j}(\boldsymbol{\theta}).
    \end{equation*}

    \While{$|\mu_{\text{max}}-\mu_{\text{min}}|>\epsilon\cdot\mu_{\infty}$}
    {
        \begin{equation*}
            \begin{aligned}
                & \mu &=&\ \ \ \ \frac{1}{2}(\mu_{\text{min}}+\mu_{\text{max}}),\\[5pt]
                & \boldsymbol{\theta}(\mu) &=&\ \ \ \ \textbf{\textit{SimplexProjection}}(\boldsymbol{\theta}^{(0)}, \mu),\\[7.5pt]
                & \nabla \mathcal{H}(\mu) &=&\ \ \ \ \frac{1}{2}\|\boldsymbol{\theta}(\mu) - \boldsymbol{\theta}_{\mathcal{A}}\|^2_2 - \beta.
            \end{aligned}
        \end{equation*}

        \eIf{$\nabla \mathcal{H}(\mu)>0$}
        {
            \begin{equation*}
                \mu_{\text{min}} = \mu,
            \end{equation*}
        }
        {   
            \begin{equation*}
                \mu_{\text{max}} = \mu.    
            \end{equation*}
        }
    }

    Update
    \begin{equation*}
        \mu =\frac{1}{2}\big(\mu_{\text{min}}+\mu_{\text{max}}\big).
    \end{equation*}

    Solve the distributionally robust estimation
    \begin{equation*}
        \boldsymbol{\hat{\theta}} = \textbf{\textit{SimplexProjection}}(\boldsymbol{\theta}^{(0)}, \mu)                
    \end{equation*}

    \Output{the distributionally robust estimation $\boldsymbol{\hat{\theta}}$.}
\end{algorithm}

\begin{algorithm}[h]
    \SetAlgoLined
    \SetKwInOut{Input}{\ \ Input}
    \SetKwInOut{Output}{Output}
    \caption{$\textbf{\textit{SimplexProjection}}(\boldsymbol{\theta}^{(0)}, \mu)$}
    \label{alg:simplex_proj}
    \Input{the partial dual problem \eqref{eq:min_max_swap}}
    Initialization: the initial step size $\eta_0$, the maximum iteration number $T_1$, $\boldsymbol{U} = [n]$, $s= 0$, $v=0$.
    
    \For{$t=0$ {\bfseries to} $T_1-1$}
    {
        \begin{equation*}
            \boldsymbol{\theta}^{(t+1)} = \boldsymbol{\theta}^{(t)} - \eta_t\nabla\mathcal{L}(\boldsymbol{\theta}^{(t)},\mu).
        \end{equation*}    
    }
    
    \While{$\boldsymbol{U}\neq \varnothing$}
    {
        Pick $k\in\boldsymbol{U}$ at random and separate $\boldsymbol{U}$ as
        \begin{equation*}
            \begin{aligned}
                & \boldsymbol{G} = \{\ j\in\boldsymbol{U}\ |\ \theta^{(T_1)}_j\geq\theta^{(T_1)}_k\ \},\\[3pt]
                & \boldsymbol{L} = \{\ j\in\boldsymbol{U}\ |\ \theta^{(T_1)}_j<\theta^{(T_1)}_k\ \}.
            \end{aligned}
        \end{equation*}
        
        Set
        \begin{equation*}
            \Delta v = |\boldsymbol{G}|,\ \Delta s = \sum_{j\in\boldsymbol{G}}\hat{\theta}^{(T_1)}_j.
        \end{equation*}

        \eIf{$(s+\Delta s)-(v+\Delta v) < 1$}
        {
            \begin{equation*}
                s\leftarrow s+\Delta s,\ v\leftarrow v+\Delta v,\ \boldsymbol{U}\leftarrow\boldsymbol{L}
            \end{equation*}
        }
        {
            \begin{equation*}
                \boldsymbol{U}\leftarrow\boldsymbol{G}/\{k\}.
            \end{equation*}
        }
    }

    \begin{equation*}
        \boldsymbol{\hat{\theta}}^{(T_1)} = \big[\boldsymbol{\theta}^{(T_1)}-\gamma\cdot\boldsymbol{1}\big]_+,
    \end{equation*}
    where
    \begin{equation*}
        \gamma = \frac{s-1}{v}.
    \end{equation*}

    \Output{$\boldsymbol{\hat{\theta}}^{(T_1)}$.}
\end{algorithm}

\begin{definition}[$p$-Wasserstein distance]
    Let $p\in[1,\infty]$. The $p$-Wasserstein distance between distributions $\mathbb{P},\ \mathbb{Q}\in\mathcal{P}(\boldsymbol{\Omega})$ is defined as
    % \begin{itemize}
        % \item $1\leq p< \infty$
        \begin{equation*}
            \begin{aligned}
                \mathcal{W}_p\ (\mathbb{P},\ \mathbb{Q})=\left\{
                \begin{array}{ll}
                    \Bigg(\underset{\gamma\in\Gamma(\mathbb{P},\ \mathbb{Q})}{\textbf{\textit{min}}}{\int}_{\boldsymbol{\Omega}\times\boldsymbol{\Omega}}\big[d(\boldsymbol{p},\ \boldsymbol{q})\big]^p\gamma\big(\mathrm{d}\boldsymbol{p},\ \mathrm{d}\boldsymbol{q}\big)\Bigg)^{\frac{1}{p}},&p<\infty\\[20pt]
                    \underset{\gamma\in\Gamma(\mathbb{P},\ \mathbb{Q})}{\textbf{\textit{inf}}}\ \underset{\vphantom{\gamma\in\Gamma(\mathbb{P},\ \mathbb{Q})}\boldsymbol{\Omega}\times\boldsymbol{\Omega}}{\gamma\textnormal{-}\textbf{\textit{ess sup}}}\ d\big(\boldsymbol{p},\ \boldsymbol{q}\big),& p=\infty,
                \end{array}
                \right.
            \end{aligned}      
        \end{equation*}
    where $\Gamma(\mathbb{P},\ \mathbb{Q})$ denotes the set of all Borel probability distributions on $\boldsymbol{\Omega}\times\boldsymbol{\Omega}$ with marginal distributions $\mathbb{P}$ and $\mathbb{Q}$, $d:\boldsymbol{\Omega}\times\boldsymbol{\Omega}\rightarrow\mathbb{R}_+$ is a nonnegative function, and $\gamma\textnormal{-}\textbf{\textit{ess sup}}$ expresses the essential supremum of $d(\cdot,\ \cdot)$ with respect to the measure $\gamma$.
\end{definition}
The Wasserstein distance arises in the problem of optimal transport \cite{monge1781memoire,villani2008optimal}: for any coupling $\gamma\in\Gamma(\mathbb{P},\ \mathbb{Q})$, the conditional distribution $\gamma_{\boldsymbol{w}\vert\boldsymbol{w}'}$ can be viewed as a randomized overhead for ‘transporting’ a unit quantity of some material from a random location $\boldsymbol{w}\sim\mathbb{P}$ to another location $\boldsymbol{w}'\sim\mathbb{Q}$. If the cost of transportation from $\boldsymbol{w}\in\boldsymbol{\Omega}$ to $\boldsymbol{w}'\in\boldsymbol{\Omega}$ is given by $[d(\boldsymbol{w},\boldsymbol{w}')]^p$, $\mathcal{W}_p\ (\mathbb{P},\ \mathbb{Q})$ will be the minimum expected transport cost \cite{peyre2019computational}. 
% \begin{equation}
%     \label{opt:dual_problem}
%     \begin{aligned}
%         \underset{\boldsymbol{\theta}\in\boldsymbol{\Theta}^{\beta}_{\mathcal{A}}}{\textbf{\textit{min}}}\ h(\boldsymbol{\theta}),
%     \end{aligned}
% \end{equation}
% where
% \begin{equation}
%     \begin{aligned}
%         & h(\boldsymbol{\theta}) &=&\ \ \sqrt{\ \alpha\underset{(i,j)}{\sum}\big[\log(1+\exp(\theta_j-\theta_i))\big]^2}\\
%         & & &\ \ {\phantom{\frac{\alpha}{N}}}+\underset{(i,j)}{\sum}p_{i,j}\log(1+\exp(\theta_j-\theta_i)),
%     \end{aligned}
% \end{equation}
% \begin{equation}
%     \boldsymbol{\Theta}^{\beta}_{\mathcal{A}} = \left\{\ \boldsymbol{\theta}\ \Big|\ \frac{1}{2}\left\|\boldsymbol{\theta}-\boldsymbol{\theta}_{\mathcal{A}}\right\|_2^2\leq\beta,\ \boldsymbol{\theta}\succeq 0,\ \boldsymbol{1}^\top\boldsymbol{\theta}=1\right\}
% \end{equation}

Introducing a dual variable $\mu$ for the constraint $1/2\|\boldsymbol{\theta}-\boldsymbol{\theta}_{\mathcal{A}}\|_2^2\leq\beta$, we solve the original optimization problem \eqref{opt:dual_problem} by maximizing its dual problem. 
% With loss of generality, we assume that $\boldsymbol{\theta}_{\mathcal{A}}=[\tilde{\theta}_1,\dots,\tilde{\theta}_n]$ holds $\tilde{\theta}_1\leq\tilde{\theta}_2\leq\dots\leq\tilde{\theta}_n$ and $\boldsymbol{1}^\top\boldsymbol{\theta}_{\mathcal{A}}=1$, which would not change the solution of the original optimization \eqref{opt:dual_problem}. 
Furthermore, the strong duality stands for \eqref{opt:dual_problem} as the Slater condition is satisfied by $\boldsymbol{\theta}_{\mathcal{A}}$. Then the standard min-max swap will be performed as  
\begin{equation}
    \label{eq:min_max_swap}
    % \begin{aligned}
    %     & & &\ \ \underset{\lambda \geq 0}{\ \ \textbf{\textit{max}}\phantom{f}}g(\lambda):=\underset{\boldsymbol{\theta}}{\textbf{\textit{inf}}}\Bigg\{\frac{\lambda}{2}\left\|\boldsymbol{\theta}-\boldsymbol{\theta}_{\mathcal{A}}\right\|^2_2-\lambda\epsilon+F(\boldsymbol{\theta})\\
    %     & & &\ \ \Bigg|\ \boldsymbol{\theta}\succeq 0,\ \boldsymbol{1}^\top\boldsymbol{\theta}=1\Bigg\}
    % \end{aligned}
    \begin{aligned}
        & \underset{\mu \geq 0}{\ \ \textbf{\textit{max}}\phantom{f}}\mathcal{H}(\mu)
        &:=&\ \ \underset{\boldsymbol{\theta}\in\mathbb{R}^n}{\textbf{\textit{inf}}}\ \Big\{\ \mathcal{L}_1(\boldsymbol{\theta}, \mu)\ \Big|\ \boldsymbol{\theta}\succeq 0,\ \boldsymbol{1}^\top\boldsymbol{\theta}=1\ \Big\}\\[5pt]
        & &=&\ \ \underset{\boldsymbol{\theta}\in\mathbb{R}^n}{\textbf{\textit{inf}}}\Big\{\frac{\mu}{2}\big\|\boldsymbol{\theta}-\boldsymbol{\theta}_{\mathcal{A}}\big\|^2_2-\mu\beta+h(\boldsymbol{\theta})\ \Big|\ \boldsymbol{\theta}\succeq 0,\boldsymbol{1}^\top\boldsymbol{\theta}=1\Big\}.
        % \begin{matrix}
        % \displaystyle\frac{\mu}{2}\big\|\boldsymbol{\theta}-\boldsymbol{\theta}_{\mathcal{A}}\big\|^2_2\\[10pt]
        % -\mu\beta+h(\boldsymbol{\theta})
        % \end{matrix}
        % \ \left|\ 
        % \begin{matrix}
        %     \boldsymbol{\theta}\succeq 0,\\[10pt]
        %     \boldsymbol{1}^\top\boldsymbol{\theta}=1
        % \end{matrix}
        % \ \right.\right\}.\\[5pt]
        \end{aligned}
\end{equation}
By Corollary 4.4.5 of Chapter VI in \cite{hiriart2013convex}, we know that
\begin{equation}
    \begin{aligned}
        & \nabla \mathcal{H}(\mu)&=&\ \ \nabla_{\mu}\ \mathcal{L}_1(\boldsymbol{\theta}(\mu), \mu)\\[5pt]
        & &=&\ \ \nabla_{\mu}\ \Bigg\{\frac{\mu}{2}\big\|\boldsymbol{\theta}(\mu)-\boldsymbol{\theta}_{\mathcal{A}}\big\|^2_2-\mu\beta+h(\boldsymbol{\theta}(\mu))\Bigg\}\\[5pt]
        & &=&\ \ \frac{1}{2}\big\|\boldsymbol{\theta}(\mu)-\boldsymbol{\theta}_{\mathcal{A}}\big\|^2_2-\beta.
    \end{aligned}
\end{equation}
If $\nabla \mathcal{H}(\mu)$ can be solved efficiently, it is possible to binary search with $\log(1/\varepsilon)$ iterations and the accuracy $\varepsilon$ for the $\mu^*$ as $\|\nabla \mathcal{H}(\mu^*)\|\leq \varepsilon$.

% For a collection $\{g_{\boldsymbol{\theta}}\}_{\boldsymbol{\theta}\in\boldsymbol{\Theta}}$, any vector $\nabla g_{\boldsymbol{\theta}_0}(\lambda)$ is a super-gradient of $g(\lambda)$ if 
% \begin{equation}
%     \textbf{\textit{inf}}\ g(\lambda) = \underset{\boldsymbol{\theta}\in\boldsymbol{\Theta}}{\ \ \textbf{\textit{inf}}}\ g_{\boldsymbol{\theta}}(\lambda)
% \end{equation}
% is attained at some $\boldsymbol{\theta}_0$. Let $\boldsymbol{\theta}(\lambda)$ be the minimum solution of $\boldsymbol{\theta}$, the objective of \eqref{eq:min_max_swap} turns to be
% \begin{equation}
%     g(\lambda) = \frac{\lambda}{2}\left\|\boldsymbol{\theta}(\lambda)-\boldsymbol{\theta}_{\mathcal{A}}\right\|^2_2-\lambda\epsilon+F(\boldsymbol{\theta}(\lambda)),
% \end{equation}
% whose derivative \textit{w.r.t} $\lambda$ (holding $\boldsymbol{\theta}$ fixed) is
% \begin{equation}
%     g'(\lambda) = \frac{1}{2}\left\|\boldsymbol{\theta}(\lambda)-\boldsymbol{\theta}_{\mathcal{A}}\right\|^2_2-\epsilon.
% \end{equation}
% Then we can obtain the $\lambda$

Given $\mu^*$ and denote the probabilistic simplex as
\begin{equation}
    \Delta=\{\ \boldsymbol{\theta}\in\mathbb{R}^n\ |\ \boldsymbol{\theta}\succeq 0,\ \boldsymbol{1}^\top\boldsymbol{\theta}=1\ \}, 
\end{equation}
the corresponding $\boldsymbol{\theta}(\lambda^*)$ can be obtained by the projected sub-gradient descent which minimizes the objective $\mathcal{L}$ with the following sequence $\{\boldsymbol{\theta}^{(t)}\}_{t=1}^{T}$,
\begin{equation}
    \boldsymbol{\theta}^{(t+1)} = \textbf{Proj}_{\Delta}\big(\boldsymbol{\theta}^{(t)}-\eta_t \nabla\mathcal{L}(\boldsymbol{\theta}^{(t)},\mu^*)\big),
\end{equation}
where $\nabla\mathcal{L}(\boldsymbol{\theta}^{(t)},\mu^*)$ is the (sub)gradient of $\mathcal{L}(\boldsymbol{\theta},\mu^*)$ at $\boldsymbol{\theta}^{(t)}$, $\eta_t$ is the positive step size and $\textbf{Proj}_{\Delta}(\boldsymbol{\varphi})$ is the Euclidean projection of $\boldsymbol{\varphi}$ onto $\Delta$, \textit{i.e.} the solution of
\begin{equation}
    \underset{\boldsymbol{\theta}\in\mathbb{R}^n}{\textbf{\textit{min}}}\ \frac{1}{2}\|\boldsymbol{\varphi}-\boldsymbol{\theta}\|^2_2,\ \textit{s.t.}\ \boldsymbol{\theta}\succeq 0,\ \boldsymbol{1}^\top\boldsymbol{\theta} = 1.
\end{equation}
The projection $\boldsymbol{\theta}^{(t+1)}$ has the form
\begin{equation}
    \begin{aligned}
        & \boldsymbol{\theta}^{(t+1)} &=&\ \ [\boldsymbol{\varphi}^{(t+1)}-\gamma\boldsymbol{1}]_{+}\\
        % & &=&\ \ [\boldsymbol{\theta}^{(t)}-\eta_t G(\boldsymbol{\theta}^{(t)})-\gamma\boldsymbol{1}]_{+}
    \end{aligned}
\end{equation}
where $[\cdot]_+ = \textbf{\textit{max}}\{\cdot, 0\}$ and $\gamma$ holds
\begin{equation}
    \gamma = \frac{1}{\rho}\left(\sum_{i=1}^{\rho}\varphi^{(t+1)}_i-1\right),
\end{equation}
\begin{equation}
    \rho = \textbf{\textit{max}}\left\{j\in[n]\ \left|\ \psi_j-\frac{1}{j}\left(\sum_{r=1}^j\psi_r-1\right)>0\right.\right\},
\end{equation}
where $\boldsymbol{\psi}$ is the sorted version of $\boldsymbol{\varphi}^{(t+1)}$ with descent order.

\section*{Details of Algorithm \ref{alg:mirror_descent}}

\begin{algorithm}[ht]
    \SetAlgoLined
    \SetKwInOut{Input}{\ \ Input}
    \SetKwInOut{Output}{Output}    
    \caption{$\textbf{\textit{MirrorDescent}}(\boldsymbol{S}, f, \boldsymbol{\Theta}^{\beta}_{\mathcal{A}}, \boldsymbol{\theta}^{(m)})$}
    \label{alg:mirror_descent}
    \Input{The \textbf{MLE} estimator $\boldsymbol{\hat{\theta}}$ and the total number of iterations $L$.}
    
    Initialization: A starting point $\boldsymbol{\lambda}^0$ and a constant $c_0>0$.

    \For{$l=1$ {\bfseries to} $L$}
    {
        Solve the maximizing sub-problem
        \begin{equation*}
            \boldsymbol{\theta}(\boldsymbol{\lambda}^{l-1})\in\underset{\boldsymbol{\theta}\in\boldsymbol{\Theta}:r(\boldsymbol{\theta})\neq r(\boldsymbol{\hat{\theta}})}{\textbf{\textit{arg max}}}\ -g(\boldsymbol{\theta};\boldsymbol{\lambda}^{l-1}, \boldsymbol{\hat{\theta}}),
        \end{equation*}
        where
        \begin{equation*}
            g(\boldsymbol{\theta};\boldsymbol{\lambda}^{l-1}, \boldsymbol{\hat{\theta}})=\underset{(i,j)\in\mathcal{A}}{\sum}\ \lambda^{l-1}_{i,j}D^{i,j}(\boldsymbol{\hat{\theta}}\|\boldsymbol{\theta}).
        \end{equation*}

        Calculate the sub-gradient $d(\boldsymbol{\lambda}^{l-1})$
        \begin{equation*}
            d(\boldsymbol{\lambda}^{l-1})=\big(d(\boldsymbol{\lambda}^{l-1})_{1,2}, \dots, d(\boldsymbol{\lambda}^{l-1})_{n,n-1}\big),
        \end{equation*}
        where
        \begin{equation*}
            d(\boldsymbol{\lambda}^{l-1})_{i,j} = -D^{i,j}(\boldsymbol{\hat{\theta}}\|\boldsymbol{\theta}(\boldsymbol{\lambda}^{l-1})).
        \end{equation*}

        Solve the minimizing sub-problem
        \begin{equation*}
            \boldsymbol{\lambda}^{l} \in \underset{\boldsymbol{\lambda}\in\Delta}{\textbf{\textit{arg min}}}\ \eta_l\langle d(\boldsymbol{\lambda}^{l-1}), \boldsymbol{\lambda}\rangle + D(\boldsymbol{\lambda}\|\boldsymbol{\lambda}^{l-1}), 
        \end{equation*}
        where $\eta_l = c_0/\sqrt{~l~}$ and 
        \begin{equation*}
            D(\boldsymbol{\lambda}\|\boldsymbol{\lambda}^{l-1}) = \underset{(i,j)\in\mathcal{A}}{\sum}\lambda_{i,j}\log\frac{\lambda_{i,j}}{\lambda^{l-1}_{i,j}}
        \end{equation*}
    }
    The categorical distribution is obtained through averaging the sequence $\{\boldsymbol{\lambda}^1,\dots, \boldsymbol{\lambda}^L\}$ like
    \begin{equation*}
        \boldsymbol{\hat{\lambda}} = \frac{1}{L}\sum_{l=1}^{L}\boldsymbol{\lambda}^l.
    \end{equation*}

    \Output{The categorical distribution $\boldsymbol{\hat{\lambda}}$.}
\end{algorithm}

\begin{equation*}
    \boldsymbol{\hat{\lambda}}^{(m)} = \underset{\phantom{\boldsymbol{\hat{\theta}}^{(m)}}\boldsymbol{\lambda}\in\Delta\phantom{\boldsymbol{\hat{\theta}}^{(m)}}}{\textbf{\textit{arg max}}}\underset{\boldsymbol{\theta}\in\boldsymbol{\Theta}:\boldsymbol{\pi}(\boldsymbol{\theta})\neq \boldsymbol{\pi}(\boldsymbol{\hat{\theta}}^{(m)})}{\ \textbf{\textit{min}\phantom{g}}}\ g(\boldsymbol{\theta}, \boldsymbol{\lambda};\boldsymbol{\hat{\theta}}^{(m)}),
\end{equation*}
where $\boldsymbol{\lambda} = (\lambda_{1,2},\dots,\lambda_{n,n-1})$,
\begin{equation*}
    g(\boldsymbol{\theta}, \boldsymbol{\lambda};\boldsymbol{\hat{\theta}}^{(m)})=\underset{(i,j)\in\mathcal{A}}{\sum}\ \lambda_{i,j}D^{i,j}(\boldsymbol{\hat{\theta}}^{(m)}\|\boldsymbol{\theta}),
\end{equation*}
and $D^{i,j}(\boldsymbol{\hat{\theta}}\|\boldsymbol{\theta})$ is the Kullback-Leibler (KL) divergence from $f(\boldsymbol{\theta};(i,j),y)$ to $f(\boldsymbol{\hat{\theta}};(i,j),y)$ as
\begin{equation*}
    D^{i,j}(\boldsymbol{\hat{\theta}}\|\boldsymbol{\theta}) = \underset{y\in\{-1,1\}}{\sum}f(\boldsymbol{\hat{\theta}};(i,j),y)\log\frac{f(\boldsymbol{\hat{\theta}};(i,j),y)}{f(\boldsymbol{\theta};(i,j),y)}.
\end{equation*}
\begin{equation*}
    \boldsymbol{\theta}(\boldsymbol{\lambda}^{(l-1)})\in\underset{\boldsymbol{\theta}\in\boldsymbol{\Theta}:r(\boldsymbol{\theta})\neq r(\boldsymbol{\hat{\theta}})}{\textbf{\textit{arg max}}}\ -g(\boldsymbol{\theta};\boldsymbol{\lambda}^{(l-1)}, \boldsymbol{\hat{\theta}}),
\end{equation*}

\begin{equation*}
    \boldsymbol{\theta}(\boldsymbol{\lambda}^{(l-1)})\in\underset{\boldsymbol{\theta}\in\boldsymbol{\Theta}^{\beta}_{\mathcal{A}}:\boldsymbol{\pi}(\boldsymbol{\theta})\neq \boldsymbol{\pi}(\boldsymbol{\hat{\theta}})}{\textbf{\textit{arg min}}}\ g(\boldsymbol{\theta};\boldsymbol{\lambda}^{(l-1)}, \boldsymbol{\hat{\theta}}),
\end{equation*}
% \begin{equation}
%     \begin{aligned}
%         & & &\ \ \frac{\partial}{\partial \theta_i} g(\boldsymbol{\theta};\boldsymbol{\lambda}^{(l-1)},\boldsymbol{\hat{\theta}})\\
%         & &=&\ \ \sum_{j\in[n]}\lambda^{(l-1)}_{j,i}\frac{1}{\exp(\hat{\theta}_i-\hat{\theta}_j)+1}\frac{1}{\exp(\theta^{t}_j-\theta^{t}_i)+1}\\
%         & & &\ \ \phantom{\sum_{j\in[n]}}-\lambda^{(l-1)}_{i,j}\frac{1}{\exp(\hat{\theta}_j-\hat{\theta}_i)+1}\frac{1}{\exp(\theta^{t}_i-\theta^{t}_j)+1}
%     \end{aligned}
% \end{equation}
\begin{equation*}
    \begin{aligned}
        & & &\ \ \underset{\delta\geq 0}{\textbf{\textit{max}}}\ \ \mathcal{G}(\delta)\\
        & &:=&\ \ \underset{\boldsymbol{\theta}\in\mathbb{R}^n}{\textbf{\textit{inf}}}\ \Big\{\ \mathcal{L}_2(\boldsymbol{\theta}, \delta)\ \Big|\ \boldsymbol{\theta}\succeq 0,\ \boldsymbol{1}^\top\boldsymbol{\theta}=1,\boldsymbol{\pi}(\boldsymbol{\theta}) \neq \boldsymbol{\pi}(\boldsymbol{\hat{\theta}})\ \Big\}\\
        & &=&\ \underset{\boldsymbol{\theta}\in\mathbb{R}^n}{\textbf{\textit{inf}}}\ \left\{\ 
        \begin{matrix}
            \displaystyle\frac{\delta}{2}\big\|\boldsymbol{\theta}-\boldsymbol{\theta}_{\mathcal{A}}\big\|^2_2-\delta\beta\\[5pt]
            +\ g(\boldsymbol{\theta};\boldsymbol{\lambda}^{(l-1)}, \boldsymbol{\hat{\theta}})
        \end{matrix}
        \ \left|\ 
        \begin{matrix}
            \boldsymbol{\theta}\succeq 0,\boldsymbol{1}^\top\boldsymbol{\theta}=1,\\[10pt]
            \boldsymbol{\pi}(\boldsymbol{\theta}) \neq \boldsymbol{\pi}(\boldsymbol{\hat{\theta}})
        \end{matrix}
        \right.\right\}
    \end{aligned}
\end{equation*}

Balance the choice probability of the selection rule
\begin{equation*}
    \lambda^{(m)}_{i,j} = p\cdot\frac{2}{n(n-1)}+(1-p)\hat{\lambda}^{(m)}_{i,j},\ i,j\in[n],\ i\neq j.
\end{equation*}

is chosen such that $\sum_i\theta^{(t+1)}_i=1$. With Lemma 2 of \cite{DBLP:conf/icml/DuchiSSC08}, we know that $\gamma$ plays the same role as the unique index $i$ which 
\begin{equation}
    \label{eq:index_condition}
    \begin{aligned}
        & \sum_{j=1}^{i}\Big(\varphi^{(t+1)}_{j}-\varphi^{(t+1)}_{i}\Big)<1\\
        & \sum_{j=1}^{i+1}\Big(\varphi^{(t+1)}_{j}-\varphi^{(t+1)}_{i+1}\Big)\geq1,
    \end{aligned}
\end{equation}
or $i=n$ if there does not exist any index satisfies \eqref{eq:index_condition}. 

The projection $\boldsymbol{\theta}(\lambda)=[\theta_i(\lambda),\dots,\theta_n(\lambda)]$ has the form
\begin{equation}
    \theta_i(\lambda) = \left[\tilde{\theta}_i-\frac{1}{\lambda}\frac{\partial F(\boldsymbol{\theta})}{\partial\theta_i}\right]
\end{equation}

finding the the Euclidean projection of vector $\boldsymbol{v}(\lambda)=[v_1,\dots,v_n]\in\mathbb{R}^n$
\begin{equation}
    v_i = 
\end{equation}
onto the probabilistic simplex. 
Then the projection $\boldsymbol{\theta}(\lambda)$ has the form $\theta_i(\lambda)=(v_i-\eta)_+$ for some $\eta\in\mathbb{R}$ where $\eta$ is chosen such that $\boldsymbol{\theta}(\lambda)$ will satisfy the probabilistic unit sphere constraint. It is equivalent to find the unique index $i$ such that
\begin{equation}
    \label{eq:eta_condition}
    \sum_{j=1}^i(v_j-v_i)<1,\ \text{and}\ \sum_{j=1}^{i+1}(v_j-v_{i+1})\geq 1,
\end{equation}
and $i=n$ if there does not exist such an index. By \eqref{eq:eta_condition}, we know that
\begin{equation}
    v_j-\eta \left\{
    \begin{array}{ll}
        \geq 0, & j\leq i,\\
        \leq 0, & j\geq i
    \end{array}
    \right.
\end{equation}

\end{document}